\documentclass[preprint,nonblindrev]{informs4}

\OneAndAHalfSpacedXI

\usepackage{amsmath,amssymb,amsfonts,mathtools}
\usepackage{algorithm,algpseudocode}
\usepackage{booktabs,enumitem}
\usepackage{graphicx}
\usepackage{endnotes}
\usepackage{array}
\usepackage{tabularx}
\let\footnote=\endnote

\usepackage{natbib}
\usepackage{multibib}
\newcites{EC}{References}

\makeatletter
\renewcommand{\bibcite}[2]{%
  \@ifundefined{b@#1\@extra@binfo}
    {}
    {%
      \def\NAT@temp{#2}%
      \expandafter\ifx
        \csname b@#1\@extra@binfo\endcsname\NAT@temp
      \else
        \NAT@citemultiple
        \PackageWarningNoLine{natbib}
          {Citation `#1' multiply defined}%
      \fi
    }%
  \global\@namedef{b@#1\@extra@binfo}{#2}%
}
\makeatother

\usepackage[hidelinks]{hyperref}
\bibpunct[, ]{(}{)}{,}{a}{}{,}

\def\bibfont{\small}
\def\bibsep{\smallskipamount}

\newcommand{\R}{\mathbb{R}}
\newcommand{\op}{\mathrm{op}}
\newcommand{\E}{\mathbb{E}}
\newcommand{\diag}{\operatorname{diag}}
\newcommand{\vecc}{\operatorname{vec}}
\newcommand{\tr}{\operatorname{tr}}
\newcommand{\T}{\mathcal{T}}
\newcommand{\C}{\mathcal{C}}

\newcommand{\MT}{\mathcal{M}_{\T}}
\newcommand{\MC}{\mathcal{M}_{\C}}
\newcommand{\N}{\mathcal{N}}
\newcommand{\core}{\mathrm{core}}
\newcommand{\con}{\mathrm{con}}
\newcommand{\mask}{\mathrm{mask}}
\newcommand{\MTClass}{\mathrm{MT}}

\newcommand{\od}{\odot}
\newcommand{\os}{\oslash}
\newcommand{\indep}{\perp\!\!\!\perp}
\newcommand{\norm}[1]{\left\lVert #1\right\rVert}

\usepackage{xspace}
\newcommand{\OU}{\textsc{ou}\xspace}

\newcommand{\DID}{\textsc{did}\xspace}
\newcommand{\SC}{\textsc{sc}\xspace}
\newcommand{\SDID}{\textsc{sdid}\xspace}

\newcommand{\NMC}{\textsc{nmc}\xspace}
\newcommand{\TFI}{\textsc{tfi}\xspace}
\newcommand{\ConvDiff}{\textsc{conv-diff}\xspace}
\newcommand{\TuckerDiff}{\textsc{tucker-diff}\xspace}
\newcommand{\CFTDiff}{\textsc{cft-diff}\xspace}

\newcommand{\CompactDisplaySkips}{%
  \setlength{\abovedisplayskip}{6pt plus 2pt minus 2pt}%
  \setlength{\belowdisplayskip}{6pt plus 2pt minus 2pt}%
  \setlength{\abovedisplayshortskip}{4pt plus 2pt minus 1pt}%
  \setlength{\belowdisplayshortskip}{6pt plus 2pt minus 2pt}%
}

\TheoremsNumberedBySection
\EquationsNumberedBySection
\makeatletter
\renewcommand{\ECEquationsNumberedBySection}{%
  \setcounter{equation}{0}%
  \def\theequation{EC.\arabic{section}.\arabic{equation}}%
}

\renewcommand{\subsection}{%
  \@startsection{subsection}{2}{\z@}%
  {-6pt plus -1pt minus -1pt}
  {2pt}
  {\TEN\sffamily\bfseries\RAGG}%
}
\makeatother
\JOURNAL{}
\makeatletter
\long\def\theARTICLETOPLEFT{%
{\tabcolsep0pt\fs.10.12.\bf\sffamily%
\begin{tabular}[t]{@{}l}%
\rule{0pt}{14pt}\\[-4pt]
{\fontsize{7}{8}\selectfont\strut}\\
\HD{25}{0}\strut\\
\vphantom{manuscript no.}\end{tabular}}}%
\def\theARTICLETOPRIGHT{\llap{{\tabcolsep0pt\bf\sffamily%
\begin{tabular}[t]{r@{}}%
\text{}\\[-22pt]
{\fs.12.14.\strut}\\[0pt]
{\fs.7.9.\strut}\\
{\fs.7.9.\strut}
\end{tabular}}}}%
\makeatother
\RRHSecondLine{}
\LRHSecondLine{}
\MANUSCRIPTNO{}
\RUNAUTHOR{Kong et al.}
\RUNTITLE{Conditional Tensor Diffusion}
\TITLE{Conditional Tensor Diffusion: Distributional Counterfactual Learning and Inference}

\ARTICLEAUTHORS{%
\AUTHOR{Xinbing Kong,\textsuperscript{a} Zeyu Li,\textsuperscript{a} Junfan Mao,\textsuperscript{a} Bin Wu\textsuperscript{b}}
\AFF{\textsuperscript{a}Southeast University; \textsuperscript{b}University of Science and Technology of China}
\AFF{{\bf Contact:} xinbingkong@126.com (XK), zeyuli@seu.edu.cn (ZL), maojf@seu.edu.cn (JM), bin.w@ustc.edu.cn (BW)}
}

\ABSTRACT{
Causal inference guides operational and managerial decisions but remains challenging in high-dimensional panel or tensor settings, where decisions may depend on the joint conditional distribution of missing control outcomes. We develop \emph{Counterfactual Tucker Diffusion} (\CFTDiff), which integrates the treatment mask and latent Tucker structure into conditional diffusion to recover this distribution given observed control outcomes through efficient nonlinear score learning in a low-dimensional core. The masked Tucker score preserves dependence across tensor modes while reducing the dimension of nonlinear score learning from the product of mode dimensions to the much smaller product of Tucker ranks. We establish high-probability error bounds for conditional score estimation that depend on the Tucker ranks, largest mode dimension, and the factor-strength-adjusted number of missing outcomes, and show how these bounds translate into recovery guaranties for the conditional distribution of the missing control outcomes. Across missing rates, simulations show more accurate point recovery than common causal panel and matrix/tensor completion methods; comparisons with nested diffusion specifications further demonstrate the gains from masked conditioning and Tucker dimension reduction. In Norway's iFlex experiment, \CFTDiff recovers missing outcomes more accurately than competing methods; when applied to causal analysis, its estimated conditional distributions yield counterfactual prediction intervals and target-attainment probabilities, allowing pricing interventions to be evaluated by demand-reduction magnitude and reliability.
}

\begin{document}
\maketitle
\fontsize{11pt}{16pt}\selectfont
\CompactDisplaySkips

\section{Introduction}\label{sec:intro}

Causal inference is central to pricing, policy evaluation, and intervention design, yet the outcome needed to evaluate an intervention is missing where treatment occurs. At a treated entry in a panel, the intervention outcome is observed, but the corresponding control outcome is missing. Point estimates of the missing control outcomes may be sufficient for reporting an average effect. They are not sufficient when a decision maker must assess the probability of meeting a demand, service, or risk target, quantify uncertainty in the missing outcomes, or compare the reliability of alternative interventions. The central statistical problem is therefore to recover the joint conditional distribution of the missing control outcomes given the observed control outcomes, rather than a single point estimate.

Recovering this distribution is difficult in high-dimensional panel or tensor settings. Outcomes may be indexed by units, dates, hours, products, or locations, so the full dimension grows as the product of the mode dimensions. A tensor formulation preserves these indexing dimensions and their dependence, while including vector and matrix settings. Treatment determines where the control outcomes are missing, often producing a structured missingness pattern. The missing outcomes may remain dependent across units, time periods, and other dimensions. Estimating each missing outcome separately does not capture this dependence and therefore cannot evaluate probabilities involving multiple missing outcomes. Directly learning a conditional distribution in the tensor space is also statistically and computationally demanding. A useful method must therefore incorporate the treatment mask, use information from the observed control outcomes, and exploit the intrinsic structure of the control-outcome tensor.

We develop \emph{Counterfactual Tucker Diffusion} (\CFTDiff), a masked conditional diffusion model for recovering the joint conditional distribution of the missing control outcomes. Its central construction is a masked Tucker score that incorporates the treatment pattern and observed control outcomes while confining nonlinear score learning to a low-dimensional Tucker core. The method therefore provides joint conditional generation rather than point recovery and avoids learning an unrestricted nonlinear score in the full tensor space. During training, the treatment mask is synthetically applied to complete training tensors, allowing the model to learn the distribution of the masked outcomes given the observed outcomes. During causal analysis, the observed control outcomes remain fixed and only the missing control outcomes are generated; observed treated outcomes enter subsequently when causal contrasts are formed. Repeated conditional draws provide point estimates, quantiles, counterfactual prediction intervals, target-attainment probabilities, and other summaries from a single learned distribution.

The key structural result explains why this conditional distribution can be learned efficiently. Mode-specific linear encoders combine the noised outcomes in the treated region and the observed control outcomes into a core-level statistic, and a Tucker decoder maps the learned core representation back to the original coordinates. Only the conditional map within the core is nonlinear. Let $D$ denote the number of modes, $p_d$ the dimension of mode $d$, and $r_d$ its Tucker rank. The full tensor has $p=\prod_{d=1}^{D}p_d$ coordinates, whereas the Tucker core has only $r=\prod_{d=1}^{D}r_d$ coordinates, with $r\ll p$. Nonlinear learning therefore operates in dimension $r$ rather than $p$, while the separately parameterized loading matrices contribute through $p_{\max}=\max_d p_d$ rather than the product $p$.

We establish nonasymptotic, high-probability error bounds for conditional-score estimation. Let $n$ denote the number of training tensors. The bound separates the errors from learning the nonlinear core regression and estimating the mode-specific loading spaces. The dominant sample-size term for the nonlinear component is $n^{-2/(r+5)}$, whose exponent depends on the intrinsic core dimension $r$ rather than the full tensor dimension $p$. Because $r\ll p$ under low Tucker rank, this rate yields more accurate score estimation from the same training sample size than a rate governed by the full tensor dimension. Estimating the loading spaces contributes $p_{\max}n^{-(r+3)/(r+5)}$, so this component depends on the largest mode dimension $p_{\max}$ rather than the product of all mode dimensions $p$. The number of missing control outcomes $p_{\T}$ and the strength of the low-rank signal $\beta$ enter through the effective dimension $d_{\T,\beta}=r+p_{\T}p^{-\beta}$. Thus, the contribution of the missing region enters as $p_{\T}p^{-\beta}$, rather than $p_{\T}$ itself, and remains bounded when $p_{\T}\lesssim p^\beta$ and the Tucker ranks are fixed. We further show that conditional-score accuracy transfers to recovery of the conditional distribution of the missing outcomes. Under the stated growth and regularity conditions, prediction intervals constructed from samples drawn from the estimated conditional distribution attain nominal conditional coverage as the training and generation sample sizes increase, without requiring a separate Gaussian approximation or asymptotic variance estimation.

Simulations with known conditional distributions show that \CFTDiff improves point recovery over the causal panel and matrix/tensor completion methods considered. Comparisons with two nested diffusion specifications attribute the additional gains to Tucker dimension reduction and conditioning on the observed control outcomes through the treatment mask; \CFTDiff also improves distributional recovery and produces sharper prediction intervals while retaining coverage close to the nominal level.

We apply \CFTDiff to Norway's iFlex dynamic-pricing experiment. Before estimating demand responses, we artificially mask fully observed no-policy (control) outcomes and evaluate recovery under switchback, staggered-adoption, and simultaneous-adoption designs. \CFTDiff attains the lowest point-recovery errors among the causal panel and matrix/tensor completion methods and the nested diffusion specifications, and it achieves more accurate distributional recovery than both nested diffusion specifications across the treatment patterns and missing proportions considered. For intervention days, the estimated conditional distributions yield counterfactual prediction intervals, characterize variation in demand reductions across interventions, and quantify the probability that a price signal achieves a specified peak-demand reduction. The results show that price signals with similar average demand reductions can differ substantially in their probability of achieving a specified reduction target. This information allows price signals to be compared by both their average demand reduction and how reliably they achieve a desired peak reduction.

\subsection{Contributions}\label{subsec:contributions}

We make four contributions. First, we formulate counterfactual distribution recovery for general multi-way outcome arrays and develop a conditional diffusion framework for this problem. The formulation includes vectors and matrices as special cases, thereby covering conventional panels as well as higher-order tensors. \CFTDiff incorporates prespecified intervention masks into training, allowing the model to learn the distribution of masked outcomes given the observed control outcomes. Because the construction does not rely on random missingness or a particular treatment pattern, it accommodates switchback, staggered-adoption, simultaneous-adoption, and other structured intervention patterns under the stated conditions. During counterfactual generation, the observed control outcomes remain fixed and only the missing control outcomes are generated, producing samples from their estimated joint conditional distribution. Existing tensor-based causal methods and common causal panel and completion methods primarily target point counterfactuals, common components, or low-dimensional causal summaries. By recovering the conditional distribution, \CFTDiff also permits distributional comparisons and provides quantiles, counterfactual prediction intervals, target-attainment probabilities, and other summaries without requiring a separate model for each quantity.

Second, we introduce a masked Tucker score that makes high-dimensional conditional generation for tensor data statistically tractable and computationally efficient. Generic diffusion networks learn nonlinear scores in the full data space, while conventional low-rank completion methods do not recover a joint conditional distribution. Under the control-outcome factor model, we show that the conditional score admits a representation in which mode-wise linear maps preserve the tensor structure and nonlinear learning takes place through a core-dimensional conditional regression. The neural approximation burden, therefore, depends on the Tucker core dimension $r$, rather than the full tensor dimension $p$. Our nonasymptotic score bounds formalize this intrinsic-dimensional advantage: nonlinear estimation is governed by the Tucker core, loading-space estimation by $p_{\max}$, and the joint effect of the size of the treated region and factor strength by $d_{\T,\beta}$. The result provides a statistical foundation for using low-rank Tucker structure in conditional diffusion.

Third, we establish end-to-end guaranties that carry conditional-score accuracy through to the recovery of the missing-outcome distribution, weighted causal summaries, and counterfactual prediction intervals. We derive high-probability bounds for the recovery of the joint conditional distribution of the missing control outcomes and extend these bounds to weighted summaries of the missing outcomes. Combined with the corresponding observed treated outcomes, these summaries yield familiar causal estimands (e.g., average treatment effects). We further establish coverage guaranties for counterfactual prediction intervals derived from the estimated conditional distribution and explicitly account for the additional error arising from a finite generation sample. The interval theory requires neither a Gaussian approximation nor asymptotic variance estimation. These intervals quantify uncertainty about the missing control outcomes conditional on the observed control outcomes. The results therefore carry score-learning accuracy through to the causal estimands and counterfactual uncertainty measures used in subsequent analysis.

Finally, simulations and the iFlex experiment identify the gains from the two components of \CFTDiff and demonstrate the decision value of recovering a conditional distribution. Simulations with known counterfactual distributions compare \CFTDiff with established causal panel and matrix/tensor completion methods in point recovery and use distributional measures to evaluate \CFTDiff and its nested diffusion specifications beyond point estimates. The two nested specifications separately identify the gains from Tucker dimension reduction and conditioning on the observed control outcomes. We then validate the method on observed no-policy outcomes in the iFlex data under switchback, staggered-adoption, and simultaneous-adoption patterns. In the causal application, the recovered distribution is used to compare dynamic-pricing signals in terms of both the magnitude of demand reduction and the probability of attaining specified peak-reduction targets. The results show that price signals with similar estimated average effects can differ materially in their reliability, illustrating the additional decision-relevant information provided by counterfactual distribution recovery.

\subsection{Related Literature}\label{subsec:related_literature}

\paragraph{Causal panel methods and counterfactual recovery.}
Existing causal panel methods address important parts of this problem, but their primary targets are different. Difference-in-differences and event-study estimators recover average effects under restrictions such as parallel trends and no anticipation, with modern procedures allowing heterogeneous treatment effects \citep{borusyak2024revisiting}. Synthetic-control methods use pretreatment outcomes and a suitable donor pool to construct point estimates under stable preintervention relationships. Synthetic difference-in-differences combines these comparison strategies through outcome balancing and time and unit weighting \citep{arkhangelsky2021synthetic}. Low-rank causal panel methods represent the untreated outcome surface using interactive factors and view treated coordinates as missing entries \citep{athey2021matrix,bai2021matrix}. Subsequent work permits nonrandom or more general observation patterns and develops entrywise inference \citep{choi2024matrix,xiong2023large,duan2024factor}. A small recent literature uses tensor completion directly for causal recovery with treatment histories, multiple outcomes, or multiple exposures \citep{mandal2019weighted,auerbach2022tensor,zhen2024nonnegative,gao2025causal,zhou2026spatial}. Synthetic interventions extend this logic to several treatments and tensor-valued potential outcomes and estimate potential outcomes or low-dimensional causal summaries \citep{agarwal2026synthetic}. \CFTDiff shares the potential-outcome and low-rank viewpoints but targets a different statistical object: the joint conditional distribution of the missing control outcomes given the observed control outcomes. The treatment design and identifying assumptions determine the causal interpretation of the resulting comparisons; \CFTDiff addresses estimation of this conditional distribution.

\paragraph{Matrix and tensor completion.}
Matrix and tensor completion methods exploit low rank structures to recover missing entries, but their guaranties often depend on how the observed coordinates cover the matrices or tensors \citep{xia2019polynomial,xia2021statistically,xia2021statistical}. Causal panels are especially demanding because treatment creates structured and potentially nonrandom observation patterns rather than dispersed incidental missingness. Factor methods have been adapted to these patterns, but their outputs remain point imputations or entrywise sampling distributions. Treatment designs can also create very different patterns, including switchback schedules and staggered or simultaneous adoption \citep{bojinov2023design,chen2026efficient}. Tensor factor models preserve mode-specific dependence in matrix sequences and tensors \citep{wang2019factor,chang2023modelling,chang2026identification,chang2026cp,han2024tensor}, and tensor time-series imputation allows broad missing patterns \citep{cen2025tensor}. \CFTDiff uses the same multilinear structure for a different purpose. It incorporates the treatment mask into training and uses the observed control outcomes to generate the missing control outcomes, thereby providing their joint conditional distribution in addition to point estimates. The fixed-mask formulation does not require a random-missingness model for estimation. 

\paragraph{Marginal counterfactual distributions.}
A separate literature studies how interventions change outcome distributions. Changes-in-changes models identify nonlinear and quantile effects under distributional restrictions \citep{athey2006changes}; counterfactual distribution methods provide inference for population distributions constructed from conditional models \citep{chernozhukov2013inference}; and distributional synthetic controls match the distribution of an aggregate treated unit over time \citep{gunsilius2023distributional}. Related work studies treatment-effect risk, formalized through the conditional value at risk of the individual treatment-effect distribution, and develops bounds and inference based on the conditional average treatment effect \citep{kallus2023treatment}. Bayesian and semiparametric procedures characterize posterior or sampling uncertainty for low-dimensional causal parameters such as the average treatment effect \citep{breunig2025double}. These methods answer questions about marginal counterfactual distributions, distributional treatment effects, treatment-effect risk, or estimator uncertainty. Our target is the joint conditional distribution of the missing control outcomes given the observed control outcomes. Recovering this distribution provides prediction intervals for individual and aggregate counterfactuals, tail probabilities, and probabilities of joint target attainment, enabling interventions to be evaluated by both effect magnitude and reliability.

\paragraph{Diffusion models and conditional generation.}
Diffusion models can simulate from complex multivariate distributions by estimating scores along a noise-perturbation path and running the reverse dynamics \citep{song2019generative,ho2020denoising,song2021score}. Conditional score models have been used for probabilistic time-series imputation \citep{tashiro2021csdi} and conditional response generation for bootstrap inference \citep{chang2026deep}. Causal diffusion methods have generated structural counterfactuals for images or learned individual potential-outcome distributions conditional on covariates, with the latter using an orthogonal loss to address treatment-selection bias \citep{sanchez2022diffusion,ma2024diffpo}. General conditional-diffusion theory and recovery results under low-dimensional structure supply important foundations \citep{fu2024unveil,chen2023score,fan2025optimal}. Related work integrates factor structure into high-dimensional financial simulation \citep{chen2025diffusion}, develops diffusion methods for heavy-tailed and conditional simulation in operations research \citep{liu2026learning}, and studies unconditional Tucker diffusion \citep{guo2026tucker}. These studies establish the feasibility of using diffusion models for related missing-data and causal tasks. For the target studied here, however, an unconditional diffusion model cannot directly exploit the observed control outcomes. A general conditional diffusion model permits such conditioning, but does not by itself determine how the treatment pattern enters training and generation or how the resulting distribution is linked to the missing control outcomes. Learning the conditional score directly in the full tensor space would also require a high-dimensional nonlinear model whose input and output dimensions grow with the total number of tensor entries. Building on these foundations, \CFTDiff incorporates the treatment mask and observed control outcomes into a conditional score for the missing control outcomes, with nonlinear learning conducted in the Tucker core.

\paragraph{Organization.}
Section~\ref{sec:main-method} formulates counterfactual distribution recovery for tensor data, develops \CFTDiff, and establishes its score-learning guaranties. Section~\ref{sec:distributional_recovery} studies recovery of the conditional distribution of the missing control outcomes, weighted summaries, and prediction intervals. Sections~\ref{sec:simulation} and~\ref{sec:empirical} report the simulation and empirical studies. Section~\ref{sec:conclusion} concludes. The Appendix contains complete proofs and additional results.

\paragraph{Notation.} 
For random elements $X$ and $Y$, $\mathcal L(X)$ denotes the probability law of $X$, $\mathcal L(X\mid Y)$ denotes its conditional law given $Y$, and $\mathcal L(X\mid Y=y)$ denotes the conditional law evaluated at $Y=y$. For $m\in\mathbb N$, let $[m]=\{1,\ldots,m\}$. The operator $\vecc(\cdot)$ denotes vectorization under a fixed convention. The symbol $\mathbf 1$ denotes an all-ones array whose dimensions are determined by context. For a vector $g$, $\operatorname{reshape}\{g;(m_1,\ldots,m_D)\}$ denotes its rearrangement into an $m_1\times\cdots\times m_D$ tensor according to the vectorization convention used by $\vecc$. We write $\|\cdot\|_2$, $\|\cdot\|_F$, and $\|\cdot\|_{\op}$ for the Euclidean norm of a vector, the Frobenius norm of a matrix or tensor, and the operator norm of a matrix, respectively. The notation $O(\cdot)$ suppresses constants depending only on fixed model parameters, whereas $\widetilde{\mathcal O}(\cdot)$ additionally suppresses polylogarithmic factors in the asymptotic quantities specified in the result. In this paper, $\od$ and $\os$ denote entrywise multiplication and division. We write $\operatorname{KL}(\mu\|\nu)$ for relative entropy and $\operatorname{TV}(\mu,\nu)=\sup_A|\mu(A)-\nu(A)|$ for total variation distance. For probability measures $\mu$ and $\nu$ on $\mathbb R$, we use the bounded--Lipschitz metric $d_{\mathrm{BL}}(\mu,\nu):=\sup_{\substack{\|\varphi\|_\infty\leq1\\ \operatorname{Lip}(\varphi)\leq1}}\left|\int \varphi\,\mathrm d\mu-\int \varphi\,\mathrm d\nu\right|$, where $\operatorname{Lip}(\varphi)$ denotes the Lipschitz constant of $\varphi$.



\section{Diffusion Models: Preliminaries}\label{sec:Preliminaries}

Diffusion models represent a target distribution through two linked stochastic processes: a known forward process that gradually perturbs data into noise and a reverse-time process that transforms noise back into data. The reverse-time process depend on the scores of the noise-perturbed distributions along the forward path; these scores must be learned from data. By outlining the basic principles of diffusion models in a simple scalar setting, this section provides the background needed to understand our framework for tensor generation developed in the paper.

\subsection{Forward Diffusion and Reverse-Time Generation}\label{subsec:forward_reverse_diffusion}

\paragraph{Forward diffusion.}
Let $X_0\in\mathbb R$ have data distribution $P_0$. We consider the variance-preserving Ornstein--Uhlenbeck (\OU) diffusion,\footnote{Alternative forward processes, including variance-exploding and sub-variance-preserving diffusions, are also possible. We adopt the variance-preserving \OU process because its transition law is available in closed form and it drives the data distribution toward a standard Gaussian reference distribution, thereby simplifying score estimation and theoretical analysis. It is also a standard and widely used specification in diffusion modeling \citep{ho2020denoising,song2021score}.} a continuous-time formulation of the Gaussian noising mechanism used in diffusion models \citep{ho2020denoising,song2021score}:
\begin{equation}
    \mathrm dX_t=-\frac{1}{2}\eta(t)X_t\,\mathrm dt+\eta(t)^{1/2}\,\mathrm dW_t,\qquad X_0\sim P_0,\qquad t\in[0,T],
    \label{eq:scalar_forward_sde}
\end{equation}
where $\{W_t\}_{t\geq0}$ is a standard Wiener process and $\eta(t)>0$ is a deterministic noise schedule. Define $\alpha_t=\exp(-\frac{1}{2}\int_0^t\eta(v)\,\mathrm dv)$ and $h_t=1-\alpha_t^2$. For each $t$, the linear structure of \eqref{eq:scalar_forward_sde} gives the explicit perturbation representation
\begin{equation*}
    X_t=\alpha_tX_0+h_t^{1/2}Z_t,\qquad Z_t\sim\mathcal N(0,1),\qquad Z_t\indep X_0.
\end{equation*}

Let $q_t(\cdot\mid x_0)$ denote the Gaussian transition density of $X_t$ conditional on $X_0=x_0$. The corresponding time-$t$ marginal density is
\[
    p_t(x)=\int_{\mathbb R}q_t(x\mid x_0)\,P_0(\mathrm dx_0),\qquad P_t=\mathcal L(X_t).
\]
The distinction between $q_t$ and $p_t$ is important. The transition density $q_t(\cdot\mid x_0)$ is the known Gaussian perturbation kernel selected by the modeler, whereas $p_t$ is the generally unknown marginal density obtained by mixing this kernel over the data distribution $P_0$. Although $q_t$ remains Gaussian, its mean and variance vary with $t$; the marginal density $p_t$ need not belong to a fixed parametric family, and its form generally changes with $t$.

As $t$ increases, the contribution of $\alpha_tX_0$ diminishes and the Gaussian component becomes dominant. Consequently, when $\alpha_T$ is sufficiently small, the terminal distribution $P_T$ is close to the standard normal distribution $\mathcal N(0,1)$.

\paragraph{Reverse-time generation.}
For $t>0$, define the time-$t$ score of the marginal distribution $P_t$ by $s_t(x):=\nabla_x\log p_t(x)$. Thus, the score used in the reverse-time process is the gradient of the log marginal density $p_t$, not the score of the Gaussian transition kernel $q_t(\cdot\mid x_0)$.

Under standard regularity conditions, the forward diffusion admits a reverse-time representation \citep{anderson1982reverse,song2021score}. Using $u$ to denote time elapsed in the reverse direction, the exact reverse-time process satisfies
\begin{equation}
\begin{aligned}
    \mathrm dX_u^{\leftarrow}=\biggl[\frac{1}{2}\eta(T-u)X_u^{\leftarrow}+\eta(T-u)s_{T-u}
        \bigl(X_u^{\leftarrow}\bigr)\biggr]\mathrm du+\eta(T-u)^{1/2}\,\mathrm d\overline W_u,\qquad u\in[0,T-t_0],
    \label{eq:scalar_reverse_sde}
\end{aligned}
\end{equation}
where $\{\overline W_u\}_{u\geq0}$ is a standard Wiener process and $t_0>0$ is a small early-stopping time. If
$X_0^{\leftarrow}\sim P_T$, then $X_u^{\leftarrow}\sim P_{T-u}$ for $u\in[0,T-t_0]$, and hence $X_{T-t_0}^{\leftarrow}\sim P_{t_0}$.

The forward process therefore provides analytically noised observations for learning, whereas the reverse-time process provides the mechanism for generation.\footnote{The reverse-time equality is distributional: the process reproduces the forward marginal distributions but does not retrace an individual forward sample path.} Once the forward coefficients are specified, the only unknown quantity in the reverse drift is the time-dependent marginal score.

\subsection{Conditional Scores, Denoising Score Matching, and Generation}
\label{subsec:conditional_score_dsm}

\paragraph{Conditional target and reverse dynamics.}
Let $C$ denote observed conditioning information, and suppose that the target is the conditional distribution $P_0(\,\cdot\mid c)=\mathcal L(X_0\mid C=c)$. Conditional diffusion perturbs the target variable $X_0$ while leaving the conditioning information unchanged. Thus,
\[
    X_t=\alpha_tX_0+h_t^{1/2}Z_t,\qquad Z_t\sim\mathcal N(0,1),\qquad Z_t\indep(X_0,C).
\]
For a realized condition $C=c$, the conditional time-$t$ density is $p_t(x\mid c)=\int_{\mathbb R}q_t(x\mid x_0)\,P_0(\mathrm dx_0\mid c)$. It is therefore a conditional mixture of the known perturbation kernel. Its score is $s_t(x\mid c) =\nabla_x\log p_t(x\mid c)$. Both arguments are essential: the functional form of $p_t(\cdot\mid c)$, and hence that of $s_t(\cdot\mid c)$, generally changes with the diffusion time $t$ and the conditioning value $c$.

Conditional on $C=c$, replacing the marginal score in \eqref{eq:scalar_reverse_sde} with $s_t(x\mid c)$ gives the exact conditional reverse-time process:
\begin{equation}
\begin{aligned}
    \mathrm dX_u^{\leftarrow}=\biggl[\frac{1}{2}\eta(T-u)X_u^{\leftarrow}+\eta(T-u)s_{T-u}\bigl(X_u^{\leftarrow}\mid c\bigr)\biggr]\mathrm du+\eta(T-u)^{1/2}\,\mathrm d\overline W_u,\qquad u\in[0,T-t_0].
    \label{eq:scalar_conditional_reverse}
\end{aligned}
\end{equation}
If $X_0^{\leftarrow}\sim P_T(\,\cdot\mid c)$, then $X_u^{\leftarrow}\sim P_{T-u}(\,\cdot\mid c)$ and $X_{T-t_0}^{\leftarrow}\sim P_{t_0}(\,\cdot\mid c)$. For each fixed $c$, the term $\alpha_TX_0$ becomes negligible when $T$ is sufficiently large. Hence, $P_T(\cdot\mid c)$ is close to the standard normal distribution $\mathcal N(0,1)$, which supplies a tractable initialization for conditional generation.

\paragraph{Denoising score matching.}
The conditional reverse process cannot be implemented directly because
$p_t(\cdot\mid c)$, and therefore its score, is unknown. For a candidate score function $s(x,c,t)$, the natural population score-matching criterion is
\begin{equation}
    \mathcal L_{\mathrm{SM}}(s)=\int_{t_0}^T \lambda(t)\,\mathbb E\left[\left(s(X_t,C,t)-s_t(X_t\mid C)\right)^2\right]\mathrm dt,
    \label{eq:scalar_conditional_score_matching}
\end{equation}
where $\lambda(t)>0$ and $\int_{t_0}^T \lambda(t)\,\mathrm dt=1$.\footnote{Throughout, a time subscript in score denotes the population counterpart, whereas the final time argument denotes a candidate or estimated score function.} The expectation is taken under the joint distribution of $(X_0,C)$ and the forward perturbation.

The forward perturbation kernel, by contrast, is known. Since $q_t(\cdot\mid x_0)$ is the density of $\mathcal N(\alpha_tx_0,h_t)$, its score (transition score) is available in closed form:
\begin{equation}
    \nabla_{x_t}\log q_t(x_t\mid x_0)=\frac{\alpha_tx_0-x_t}{h_t},\qquad\nabla_{X_t}\log q_t(X_t\mid X_0)=-\frac{Z_t}{h_t^{1/2}}.
    \label{eq:scalar_transition_score}
\end{equation}
The second expression follows by substituting $X_t=\alpha_tX_0+h_t^{1/2}Z_t$ into the first. Because $Z_t$ is independent of $C$ conditional on $X_0$, the conditional denoising identity gives $s_t(x\mid c)=\mathbb E\left[\frac{\alpha_tX_0-x}{h_t}\,\middle|\,X_t=x,\ C=c\right]$. In other words, the conditional marginal score is the regression function of the analytically available transition score on $(X_t,C)$.

The conditional-expectation projection identity then implies that replacing $s_t(X_t\mid C)$ in \eqref{eq:scalar_conditional_score_matching} by the transition score in \eqref{eq:scalar_transition_score} changes the loss only by an additive term that does not depend on $s$. The two population objectives therefore have the same minimizer. This yields the conditional denoising score-matching criterion
\begin{equation}
\begin{aligned}
    \mathcal L_{\mathrm{DSM}}(s)=\int_{t_0}^T \lambda(t)\,\mathbb E \left[ \left\{ s(X_t,C,t)-\frac{\alpha_tX_0-X_t}{h_t}\right\}^2\right]\mathrm dt.
    \label{eq:scalar_conditional_dsm}
\end{aligned}
\end{equation}

\paragraph{Empirical score learning.}
Given finite training observations $\{(X_0^{(i)},C^{(i)})\}_{i=1}^n$, we take the uniform time weight $\lambda(t)=1/(T-t_0)$. The corresponding empirical counterpart of \eqref{eq:scalar_conditional_dsm} is
\begin{equation*}
\begin{aligned}
    \widehat{\mathcal L}_{\mathrm{DSM}}(s)=\frac{1}{n}\sum_{i=1}^n\frac{1}{T-t_0}\int_{t_0}^T\mathbb E_{Z_t}\left[\left\{s\left(\alpha_tX_0^{(i)}+h_t^{1/2}Z_t,C^{(i)},t\right)+\frac{Z_t}{h_t^{1/2}}\right\}^2\right]\mathrm dt.
\end{aligned}
\end{equation*}
For a parameterized score class $\mathcal S$, the trained conditional score is defined by $\widehat s\in\arg\min_{s\in\mathcal S}\widehat{\mathcal L}_{\mathrm{DSM}}(s)$. In computation, the diffusion time $t$ and Gaussian noise $Z_t$ are sampled, so this optimization can be carried out by stochastic gradient methods. The cutoff $t_0>0$ excludes the near-zero-noise region, where the variance $h_t^{-1}$ of the denoising target diverges as $t\downarrow0$ (making the score-learning problem increasingly ill-conditioned).

\paragraph{Conditional generation.}
Let $\widehat s$ denote the learned (trained) conditional score. Replacing the unknown conditional score in \eqref{eq:scalar_conditional_reverse} with $\widehat s$ and initializing the reverse process from the Gaussian approximation to $P_T(\cdot\mid c)$ yields the implementable generator
\begin{equation}
\begin{aligned}
    \mathrm d\widehat X_u^{\leftarrow}&=\biggl[\frac{1}{2}\eta(T-u)\widehat X_u^{\leftarrow}+\eta(T-u)\widehat s\bigl(\widehat X_u^{\leftarrow},c,T-u\bigr)\biggr]\mathrm du+\eta(T-u)^{1/2}\,\mathrm d\overline W_u,\quad u\in[0,T-t_0],
    \label{eq:scalar_implementable_reverse}
\end{aligned}
\end{equation}
with $\widehat X_0^{\leftarrow}\sim\mathcal N(0,1)$. For a fixed condition $c$, repeated independent solutions of \eqref{eq:scalar_implementable_reverse} produce draws from the learned conditional distribution. Omitting the conditioning variable $C$ recovers the unconditional diffusion model as a special case \citep{chen2023score,fu2024unveil}.

Section~\ref{sec:main-method} extends this conditional-diffusion principle to tensor data for counterfactual distribution recovery. The framework keeps the observed control outcomes fixed while generating the missing control outcomes conditional on them. For notational simplicity, from Section~\ref{sec:main-method} onward we set $\eta(t)\equiv1$, so that $\alpha_t=e^{-t/2}$ and $h_t=1-e^{-t}$.

\section{Learning the Missing Control-World Distribution}\label{sec:main-method}

This section develops \CFTDiff, a tensor conditional-diffusion framework for recovering the conditional distribution of missing control outcomes. We formulate the counterfactual problem under a fixed treatment mask and introduce a masked diffusion that perturbs only the missing outcomes while holding the observed control outcomes fixed. Tucker structure then reduces nonlinear score learning to the latent core, leading to the estimation procedure and its intrinsic-dimensional guaranties. Figure~\ref{fig:conditional-tensor-counterfactual} summarizes its implementation.

\begin{figure}[t]
\FIGURE
{\includegraphics[width=1.02\linewidth]
{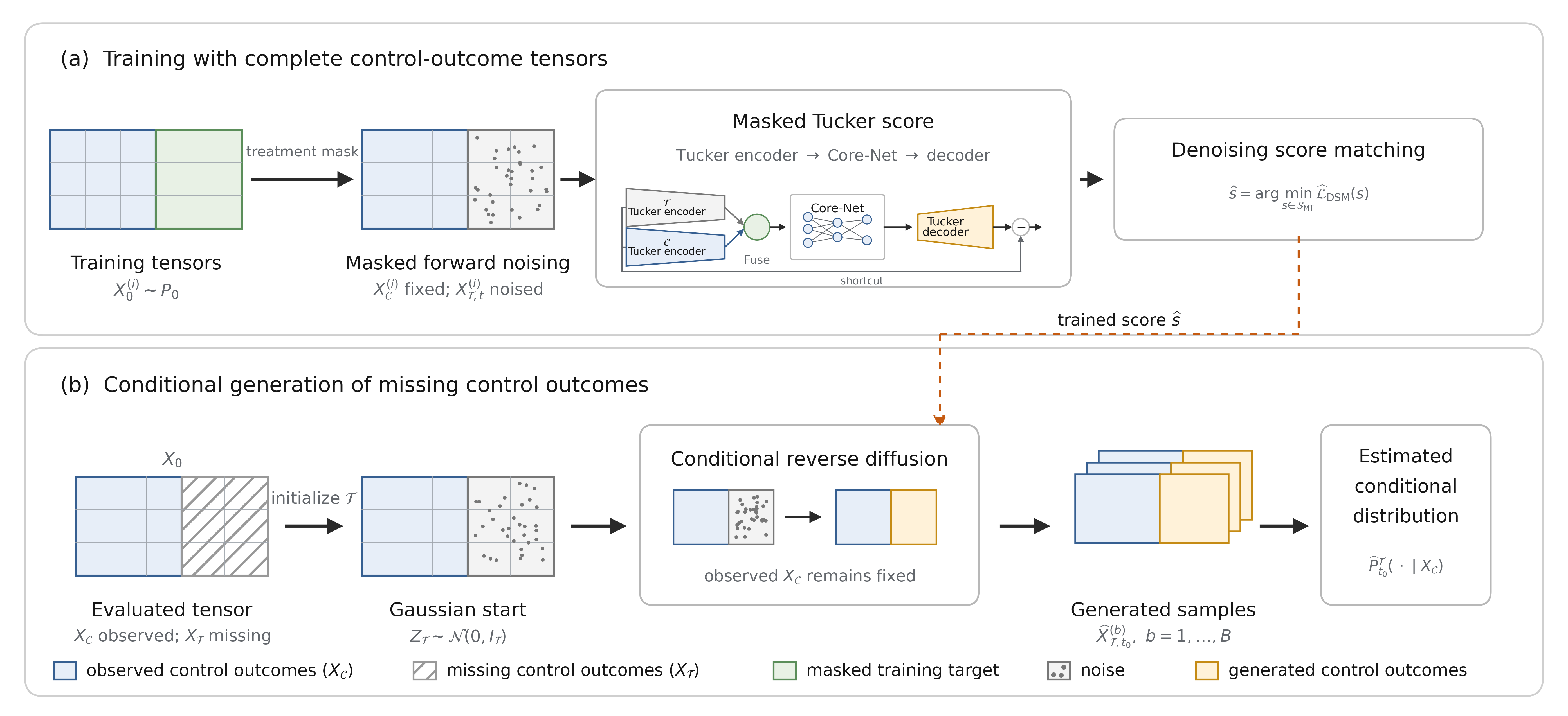}}
{Training and conditional generation in \CFTDiff.\label{fig:conditional-tensor-counterfactual}}
{Panel (a) shows the masked Tucker score architecture and its estimation from training tensors. The missing-outcome and observed-outcome branches are encoded into the Tucker core, combined by the Core-Net, and decoded back to the treated region; only the missing outcomes are noised during training. Panel (b) holds the observed control outcomes fixed and applies the learned score to transform Gaussian noise in the missing region into samples from $\widehat P_{t_0}^{\mathcal T}(\cdot\mid X_{\mathcal C})$. Observed treated outcomes are reserved for subsequent causal analysis.}
\end{figure}

\subsection{Counterfactual Target and Training Sample}\label{subsec:target_training}

\paragraph{Counterfactual learning as tensor completion.}
Consider a $D$-way outcome array with $p=\prod_{d=1}^Dp_d$ entries and index set $\mathcal I=[p_1]\times\cdots\times[p_D]$. For each coordinate $\boldsymbol i=(i_1,\ldots,i_D)\in\mathcal I$, let $Y_{\boldsymbol i}(a)$ denote the potential outcome under intervention $a\in\{0,1,\ldots,\mathbb A\}$, where $a=0$ denotes control. The tensor modes may represent units, time periods, outcome dimensions, or other sources of multi-way variation. 

Let the mutually disjoint sets $\mathcal T_a\subseteq\mathcal I$ collect the coordinates assigned to intervention $a$, and define $\mathcal T=\bigcup_{a=1}^{\mathbb A}\mathcal T_a$ and $\mathcal C=\mathcal I\setminus\mathcal T$. Under consistency, the observed outcome satisfies
\begin{equation*}
Y_{\boldsymbol i}^{\mathrm{obs}}
=
\begin{cases}
Y_{\boldsymbol i}(0),
    & \boldsymbol i\in\mathcal C,\\[1mm]
Y_{\boldsymbol i}(a),
    & \boldsymbol i\in\mathcal T_a,\quad a=1,\ldots,\mathbb A.
\end{cases}
\end{equation*}
Thus, every coordinate has an observed outcome, but a treated coordinate does not reveal the outcome that would have occurred under control. Equivalently, the missing control outcomes are $\{Y_{\boldsymbol i}(0):\boldsymbol i\in\mathcal T\}$.

Collect all control potential outcomes in the control-world tensor $X_0=\bigl[Y_{\boldsymbol i}(0)\bigr]_{\boldsymbol i\in\mathcal I} \in\mathbb R^{p_1\times\cdots\times p_D}$. Let $P_0=\mathcal L(X_0)$ denote its joint distribution. Let $M_{\mathcal T}$ be the binary tensor equal to one on $\mathcal T$ and zero elsewhere, and set $M_{\mathcal C}=\mathbf 1-M_{\mathcal T}$. The treatment design partitions $X_0$ into
\begin{equation}
    X_{\mathcal C}
    =M_{\mathcal C}\odot X_0
    =M_{\mathcal C}\odot Y^{\mathrm{obs}},
    \qquad
    X_{\mathcal T}
    =M_{\mathcal T}\odot X_0.
    \label{eq:counterfactual_tensor_partition}
\end{equation}
Both $X_{\mathcal C}$ and $X_{\mathcal T}$ retain the dimensions of $X_0$ and are padded with zeros outside the coordinates selected by their respective masks. Because $M_{\mathcal C}+M_{\mathcal T}=\mathbf 1$, this partition satisfies $X_0=X_{\mathcal C}+X_{\mathcal T}$. When vector notation is needed, we write $\MT=\diag\{\vecc(M_{\mathcal T})\}$ and $\MC=I_p-\MT$. The block $X_{\mathcal C}$ contains the observed control outcomes, whereas $X_{\mathcal T}$ contains the control outcomes missing at treated coordinates. Throughout, ``treated region'' refers to the coordinate set $\mathcal T$; $X_{\mathcal T}$ therefore contains outcomes under control for coordinates at which an intervention occurs. Treatment creates a structured missing region in the control-outcome tensor even though the realized-outcome tensor $Y^{\mathrm{obs}}$ is fully observed. Unlike conventional tensor completion, our objective is to recover the conditional distribution of the missing control outcomes rather than merely impute each missing entry. This distribution directly yields conditional means, quantiles, and other distributional summaries.

Specifically, conditional on the observed control outcomes, our target is
\begin{equation}
    P_0^{\mathcal T}(\,\cdot\mid X_{\mathcal C}):=\mathcal L(X_{\mathcal T}\mid X_{\mathcal C}).
    \label{eq:counterfactual_conditional_law}
\end{equation}
The treatment mask is prespecified and held fixed,\footnote{The tensor-completion step conditions on the specified treatment mask. Thus, even when the mask arises from randomized assignment, its sampling distribution need not be modeled by the conditional generator.} so it is suppressed from the conditioning set. Recovering \eqref{eq:counterfactual_conditional_law} provides the full conditional distribution of the missing control outcomes in the treated region.

To connect this distributional target to downstream causal summaries, let $M_{\mathcal T_a}$ denote the binary mask for the intervention-specific support $\mathcal T_a$. For intervention $a$, a deterministic weight vector $w\in\mathbb R^p$ supported on $\mathcal T_a$ defines the control-outcome summary
\begin{equation}
U_{0,a}(w)=w^\top\vecc\bigl(M_{\mathcal T_a}\odot X_0\bigr).
\end{equation}
The vector $w$ specifies which coordinates enter the comparison and how they are aggregated. Because $U_{0,a}(w)$ is a linear functional of $X_{\mathcal T}$, its conditional distribution is determined by \eqref{eq:counterfactual_conditional_law}. A coordinate-selection weight isolates the control outcome at a specified tensor coordinate, whereas equal weights yield an average over selected coordinates.\footnote{For example, if $\mathcal T_a=\{\boldsymbol i_1,\boldsymbol i_2,\boldsymbol i_3\}$ and $w$ assigns weight $1/3$ to each coordinate, then $U_{0,a}(w)=\frac{1}{3}\sum_{j=1}^3Y_{\boldsymbol i_j}(0)$.} The corresponding weighted outcome under intervention $a$ is observed and enters only the subsequent causal analysis; it is not supplied to the conditional generator.

\paragraph{Training sample.}
We learn \eqref{eq:counterfactual_conditional_law} using a finite training sample $\{X_0^{(i)}\}_{i=1}^n$ of complete tensors drawn from $P_0$. These tensors may correspond to comparable historical periods or external panels observed entirely under control.

Applying the fixed treatment mask synthetically to each training tensor gives $X_{\mathcal T}^{(i)}=M_{\mathcal T}\odot X_0^{(i)}$ and $X_{\mathcal C}^{(i)}=M_{\mathcal C}\odot X_0^{(i)}$. This construction produces paired missing and observed control outcomes for learning the conditional score. For the data used in causal analysis, the observed control outcomes $X_{\mathcal C}$ are held fixed and provided to the learned conditional generator, which generates the missing control outcomes $X_{\mathcal T}$ conditionally on them. The observed treated outcomes remain outside the generator and enter only the subsequent causal analysis.

\subsection{Masked Conditional Diffusion in the Treated Region}\label{subsec:masked_diffusion}

We now develop the masked conditional diffusion underlying \CFTDiff for the tensor partition in \eqref{eq:counterfactual_tensor_partition}. Under the unit noise schedule adopted above, $\alpha_t=e^{-t/2}$ and $h_t=1-e^{-t}$. The forward process perturbs only the missing outcomes $X_{\mathcal T}$, while the observed control outcomes $X_{\mathcal C}$ remain fixed and condition the diffusion throughout.

Let $\{W_t\}_{t\geq0}$ be a tensor-valued Wiener process with independent standard Brownian entries. The forward diffusion of the missing outcomes is
\begin{equation}
    \mathrm dX_{\mathcal T,t}=-\frac12X_{\mathcal T,t}\,\mathrm dt + M_{\mathcal T}\odot\mathrm dW_t,\qquad
    X_{\mathcal T,0}=X_{\mathcal T}.
    \label{eq:main-forward-sde}
\end{equation}
For every fixed $t\in[0,T]$, the solution of \eqref{eq:main-forward-sde} has the explicit form
\begin{equation}
    X_{\mathcal T,t}
    =
    \alpha_tX_{\mathcal T}
    +h_t^{1/2}(M_{\mathcal T}\odot Z_t),
    \qquad
    Z_t\indep(X_{\mathcal T},X_{\mathcal C}),
    \label{eq:main-forward}
\end{equation}
where $Z_t$ has independent standard Gaussian entries. 

Let $\mathbb X_{\mathcal T}=\left\{X\in\mathbb R^{p_1\times\cdots\times p_D}: M_{\mathcal T}\odot X=X\right\}$ be the tensor space supported on $\mathcal T$. Conditional on $X_{\mathcal C}$, let $P_t^{\mathcal T}(\cdot\mid X_{\mathcal C})=\mathcal L(X_{\mathcal T,t}\mid X_{\mathcal C})$ denote the conditional distribution of $X_{\mathcal T,t}$ on this space, and let $p_t^{\mathcal T}(\cdot\mid X_{\mathcal C})$ denote its density with respect to Lebesgue measure over the coordinates in the treated region. The corresponding conditional score is $ s_t(X\mid X_{\mathcal C})=\nabla_{\mathcal T}\log p_t^{\mathcal T}(X\mid X_{\mathcal C})$ for $X\in\mathbb X_{\mathcal T}$, where $\nabla_{\mathcal T}$ differentiates only with respect to entries on $\mathcal T$ and pads the coordinates outside $\mathcal T$ with zeros. Consequently, $M_{\mathcal T}\odot s_t(X\mid X_{\mathcal C})=s_t(X\mid X_{\mathcal C})$ for $X\in\mathbb X_{\mathcal T}$.

With this conditional score, the exact reverse-time diffusion is
\begin{equation}
    \mathrm dX_{\mathcal T,u}^{\leftarrow}
    ={}
    \left\{
        \frac12X_{\mathcal T,u}^{\leftarrow}
        +
        s_{T-u}\bigl(X_{\mathcal T,u}^{\leftarrow}\mid X_{\mathcal C}\bigr)
    \right\}\mathrm du
    +M_{\mathcal T}\odot\mathrm d\overline W_u,\qquad u\in[0,T-t_0].
    \label{eq:main-exact-reverse}
\end{equation}
Here $\{\overline W_u\}_{u\geq0}$ is a tensor-valued Wiener process with independent standard Brownian entries. Conditional on $X_{\mathcal C}$, if $X_{\mathcal T,0}^{\leftarrow}\sim P_T^{\mathcal T}(\cdot\mid X_{\mathcal C})$, then $X_{\mathcal T,u}^{\leftarrow}\sim P_{T-u}^{\mathcal T}(\cdot\mid X_{\mathcal C})$ and $X_{\mathcal T,T-t_0}^{\leftarrow}\sim P_{t_0}^{\mathcal T}(\cdot\mid X_{\mathcal C})$. The conditional score therefore determines the reverse-time process and is the central object to be learned. Because an unrestricted tensor score is high dimensional, the next subsection develops a Tucker-structured representation whose nonlinear component operates only on a low-dimensional core.

\subsection{Low-Dimensional Tucker Structure of the Conditional Score}\label{subsec:score_structure}

The conditional score in Section~\ref{subsec:masked_diffusion} is a tensor-valued function whose inputs and output are all high dimensional. Estimating and repeatedly evaluating such a function can be prohibitively costly, particularly because reverse-time generation requires score evaluation at every discretization step. Tucker structure provides a parsimonious representation of large tensors and can substantially reduce the cost of tensor computation \citep{kolda2009tensor}. We, therefore, posit a low-Tucker-rank model for the complete control-outcome (control-world) tensor. This structure exploits dependence across tensor modes, organizes score learning and generation around a low-dimensional core, and yields the conditional-score representation derived below.

Recall that $p=\prod_{d=1}^Dp_d$, and define $p^\beta=\prod_{d=1}^Dp_d^{\beta_d}$, where $\beta_d\in[0,1]$ indexes factor strength along mode $d$.\footnote{Under the normalization below, $\beta_d=0$ gives the weak factor boundary, with no dimension-driven attenuation of idiosyncratic noise along mode $d$, whereas $\beta_d=1$ gives the strong-factor boundary, with attenuation at a rate of $p_d^{-1/2}$. Values in $(0,1)$ accommodate intermediate strength, and allowing $\beta_d$ to vary across modes permits heterogeneous factor strengths; related prediction theory under weak factors is developed by \citet{giglio2026prediction}.} For compactness, write $F\times_{d=1}^DA_d:=F\times_1A_1\times_2\cdots\times_DA_D$. We assume
\begin{equation}
    X_0
    =
    F\times_{d=1}^D A_d
    +p^{-\beta/2}E,
    \qquad
    A_d^\top A_d=I_{r_d},\quad d\in[D].
    \label{eq:main-factor-model}
\end{equation}
Here $F\in\mathbb R^{r_1\times\cdots\times r_D}$ is the Tucker core, $A_d\in\mathbb R^{p_d\times r_d}$ is the loading matrix for mode $d$, and $E$ is idiosyncratic tensor noise. The full dimension is $p$, whereas $r=\prod_{d=1}^Dr_d$ is the intrinsic core dimension. The scaling $p^{-\beta/2}$ allows factor strength to increase with the tensor dimensions.

\begin{remark}[Connection with vector and matrix factor models]
The Tucker representation nests familiar factor models as special cases. When $D=1$, \eqref{eq:main-factor-model} becomes $X_0=A_1F+p^{-\beta/2}E$, where $F\in\mathbb R^{r_1}$ is a vector of latent factors and $A_1$ contains their loading vectors, as in the standard vector factor model \citep{bai2003inferential,bai2023approximate}. When $D=2$, it becomes $X_0=A_1FA_2^\top+p^{-\beta/2}E$, where $A_1$ and $A_2$ are the row and column loading matrices, respectively, and $F\in\mathbb R^{r_1\times r_2}$ is the matrix-valued factor. This is the bilinear, or two-way, matrix factor specification \citep{wang2019factor, yu2023testing,yuan2023twoway, zhang2025modeling, he2025new}. The $D$-way formulation retains this loading-factor interpretation while representing the structured component through $r$ core coordinates rather than $p$ full coordinates. Consequently, the proposed framework also covers counterfactual recovery with vector- or matrix-valued outcomes by taking $D=1$ or $D=2$, respectively.
\end{remark}

\begin{assumption}[Control-world tensor factor model]
\label{ass:tensor}
The complete control-world tensor follows \eqref{eq:main-factor-model}. The vectorized core $f=\vecc(F)\in\mathbb R^r$ has density $p_{\core}$, mean zero, and finite second moment. It is independent of $e=\vecc(E)$, where $e\sim\mathcal N(0,\Sigma_e^{\otimes})$ and $\Sigma_e^{\otimes}=\Sigma_{e,D}\otimes\cdots\otimes\Sigma_{e,1}$, with $\Sigma_{e,d}=\diag(\sigma_{d1}^2,\ldots,\sigma_{dp_d}^2)$ and $0<\sigma_{\min}^2\leq \sigma_{dj}^2\leq \sigma_{\max}^2<\infty$ for $d\in[D],\ j\in[p_d]$.
\end{assumption}

Assumption~\ref{ass:tensor} concerns the joint distribution of the control outcomes; no factor model is imposed on the outcomes under active interventions. The separable diagonal covariance makes the conditional score under the treatment mask analytically tractable, while the uniform variance bounds rule out degenerate coordinates. To expose the resulting score structure, let $A_\otimes=A_D\otimes\cdots\otimes A_1$ and define the idiosyncratic variance tensor $\mathcal Q_\sigma$ by $(\mathcal Q_\sigma)_{i_1,\ldots,i_D}=\prod_{d=1}^D\sigma_{d i_d}^2$.

The construction below combines the noised missing outcomes and the observed control outcomes as two measurements of the same latent core. Conditional on $f$, $X_{\mathcal T,t}$ is a Gaussian measurement of the scaled core $\alpha_t f$ in the treated region, with coordinatewise variances $h_t+\alpha_t^2p^{-\beta}\mathcal Q_\sigma$. The observed control outcomes $X_{\mathcal C}$ measure $f$ on $\mathcal C$, or equivalently $\alpha_t f$ after rescaling by $\alpha_t$. Because the two measurements are conditionally independent, they can be combined by precision-weighted least squares in the core space.

Define their coordinatewise precision tensors by
\[
    \mathcal W_{\mathcal T,t}
    =
    M_{\mathcal T}\os
    \bigl(h_t\mathbf 1+\alpha_t^2p^{-\beta}\mathcal Q_\sigma\bigr),
    \qquad
    \mathcal W_{\mathcal C}
    =
    M_{\mathcal C}\os
    \bigl(p^{-\beta}\mathcal Q_\sigma\bigr).
\]
Each tensor assigns the inverse conditional variance to coordinates in its corresponding support and zero weight elsewhere. Projecting these
precisions onto the Tucker loading space gives the core-level information matrices
\[
\begin{aligned}
    H_{\mathcal T,t}
    =
    A_\otimes^\top
    \diag\{\vecc(\mathcal W_{\mathcal T,t})\}
    A_\otimes,
    \qquad
    H_{\mathcal C}
    =
    A_\otimes^\top
    \diag\{\vecc(\mathcal W_{\mathcal C})\}
    A_\otimes.
\end{aligned}
\]

Because $M_{\mathcal T}$ and $M_{\mathcal C}$ partition the tensor coordinates and $A_\otimes^\top A_\otimes=I_r$, the combined core-level precision matrix is positive definite. We therefore define
\[
    V_t=\left(H_{\mathcal T,t}+\alpha_t^{-2}H_{\mathcal C}\right)^{-1}.
\]

The key step is to combine the two precision-weighted encoder outputs into a single core statistic:
\begin{equation}\label{eq:main-core-statistic}
g_t=
V_t\left[
\vecc\left\{ \underbrace{(\mathcal W_{\T,t}\od X_{\T,t})\times_{d=1}^D A_d^\top}_{\rm{missing\mbox{-}outcome~encoder}}
\right\}
+
\alpha_t^{-1}
\vecc\left\{\underbrace{(\mathcal W_\C\od X_\C)\times_{d=1}^D A_d^\top}_{\rm{observed\mbox{-}outcome~encoder}}
\right\}
\right],
\end{equation}
Let $G_t=\operatorname{Tucker}(g_t)$, where $\operatorname{Tucker}(g):=\operatorname{reshape}\{g;(r_1,\ldots,r_D)\}$. Standard Gaussian least-squares calculations then give $g_t\mid f\sim\mathcal N(\alpha_t f,V_t)$. Thus, $g_t$ is an $r$-dimensional sufficient statistic for the scaled core signal $\alpha_t f$: it pools the information retained in the noised missing outcomes with the undiffused information in the observed control outcomes.

\begin{proposition}[Conditional Tucker score decomposition]\label{prop:main-score-decomposition}
Under Assumption~\ref{ass:tensor}, for every $t\in[t_0,T]$, the conditional tensor score satisfies
\begin{equation}
\begin{aligned}
    s_t(X_{\mathcal T,t}\mid X_{\mathcal C})=
    \underbrace{
        \mathcal W_{\mathcal T,t}
    }_{\text{treated-region precision}}
    \odot
    \left[
        \left\{
            \underbrace{
                \mathbb E(\alpha_tF\mid G_t)
            }_{\text{core regression}}
            \underbrace{
                \times_{d=1}^D A_d
            }_{\text{Tucker decoder}}
        \right\}
        \underbrace{
            -X_{\mathcal T,t}
        }_{\text{skip connection}}
    \right].
    \label{eq:main-score-decomposition}
\end{aligned}
\end{equation}
The score is supported on $\mathcal T$ and is padded with zeros on
$\mathcal C$.
\end{proposition}
The statistic $g_t$ pools the noised outcomes in the treated region and the observed control outcomes according to their precisions. Its conditional core mean is decoded to the original tensor coordinates, yielding a precision-weighted denoising score in the treated region, while the observed control outcomes remain fixed.

\subsection{\CFTDiff: Estimation and Generation}\label{subsec:cftdiff_estimation}

\paragraph{Masked Tucker score class.}
Proposition~\ref{prop:main-score-decomposition} depends on the unknown loading matrices, idiosyncratic variances, and conditional core
regression. To obtain an estimable score class, we replace them with trainable counterparts. Let $\Gamma=\{\Gamma_d\}_{d=1}^D$ denote candidate Tucker loading matrices satisfying $\Gamma_d^\top \Gamma_d=I_{r_d}$ for $d\in[D]$, and let $\omega=\{\omega_{dj}\}$ denote candidate idiosyncratic variances satisfying $\sigma_{\min}^2\leq \omega_{dj}\leq \sigma_{\max}^2$ for $d\in[D]$ and $\ j\in[p_d]$. For any $(\Gamma,\omega)$, let $\mathcal W_{\mathcal T,t}(\omega)$, $\mathcal W_{\mathcal C}(\omega)$, $V_t(\Gamma,\omega)$, and $G_t(\Gamma,\omega)$ denote the quantities obtained from the precision-weighted construction defined in Section~\ref{subsec:score_structure} after replacing $(A_d,\sigma_{dj}^2)$ with $(\Gamma_d,\omega_{dj})$. Thus, $G_t(\Gamma,\omega)$ is the candidate Tucker-core summary of $(X_{\mathcal T,t},X_{\mathcal C})$, with $r=\prod_{d=1}^Dr_d$ coordinates.

Let $\mathcal Z_\theta:\mathbb R^{r_1\times\cdots\times r_D}\times[t_0,T] \to\mathbb R^{r_1\times\cdots\times r_D}$ be a Core-Net. The corresponding score network is
\begin{equation}
\begin{aligned}
    s_{\Gamma,\omega,\theta}(X_{\mathcal T,t},X_{\mathcal C},t)=\mathcal W_{\mathcal T,t}(\omega)\od
    \left[\mathcal Z_\theta\left(G_t(\Gamma,\omega),t\right)\times_{d=1}^D \Gamma_d-X_{\mathcal T,t}\right].
    \label{eq:main-mtucker}
\end{aligned}
\end{equation}
For $g\in\mathbb R^r$, define $\zeta_\theta(g,t)=\vecc\left[\mathcal Z_\theta(\operatorname{Tucker}(g),t)\right]$. Let $\mathcal F_{\core}(L,m,J,K,\kappa,\gamma_g,\gamma_t)$ denote the class of feed-forward ReLU networks with depth at most $L$, hidden-layer width at most $m$, at most $J$ nonzero parameters, and parameter magnitudes bounded by $\kappa$. The constants $K$, $\gamma_g$, and $\gamma_t$ bound the network output and its Lipschitz variation in the core and time arguments, respectively. A formal definition is provided in Appendix~\ref{appendix_subsec:core-network-class}. For a fixed constant $C_V>0$, the \emph{masked Tucker class} is
\[
\begin{aligned}
    \mathcal S_{\MTClass}
    =
    \biggl\{
        s_{\Gamma,\omega,\theta}:\;&
        \Gamma_d^\top \Gamma_d=I_{r_d},\quad d\in[D],
        \sigma_{\min}^2
        \leq\omega_{dj}
        \leq\sigma_{\max}^2,\\
        &
        \sup_{t\in[t_0,T]}
        \|V_t(\Gamma,\omega)\|_{\op}
        \leq C_V,\quad
        \zeta_\theta\in
        \mathcal F_{\core}
        (L,m,J,K,\kappa,\gamma_g,\gamma_t)
    \biggr\}.
\end{aligned}
\]
The bound on $V_t(\Gamma,\omega)$ controls the amplification induced by aggregating the missing-outcome and observed-outcome branches in the core space, while the constraint on $\zeta_\theta$ controls the complexity of the nonlinear core regression. As displayed in \eqref{eq:main-mtucker}, each score evaluation pools the two branches in the core space, applies the Core-Net, decodes the result through the mode-wise loadings, and forms the precision-weighted residual on $\mathcal T$.

\paragraph{Score estimation.}
The Gaussian transition in \eqref{eq:main-forward} has treated-region score $h_t^{-1}(\alpha_tX_{\mathcal T}-X_{\mathcal T,t})=-h_t^{-1/2}(M_{\mathcal T}\od Z_t)$. Applying the conditional denoising identity from Section~\ref{subsec:conditional_score_dsm}, the population masked score-matching loss is
\begin{equation*}
\begin{aligned}
    \mathcal L_{\mask}(s)
    =
    \frac{1}{T-t_0}
    \int_{t_0}^T
    \mathbb E
    \left[
        \left\|
            s(X_{\mathcal T,t},X_{\mathcal C},t)
            -h_t^{-1}(\alpha_tX_{\mathcal T}-X_{\mathcal T,t})
        \right\|_F^2
    \right]
    \mathrm dt.
\end{aligned}
\end{equation*}
The expectation is over a complete control-outcome tensor and its forward perturbation. Its population minimizer is $s_t(X_{\mathcal T,t}\mid X_{\mathcal C})$. For the training sample introduced in Section~\ref{subsec:target_training}, define
\begin{equation}
\begin{aligned}
    \widehat s
    &\in
    \operatorname*{arg\,min}_{s\in\mathcal S_{\MTClass}}
    \widehat{\mathcal L}_{\mask}(s),
    \qquad
    \widehat{\mathcal L}_{\mask}(s)
    =
    \frac1n\sum_{i=1}^n\ell(X_0^{(i)};s),
    \\
    \ell(X_0;s)
    &=
    \frac{1}{T-t_0}
    \int_{t_0}^T
    \mathbb E_{X_{\mathcal T,t}\mid X_{\mathcal T}}
    \left[
        \left\|
            s(X_{\mathcal T,t},X_{\mathcal C},t)
            -h_t^{-1}(\alpha_tX_{\mathcal T}-X_{\mathcal T,t})
        \right\|_F^2
    \right]
    \mathrm dt.
    \label{eq:mask_empirical_loss}
\end{aligned}
\end{equation}
Stochastic training samples a training tensor, a diffusion time, and Gaussian perturbation noise. The same fixed mask is used throughout, so the learned score is aligned with the missing-region pattern used in conditional generation.

\paragraph{Conditional generation.}
Given the observed control outcomes $X_{\mathcal C}$, \CFTDiff\ uses the learned score $\widehat s$ in the reverse dynamics \eqref{eq:main-exact-reverse} and initializes the missing-outcome tensor from its large-$T$ Gaussian approximation:
\begin{equation}
    \mathrm d\widehat X_{\mathcal T,u}^{\leftarrow}
    ={}
    \left\{
        \frac12 \widehat X_{\mathcal T,u}^{\leftarrow}
        +
        \widehat s
        \bigl(\widehat X_{\mathcal T,u}^{\leftarrow},X_{\mathcal C},T-u\bigr)
    \right\}\mathrm du
    +M_{\mathcal T}\odot\mathrm d\overline W_u,\quad
    u\in[0,T-t_0],
    \quad
    \widehat X_{\mathcal T,0}^{\leftarrow}
    =M_{\mathcal T}\odot Z_0,
    \label{eq:main-reverse}
\end{equation}
with $\vecc(Z_0)\sim\mathcal N(0,I_p)$. The output $\widehat X_{\mathcal T,T-t_0}^{\leftarrow}$ is one generated sample of the missing control outcomes in the treated region. Conditional on the observed control outcomes, denote its early-stopped distribution by $\widehat P_{t_0}^{\mathcal T}(\cdot\mid X_{\mathcal C}):=\mathcal L(\widehat X_{\mathcal T,T-t_0}^{\leftarrow}\mid X_{\mathcal C})$. By the support convention following \eqref{eq:counterfactual_tensor_partition}, $\widehat X_{\mathcal T,T-t_0}^{\leftarrow}$ is zero on $\mathcal C$. Consequently, $X_{\mathcal C}+\widehat X_{\mathcal T,T-t_0}^{\leftarrow}$ is the corresponding completed control-outcome tensor. Independent runs of \eqref{eq:main-reverse} yield samples from $\widehat P_{t_0}^{\mathcal T}(\cdot\mid X_{\mathcal C})$. 


\subsection{Intrinsic-Dimensional Guaranties for Conditional Score Learning}\label{subsec:score_guarantees}

The Tucker representation reduces the nonlinear component of the conditional score to the $r$-dimensional core space. Its statistical analysis depends on two features: concentration of the latent core and regularity of the conditional core regression. We state these conditions below.

Recall from \eqref{eq:main-core-statistic} that $g_t\mid f\sim\mathcal N(\alpha_tf,V_t)$. For $R\geq1$, define $\mathcal G_R=\left\{g\in\mathbb R^r:\|g\|_\infty\leq R\right\}$ and let
\begin{equation*}
    \xi(g,t):=\mathbb E(\alpha_tf\mid g_t=g)=
    \frac{
        \displaystyle
        \int \alpha_tf\,
        \phi(g;\alpha_tf,V_t)
        p_{\core}(f)\,\mathrm df
    }{
        \displaystyle
        \int
        \phi(g;\alpha_tf,V_t)
        p_{\core}(f)\,\mathrm df
    },
\end{equation*}
where $\phi(\cdot;\mu,V)$ is the Gaussian density with mean $\mu$ and covariance $V$. This vector-valued map is the core representation of the
conditional mean in Proposition~\ref{prop:main-score-decomposition}: $\mathbb E(\alpha_tF\mid G_t)=\operatorname{Tucker}\{\xi(g_t,t)\}$.

\begin{assumption}[Sub-gaussianlity of the latent core]\label{ass:mask_core_tail}
Let $\Sigma_f=\operatorname{Cov}(f)$. The eigenvalues of $\Sigma_f$ are bounded away from zero and infinity. Moreover, there exist constants $B_f,C_1,C_2>0$ such that, whenever $\|f\|_2\geq B_f$, $p_{\core}(f)\leq(2\pi)^{-r/2}C_1\exp\bigl(-C_2\|f\|_2^2/2\bigr)$.
\end{assumption}

Assumption~\ref{ass:mask_core_tail} is a standard concentration condition for light-tailed factor models and is rather mild. The covariance bounds only rule out degenerate or excessively dispersed factor directions, while the density bound controls the probability of extreme core realizations. Together with the Gaussian noise in $g_t\mid f$, this condition ensures that $g_t$ lies in $\mathcal G_R$ with high probability for the radius $R$ chosen below. The next assumption controls the conditional core regression on this high-probability region.

\begin{assumption}[Regularity of the conditional core regression]\label{ass:mask_xi_regular}
For the radius $R$ under consideration, suppose that, uniformly over $t\in[t_0,T]$, $V_t\preceq v_{\max}I_r$ for a constant $v_{\max}<\infty$. Assume that $\xi$ is $L_g$-Lipschitz in $g$ and $L_t$-Lipschitz in $t$ on $\mathcal G_R\times[t_0,T]$. In addition, $K_0:=1+\sup_{t\in[t_0,T]}\|\xi(0,t)\|_2<\infty$.
\end{assumption}

Assumption~\ref{ass:mask_xi_regular} is imposed only on the $r$-dimensional conditional core regression, rather than on the full $p$-dimensional score. The bound on $V_t$ uniformly limits the conditional variance of $g_t\mid f$ in every core-space direction. The Lipschitz conditions ensure that $\xi$ can be approximated uniformly by the Core-Net on $\mathcal G_R\times[t_0,T]$, while $K_0$ anchors the magnitude of the regression at the origin. Because these restrictions are required only in a high-probability region, they are weaker than the corresponding global smoothness conditions.

For example, if $f\sim\mathcal N(0,\Sigma_f)$, then $\xi(g,t)=\alpha_t^2\Sigma_f\bigl(\alpha_t^2\Sigma_f+V_t\bigr)^{-1}g$. In this case, $\xi$ is linear in $g$ and $K_0=1$, and the Lipschitz condition in $g$ follows directly. Meanwhile, the condition in $t$ holds whenever $\alpha_t$ and $V_t$ vary smoothly over $[t_0,T]$. When $V_t$ is small relative to $\alpha_t^2\Sigma_f$, the statistic $g_t$ is highly informative about $\alpha_tf$, and $\xi(g,t)$ is close to $g$.

\paragraph{Approximation in the core space.}
A central implication of Proposition~\ref{prop:main-score-decomposition} is that nonparametric score approximation can be carried out after
Tucker compression. All masking, precision weighting, and mode-wise transformations are represented exactly. The only function to be approximated is the conditional core regression $\xi$. The result below quantifies this reduction: the nonlinear approximation depends on the $r$ core coordinates and time, rather than on the $p$ entries of the original tensor.

For an approximation tolerance $\epsilon\in(0,1)$, define $R_\epsilon=C_R\sqrt{\log\{c_{\mathrm{tail}}r/(t_0\epsilon)\}}$, where $C_R$ and $c_{\mathrm{tail}}$ are sufficiently large constants. This radius contains the core statistic with sufficiently high probability, so the Core-Net need only approximate $\xi$ accurately on $\mathcal G_{R_\epsilon}\times[t_0,T]$.

Consider a Core-Net class whose complexity parameters satisfy
\begin{equation}
\begin{gathered}
    m=O\!\left((1+L_gR_\epsilon)^r(1+TL_t)\epsilon^{-(r+1)}\right),\quad
    K=O(K_0+L_gR_\epsilon),\quad
    L=O\!\left(\log(K/\epsilon)+1\right),\\[1mm]
    J=O(mL),\quad
    \kappa=O\!\left(\max\left(K_0+L_gR_\epsilon,\,TL_t,\,T^{-1}\right)\right),\quad
    \gamma_g=C_rL_g,\quad
    \gamma_t=C_rL_t,
\end{gathered}
\label{eq:core-net-approximation-configuration}
\end{equation}  
where $C_r$ depends only on the core dimension. 

\begin{theorem}[Conditional score approximation]\label{thm:main-approximation}
Suppose Assumptions~\ref{ass:tensor} and \ref{ass:mask_core_tail} hold, and Assumption~\ref{ass:mask_xi_regular} holds with $R=R_\epsilon$. Let $0<t_0\leq1$, $T>t_0$, and $C_V\geq v_{\max}$, and configure the Core-Net class according to \eqref{eq:core-net-approximation-configuration}. Let $A_\bullet=(A_1,\ldots,A_D)$ denote the population loading matrices, and let $\omega^0=(\sigma_{dj}^2)_{d,j}$ denote the population idiosyncratic variances. Then there exists a Core-Net parameter set $\bar\theta$ such that $s_{A_\bullet,\omega^0,\bar\theta}\in\mathcal S_{\MTClass}$ and, for every $t\in[t_0,T]$,
\begin{equation*}
    \mathbb E_{X_{\mathcal C}}\mathbb E_{X_{\mathcal T,t}\mid X_{\mathcal C}}
    \left[
        \left\|
            s_{A_\bullet,\omega^0,\bar\theta}
            (X_{\mathcal T,t},X_{\mathcal C},t)
            -
            s_t(X_{\mathcal T,t}\mid X_{\mathcal C})
        \right\|_F^2
    \right]
    \leq
    h_t^{-2}
    (\sqrt r+1)^2
    \epsilon^2.
\end{equation*}
\end{theorem}

Theorem~\ref{thm:main-approximation} isolates the approximation benefit of the Tucker score representation. At the population loading spaces and idiosyncratic variances, approximating the $p$-dimensional conditional score reduces to approximating $\xi$ on $\mathbb R^r\times[t_0,T]$. Accordingly, the network width in \eqref{eq:core-net-approximation-configuration} scales as $\epsilon^{-(r+1)}$. The exponent $r+1$ is the input dimension of $\xi$: $r$ core coordinates and one scalar time input. It does not involve the full dimension $p$. The factor $h_t^{-2}$ reflects the increasing sensitivity of the score near the data distribution and explains the use of the early-stopping time $t_0>0$. The next result incorporates estimation of the loading spaces and variance parameters from the training samples.

\paragraph{Learning from finitely many training tensors.}
The preceding approximation result fixes the loading spaces and idiosyncratic variances at their population values. We now allow these components, together with the conditional core regression, to be learned from the training sample. The resulting bound separates the cost of learning the nonlinear core map from that of estimating the mode-wise loading spaces.

Define $p_{\mathcal T}=\tr(\MT)$, $d_{\mathcal T,\beta}=r+p_{\mathcal T}p^{-\beta}$, and $p_{\max}=\max_{d\in[D]}p_d$. Here $p_{\mathcal T}$ is the number of missing control outcomes. The quantity $d_{\mathcal T,\beta}$ combines the $r$ core coordinates with the aggregate scale $p_{\mathcal T}p^{-\beta}$ of the idiosyncratic component on that support, while $p_{\max}$ is the largest mode dimension. For the sample size $n$, set $\delta_n=(r+12)\log\log(n)/\{2\log(n)\}$ and $\epsilon=n^{-(1-\delta_n)/(r+5)}$.

\begin{theorem}[Conditional score generalization]
\label{thm:main-generalization}
Suppose Assumptions~\ref{ass:tensor} and \ref{ass:mask_core_tail} hold, and suppose Assumption~\ref{ass:mask_xi_regular} holds on $\mathcal G_{R_{\epsilon}}\times[t_0,T]$. Configure the Core-Net class according to \eqref{eq:core-net-approximation-configuration} at the value of $\epsilon$ specified above, and choose $C_V\geq v_{\max}$. Suppose that $D$, the mode ranks $(r_1,\ldots,r_D)$, $L_g$, $L_t$, and $K_0$ are fixed. For fixed constants $c_0,c_T>0$ and all sufficiently large $n$, assume $c_0^{-1}p^{-\beta}<t_0\leq1$ and $T-t_0\geq c_T$. Then, with probability at least $1-1/n$ over the training sample,
\begin{equation}\footnotesize
    \frac{1}
    {d_{\mathcal T,\beta}(T-t_0)}
    \int_{t_0}^T
    \mathbb E_{X_{\mathcal C}}
    \mathbb E_{X_{\mathcal T,t}\mid X_{\mathcal C}}
    \left[
        \left\|
            \widehat s
            (X_{\mathcal T,t},X_{\mathcal C},t)
            -
            s_t
            (X_{\mathcal T,t}\mid X_{\mathcal C})
        \right\|_F^2
    \right]
    \mathrm dt
    =
    \widetilde{\mathcal O}\!\left[
        \left(
            \frac{1}{t_0}+T
        \right)
        \left\{
            n^{-\frac{2-2\delta_n}{r+5}}
            +
            p_{\max}
            n^{-\frac{r+3+2\delta_n}{r+5}}
        \right\}
    \right].
    \label{eq:main-generalization-rate}
\end{equation}
The notation $\widetilde{\mathcal O}$ suppresses fixed model constants and polylogarithmic factors in $n$, $p_{\max}$, $t_0^{-1}$, $\epsilon^{-1}$, and $\alpha_T^{-1}$.
\end{theorem}

The left-hand side of \eqref{eq:main-generalization-rate} is the time-integrated prediction error of the learned conditional score, normalized by the effective dimension $d_{\mathcal T,\beta}$. The two terms in the bound correspond to distinct learning tasks. The first is the cost of estimating the nonlinear conditional core regression $\xi$; because $\delta_n\to0$, its leading rate is $n^{-2/(r+5)}$ up to logarithmic factors. The second is the cost of estimating the mode-wise loading spaces. It depends on $p_{\max}$ rather than on the vectorized dimension $p$, and vanishes provided $p_{\max}n^{-\frac{r+3+2\delta_n}{r+5}}\to0$.

The size of the missing region enters only through $d_{\mathcal T,\beta}=r+p_{\mathcal T}p^{-\beta}$. In the strong-signal regime $p_{\mathcal T}\lesssim p^\beta$, this effective dimension remains bounded when the mode ranks are fixed. Thus, the nonlinear component of score learning is governed by the Tucker-core dimension, while the cost of recovering the full loading structure depends only on the largest individual mode.


Section~\ref{sec:distributional_recovery} translates this score bound into recovery guaranties for the conditional distribution of the missing control outcomes.

\section{Counterfactual Distribution Recovery}\label{sec:distributional_recovery}

This section shows how conditional-score error propagates to the generated distribution and establishes recovery and prediction-interval guaranties for the missing control outcomes and their weighted summaries.

\subsection{A Score-to-Distribution Transfer Bound}\label{subsec:score_distribution_transfer}

Recall that $P_t^{\mathcal T}(\cdot\mid X_{\mathcal C})$ is the conditional distribution of the forward-diffused missing outcomes $X_{\mathcal T,t}$, whereas $\widehat P_{t_0}^{\mathcal T}(\cdot\mid X_{\mathcal C})$ is the conditional distribution generated by the trained reverse-time process \eqref{eq:main-reverse}. Both distributions are defined over the $p_{\mathcal T}$ coordinates in the treated region, where $p_{\mathcal T}=\tr(\MT)$; at $t_0>0$, they admit densities on $\mathbb R^{p_{\mathcal T}}$. Recall also that $p_{\max}=\max_{d\in[D]}p_d$.

The following result isolates the mechanism by which conditional-score error affects the generated distribution. It applies to a fixed realization of the training sample and therefore treats the trained score $\widehat s$ as given.

\begin{proposition}[Score-to-distribution transfer]\label{prop:score_distribution_transfer}
Suppose that the reverse SDE driven by $\widehat s$ is well defined and that
\[
    \int_{t_0}^T
    \mathbb E_{X_{\mathcal C}}
    \mathbb E_{X_{\mathcal T,t}\mid X_{\mathcal C}}
    \left\|
        \widehat s(X_{\mathcal T,t},X_{\mathcal C},t)
        -s_t(X_{\mathcal T,t}\mid X_{\mathcal C})
    \right\|_F^2
    \mathrm dt
    <\infty.
\]
Then
\begin{equation}
\begin{aligned}
    \mathbb E_{X_{\mathcal C}}\left[
    \operatorname{KL}\!\left(
        P_{t_0}^{\mathcal T}(\cdot\mid X_{\mathcal C})
        \,\middle\|\,
        \widehat P_{t_0}^{\mathcal T}(\cdot\mid X_{\mathcal C})
    \right)\right]\leq &
    \frac12
    \int_{t_0}^T
    \mathbb E_{X_{\mathcal C}}
    \mathbb E_{X_{\mathcal T,t}\mid X_{\mathcal C}}
    \left\|
        \widehat s(X_{\mathcal T,t},X_{\mathcal C},t)
        -s_t(X_{\mathcal T,t}\mid X_{\mathcal C})
    \right\|_F^2
    \mathrm dt\\
    &+\mathbb E_{X_{\mathcal C}}\left[
    \operatorname{KL}\!\left(
        P_T^{\mathcal T}(\cdot\mid X_{\mathcal C})
        \,\middle\|\,
        \mathcal N(0,I_{p_{\mathcal T}})
    \right)\right].
    \label{eq:score_distribution_transfer}
\end{aligned}
\end{equation}
Under Assumption~\ref{ass:tensor}, for fixed $D$, there exist constants $C,c>0$ such that
\begin{equation}
    \mathbb E_{X_{\mathcal C}}\left[
    \operatorname{KL}\!\left(
        P_T^{\mathcal T}(\cdot\mid X_{\mathcal C})
        \,\middle\|\,
        \mathcal N(0,I_{p_{\mathcal T}})
    \right)\right]
    \leq
    C\exp(-cT)
    \left\{r+p_{\mathcal T}\log(1+p_{\max})\right\}.
    \label{eq:terminal_kl_bound}
\end{equation}
\end{proposition}

The two terms in \eqref{eq:score_distribution_transfer} correspond to the two approximations made by the implementable sampler. The integral accumulates the error introduced along the reverse-time process when the population score is replaced by $\widehat s$. The second term arises from initialization. The exact reverse-time process would start from the conditional terminal distribution $P_T^{\mathcal T}(\cdot\mid X_{\mathcal C})$, which is unknown, whereas the implementable sampler starts from $\mathcal N(0,I_{p_{\mathcal T}})$. Under the Ornstein--Uhlenbeck forward process, the contribution of the initial missing outcomes diminishes, and $P_T^{\mathcal T}(\cdot\mid X_{\mathcal C})$ approaches the standard normal distribution as $T$ increases. Bound \eqref{eq:terminal_kl_bound} makes this approximation precise: its contribution decays exponentially in $T$, up to the displayed dimension factor.\footnote{The score-error term follows from a change of measure between the exact and trained reverse-time processes and contraction under the endpoint map \citep{girsanov1960transforming,follmer2005entropy,csiszar1967information,fu2024unveil}; the terminal bound follows from entropy contraction of the Ornstein--Uhlenbeck forward process \citep{bakry2014analysis}.} Consequently, for sufficiently large $T$, the initialization approximation contributes little to the error in the generated distribution, which is then governed primarily by the integrated score error.

Proposition~\ref{prop:score_distribution_transfer} compares the true and generated distributions at the same early-stopping time $t_0$. This common-time comparison excludes the error caused by stopping before time zero. The remaining difference between $P_{t_0}^{\mathcal T}(\cdot\mid X_{\mathcal C})$ and the original target distribution $P_0^{\mathcal T}(\cdot\mid X_{\mathcal C})$ is the early-stopping error, which vanishes as $t_0\downarrow0$.

\subsection{Recovery of the Missing-Outcome Distribution}\label{subsec:missing_law_recovery}

We now combine the conditional-score bound in Theorem~\ref{thm:main-generalization} with Proposition~\ref{prop:score_distribution_transfer} to establish recovery of the generated distribution. Recall that $d_{\mathcal T,\beta}=r+p_{\mathcal T}p^{-\beta}$ is the effective dimension of the missing region. Recall that $\delta_n=\frac{(r+12)\log\log(n)}{2\log(n)}$ and $\epsilon=n^{-\frac{1-\delta_n}{r+5}}$, and define $a_n=\frac{1-\delta_n}{3r+15}$. Under the choices of $t_0$ and $T$ stated below, the polylogarithmic factors suppressed in Theorem~\ref{thm:main-generalization} are bounded by $L_n=(1+\log(n)+\log(p_{\max}))^{c_L}$ for a sufficiently large fixed constant $c_L>0$. Define
\begin{equation}
    \mathfrak R_n=L_n
    \left[
        n^{-\frac{5(1-\delta_n)}{6(r+5)}}
        +
        \sqrt{p_{\max}}\,
        n^{-\frac{3r+8+7\delta_n}{6(r+5)}}
    \right].
    \label{eq:main-rate-envelope}
\end{equation}
The exponent $c_L$ only records the logarithmic factors hidden by the $\widetilde{\mathcal O}$ notation and is not an implementation parameter.

\begin{theorem}[Recovery of the missing-outcome distribution]
\label{thm:main-distribution}
Suppose Assumptions~\ref{ass:tensor} and \ref{ass:mask_core_tail} hold and Assumption~\ref{ass:mask_xi_regular} holds on $\mathcal G_{R_\epsilon}\times[t_0,T]$. Configure $\widehat s$ as in Theorem~\ref{thm:main-generalization} with the value of $\epsilon$ displayed above. Suppose that $D$, the mode ranks $(r_1,\ldots,r_D)$, $L_g$, $L_t$, and $K_0$ are fixed. Choose $t_0\asymp n^{-a_n}$, $T=C_T\{\log(n)+\log(p_{\max})\}$, and suppose that $p^\beta\geq C_\beta n^{a_n}$, where $C_T$ and $C_\beta$ are sufficiently large constants. Then there exists an event $\mathcal A_n$, measurable with respect to the training sample, such that $\mathbb{P}(\mathcal A_n)\geq1-1/n$ and, on $\mathcal A_n$,
\begin{equation}
    \mathbb E_{X_{\mathcal C}}\left[
    \operatorname{KL}\!\left(
        P_{t_0}^{\mathcal T}(\cdot\mid X_{\mathcal C})
        \,\middle\|\,
        \widehat P_{t_0}^{\mathcal T}(\cdot\mid X_{\mathcal C})
    \right)\right]
    \leq
    C d_{\mathcal T,\beta}\mathfrak R_n^2,
    \label{eq:main-distribution-kl}
\end{equation}
with $\mathfrak R_n$ defined in \eqref{eq:main-rate-envelope}. Consequently,
\begin{equation}
    \mathbb E_{X_{\mathcal C}}\left[
    \operatorname{TV}\!\left(
        P_{t_0}^{\mathcal T}(\cdot\mid X_{\mathcal C}),
        \widehat P_{t_0}^{\mathcal T}(\cdot\mid X_{\mathcal C})
    \right)\right]
    \leq
    C d_{\mathcal T,\beta}^{1/2}\mathfrak R_n.
    \label{eq:main-distribution-tv}
\end{equation}
\end{theorem}

Theorem~\ref{thm:main-distribution} converts the conditional-score guaranty in Theorem~\ref{thm:main-generalization} into a recovery guaranty for the generated conditional distribution. The first component of $\mathfrak R_n$ arises from learning the nonlinear regression in the $r$-dimensional core, whereas the second arises from estimating the mode-wise loading spaces and depends on $p_{\max}$ rather than on the full dimension $p$. The factor $d_{\mathcal T,\beta}$ captures the idiosyncratic variation remaining in the treated region. In the strong-signal regime $p_{\mathcal T}\lesssim p^\beta$, this factor remains bounded when the mode ranks are fixed. More generally, if $d_{\mathcal T,\beta}\mathfrak R_n^2\to0$, the average conditional KL divergence in \eqref{eq:main-distribution-kl} converges to zero; \eqref{eq:main-distribution-tv} then gives the same conclusion for the average conditional total variation distance. Hence, the generated conditional distribution consistently recovers the early-stopped target distribution $P_{t_0}^{\mathcal T}(\cdot\mid X_{\mathcal C})$ under both metrics.

The prescribed choices of $t_0$ and $T$ control different approximation errors. The terminal time $T=C_T(\log(n)+\log(p_{\max}))$ grows only logarithmically with the sample size and the largest mode dimension. Substituting this choice into \eqref{eq:terminal_kl_bound} changes $\exp(-cT)$ into $n^{-cC_T}p_{\max}^{-cC_T}$. Because $D$ is fixed, a sufficiently large $C_T$ makes the Gaussian-initialization error negligible relative to the score-learning error. The early-stopping time $t_0$ has a different role. Choosing $t_0\asymp n^{-a_n}$ allows the perturbation of the original target to vanish while preventing the factor $t_0^{-1}$ in the score bound from increasing too rapidly. The condition $p^\beta\geq C_\beta n^{a_n}$ ensures that this choice remains above the idiosyncratic-noise scale required by Theorem~\ref{thm:main-generalization}. Because $t_0\to0$, the forward coupling in \eqref{eq:main-forward} implies that $P_{t_0}^{\mathcal T}(\cdot\mid X_{\mathcal C})$ approaches the original conditional distribution of the missing control outcomes $P_0^{\mathcal T}(\cdot\mid X_{\mathcal C})$.

\begin{remark}[Implementation implications of the recovery rate]
The recovery bound also clarifies how statistical accuracy depends on the diffusion horizon and the intrinsic dimensions of the problem. The early-stopping time $t_0$ balances approximation and estimation. It should decrease with the training sample size, with $t_0\asymp n^{-a_n}$ providing a theoretically justified benchmark. Moving $t_0$ closer to zero leaves less perturbation in the target distribution but makes the score more difficult to estimate near the data distribution. At the other end of the diffusion path, a terminal time $T$ that grows logarithmically is sufficient for the forward-diffused missing region to approach a standard Gaussian. The bound further shows that recovery depends on $r$, $p_{\max}$, and $d_{\mathcal T,\beta}$ rather than on the full dimension $p$. Because nonlinear score learning is governed by $r$, network capacity is most effectively allocated to the Tucker core. When the term involving $p_{\max}$ dominates, estimation of the loading spaces becomes the principal bottleneck. A larger missing region or a weaker factor signal increases $d_{\mathcal T,\beta}$ and therefore requires a larger training sample to attain comparable recovery accuracy.
\end{remark}

\subsection{Recovery of Weighted Counterfactual Summaries}\label{subsec:weighted_summary_recovery}

Theorem~\ref{thm:main-distribution} concerns the entire conditional distribution of the missing control outcomes in the treated region. Most causal analyzes, however, concern selected entries or averages of these outcomes. We therefore consider prespecified weighted summaries of the recovered control outcomes.

For intervention $a$, let $M_{\mathcal T_a}$ denote the binary mask for the intervention-specific region $\mathcal T_a$. A deterministic weight vector $w\in\mathbb R^p$ supported on $\mathcal T_a$ defines the control-outcome summary
\begin{equation}
U_{0,a}(w)=w^\top\vecc\bigl(M_{\mathcal T_a}\odot X_0\bigr).
\label{eq:control_counterfactual_summary}
\end{equation}
The vector $w$ specifies which coordinates enter the comparison and how they are aggregated. Because $U_{0,a}(w)$ is a linear functional of $X_{\mathcal T}$, its conditional distribution is determined by \eqref{eq:counterfactual_conditional_law}. A coordinate-selection weight isolates the control outcome at a specified tensor coordinate, whereas equal weights yield an average over selected coordinates.\footnote{For example, if $\mathcal T_a={\boldsymbol i_1,\boldsymbol i_2,\boldsymbol i_3}$ and $w$ assign weight $1/3$ to each coordinate, then $U_{0,a}(w)=\frac{1}{3}\sum_{j=1}^3Y_{\boldsymbol i_j}(0)$.}

To accommodate causal analysis involving several tensors, let $n_e$ tensors be used in the causal analysis, hereafter called the analysis tensors. For analysis tensor $\ell$, let $X_{\mathcal C}^{(\ell)}$ denote its observed control outcomes and $X_{\mathcal T}^{(\ell)}$ its missing control outcomes in the treated region, with conditional distribution $X_{\mathcal T}^{(\ell)}\mid X_{\mathcal C}^{(\ell)}\sim P_0^{\mathcal T}\bigl(\,\cdot\mid X_{\mathcal C}^{(\ell)}\bigr)$ for $\ell=1,\ldots,n_e$. The collection $X_{\mathcal C}^{1:n_e}=\bigl(X_{\mathcal C}^{(1)},\ldots,X_{\mathcal C}^{(n_e)}\bigr)$ contains the observed control outcomes supplied to the conditional generator. These analysis tensors are separate from the training sample of complete control-outcome tensors used to estimate the score. They are independent draws from the population distribution and are independent of the training sample.

For a deterministic vector $w=\MT w\in\mathbb R^p$, define the control-outcome summary for the analysis tensor $\ell$ by $U_{w,\ell,0}=w^\top\vecc\bigl(X_{\mathcal T}^{(\ell)}\bigr)$. If the comparison concerns intervention $a$ and $w$ is supported on $\mathcal T_a$, then $U_{w,\ell,0}$ is the analysis-tensor counterpart of $U_{0,a}(w)$ in \eqref{eq:control_counterfactual_summary}. When a causal comparison combines several analysis tensors, define their average control-outcome summary as follows: $V_{n_e,w,0} = \frac{1}{n_e} \sum_{\ell=1}^{n_e}U_{w,\ell,0}$. Here $n_e$ is the number of analysis tensors entering the downstream comparison, rather than the number of generated samples. Thus, $V_{n_e,w,0}$ is directly interpretable as an average missing control outcome and can be compared with the corresponding average observed outcome under treatment.

For each analysis tensor $\ell$, let $\widehat X_{\mathcal T,t_0}^{(\ell)}:=\widehat X_{\mathcal T,T-t_0}^{\leftarrow,(\ell)}$ be generated conditional on $X_{\mathcal C}^{(\ell)}$ by the reverse-time process \eqref{eq:main-reverse} driven by the estimated score. Then $\widehat X_{\mathcal T,t_0}^{(\ell)}\sim\widehat P_{t_0}^{\mathcal T}(\cdot\mid X_{\mathcal C}^{(\ell)})$, and the generated counterpart of $V_{n_e,w,0}$ is $\widehat V_{n_e,w,t_0} = \frac{1}{n_e} \sum_{\ell=1}^{n_e} w^\top\vecc (\widehat X_{\mathcal T,t_0}^{(\ell)})$. Conditional on $X_{\mathcal C}^{1:n_e}$ and the learned score, the generated vector follows the product conditional distribution $\bigotimes_{\ell=1}^{n_e} \widehat P_{t_0}^{\mathcal T} (\cdot\mid X_{\mathcal C}^{(\ell)})$. Let $\mathcal F_{\mathrm{tr}} =\sigma(X_0^{(1)},\ldots,X_0^{(n)})$ denote the sigma-field generated by the training sample. For a fixed realization of $\mathcal F_{\mathrm{tr}}$ belonging to $\mathcal A_n$, the learned score and its induced conditional distributions are fixed; this conditioning is suppressed below. The following result uses $d_{\mathrm{BL}}$ to compare the target and generated conditional distributions of the weighted aggregate.

Because $V_{n_e,w,0}$ is an arithmetic average, its distribution may contract mechanically as $n_e$ increases. We therefore state distributional recovery on the fluctuation scale $\sqrt{n_e}V_{n_e,w,0}$. This deterministic normalization is neither a studentization nor an appeal to a central limit theorem.

\begin{corollary}[Recovery of weighted counterfactual summaries]
\label{cor:main-linear-functional}
Under the conditions of Theorem~\ref{thm:main-distribution}, for every $n_e\in\mathbb N$ and every deterministic $w=\MT w$, on $\mathcal A_n$,
\begin{equation}
\begin{aligned}
    &\mathbb E_{X_{\mathcal C}^{1:n_e}}
    d_{\mathrm{BL}}\!\left[
        \mathcal L\{\sqrt{n_e}V_{n_e,w,0}\mid X_{\mathcal C}^{1:n_e}\},
        \mathcal L\{\sqrt{n_e}\widehat V_{n_e,w,t_0}\mid X_{\mathcal C}^{1:n_e}\}
    \right]
    \\
    &\qquad
    \leq
    C\|w\|_2 n^{-a_n/2}
    +
    C\sqrt{n_e}\,
    d_{\mathcal T,\beta}^{1/2}\mathfrak R_n.
\end{aligned}
    \label{eq:main-linear-functional-bound}
\end{equation}
\end{corollary}

The first term in \eqref{eq:main-linear-functional-bound} bridges the original and early-stopped distributions. Independence and centering keep the mean-square forward perturbation of $\sqrt{n_e}V_{n_e,w,0}$ of order $\|w\|_2^2 t_0$.\footnote{Let $U_{w,\ell,t}=w^\top\vecc(X_{\mathcal T,t}^{(\ell)})$. Under the forward coupling, $U_{w,\ell,t_0}-U_{w,\ell,0}=(\alpha_{t_0}-1)U_{w,\ell,0}+h_{t_0}^{1/2}w^\top\vecc(Z_\ell)$. Independence across analysis tensors, centering, and $\mathbb E U_{w,\ell,0}^2\leq C\|w\|_2^2$ therefore give $\mathbb E\!\left[\left|\sqrt{n_e}\{V_{n_e,w,t_0}-V_{n_e,w,0}\}\right|^2\right]\leq C\|w\|_2^2\{(1-\alpha_{t_0})^2+h_{t_0}\}\leq C\|w\|_2^2t_0$. Taking the square root yields the contribution $C\|w\|_2\sqrt{t_0}=C\|w\|_2n^{-a_n/2}$ in \eqref{eq:main-linear-functional-bound}.} The second term is the conditional-distribution estimation error: relative entropy adds over the $n_e$ product components, and Pinsker's inequality yields the factor $\sqrt{n_e}$. Thus, if $n_e$ and $w=w_n$ vary with $n$, consistency follows when $\|w_n\|_2 n^{-a_n/2}\to 0$ and $\sqrt{n_e}\,d_{\mathcal T,\beta}^{1/2}\mathfrak R_n\to 0$. The target $V_{n_e,w,0}$ itself remains on the average-outcome scale used below.

\subsection{Counterfactual Prediction Intervals}\label{subsec:counterfactual_prediction_intervals}

Recovering the conditional distribution also provides quantiles and prediction intervals for the missing control outcomes. With suitable weights, the same construction applies to averages used in treatment-effect comparisons. Combining such an interval with the corresponding observed treated outcome yields an interval for the resulting treatment-effect contrast.

We next establish the coverage of these intervals. Although Corollary~\ref{cor:main-linear-functional} controls the bounded--Lipschitz distance between the weighted-summary distributions, interval coverage is not a continuous functional of the underlying distribution. We therefore impose a mild anti-concentration condition. Suppose that, for almost every $X_{\mathcal C}^{1:n_e}$, the conditional density of $\sqrt{n_e}V_{n_e,w,0}$ is uniformly bounded over its argument by a finite constant $M_{n,n_e,w}$. Equivalently, the conditional density of $V_{n_e,w,0}$ is bounded by $\sqrt{n_e}M_{n,n_e,w}$. This condition prevents excessive probability mass from accumulating near a prediction-interval endpoint.

For a realization $x=(x_{\mathcal C}^{(1)},\ldots,x_{\mathcal C}^{(n_e)})$ of the observed control outcomes, let $\widehat F_x$ denote the conditional distribution function of $\widehat V_{n_e,w,t_0}$. The density condition excludes the degenerate choice $w=0$. Because the reverse diffusion is nondegenerate on the positions in the treated region, $\widehat F_x$ is continuous. Define $\widehat q_x(u)=\inf\{v:\widehat F_x(v)\geq u\}$ for $u\in(0,1)$. We report intervals on the average-outcome scale. Multiplying both the statistic and the interval by $\sqrt{n_e}$ leaves the coverage event unchanged. The central interval induced by the estimated conditional distribution is
\begin{equation}
    \widehat I_{n_e,w,1-\alpha}(x)=\left[
        \widehat q_x(\alpha/2),
        \widehat q_x(1-\alpha/2)
    \right],
    \qquad 0<\alpha<1.
    \label{eq:infinite_generation_interval}
\end{equation}
In computation, let $\left\{\widehat V_{n_e,w,t_0}^{(b)}\right\}_{b=1}^B$ be $B$ independent draws from the estimated conditional distribution. Let $\widehat q_x^{(B)}(u)$ denote their empirical $u$-quantile and define $\widehat I_{n_e,w,1-\alpha}^{(B)}(x)=\left[\widehat q_x^{(B)}(\alpha/2),\widehat q_x^{(B)}(1-\alpha/2)\right]$. Also, let $\mathcal F_{B}^{\mathrm{gen}}=\mathcal F_{\mathrm{tr}}
    \vee
    \sigma\!\{
        \widehat V_{n_e,w,t_0}^{(1)},\ldots,
        \widehat V_{n_e,w,t_0}^{(B)}\}$.

\begin{corollary}[Coverage of counterfactual prediction intervals]
\label{cor:counterfactual_prediction_interval}
Under the conditions of Theorem~\ref{thm:main-distribution} and the conditional-density bound above, on $\mathcal A_n$,
\begin{equation}
\begin{aligned}
    &\mathbb E_{X_{\mathcal C}^{1:n_e}}
    \left|
        \mathbb{P}\!\left\{
            V_{n_e,w,0}
            \in
            \widehat I_{n_e,w,1-\alpha}
            (X_{\mathcal C}^{1:n_e})
            \,\middle|\,
            X_{\mathcal C}^{1:n_e},\mathcal F_{\mathrm{tr}}
        \right\}
        -(1-\alpha)
    \right|
    \\
    &\qquad\leq
    C\sqrt{n_e}\,
    d_{\mathcal T,\beta}^{1/2}\mathfrak R_n
    +
    C\bigl(M_{n,n_e,w}\|w\|_2\bigr)^{2/3}
    n^{-a_n/3}.
    \label{eq:main-original-interval}
\end{aligned}
\end{equation}
Moreover,
\begin{equation}
\begin{aligned}
    &\mathbb E_{X_{\mathcal C}^{1:n_e}}
    \mathbb E_{\mathrm{gen}\mid
        X_{\mathcal C}^{1:n_e},\mathcal F_{\mathrm{tr}}}
    \left|
        \mathbb{P}\!\left\{
            V_{n_e,w,0}
            \in
            \widehat I_{n_e,w,1-\alpha}^{(B)}
            (X_{\mathcal C}^{1:n_e})
            \,\middle|\,
            X_{\mathcal C}^{1:n_e},\mathcal F_{B}^{\mathrm{gen}}
        \right\}
        -(1-\alpha)
    \right|
    \\
    &\qquad\leq
    C\sqrt{n_e}\,
    d_{\mathcal T,\beta}^{1/2}\mathfrak R_n
    +
    C\bigl(M_{n,n_e,w}\|w\|_2\bigr)^{2/3}
    n^{-a_n/3}
    +CB^{-1/2}.
    \label{eq:finite_generation_interval}
\end{aligned}
\end{equation}
Here the inner expectation is taken over the finite generation sample, conditional on the observed control outcomes and the training sample.
\end{corollary}

The three terms in \eqref{eq:finite_generation_interval} arise from recovery of the early-stopped distribution, the early-stopping approximation, and finite generation, respectively. On the fluctuation scale, the forward coupling perturbs the target by $O(\|w\|_2\sqrt{t_0})$; combining this bound with $M_{n,n_e,w}$ gives the second term. The final term follows from replacing the quantiles of the estimated conditional distribution with empirical quantiles based on $B$ conditional draws \citep{massart1990tight}. Hence, the finite-generation interval is asymptotically calibrated whenever
\[
    \sqrt{n_e}\,d_{\mathcal T,\beta}^{1/2}\mathfrak R_n\to0,\qquad M_{n,n_e,w}\|w\|_2 n^{-a_n/2}\to0,\qquad B\to\infty.
\]

These intervals concern unobserved counterfactual realizations rather than a fixed population parameter. Because $V_{n_e,w,0}$ is an arithmetic average, the interval endpoints retain the units relevant for an average causal comparison. The fluctuation-scale density condition is compatible with this presentation because common scaling leaves coverage unchanged. The result does not use an additional Gaussian approximation or require estimation of an asymptotic variance; its quantiles are obtained directly from conditional draws generated by \CFTDiff.

The control-world intervals can be translated directly into intervals for treatment effects on the analysis tensors. For intervention $a$, let $w$ be supported on $\mathcal T_a$ and define the corresponding average of the observed treated outcomes by $V_{n_e,w,a}^{\mathrm{obs}}=\frac{1}{n_e}\sum_{\ell=1}^{n_e} w^\top\vecc\!\left(M_{\mathcal T_a}\odot Y^{\mathrm{obs},(\ell)}\right)$. The associated weighted sample-average counterfactual contrast is
\begin{equation}
\begin{aligned}
\tau_{n_e,w,a}&=\frac{1}{n_e}\sum_{\ell=1}^{n_e}w^\top\vecc\!\left[M_{\mathcal T_a}\odot\left(Y^{\mathrm{obs},(\ell)}-X_{\mathcal T}^{(\ell)}\right)\right]=V_{n_e,w,a}^{\mathrm{obs}}-V_{n_e,w,0}.
\label{eq:weighted_sample_ate}
\end{aligned}
\end{equation}
With equal weights over the relevant treated coordinates, this quantity is the sample average treatment effect for the analysis tensors.

Because the treated outcomes are observed, a prediction interval for the contrast in \eqref{eq:weighted_sample_ate} is obtained by translating the interval for $V_{n_e,w,0}$. In particular, the interval induced by
\eqref{eq:infinite_generation_interval} is
\begin{equation*}
\begin{aligned}
\widehat I_{\tau,n_e,w,a,1-\alpha}(x)=\left[V_{n_e,w,a}^{\mathrm{obs}}-\widehat q_x(1-\alpha/2),V_{n_e,w,a}^{\mathrm{obs}}-\widehat q_x(\alpha/2)\right].
\end{aligned}
\end{equation*}
Its finite-$B$ counterpart is obtained by replacing $\widehat q_x$ with $\widehat q_x^{(B)}$. The reversal of the quantile endpoints reflects that the treatment effect is the observed treated outcome minus the unobserved control outcome. Treating the realized treated summary as fixed, this transformation preserves the conditional coverage guarantee of the underlying control-world interval and requires no model for the treated-outcome distribution. The resulting interval quantifies counterfactual uncertainty in the weighted sample-average treatment effect.\footnote{It is not a sampling confidence interval for a population average treatment effect; population-level inference would additionally require assumptions linking the analysis tensors to the target population.}

\section{Simulation Study}\label{sec:simulation}

This section evaluates the finite-sample performance of \CFTDiff\ in recovering both missing control outcomes and their conditional distribution. We vary the size of the missing region and the strength of the latent factor structure and compare \CFTDiff\ with established point-recovery methods and two nested diffusion benchmarks. The main text reports the more challenging setting, $\beta=0.50$; results for $\beta=0.75$ are provided in Appendix~\ref{app:Additional Simulations}.

\subsection{Simulation Setting}\label{subsec:simulation_setting}

For the tensor design, we set $\beta_1=\beta_2=\beta$ and generate $64\times64$ control-world matrices from the bilinear factor model $X=A_1FA_2^\top+p^{-\beta/2}E$, where $F$ is a $16\times16$ Gaussian factor matrix and $E$ contains independent Gaussian idiosyncratic noise. Each replication contains 360 complete training matrices and 12 independently generated test matrices.

For each evaluation matrix, we remove a contiguous rectangular block containing approximately 25\%, 50\%, or 75\% of its entries. The remaining entries are the observed control outcomes supplied to the methods, and recovery is evaluated only in the missing region. Because the data-generating process is Gaussian, the exact conditional distribution of the missing outcomes given the observed control outcomes is available and provides a benchmark for distributional recovery.

We compare \CFTDiff\ with five competing point-recovery methods, organized into causal panel estimators and matrix/tensor completion methods. The first group includes difference-in-differences (\DID), synthetic control (\SC), and synthetic difference-in-differences (\SDID) \citep{arkhangelsky2021synthetic}. The second includes nuclear-norm matrix completion (\NMC) \citep{athey2021matrix} and tensor factor imputation (\TFI) based on the Tucker factor model of \citet{cen2025tensor}. All methods use the same observed entries and are evaluated on the same missing regions.

To distinguish the sources of performance gains, we construct two nested diffusion benchmarks. Convolutional Diffusion (\ConvDiff) replaces the structured score network with a conventional convolutional U-Net and learns directly in the full matrix space. Tucker Diffusion (\TuckerDiff) employs the Tucker-U-Net architecture, thereby exploiting multilinear low-rank structure without the masked conditional construction of \CFTDiff. Our method combines Tucker-based dimension reduction with conditioning on the observed control outcomes through the treatment mask. This nested design separates the gains from multilinear structure and conditional generation.

The five competing methods provide point estimates. For \ConvDiff, \TuckerDiff, and \CFTDiff, the mean of the generated draws serves as the point estimate, while the draws themselves provide an estimated conditional distribution. Point recovery is evaluated by mean absolute error (MAE) and root mean squared error (RMSE). Distributional recovery is evaluated by the energy score and continuous ranked probability score (CRPS), together with the coverage and width of 90\% intervals, the interval score (IS), and the weighted interval score (WIS). Results are averaged over five independent simulation replications. Appendix~\ref{app:simulation_design} provides the complete data-generating process, implementation details, and metric definitions.

\subsection{Simulation Results}\label{subsec:simulation_results}

We first examine counterfactual recovery performance as the missing rate increases. Table~\ref{tab:recovery_beta050} reports point, distributional, and interval recovery at missing rates of 25\%, 50\%, and 75\%. These designs progressively reduce the observed control outcomes and therefore assess how recovery changes as less conditioning information remains available.

\begin{table}[htb]
\TABLE
{Counterfactual recovery across missing rates ($\beta=0.50$).
\label{tab:recovery_beta050}}
{
\scriptsize
\setlength{\tabcolsep}{1.25pt}
\resizebox{\textwidth}{!}{%
\begin{tabular}{@{}lcccccccc@{}}
\toprule
& \multicolumn{2}{c}{Point Recovery} & \multicolumn{2}{c}{Distributional Recovery} & \multicolumn{4}{c}{Interval Recovery} \\[-1pt]
\cmidrule(lr){2-3}\cmidrule(lr){4-5}\cmidrule(lr){6-9}
Method & MAE & RMSE & Energy & CRPS & Coverage & Width & IS & WIS \\
\midrule
\multicolumn{9}{c}{\textit{Panel A: Missing Rate $=25\%$}} \\
\midrule
\NMC & 0.0588 (0.0042) & 0.0740 (0.0054) & -- & -- & -- & -- & -- & -- \\
\DID & 0.1688 (0.0129) & 0.2170 (0.0169) & -- & -- & -- & -- & -- & -- \\
\SC & 0.0533 (0.0035) & 0.0674 (0.0045) & -- & -- & -- & -- & -- & -- \\
\SDID & 0.1172 (0.0075) & 0.1502 (0.0095) & -- & -- & -- & -- & -- & -- \\
\TFI & 0.0586 (0.0035) & 0.0745 (0.0045) & -- & -- & -- & -- & -- & -- \\
\ConvDiff & 0.1191 (0.0075) & 0.1527 (0.0095) & 3.4293 (0.2163) & 0.0851 (0.0054) & 0.8831 (0.0120) & 0.4785 (0.0210) & 0.6469 (0.0398) & 0.0640 (0.0041) \\
\TuckerDiff & 0.0497 (0.0027) & 0.0637 (0.0033) & 1.4746 (0.0770) & 0.0355 (0.0019) & 0.9614 (0.0101) & 0.2745 (0.0169) & 0.2931 (0.0161) & 0.0272 (0.0015) \\
\CFTDiff & 0.0371 (0.0023) & 0.0470 (0.0029) & 1.0679 (0.0655) & 0.0263 (0.0017) & 0.9225 (0.0094) & 0.1744 (0.0098) & 0.2056 (0.0125) & 0.0199 (0.0013) \\[2pt]
\midrule
\multicolumn{9}{c}{\textit{Panel B: Missing Rate $=50\%$}} \\
\midrule
\NMC & 0.0688 (0.0054) & 0.0871 (0.0068) & -- & -- & -- & -- & -- & -- \\
\DID & 0.1711 (0.0142) & 0.2195 (0.0184) & -- & -- & -- & -- & -- & -- \\
\SC & 0.0635 (0.0057) & 0.0815 (0.0075) & -- & -- & -- & -- & -- & -- \\
\SDID & 0.1188 (0.0086) & 0.1523 (0.0109) & -- & -- & -- & -- & -- & -- \\
\TFI & 0.0769 (0.0057) & 0.0983 (0.0072) & -- & -- & -- & -- & -- & -- \\
\ConvDiff & 0.1206 (0.0087) & 0.1546 (0.0110) & 4.9108 (0.3529) & 0.0861 (0.0062) & 0.8830 (0.0117) & 0.4839 (0.0235) & 0.6560 (0.0469) & 0.0648 (0.0047) \\
\TuckerDiff & 0.0699 (0.0050) & 0.0900 (0.0063) & 2.8968 (0.2026) & 0.0500 (0.0036) & 0.9500 (0.0092) & 0.3544 (0.0253) & 0.3895 (0.0270) & 0.0378 (0.0027) \\
\CFTDiff & 0.0490 (0.0037) & 0.0631 (0.0049) & 2.0091 (0.1558) & 0.0347 (0.0027) & 0.9148 (0.0144) & 0.2195 (0.0183) & 0.2668 (0.0202) & 0.0261 (0.0020) \\[2pt]
\midrule
\multicolumn{9}{c}{\textit{Panel C: Missing Rate $=75\%$}} \\
\midrule
\NMC & 0.1081 (0.0076) & 0.1394 (0.0098) & -- & -- & -- & -- & -- & -- \\
\DID & 0.1693 (0.0126) & 0.2176 (0.0164) & -- & -- & -- & -- & -- & -- \\
\SC & 0.0963 (0.0065) & 0.1253 (0.0086) & -- & -- & -- & -- & -- & -- \\
\SDID & 0.1185 (0.0081) & 0.1521 (0.0101) & -- & -- & -- & -- & -- & -- \\
\TFI & 0.1023 (0.0063) & 0.1312 (0.0080) & -- & -- & -- & -- & -- & -- \\
\ConvDiff & 0.1197 (0.0081) & 0.1535 (0.0102) & 5.9757 (0.4011) & 0.0855 (0.0058) & 0.8867 (0.0123) & 0.4862 (0.0204) & 0.6517 (0.0430) & 0.0644 (0.0043) \\
\TuckerDiff & 0.0937 (0.0067) & 0.1207 (0.0086) & 4.7137 (0.3321) & 0.0668 (0.0047) & 0.9301 (0.0093) & 0.4330 (0.0269) & 0.5010 (0.0344) & 0.0501 (0.0035) \\
\CFTDiff & 0.0801 (0.0065) & 0.1037 (0.0084) & 4.0358 (0.3276) & 0.0570 (0.0047) & 0.9192 (0.0129) & 0.3564 (0.0287) & 0.4253 (0.0347) & 0.0426 (0.0035) \\
\bottomrule
\end{tabular}%
}
}
{The table reports recovery at three missing rates. Entries are means across five independent simulation replications, with standard deviations in parentheses. MAE and RMSE evaluate point recovery; Energy and CRPS evaluate the generated conditional distribution. Coverage and Width denote the empirical coverage and average width of nominal 90\% intervals, and IS and WIS denote the interval score and weighted interval score. Smaller MAE, RMSE, Energy, CRPS, IS, and WIS indicate better performance. Width is interpreted jointly with Coverage, whose nominal level is 0.90. Dashes indicate measures that are not applicable to point estimators. Each diffusion method uses 100 conditional draws per evaluation matrix.}
\end{table}

Among the five established point-recovery methods, \SC\ has the lowest MAE and RMSE at each missing rate, while \NMC\ and \TFI\ are also competitive when the missing region is relatively small. Their recovery errors nevertheless increase as the number of observed control outcomes decreases. \CFTDiff\ attains the lowest MAE and RMSE in every design. Relative to \SC, the strongest of the five established methods, it reduces MAE by 17\%--30\%; even after including the two nested generative specifications, its MAE remains 15\%--25\% below that of the next-best method.

The nested specifications further clarify the source of these gains. \TuckerDiff\ improves substantially on \ConvDiff, demonstrating the value of Tucker dimension reduction, while \CFTDiff\ delivers a further improvement by incorporating the observed control outcomes through the masked conditional construction. Relative to \TuckerDiff, \CFTDiff\ reduces both Energy and CRPS by 14\%--31\% across the three missing rates. Its 90\% coverage remains between 0.915 and 0.923, while its intervals are 18\%--38\% narrower than those of \TuckerDiff. The corresponding reductions in IS and WIS show that the sharper intervals retain coverage close to the nominal level.

Figure~\ref{fig:pointwise_conditional_distribution_beta050} complements these aggregate measures with pointwise conditional distributions at a 25\% missing rate. The three rows cover lower, central, and upper realized values within the missing region. This comparison examines whether the methods recover not only a point prediction but also the local shape and scale of the exact conditional distribution.

\begin{figure}[htb]
\FIGURE
{\includegraphics[width=\textwidth]{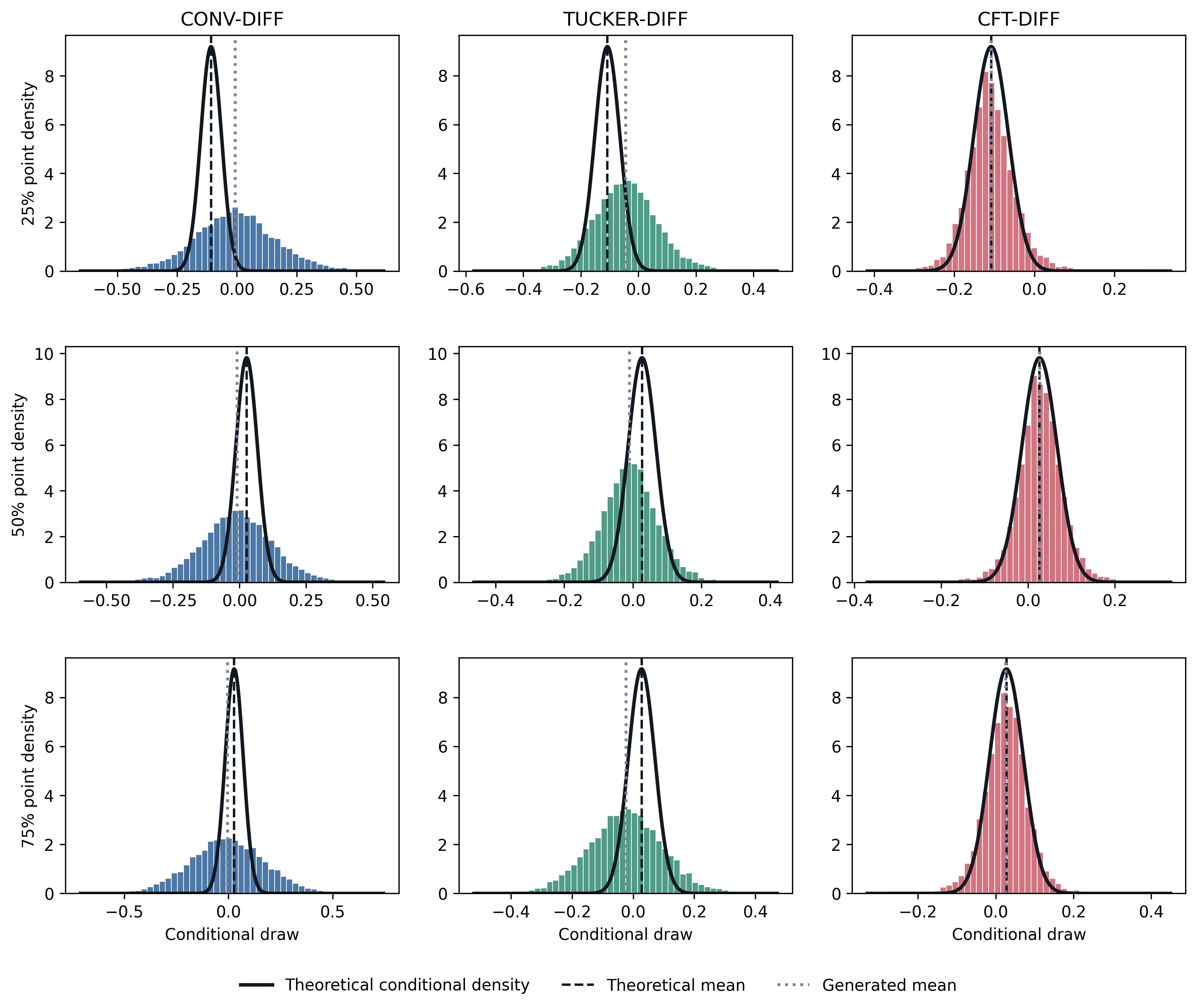}}
{Pointwise conditional distribution recovery ($\beta=0.50$).
\label{fig:pointwise_conditional_distribution_beta050}}
{Columns correspond to \ConvDiff, \TuckerDiff, and \CFTDiff. Rows display entries whose realized missing values are near the 25th, 50th, and 75th percentiles within the missing region. Each histogram is based on 10,000 conditional draws, and the missing rate is 25\%. The solid curve shows the exact conditional density; the dashed and dotted vertical lines indicate the exact conditional mean and the mean of the generated draws, respectively.}
\end{figure}

Across the three displayed entries, the \CFTDiff\ histograms closely track the location and concentration of the exact conditional densities, and their means nearly coincide with the exact conditional means. \ConvDiff\ and \TuckerDiff\ produce visibly more dispersed distributions and larger location discrepancies, particularly away from the center of the missing region's realized-value distribution. Thus, the advantage in Table~\ref{tab:recovery_beta050} reflects a closer approximation to the conditional distribution rather than only a more accurate generated mean.

\section{Distributional Demand Response to Dynamic Electricity Pricing}\label{sec:empirical}

Electricity systems increasingly rely on demand flexibility to manage short-run imbalances, especially during periods of high demand or variable supply. Dynamic pricing can shift electricity use away from costly or constrained hours, but its operational value depends on the magnitude, timing, and reliability of household response. Evaluating these interventions therefore requires counterfactual demand estimates that characterize both average demand reductions and uncertainty in whether a price signal attains its intended target.

We examine these questions using the Norwegian iFlex dynamic-pricing experiment. The experiment followed 3,746 households in five Norwegian regions during the winter of 2020--2021.\footnote{The iFlex project was conducted in two phases. We use the full-scale second phase. The first phase, conducted during the winter of 2019--2020, was stopped earlier than planned because of the onset of the COVID-19 pandemic and therefore did not cover the originally intended winter experimental window. The anonymized data are publicly available through the official iFlex repository on Zenodo (\href{https://doi.org/10.5281/zenodo.8248802} {doi:10.5281/zenodo.8248802}).} The iFlex experiment included 11 treatment groups and five regional control groups. Its within-region random allocation of households supports baseline exchangeability, while contemporaneous regional controls support counterfactual comparisons on the realized intervention dates \citep{hofmann2024demandflexibility}. On selected experiment days, each treatment group received one of 14 day-ahead hourly price signals implemented through a reward scheme. The signals differed in their peak price levels, the timing of the peak periods, and their duration. The data combine household-level hourly electricity consumption with treatment assignments, complete hourly price paths, and local temperature measurements. This design generates repeated multi-arm interventions with 24-hour outcome trajectories and substantial heterogeneity across households, dates, and pricing regimes.

For the empirical analysis, we use hourly electricity consumption from December 1, 2020, through March 26, 2021, and restrict the sample to weekdays excluding holidays. We also exclude households that did not satisfy the recruitment criteria, withdrew from the study, lacked complete electricity-consumption records over the study period, or exhibited outlying consumption profiles. The resulting balanced panel contains 3,746 households, each represented by a $74\times24$ day-by-hour outcome matrix. Households assigned to the five control groups form the training sample used to learn the conditional generator. For each treatment-group household, \CFTDiff generates the missing 24-hour no-policy (control) outcomes on intervention days conditional on the observed no-policy outcomes. We use this setting to assess point and counterfactual distribution recovery across competing methods and to quantify the average and distributional effects of the dynamic-pricing interventions.

\subsection{The iFlex Experiment and Counterfactual Structure}
The iFlex experiment assigns dynamic-pricing interventions at the experimental-group-by-day level. Let $G_i$ denote the experimental group of household $i$, and let $D_{gt}\in\{0,1,\ldots,14\}$ denote the pricing regime assigned to group $g$ on working day $t$, where $D_{gt}=0$ denotes no active intervention. We define the corresponding household-level assignment by $D_{it}=D_{G_i t}$. Thus, all households within the same treatment group are assigned the same day-ahead 24-hour price path, hereafter referred to as a price signal, on an intervention day.\footnote{Each treatment group was associated with a prespecified set of four candidate price signals, one of which was assigned whenever the group was treated. Consequently, not every signal was available to every treatment group; the group-specific candidate sets are reported in Appendix Table~\ref{tab:iflex_group_signal_sets}.}
\begin{figure}[htb]
	\FIGURE
	{\includegraphics[scale=0.52]{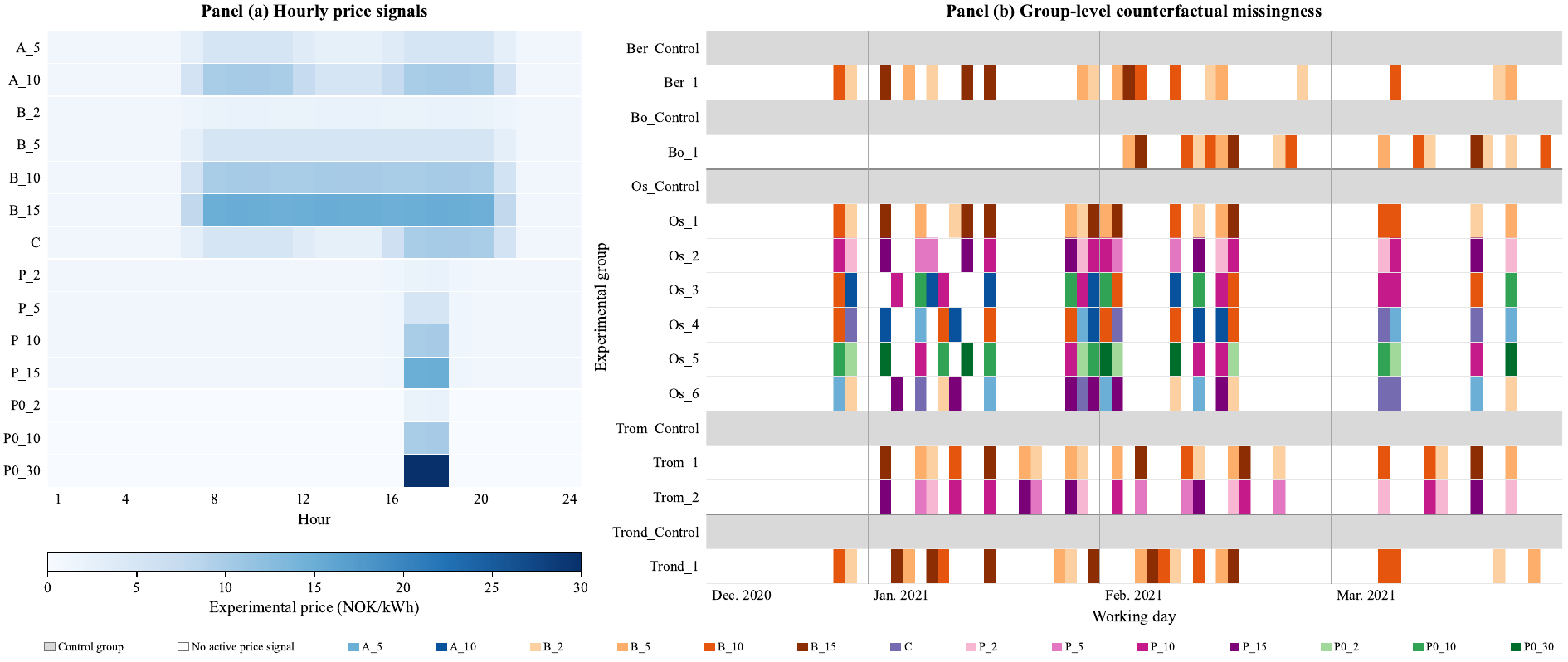}}	
	{Dynamic-pricing signals and group-level intervention schedule. \label{fig:iflex_experiment_structure}} 
	{Panel (a) shows the 24-hour price paths for the 14 active dynamic-pricing interventions. Panel (b) shows the realized
	intervention schedule for the 11 treatment groups and five regional control groups over the 74 working days in the analysis sample. Gray cells denote control groups, white cells denote treatment-group days without an active intervention, and colored cells identify the assigned price signal. Because assignment occurs at the group-day level, all households within a group share the same assignment on a given day.}
\end{figure}

We take the household as the causal unit. Let $Y_{ith}(d)$ denote household $i$'s total electricity consumption at hour $h$ on day $t$ under assignment to the complete daily price signal $d$, with $d=0$ denoting the no-policy regime. Interactions among household members are incorporated into this household-level potential response. We assume no interference across households: a household's potential outcomes are unaffected by the intervention assignments of other households. This causal restriction is distinct from statistical independence, since households may share regional and temporal shocks.

Figure~\ref{fig:iflex_experiment_structure} summarizes the intervention design. Panel~(a) displays the 14 active price signals as complete hourly price paths. Profiles A and C feature separate morning and afternoon peak periods; profile B maintains elevated prices over an extended daytime period; and profiles P and P0 concentrate the peak price in a short afternoon window. The numeric suffix denotes the peak price level in NOK/kWh and distinguishes signals with the same profile but different incentive intensities; profile C is asymmetric and therefore has no single numeric suffix. Panel~(b) reports the realized group-level intervention schedule. No price signal was implemented from December 1 through December 15, 2020, and the first interventions occurred on December 16. Thereafter, the control groups remained untreated, whereas the treatment groups received price signals intermittently. Each colored cell therefore represents a common intervention for all households in the corresponding group.

In the empirical implementation, each household contributes one day-by-hour tensor. For household $i$, define the no-policy outcome matrix and its observed counterpart by $X_0^{(i)}=\bigl[Y_{ith}(0)\bigr]_{t\in[74],\,h\in[24]}\in\mathbb R^{74\times24}$ and $Y_i^{\mathrm{obs}}=\bigl[Y_{ith}^{\mathrm{obs}}\bigr]_{t\in[74],\,h\in[24]}$. Each $X_0^{(i)}$ is an order-two tensor corresponding to $X_0$ in Section~\ref{subsec:target_training}, while $i$ indexes household-level samples. The complete matrices from households in the control groups form the training sample used for score learning.

For treatment group $g$, define the intervention-day mask
$M_{\mathcal T}^{(g)}\in\{0,1\}^{74\times24}$ and its complement by
\[
M_{\mathcal T}^{(g)}(t,h)
=
\mathbf 1\{D_{gt}\neq0\},
\qquad t\in[74],\quad h\in[24],
\qquad
M_{\mathcal C}^{(g)}
=
\mathbf 1-M_{\mathcal T}^{(g)}.
\]
Thus, each intervention day corresponds to a completely masked 24-hour row, and households in the same experimental group share the same mask. Although \citet{hofmann2024demandflexibility} report no clear aggregate intraday rebound, a price signal may affect non-peak consumption through within-day load shifting or other behavioral responses. Masking the entire intervention day prevents these potentially treatment-affected outcomes from entering the conditioning set for recovery of the no-policy counterfactual.

Retaining non-intervention days as observed no-policy outcomes requires an additional temporal exclusion restriction: interventions have no carryover effects on subsequent non-intervention days, and advance notifications have no anticipatory effects on preceding non-intervention days. These are identifying assumptions, not consequences of random allocation. They permit responses across hours within an intervention day while requiring that retained non-intervention days remain unaffected. Accordingly, the mask selects complete intervention-day rows without extending to adjacent non-intervention days. The reduced consumption on non-intervention days reported by \citet{hofmann2024demandflexibility} indicates that this restriction may fail; the causal interpretation of recovery using the retained rows depends on its validity. Under consistency, no interference across households, and the
temporal exclusion restriction, a household $i$ with $G_i=g$
satisfies the observed control-outcome relation:

\[
X_{\mathcal C}^{(i)}
=
M_{\mathcal C}^{(g)}\odot X_0^{(i)}
=
M_{\mathcal C}^{(g)}\odot Y_i^{\mathrm{obs}},
\]
whereas $M_{\mathcal T}^{(g)}\odot X_0^{(i)}$ contains the missing no-policy outcomes on intervention days. Each group-specific mask is applied synthetically to the training matrices during score learning. Conditional generation then produces the missing no-policy rows of each treatment-group household matrix conditional on its observed no-policy outcomes. Observed outcomes from all 24 hours of each intervention day are excluded from the generator's conditioning set and enter only the subsequent treatment-effect comparisons.

\subsection{Counterfactual Recovery Accuracy}\label{sec:counterfactual_recovery}

Before using the estimated counterfactuals to evaluate demand responses, we examine how accurately the methods recover no-policy outcomes when these outcomes are available for validation. We artificially mask complete household--day profiles observed under no policy. The masked 24-hour trajectories are treated as missing no-policy outcomes, while their observed values provide the ground truth for measuring recovery.

We compare \CFTDiff with five point-recovery benchmarks: the causal panel methods \DID, \SC, and \SDID, and the matrix/tensor completion methods \NMC and \TFI. We also include the two diffusion specifications nested within \CFTDiff, \ConvDiff and \TuckerDiff, to assess the gains from Tucker dimension reduction and conditioning on the observed no-policy outcomes. All methods receive the same observed no-policy outcomes and are evaluated on the same artificially masked outcomes. Point recovery is measured by MAE; for the three diffusion methods, distributional recovery is assessed using the energy score, CRPS, and interval-based measures.

We consider three missingness patterns that represent different structures of missing outcomes in panel settings. Under the switchback pattern, missing household--days occur intermittently over time. Under staggered adoption, households enter the missing region at different dates, whereas under simultaneous adoption they enter at a common date. Whenever a household--day is selected as missing, all 24 hourly outcomes are removed. Appendix~\ref{sec:appendix_missingness_patterns} provides a visual illustration of the three patterns. For each pattern, we consider seven target missing proportions, ranging from 0.149 to 0.581, so that the evaluation covers missingness levels both below and above that observed in the iFlex application. Let $\rho$ denote the realized missing proportion and let $\rho_{\mathrm{emp}}=0.265$ denote the empirical treated share. Figure~\ref{fig:counterfactual_recovery} reports the relative missingness level $\rho/\rho_{\mathrm{emp}}$, with the vertical dashed line indicating $\rho=\rho_{\mathrm{emp}}$.

\begin{figure}[htb]
	\FIGURE
	{\includegraphics[scale=0.52]{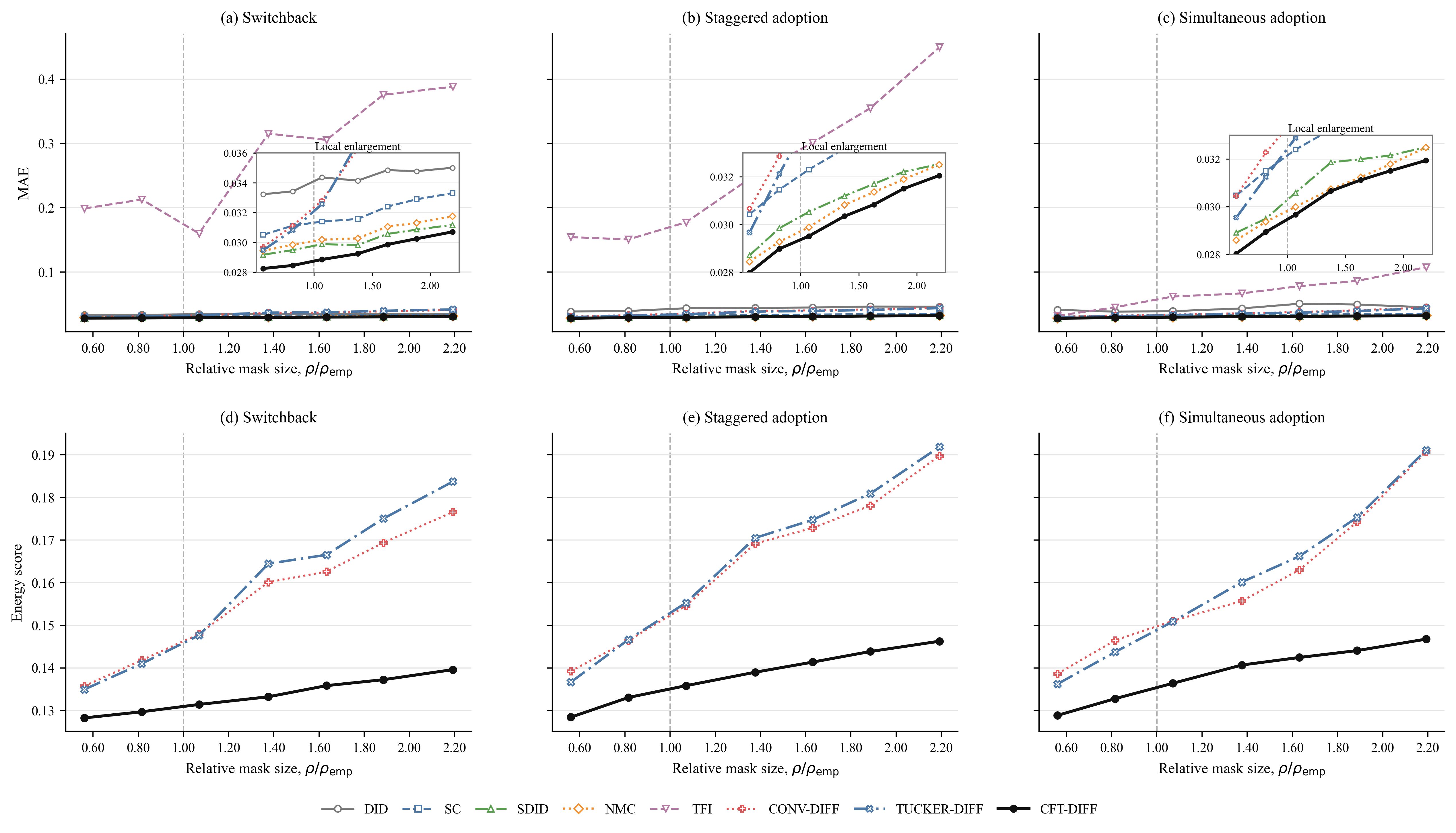}}	
	{Counterfactual recovery across missingness patterns and levels.
	\label{fig:counterfactual_recovery}} 
	{The figure evaluates counterfactual recovery by artificially masking observed no-policy outcomes and comparing the recovered outcomes with their observed values. Panels (a)--(c) report MAE across different missingness levels under the switchback, staggered-adoption, and simultaneous-adoption patterns, respectively. Panels (d)--(f) report the corresponding energy scores for \CFTDiff and its two nested diffusion specifications. The horizontal axis reports the realized missing proportion relative to the empirical treated share, $\rho/\rho_{\mathrm{emp}}$, and the vertical dashed line indicates $\rho=\rho_{\mathrm{emp}}$. Lower values indicate better recovery.}
\end{figure}

Panels (a)--(c) of Figure~\ref{fig:counterfactual_recovery} compare point recovery across the full range of missingness levels. We report MAE under the switchback, staggered-adoption, and simultaneous-adoption patterns, respectively. For the three diffusion methods, the point prediction is obtained by averaging the 100 generated draws. \CFTDiff has the lowest MAE across the three patterns and remains stable as the missing proportion increases. The inset in each panel is a local enlargement of the corresponding region in the original panel and provides a more detailed view of the methods with relatively low errors. The results show that the advantage of \CFTDiff persists across missingness levels. By contrast, the error of \TFI increases substantially as the missing region expands, especially under the switchback and staggered-adoption patterns.

Panels (d)--(f) examine counterfactual distribution recovery for \CFTDiff and its two nested diffusion specifications. We use the energy score to compare the generated distribution of each missing 24-hour trajectory with its observed trajectory. The energy score accounts for both the discrepancy between generated and observed trajectories and the dispersion of the generated trajectories; smaller values indicate better distributional recovery. Its formal definition and sample implementation, together with the definitions of the other evaluation measures, are provided in Appendix~\ref{sec:appendix_recovery_metrics}.

The lower panels show that the energy score generally increases with the missing proportion, reflecting the greater difficulty of recovering a trajectory from fewer observed no-policy outcomes. This increase is substantially smaller for \CFTDiff. Its energy score remains below those of \ConvDiff and \TuckerDiff under all three missingness patterns and throughout the range of missingness levels, with the gap tending to widen under staggered and simultaneous adoption. The point-recovery advantage of \CFTDiff is therefore accompanied by more accurate recovery of the joint distribution of the missing 24-hour trajectory.

\begin{table}[htbp]
\TABLE
{Counterfactual recovery at the missingness level closest to the empirical treated share.
\label{tab:counterfactual_recovery_main}}
{
\begin{tabular*}{\textwidth}{@{\extracolsep{\fill}}lrrrrrr}
\hline
\multicolumn{7}{l}{\textit{Panel A: Point counterfactual recovery}} \\
& \multicolumn{2}{c}{Switchback} & \multicolumn{2}{c}{Staggered adoption} & \multicolumn{2}{c}{Simultaneous adoption} \\
Method & MAE & RMSE & MAE & RMSE & MAE & RMSE \\
\hline
\DID & 0.0344 & 0.0498 & 0.0437 & 0.0664 & 0.0391 & 0.0606 \\
\SC & 0.0314 & 0.0475 & 0.0323 & 0.0488 & 0.0324 & 0.0487 \\
\SDID & \underline{0.0299} & \underline{0.0455} & 0.0305 & 0.0465 & 0.0306 & 0.0464 \\
\NMC & 0.0302 & 0.0464 & \underline{0.0299} & \underline{0.0465} & \underline{0.0300} & \underline{0.0463} \\
\TFI & 0.1599 & 0.2517 & 0.1773 & 0.3072 & 0.0618 & 0.0935 \\
\ConvDiff & 0.0328 & 0.0513 & 0.0354 & 0.0524 & 0.0337 & 0.0508 \\
\TuckerDiff & 0.0326 & 0.0507 & 0.0343 & 0.0528 & 0.0329 & 0.0508 \\
\CFTDiff & \textbf{0.0289} & \textbf{0.0449} & \textbf{0.0295} & \textbf{0.0453} & \textbf{0.0297} & \textbf{0.0454} \\
\hline
\multicolumn{7}{l}{\textit{Panel B: Counterfactual distribution recovery under the switchback mask}} \\
Method & \multicolumn{1}{r}{Energy score} & \multicolumn{1}{c}{CRPS} & \multicolumn{1}{r}{90\% coverage} & \multicolumn{1}{c}{90\% width} & \multicolumn{1}{c}{IS} & \multicolumn{1}{r}{WIS} \\
\hline
\ConvDiff & 0.1479 & 0.0239 & 0.887 & \underline{1.7246} & \underline{2.6548} & \underline{0.2569} \\
\TuckerDiff & \underline{0.1476} & \underline{0.0239} & \underline{0.907} & 1.8352 & 2.6898 & 0.2583 \\
\CFTDiff & \textbf{0.1314} & \textbf{0.0210} & \textbf{0.893} & \textbf{1.2559} & \textbf{2.0126} & \textbf{0.1819} \\
\hline
\end{tabular*}
}
{The table evaluates counterfactual recovery by artificially masking observed no-policy outcomes and comparing the recovered outcomes with their observed values. Panel A reports MAE and RMSE at the available missingness level closest to the empirical treated share, $\rho_{\mathrm{emp}}=0.265$. The target missing proportion is 0.284, with realized missing proportions of 0.283, 0.284, and 0.284 under the switchback, staggered-adoption, and simultaneous-adoption patterns, respectively. Point predictions for \CFTDiff and its two nested diffusion specifications are obtained by averaging 100 generated draws. Panel B reports distributional recovery under the switchback pattern. The energy score evaluates the joint 24-hour trajectory, whereas CRPS is computed for individual hourly outcomes. Coverage and width refer to 90\% central intervals for household--day totals. The interval score is computed for the 90\% central interval, and WIS denotes the weighted interval score combining the predictive median and the 50\%, 60\%, 70\%, 80\%, 90\%, and 95\% central intervals. Errors, scores, and interval widths are reported on the normalized experimental scale. Boldface and underlining indicate the best and second-best values in each column, respectively; coverage is ranked by absolute deviation from the nominal level of 0.90, and the remaining measures are ranked from smallest to largest.}
\end{table}

Table~\ref{tab:counterfactual_recovery_main} provides a more detailed comparison at the available missingness level closest to the empirical treated share. The common target missing proportion is 0.284, with realized missing proportions close to 0.284 under all three patterns. Panel A reports MAE and RMSE. \CFTDiff achieves the lowest value for every point-recovery measure. Under the switchback pattern, its MAE and RMSE are 0.0289 and 0.0449, compared with 0.0299 and 0.0455 for the second-best method. Under staggered adoption, the corresponding values are 0.0295 and 0.0453, while under simultaneous adoption they are 0.0297 and 0.0454. The improvement is therefore consistent across both error measures and all three missingness patterns.

Panel B evaluates the generated counterfactual distribution in more detail under the switchback pattern. The energy score evaluates the joint 24-hour trajectory, whereas CRPS evaluates the marginal distribution of individual hourly outcomes. The 90\% coverage measures how often the observed household--day total falls within the generated 90\% central interval, and should therefore be close to its nominal level of 0.90. The 90\% width measures the average width of these intervals, with a smaller value indicating a sharper interval when coverage remains comparable. The interval score (IS) jointly penalizes wide intervals and observations falling outside the interval, while the weighted interval score (WIS) combines the predictive median with several central intervals. For the energy score, CRPS, IS, and WIS, smaller values indicate better distributional recovery.

The distributional measures give a consistent picture. \CFTDiff reduces the energy score from approximately 0.148 for the two nested diffusion specifications to 0.1314 and reduces CRPS from 0.0239 to 0.0210. Its 90\% empirical coverage is 0.893, close to the nominal level of 0.90, while its average interval width is 1.2559, compared with 1.7246 for \ConvDiff and 1.8352 for \TuckerDiff. The interval score and WIS likewise decrease to 2.0126 and 0.1819, respectively. Thus, the narrower intervals produced by \CFTDiff retain coverage close to the nominal level. Together with Figure~\ref{fig:counterfactual_recovery}, these results show that \CFTDiff improves point and counterfactual distribution recovery across the missingness patterns and levels considered.

\subsection{Causal Estimands}\label{sec:causal_estimands}

Having evaluated counterfactual recovery, we define the causal measures used to assess each price signal relative to the no-policy outcome. 
Within the analysis sample, our estimands are defined by assigned price signals and include households regardless of active response, targeting the effect of assignment to the implemented intervention package. For a household--day--hour observation with $D_{it}=d$, define the demand reduction as $\Delta_{ith}(d)=Y_{ith}(0)-Y_{ith}(d)$, so that a positive value indicates lower electricity use under the price signal. Because $Y_{ith}(d)=Y_{ith}^{\mathrm{obs}}$ whenever $D_{it}=d$, estimating $\Delta_{ith}(d)$ requires recovering the corresponding no-policy outcome $Y_{ith}(0)$.

For a price signal $d$ and a set of hours $\mathcal H$, we consider two aggregate estimands. The pooled relative demand reduction is
\[
R_d(\mathcal H)=100\times\frac{\displaystyle\sum_{(i,t):D_{it}=d}\sum_{h\in\mathcal H}\left\{Y_{ith}(0)-Y_{ith}(d)\right\}}{
\displaystyle\sum_{(i,t):D_{it}=d}\sum_{h\in\mathcal H}Y_{ith}(0)},
\]
and the average absolute reduction per assigned household--day is
\[
\Delta_d^{\mathrm{abs}}(\mathcal H)=\frac{1}{N_d}\sum_{(i,t):D_{it}=d}\sum_{h\in\mathcal H}\left\{Y_{ith}(0)-Y_{ith}(d)\right\},
\]
where $N_d$ is the number of household--day observations assigned signal $d$. We evaluate both estimands over the signal-specific peak-price hours and additionally evaluate the relative reduction over the full 24-hour period. The peak-period estimands measure demand changes during the hours targeted by the price signal, whereas the full-day estimand measures the overall change in daily electricity use. These estimands characterize responses under the realized repeated-intervention schedule. Their operational relevance is supported by the sustained responsiveness documented by \citet{hofmann2024demandflexibility} during the study period.

In the empirical analysis, the unobserved $Y_{ith}(0)$ is replaced by its estimated no-policy counterfactual. For \CFTDiff, the point estimates in Section~\ref{sec:demand_response} use the mean of $B=100$ generated draws, while the complete set of draws is retained for the distributional summaries defined below.

Let $\nu=(g,t)$ denote a group--day intervention event, let $\mathcal E_d$ be the set of realized events assigned signal $d$, and let $\mathcal H_d$ denote the signal-specific peak-price hours. For the $b$th sample from $\widehat P_{t_0}^{\mathcal T}(\cdot\mid X_{\mathcal C})$, let $\widetilde Y_\nu^{(b)}(0)$ denote the generated no-policy outcomes for event $\nu$, and write $\widetilde Y_{ith}^{(b)}(0)$ for the generated no-policy outcome of household $i$ at hour $h$. For $\nu=(g,t)\in\mathcal E_d$, define the corresponding peak-period relative demand reduction by
\[
R_\nu^{(b)}=100\times\frac{\displaystyle\sum_{i:G_i=g}\sum_{h\in\mathcal H_d}\left\{\widetilde Y_{ith}^{(b)}(0)-Y_{ith}^{\mathrm{obs}}\right\}}{\displaystyle\sum_{i:G_i=g}\sum_{h\in\mathcal H_d}\widetilde Y_{ith}^{(b)}(0)},\qquad b=1,\ldots,B.
\]
The pooled response distribution for signal $d$ is represented by $\mathcal R_d=\left\{R_\nu^{(b)}:\nu\in\mathcal E_d;b=1,\ldots,B\right\}$. To measure how reliably a signal attains a target reduction $c$, define $\widehat{\pi}_d(c)=\frac{1}{B|\mathcal E_d|}\sum_{\nu\in\mathcal E_d}\sum_{b=1}^{B}\mathbb I\!\left\{R_\nu^{(b)}\geq c\right\}$. Thus, $\widehat{\pi}_d(c)$ is the fraction of pooled reductions that reach the target. The collection $\mathcal R_d$ combines variation across realized intervention events with uncertainty in the generated no-policy counterfactuals.\footnote{It is distinct from the sampling distribution of an average-effect estimator and from a joint cross-world distribution of individual treatment effects.}

\subsection{Demand Response to Dynamic Pricing}\label{sec:demand_response}

We now use the no-policy counterfactuals generated by \CFTDiff to evaluate the demand-response measures defined in Section~\ref{sec:causal_estimands}. Point estimates are obtained by averaging $B=100$ conditional draws.

\begin{figure}[htb]
	\FIGURE
	{\includegraphics[scale=0.45]{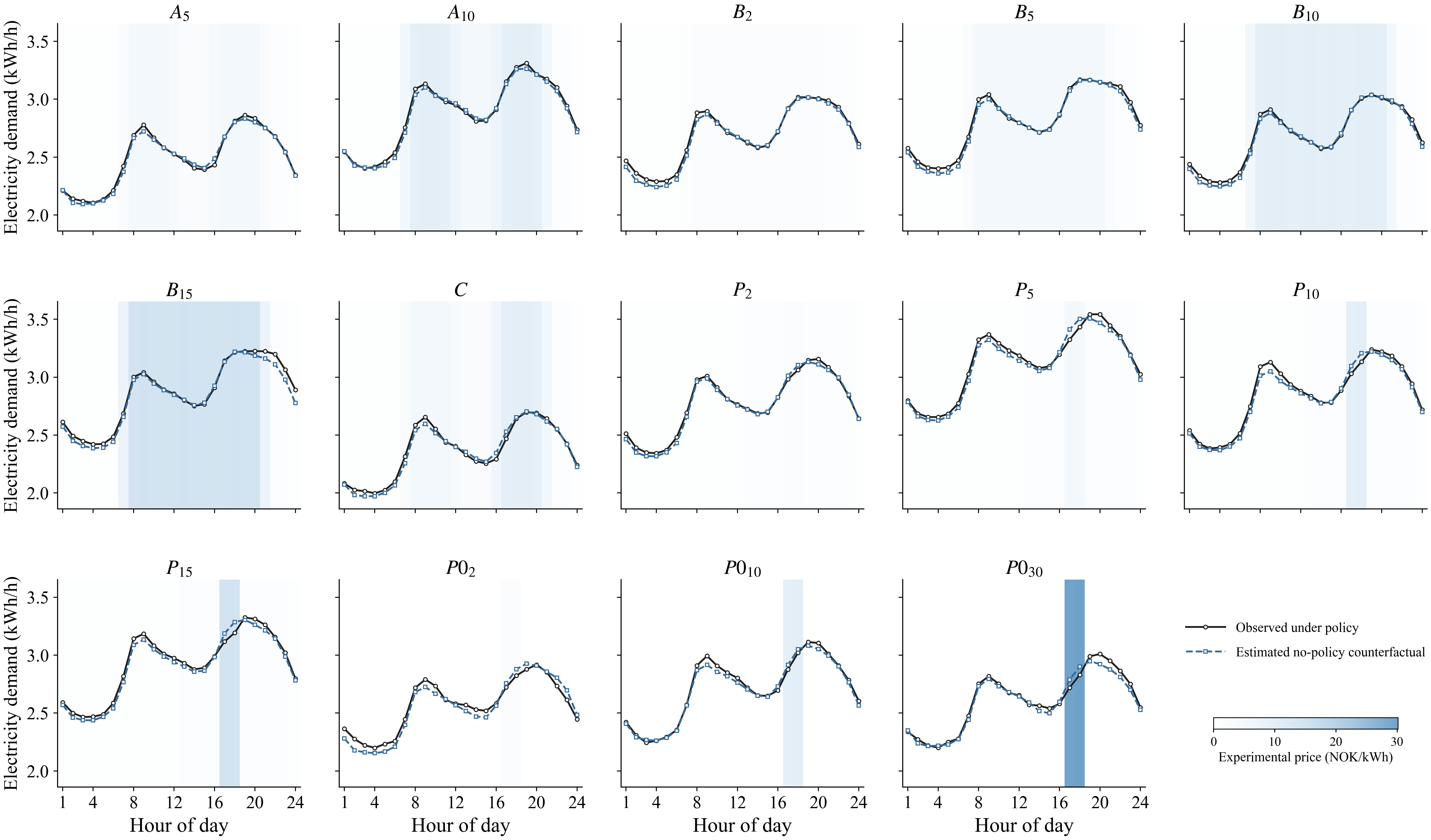}}	
	{Observed and counterfactual electricity demand by price signal.
	\label{fig:causal_hourly_demand_profiles}}
    {The solid line reports observed hourly electricity demand under the assigned price signal, and the dashed line reports the estimated no-policy counterfactual. The counterfactual curve averages 100 \CFTDiff conditional draws. Background shading indicates the experimental electricity price, with darker shading corresponding to higher prices.}

\end{figure}

Figure~\ref{fig:causal_hourly_demand_profiles} shows the intraday pattern of electricity demand and its response to the different price signals. Across signals, demand follows a similar daily profile, with a morning peak and a larger evening peak. During the high-price hours, observed demand generally remains close to its estimated no-policy counterfactual, although the direction and magnitude of the difference vary across signals. The clearest reduction occurs for $P0_{30}$, which imposes a peak price of 30 NOK/kWh over a two-hour afternoon window, whereas most other signals produce only modest changes. For several signals, observed demand exceeds its estimated counterfactual outside the high-price window. This pattern is consistent with intertemporal substitution: households may shift some deferrable electricity use to lower-price hours rather than eliminate it altogether. Within a given price profile, higher peak prices are associated with larger reductions in some cases, most visibly among the $P0$ signals, although the pattern also varies with the timing and duration of the intervention. The responses therefore differ across price designs in both magnitude and timing.

\begin{table}
	\TABLE
	{Causal and distributional effects by price signal.
	\label{tab:treatment_effects_by_signal}}
	{
	\setlength\tabcolsep{2pt}
	\begin{tabular*}{\textwidth}{@{\extracolsep{\fill}}lccccccc}
		\toprule
		Signal & Peak h & Events & Peak kWh & Peak \% & Full-day \% & $Q_{0.10}$ & $\widehat\pi_d(3\%)$ \\
		\midrule
		$A_5$     & 8  & 9  & -0.164 & -0.76 & -0.36 & -2.34 & 0.01 \\
		$A_{10}$    & 8  & 10 & -0.154 & -0.61 & -0.46 & -2.54 & 0.00 \\
		$B_2$     & 13 & 35 & -0.080 & -0.22 & -0.77 & -2.15 & 0.02 \\
		$B_5$     & 13 & 29 & -0.103 & -0.27 & -0.78 & -2.26 & 0.04 \\
		$B_{10}$    & 13 & 39 & -0.034 & -0.09 & -0.58 & -2.13 & 0.02 \\
		$B_{15}$    & 13 & 28 & -0.079 & -0.20 & -1.01 & -2.31 & 0.02 \\
		$C$      & 8  & 8  & -0.064 & -0.31 & -0.41 & -2.27 & 0.02 \\
		$P_2$     & 2  & 16 & 0.070  & 1.14  & -0.49 & -2.56 & 0.24 \\
		$P_5$     & 2  & 13 & 0.158  & 2.29  & -0.71 & -1.54 & 0.41 \\
		$P_{10}$    & 2  & 27 & 0.140  & 2.22  & -0.69 & -1.77 & 0.41 \\
		$P_{15}$    & 2  & 20 & 0.163  & 2.51  & -0.75 & -1.62 & 0.42 \\
		$P0_{2}$  & 2  & 5  & 0.089  & 1.58  & -0.80 & -1.98 & 0.31 \\
		$P0_{10}$ & 2  & 10 & 0.069  & 1.15  & -0.54 & -1.56 & 0.22 \\
		$P0_{30}$ & 2  & 5  & 0.139  & 2.43  & -0.61 & -0.04 & 0.35 \\
		\bottomrule
	\end{tabular*}%
	}
    {Positive values denote electricity-demand reductions. Peak h is the number of hours in the signal-specific peak-price period, and Events is the number of group-day intervention events assigned signal $d$. Peak kWh is the average absolute reduction per assigned household--day over the peak-price period. Peak and full-day percentages are pooled ratio-of-sums estimates based on the mean of 100 \CFTDiff conditional draws. $Q_{0.10}$ is the 10th percentile of the event-by-draw peak-reduction distribution and $\widehat\pi_d(3\%)$ is the corresponding fraction of outcomes with a peak reduction of at least 3\%.}
\end{table}

Table~\ref{tab:treatment_effects_by_signal} quantifies these patterns. Seven signals---all $P$ and $P0$ signals---have positive peak-period estimates, and the peak-period percentage exceeds the corresponding full-day percentage for 12 of the 14 signals. Thus, except for the two $A$ signals, the demand response is more favorable during the targeted hours than over the day as a whole. The largest estimated peak reduction occurs under $P_{15}$, at 2.51\%, or 0.163 kWh per assigned household--day over the two-hour peak period. The comparison between peak-period and full-day estimates is consistent with households shifting electricity use away from high-price hours rather than uniformly reducing daily consumption.

The aggregate estimates in Table~\ref{tab:treatment_effects_by_signal} do not show when the offsetting demand adjustments occur. We therefore examine the temporal pattern of the response more closely. Figure~\ref{fig:causal_dynamic_effects} reports the hourly effects and then expands the evaluation window from the signal-specific peak hours toward the full day.

\begin{figure}[htb]
	\FIGURE
	{\includegraphics[scale=0.47]{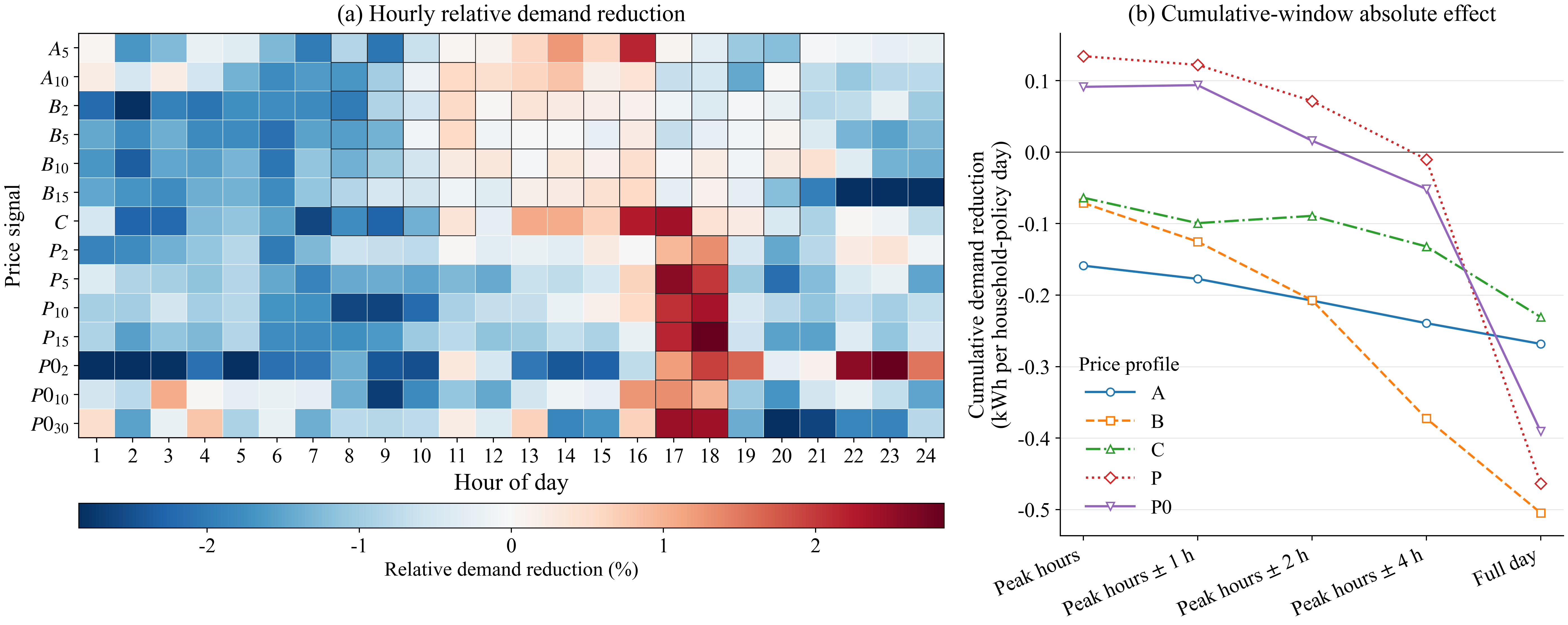}}	
	{Hourly and cumulative demand reductions.
	\label{fig:causal_dynamic_effects}} 
    {Panel (a) reports hourly relative demand reductions for each price signal, with positive values indicating lower observed demand relative to the estimated no-policy counterfactual. Thin outlines identify the signal-specific peak-price hours. Panel (b) reports cumulative demand changes in kWh per assigned household--day as the evaluation window expands from the peak-price hours to the full day. Results in Panel (b) are aggregated by price profile and weighted by the observed frequency of household--day assignments.}
\end{figure}

Panel (a) shows that, for many price signals, observed demand decreases relative to the estimated no-policy counterfactual during the high-price hours, with larger reductions inside the targeted periods than in surrounding hours. The magnitude and timing of the response vary across signals with different price levels and intervention durations. Panel (b) shows that the estimated cumulative reduction generally becomes smaller as the evaluation window expands beyond the signal-specific peak hours, turning from positive around the peak hours to negative over the full day for profiles $P$ and $P0$. This pattern is consistent with, but does not by itself establish, intertemporal shifting of electricity use from the targeted periods to other hours. From an operational perspective, shifting demand away from a constrained period may remain valuable even without a reduction in total daily consumption.

\subsection{Counterfactual Distributions, Intervals, and Response Reliability}\label{sec:distributional_effects}

The preceding analysis uses the mean of the \CFTDiff draws to study average demand responses. We now use the pooled response distributions $\mathcal R_d$ and target-attainment rates $\widehat{\pi}_d(c)$ defined in Section~\ref{sec:causal_estimands} to examine variation in event-level responses and the reliability with which each signal attains a specified reduction target. The estimated conditional distributions provide prediction intervals for the missing no-policy outcomes.

\begin{figure}[htb]
	\FIGURE
	{\includegraphics[scale=0.45]{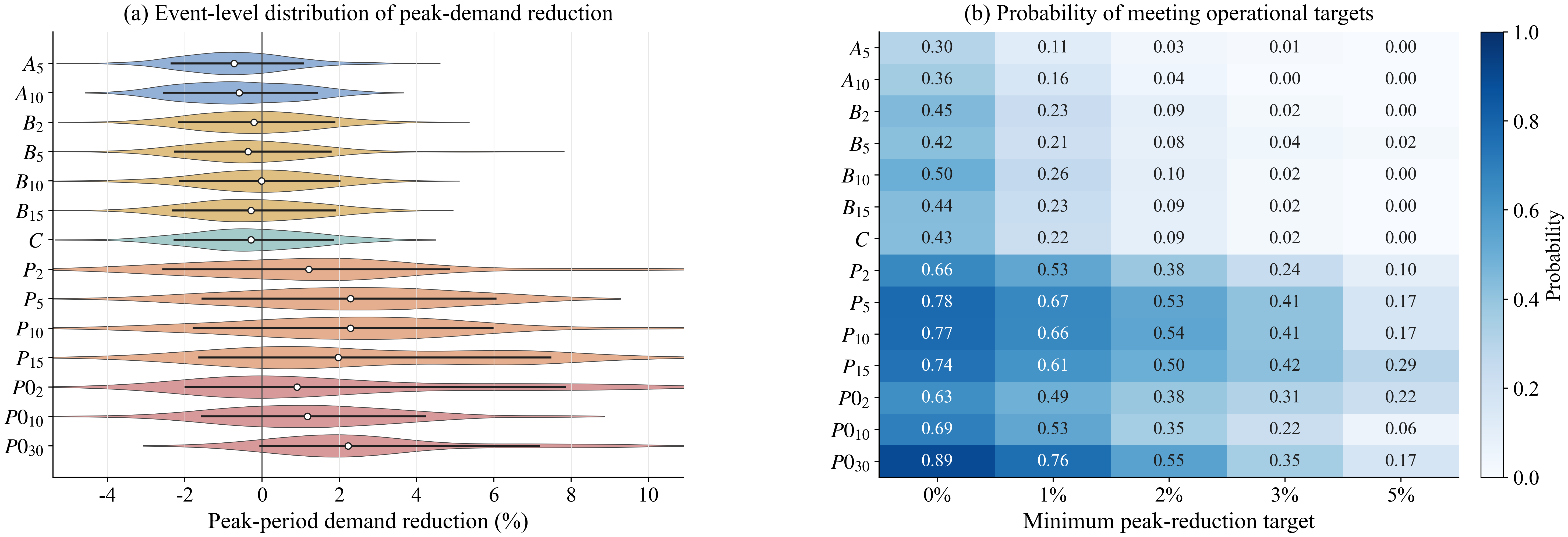}}	
	{Counterfactual distribution and reliability of demand response.
	\label{fig:causal_distribution_reliability}} 
	{For each intervention event $\nu$ and \CFTDiff draw $b$, $R_\nu^{(b)}$ denotes the peak-period relative demand reduction computed using the generated no-policy counterfactual $\widetilde{Y}_\nu^{(b)}(0)$ from $\widehat{P}_{t_0}^{\mathcal T}(\cdot\mid X_{\mathcal C})$. Panel (a) reports the empirical distribution of $R_\nu^{(b)}$ across intervention events and $B=100$ generated samples for each price signal; the white point denotes the median and the horizontal segment spans the 10th to the 90th percentile. Panel (b) reports the corresponding probability of achieving alternative peak-period reduction targets. Positive values indicate lower observed demand relative to the generated no-policy counterfactual.}
\end{figure}

Figure~\ref{fig:causal_distribution_reliability} shows substantial differences in the distribution and reliability of demand response across price signals. Panel~(a) shows that peak-period responses span a wide range across intervention events and conditional draws for most signals. The $A$ and $B$ signals are generally concentrated at lower reductions, whereas the $P$ signals place more of their distributions above zero. Panel~(b) translates these distributions into target-attainment rates. $P0_{30}$ produces a positive peak reduction in 89\% of the event--draw outcomes and a reduction of at least 3\% in 35\%. For the $A$ and $B$ signals, the corresponding fractions attaining a 3\% reduction range from 0\% to 4\%. The relative performance of the P signals also varies with the target: $P0_{30}$ has the highest probability of attaining the lower targets, whereas $P_{15}$ performs better at the 3\% and 5\% targets. Thus, the full counterfactual distribution provides information that is not available from the mean response alone: price signals with similar average effects can differ substantially in how consistently they reduce demand during the targeted high-price hours.

The conditional distributions also provide prediction intervals for the average no-policy outcomes in the treated region. Table~\ref{tab:counterfactual_intervals_by_signal} reports central 90\% prediction intervals from \CFTDiff and its two nested diffusion specifications, together with their $\sqrt{n_e}$-scaled widths and width ratios. 

\begin{table}[htb]
\TABLE
{Counterfactual prediction intervals for average no-policy outcomes by price signal.
\label{tab:counterfactual_intervals_by_signal}}
{
\setlength{\tabcolsep}{3.5pt}
\begin{tabular*}{\textwidth}
{@{\extracolsep{\fill}}lcccccccc@{}}
\hline
&
\multicolumn{3}{c}{No-policy prediction interval}
&
\multicolumn{3}{c}{Scaled interval width}
&
\multicolumn{2}{c}{Width ratio}
\\
Signal & \textsc{conv} & \textsc{tucker} & \textsc{cft} & \textsc{conv} & \textsc{tucker} & \textsc{cft} & \textsc{conv}/\textsc{cft} & \textsc{tucker}/\textsc{cft} \\
\hline
$A_{5}$ & [2.297, 2.345] & [2.233, 2.289] & [2.454, 2.482] & 1.089 & 1.265 & 0.631 & 1.725 & 2.004 \\
$A_{10}$ & [2.592, 2.654] & [2.538, 2.598] & [2.847, 2.880] & 1.394 & 1.360 & 0.747 & 1.867 & 1.822 \\
$B_{2}$ & [2.534, 2.574] & [2.464, 2.497] & [2.658, 2.678] & 1.593 & 1.347 & 0.818 & 1.948 & 1.647 \\
$B_{5}$ & [2.636, 2.682] & [2.559, 2.606] & [2.772, 2.800] & 1.665 & 1.740 & 1.039 & 1.603 & 1.675 \\
$B_{10}$ & [2.512, 2.548] & [2.437, 2.479] & [2.648, 2.671] & 1.566 & 1.793 & 1.008 & 1.553 & 1.778 \\
$B_{15}$ & [2.725, 2.771] & [2.629, 2.680] & [2.839, 2.862] & 1.707 & 1.870 & 0.867 & 1.968 & 2.156 \\
$C$ & [2.218, 2.269] & [2.178, 2.224] & [2.335, 2.369] & 1.135 & 1.033 & 0.769 & 1.475 & 1.342 \\
$P_{2}$ & [2.608, 2.654] & [2.529, 2.577] & [2.741, 2.768] & 1.231 & 1.277 & 0.747 & 1.648 & 1.709 \\
$P_{5}$ & [2.926, 2.980] & [2.861, 2.922] & [3.109, 3.143] & 1.441 & 1.656 & 0.895 & 1.610 & 1.850 \\
$P_{10}$ & [2.642, 2.687] & [2.560, 2.609] & [2.818, 2.843] & 1.578 & 1.740 & 0.858 & 1.839 & 2.027 \\
$P_{15}$ & [2.713, 2.762] & [2.635, 2.695] & [2.908, 2.931] & 1.510 & 1.862 & 0.708 & 2.132 & 2.629 \\
$P0_{2}$ & [2.339, 2.396] & [2.284, 2.341] & [2.532, 2.562] & 0.923 & 0.911 & 0.476 & 1.939 & 1.914 \\
$P0_{10}$ & [2.449, 2.504] & [2.390, 2.449] & [2.683, 2.711] & 1.251 & 1.348 & 0.637 & 1.964 & 2.116 \\
$P0_{30}$ & [2.420, 2.475] & [2.353, 2.413] & [2.579, 2.610] & 0.887 & 0.977 & 0.496 & 1.790 & 1.972 \\
\hline
\multicolumn{4}{@{}l}{Mean across signals} & 1.355 & 1.441 & 0.764 & 1.790 & 1.903 \\
\hline
\end{tabular*}
}
{The table uses all hourly outcomes in the signal-specific treated region. The \textsc{conv}, \textsc{tucker}, and \textsc{cft} columns correspond to \ConvDiff, \TuckerDiff, and \CFTDiff, respectively, and report central 90\% prediction intervals for their generated average no-policy outcomes. For each signal, the scaled interval width equals the interval width multiplied by $\sqrt{n_e}$, where $n_e$ is the number of treated households contributing to that signal. The interval endpoints remain on the original kWh scale. The width ratios compare the \textsc{conv} and \textsc{tucker} widths with the \textsc{cft} width; values above one indicate wider intervals than under \textsc{cft}. The final row reports arithmetic means of the displayed width statistics across signals.}
\end{table}

Table~\ref{tab:counterfactual_intervals_by_signal} shows that \CFTDiff produces the narrowest prediction interval for every price signal. Its mean $\sqrt{n_e}$-scaled width is 0.764, compared with 1.355 for \ConvDiff and 1.441 for \TuckerDiff. The corresponding mean width ratios are 1.790 and 1.903, so the \ConvDiff and \TuckerDiff intervals are, on average, 79.0\% and 90.3\% wider than the \CFTDiff intervals, respectively. The artificial-mask validation in Table~\ref{tab:counterfactual_recovery_main} provides the corresponding calibration evidence: the narrower \CFTDiff intervals retain empirical coverage close to the nominal level.


\section{Conclusion}\label{sec:conclusion}

This paper develops \CFTDiff to extend counterfactual analysis in panel or tensor data beyond point recovery to the joint conditional distribution of missing control outcomes. The method generates these outcomes conditional on the observed control outcomes, with efficiency gained by learning the nonlinear component of the score in a low-dimensional Tucker core. Repeated draws from the learned distribution preserve dependence among the missing outcomes and yield point estimates, quantiles, counterfactual prediction intervals, and target-attainment probabilities from a single model.

Nonasymptotic, high-probability bounds show that the difficulty of conditional-score estimation depends on the Tucker core dimension, the largest mode dimension, and the effective number of missing outcomes rather than on the full tensor dimension alone. The score bounds imply recovery guaranties for the conditional distribution and weighted summaries, while the prediction-interval guarantee accounts for a finite number of generated draws. Simulations confirm the roles of both components of \CFTDiff: Tucker dimension reduction and conditioning on the observed control outcomes each improve recovery, and \CFTDiff is more accurate than the causal panel and matrix/tensor completion methods considered. In the iFlex study, artificial-mask validation confirms these gains across three treatment patterns. The subsequent causal analysis shows that price signals with similar average demand reductions can have different target-attainment probabilities. Rankings based only on average responses can therefore conceal meaningful differences in intervention reliability.

Future research can extend \CFTDiff to causal settings with pretreatment covariates, more complex treatment patterns, or heavy-tailed outcomes. The heavy-tailed setting is especially important when decisions depend on rare or extreme counterfactual events. Another direction is to use the recovered distributions directly in intervention selection and optimization and establish guaranties for the resulting decisions. When Tucker structure is too restrictive, learned nonlinear representations such as variational autoencoders may provide an alternative form of dimension reduction. Comparing these representations with Tucker reduction would clarify which low-dimensional structure is most effective across data settings and sample sizes.

\ACKNOWLEDGMENT{All authors contributed equally and are listed in alphabetical order.}

\theendnotes


\begingroup
\def\bibfont{\fontsize{10pt}{14pt}\selectfont}
\def\bibsep{0pt}
\bibliographystyle{informs2014} 
\bibliography{ref.bib} 
\endgroup



\ECSwitch

\fontsize{11pt}{16pt}\selectfont

\ECHead{Technical Proofs and Additional Results}

This appendix contains technical proofs and additional results for the paper.
\begin{itemize}
    \item Appendix \ref{app:notation} provides the summary of notation.
    \item Appendix \ref{app:proofs} provides the proofs of theorems in the main text.
    \item Appendix \ref{appendix_sec:Framework Details} provides the details of our framework.
    \item Appendix \ref{app:Additional Simulations} provides the additional simulation results.
    \item Appendix \ref{sec:Additional Empirical Analysis} provides additional empirical results.
\end{itemize}


\section{Summary of Notation}\label{app:notation}

Table~\ref{tab:notation_main} summarizes the principal notation used in the main text, while Table~\ref{tab:notation_appendix} collects additional notation used in the e-companion. Both tables group notation by its role in the analysis.

\begin{table}[htb]
\TABLE
{Principal notation used in the main text.\label{tab:notation_main}}
{%
\begin{minipage}{6.10in}
\scriptsize
\setlength{\tabcolsep}{0pt}
\renewcommand{\arraystretch}{1.08}
\begin{tabular}{@{}
>{\raggedright\arraybackslash}p{1.05in}
@{\hspace{0.08in}}
>{\raggedright\arraybackslash}p{1.80in}
@{\hspace{0.20in}}
>{\raggedright\arraybackslash}p{1.05in}
@{\hspace{0.08in}}
>{\raggedright\arraybackslash}p{1.80in}
@{}}
\hline
\multicolumn{4}{@{}l@{}}{\emph{General notation}}\\
\hline
$\mathcal{L}(X),\mathcal{L}(X\mid Y)$
& Probability law and conditional probability law
& $\lVert\cdot\rVert_2,\lVert\cdot\rVert_F,
   \lVert\cdot\rVert_{\mathrm{op}}$
& Euclidean, Frobenius, and operator norms\\
$\operatorname{vec}(\cdot),\operatorname{Tucker}(\cdot)$
& Vectorization and reshaping into a Tucker core tensor
& $\odot,\oslash$
& Entrywise multiplication and division\\
\multicolumn{2}{@{}l}{\emph{Potential outcomes and counterfactual target}}
& \multicolumn{2}{l@{}}{\emph{Tucker factor structure}}\\
\cline{1-2}\cline{3-4}
$D,p_d,p$
& Number of modes, dimension of mode $d$, and full dimension
  $p=\prod_{d=1}^D p_d$
& $F$
& Latent Tucker core tensor\\
$\mathcal{I}$
& Tensor index set $[p_1]\times\cdots\times[p_D]$
& $A_d$
& Mode-$d$ loading matrix, with $A_d^\top A_d=I_{r_d}$\\
$Y_{\boldsymbol{i}}(a)$
& Potential outcome at coordinate $\boldsymbol{i}$ under intervention $a$;
  $a=0$ denotes control
& $E$
& Idiosyncratic tensor component\\
$Y_{\boldsymbol{i}}^{\mathrm{obs}}$
& Observed outcome at coordinate $\boldsymbol{i}$
& $r_d,r$
& Mode-$d$ Tucker rank and core dimension
  $r=\prod_{d=1}^D r_d$\\
$\mathcal{T}_a,\mathcal{T},\mathcal{C}$
& Intervention-$a$ support, treated region, and the observed-control region
& $\beta_d,p^\beta$
& Mode-$d$ factor strength and scaling
  $p^\beta=\prod_{d=1}^D p_d^{\beta_d}$\\
$M_{\mathcal{T}_a},M_{\mathcal{T}},M_{\mathcal{C}}$
& Binary masks for $\mathcal{T}_a$, $\mathcal{T}$, and $\mathcal{C}$
& $f,e$
& Vectorized core $f=\operatorname{vec}(F)$ and idiosyncratic component
  $e=\operatorname{vec}(E)$\\
$X_0,P_0$
& Complete control-world tensor and its distribution
  $P_0=\mathcal{L}(X_0)$
& $p_{\mathrm{core}}$
& Density of the vectorized core $f$\\
$X_{\mathcal{T}},X_{\mathcal{C}}$
& Missing and observed control outcomes
& $\Sigma_{e,d},\Sigma_e^{\otimes}$
& Mode-specific and separable idiosyncratic covariance matrices\\
$P_0^{\mathcal{T}}(\cdot\mid X_{\mathcal{C}})$
& Target law $\mathcal{L}(X_{\mathcal{T}}\mid X_{\mathcal{C}})$
& $A_{\otimes},\mathcal Q_\sigma$
& Kronecker loading matrix and tensor of coordinatewise idiosyncratic
  variances\\
$\{X_0^{(i)}\}_{i=1}^n$
& Training sample
& $\mathcal{W}_{\mathcal{T},t},\mathcal{W}_{\mathcal{C}}$
& Coordinatewise precision tensors for the two regions\\
$w,U_{0,a}(w)$
& Prespecified weight vector and intervention-specific control-world summary
& $H_{\mathcal{T},t},H_{\mathcal{C}}$
& Core-level information matrices for the two regions\\
&
&
$V_t,g_t,G_t$
& Core covariance, vectorized core statistic, and its tensorization\\
\multicolumn{2}{@{}l}{\emph{Masked conditional diffusion}}
& \multicolumn{2}{l@{}}{\emph{Estimation and distributional recovery}}\\
\cline{1-2}\cline{3-4}
$t,t_0,T,u$
& Diffusion time, early-stopping time, terminal time, and reverse elapsed time
& $\xi(g,t)$
& Conditional core regression
  $\mathbb{E}(\alpha_t f\mid g_t=g)$\\
$\alpha_t,h_t$
& Signal coefficient and noise variance;
  $\alpha_t=e^{-t/2}$ and $h_t=1-e^{-t}$
& $\Gamma,\omega$
& Candidate Tucker loadings and idiosyncratic variances\\
$W_t,\overline{W}_u,Z_t$
& Forward and reverse Wiener processes and Gaussian perturbation tensor
& $\mathcal{Z}_\theta,\zeta_\theta$
& Tensor-valued Core-Net and its vectorized form\\
$X_{\mathcal{T},t}$
& Forward-diffused outcomes in the treated region at time $t$
& $s_{\Gamma,\omega,\theta}$
& Candidate masked Tucker score\\
$P_t^{\mathcal{T}},p_t^{\mathcal{T}}$
& Conditional distribution and density of $X_{\mathcal{T},t}$ given
  $X_{\mathcal{C}}$
& $\mathcal{S}_{\mathrm{MT}}$
& Masked Tucker conditional-score class\\
$s_t(X\mid X_{\mathcal{C}})$
& Conditional score
  $\nabla_{\mathcal{T}}\log p_t^{\mathcal{T}}(X\mid X_{\mathcal{C}})$
& $\mathcal{L}_{\mathrm{mask}},
   \widehat{\mathcal{L}}_{\mathrm{mask}}$
& Population and empirical masked score-matching losses\\
$X_{\mathcal{T},u}^{\leftarrow}$
& Exact reverse-time process in the treated region
& $\widehat{s}$
& Conditional score learned from the training sample\\
$\widehat{X}_{\mathcal{T},u}^{\leftarrow}$
& Reverse process driven by $\widehat{s}$
& $p_{\mathcal{T}},d_{\mathcal{T},\beta},p_{\max}$
& Number of missing outcomes, effective treated dimension, and largest mode dimension\\
$\widehat{P}_{t_0}^{\mathcal{T}}
 (\cdot\mid X_{\mathcal{C}})$
& Conditional distribution generated by the trained reverse process
& $V_{n_e,w,0},\widehat{V}_{n_e,w,t_0}$
& Target and generated weighted control-world averages\\
$\widehat{I}_{n_e,w,1-\alpha}^{(B)}(x)$
& Counterfactual prediction interval based on $B$ generated draws
& $\tau_{n_e,w,a}$
& Weighted treated-minus-control counterfactual contrast\\
\hline
\end{tabular}
\end{minipage}%
}
{}
\end{table}


\begin{table}[htb]
\TABLE
{Additional notation used in the e-companion.
 \label{tab:notation_appendix}}
{%
\begin{minipage}{6.10in}
\scriptsize
\setlength{\tabcolsep}{0pt}
\renewcommand{\arraystretch}{1.08}
\begin{tabular}{@{}
>{\raggedright\arraybackslash}p{1.05in}
@{\hspace{0.08in}}
>{\raggedright\arraybackslash}p{1.80in}
@{\hspace{0.20in}}
>{\raggedright\arraybackslash}p{1.05in}
@{\hspace{0.08in}}
>{\raggedright\arraybackslash}p{1.80in}
@{}}
\hline
\multicolumn{2}{@{}l}{\emph{Vectorized proof notation}}
& \multicolumn{2}{l@{}}{\emph{Discrete training and generation}}\\
\cline{1-2}\cline{3-4}
$x_0,x_{\mathcal{T}},x_{\mathcal{C}}$
& Vectorizations of $X_0,X_{\mathcal{T}},X_{\mathcal{C}}$
& $j,N_{\mathrm{step}},\bar{\alpha}_j$
& Discrete diffusion index, number of steps, and cumulative signal coefficient\\
$x_{\mathcal{T},t}$
& Vectorized forward-diffused outcomes in the treated region
& $x_j,\widehat{\epsilon}_\theta$
& Masked input and learned noise predictor at step $j$\\
$\MT, \MC$
& Diagonal matrix forms of the region masks
& $\widehat{X}_0,\widetilde{X}_0^{(b)}$
& Inferred clean tensor and conditional completion $b$\\
$\Omega_t,\Omega_0$
& Full-dimensional noise covariance matrices at times $t$ and $0$
& $H,W,k_1,k_2$
& Matrix dimensions and implemented Tucker-core dimensions\\
$\Lambda_{\mathcal{T},t},\Lambda_{\mathcal{C}}$
& Masked precision matrices for the two regions
& $\phi_{\text{enc}},U,V$
& Feature stem and trainable row and column Tucker bases\\
$\phi(\cdot;\mu,V)$
& Gaussian density with mean $\mu$ and covariance $V$
& $E_{\mathcal{C}},E_{\mathcal{T}}$
& Encoded observed-control and treated-region features\\
$p_{\mathrm{con}}^t$
& Marginal density of the low-dimensional core statistic $g_t$
& $B$
& Number of independently generated conditional draws\\
$\lVert v\rVert_B^2$
& Quadratic seminorm $v^\top Bv$
& $x_{\max}$
& Clipping threshold in DDIM generation\\
\multicolumn{2}{@{}l}{\emph{Theoretical rate quantities}}
& \multicolumn{2}{l@{}}{\emph{Simulation and empirical evaluation}}\\
\cline{1-2}\cline{3-4}
$\Sigma_f,\mathcal{G}_R$
& Core covariance and region
  $\{g:\lVert g\rVert_\infty\le R\}$
& $\rho,\rho_{\mathrm{emp}}$
& Simulation missing rate and empirical treated proportion\\
$B_f,C_1,C_2$
& Constants in the tail condition for $p_{\mathrm{core}}$
& $\mathcal{J},N_{\mathcal{J}}$
& Set and number of missing simulation entries\\
$v_{\max},L_g,L_t,K_0$
& Core-variance, Lipschitz, and magnitude constants
& $y_i,x_i^{(b)},\widehat{y}_i$
& Realized value, generated draw, and generated-mean prediction\\
$R_\epsilon$
& High-probability core truncation radius
& $L_i,U_i$
& Endpoints of an entrywise prediction interval\\
$L,m,J,\kappa$
& Network depth, width, sparsity, and parameter bound
& $\mu_i,\sigma_i,\widehat{\mu}_{mi},\widehat{\sigma}_{mi}$
& Exact and generated conditional moments\\
$K,\gamma_g,\gamma_t,C_V$
& Network output, Lipschitz, and core-covariance bounds
& $D_{mi}$
& Pointwise location-and-scale discrepancy for method $m$\\
$A_\bullet,\omega^0$
& Population loading matrices and idiosyncratic variances
& $\mathrm{MAE},\mathrm{RMSE}$
& Point-recovery measures\\
$\delta_n,\epsilon,a_n$
& Exponent adjustment, approximation tolerance, and early-stopping exponent
& $\mathrm{Energy},\mathrm{CRPS}$
& Joint and marginal distributional scores\\
$L_n,\mathfrak R_n,\mathcal{A}_n$
& Polylogarithmic factor, recovery-rate sequence, and high-probability event
& $\mathrm{IS},\mathrm{WIS}$
& Interval and weighted interval scores\\
$\mathcal{F}_{\mathrm{tr}}$
& Sigma-field generated by the training sample
& $\mathcal{M},n_{\mathcal{M}}$
& Set and number of artificially masked empirical outcomes\\
$M_{n,n_e,w}$
& Conditional-density bound used for anti-concentration
& $U_j,\widetilde{U}_j^{(b)}$
& Observed and generated household--day totals\\
$\widehat{F}_x,\widehat{q}_x$
& Conditional CDF and quantile function of the generated weighted summary
& $L_{j,\alpha},H_{j,\alpha},m_j$
& Interval endpoints and predictive median for household--day $j$\\
$\widehat{F}_{B,x},D_{B,x}$
& Empirical generated CDF and its uniform error
& $n_B,K_C$
& Number of evaluated blocks and number of central coverage levels\\
\hline
\end{tabular}
\end{minipage}%
}
{}
\end{table}

\section{Proofs}\label{app:proofs}

Throughout this section, tensor-valued scores are identified with their vectorizations when vector notation is used. Vectorization preserves inner products and hence identifies the Frobenius norm with the Euclidean norm. In vector coordinates, \(p_t(\cdot\mid x_\C)\) denotes the density \(p_t^\T(\cdot\mid X_\C)\) under this identification. For a candidate loading--variance pair \((\Gamma,\omega)\), write
\[
\Gamma_\otimes=\Gamma_D\otimes\cdots\otimes\Gamma_1,
\qquad
g_t(\Gamma,\omega)=\vecc\{G_t(\Gamma,\omega)\},
\]
\[
\Lambda_{\T,t}(\omega)
=\diag\{\vecc(\mathcal W_{\T,t}(\omega))\},
\qquad
\Lambda_\C(\omega)
=\diag\{\vecc(\mathcal W_\C(\omega))\}.
\]

\subsection{Proof of Proposition~\ref{prop:main-score-decomposition}}
We first vectorize the tensor model in Proposition~\ref{prop:main-score-decomposition} for the sake of notational brevity, which is a proof device only. The model and algorithm are stated in
tensor notation in Section~\ref{sec:main-method}. Recall the identity $\vecc(G\times_{d=1}^D B_d)
=
(B_D\otimes\cdots\otimes B_1)\vecc(G)$. Write
\(x_{\T,t}=\vecc(X_{\T,t})\), \(x_\C=\vecc(X_\C)\),
\(f=\vecc(F)\), and
\(A_\otimes=A_D\otimes\cdots\otimes A_1\). The matrix forms of the two noise covariances and the masked precision matrices are $\Omega_t=h_tI_p+\alpha_t^2p^{-\beta}\Sigma_e^\otimes$, $\Omega_0=p^{-\beta}\Sigma_e^\otimes$, $\Lambda_{\T,t}
=
\MT\Omega_t^{-1}\MT
=
\diag\{\vecc(\mathcal W_{\T,t})\}$, and $\Lambda_\C
=
\MC\Omega_0^{-1}\MC
=
\diag\{\vecc(\mathcal W_\C)\}$. Then, we define $H_{\T,t}
=
A_\otimes^\top\Lambda_{\T,t}A_\otimes$, $H_\C
=
A_\otimes^\top\Lambda_\C A_\otimes$, and $V_t
=
\{H_{\T,t}+\alpha_t^{-2}H_\C\}^{-1}$. The vector coordinate of the core statistic is defined as
\[
g_t:=\vecc(G_t)
=
V_t\left(
A_\otimes^\top\Lambda_{\T,t}x_{\T,t}
+
\alpha_t^{-1}
A_\otimes^\top\Lambda_\C x_\C
\right).
\]
They agree directly with the tensor operations because
\[
\vecc\left(G\times_{d=1}^D A_d\right)
=A_\otimes\vecc(G),\qquad
\vecc\left\{(\mathcal W\od X)\times_{d=1}^D A_d^\top\right\}
=A_\otimes^\top\diag\{\vecc(\mathcal W)\}\vecc(X).\]
Finally, recall that
\(\xi(g,t)=\E(\alpha_t f\mid g_t=g)\). The tensor proposition is therefore equivalent to proving the
vectorized identity
\[
\nabla_{x_{\T,t}}\log p_t(x_{\T,t}\mid x_\C)
=
\Lambda_{\T,t}
\left\{
A_\otimes\xi(g_t,t)-x_{\T,t}
\right\}.
\]

In the following, we prove the vectorized identity above, where we write \(A=A_\otimes\) and \(\Sigma_e=\Sigma_e^\otimes\) for brevity. Let \(x_0\in\R^p\) denote the complete vector. Let \(M\in\{0,1\}^p\) be a fixed binary mask, where \(M_j=1\) means that coordinate \(j\) belongs to the missing-outcome block in the treated region. The mask partitions the coordinates into two deterministic sets, i.e., $\T=\{j:M_j=1\}$ and $\C=\{1,\ldots,p\}\setminus \T$.

The mask, equivalently the partition \((\T,\C)\), is fixed and non-random. We therefore suppress \(M\) from the conditioning notation for brevity. Define the diagonal mask matrices $\MT=\diag(M)$ and $\MC=I_p-\MT$. The treated-region and observed-control-outcome vectors are kept in the original \(p\)-dimensional full space, that is, $x_\T=\MT x_0$ and $x_\C=\MC x_0$. Thus, \(x_\T\) is supported only on \(\T\), while \(x_\C\) is supported only on \(\C\). Recall that we write \(A=A_\otimes\) and \(\Sigma_e=\Sigma_e^\otimes\). Consider the following vectorized factor model
\begin{equation}\label{eq:vector_factor_model}
x_0=Af+p^{-\beta/2}e.
\end{equation}
Under Assumption~\ref{ass:tensor}, \(A^\top A=I_r\), the factor $f$ has
density \(p_{\core}\), and the diagonal idiosyncratic covariance is
bounded above and below. The signal is kept at a constant
spectral scale, while the factor-strength parameter \(\beta\in[0,1]\)
enters through the idiosyncratic noise scale \(p^{-\beta/2}\). The missing-outcome distribution in the treated region is denoted by
\(P_0^\T(\cdot\mid x_\C)\), with the fixed partition \((\T,\C)\) given in
advance.

Let \(\phi(\cdot;\mu,\Sigma)\) denote the Gaussian density with mean \(\mu\) and covariance matrix \(\Sigma\). For a positive semidefinite matrix \(B\), write \(\norm{u}_B^2=u^\top Bu\), allowing this to be a seminorm. All gradients are taken with respect to the displayed argument.
Whenever a masked vector is used as the argument of a density, its
active coordinate subspace is identified with the corresponding
Euclidean space.

\medskip
\noindent\emph{Forward diffusion on positions in the treated region.}

The forward diffusion corrupts only the positions in the treated region, while
the observed-control-outcome coordinates remain fixed throughout the process. For
diffusion time \(t\in[0,T]\), we use the standard Ornstein–Uhlenbeck schedule, that is, $\alpha_t=\exp(-t/2)$ and $h_t = 1-\alpha_t^2 = 1-\exp(-t)$. Hence, the noise-treated-region coordinates admit the marginal representation $x_{\T,t} =
    \alpha_t x_{\T}+
    h_t^{1/2}\MT z_t$, where $z_t\sim\N(0,I_p)$ and $z_t\indep x_0$.

Define the full-dimensional noise scales as $\Omega_t=h_tI_p+\alpha_t^2p^{-\beta}\Sigma_e$, $\Omega_0=p^{-\beta}\Sigma_e$, and define the masked precision matrices as $\Lambda_{\T,t}=\MT\Omega_t^{-1}\MT$, $\Lambda_\C=\MC\Omega_0^{-1}\MC$. The matrices \(\Lambda_{\T,t}\) and \(\Lambda_\C\) are \(p\times p\) matrices. They are singular as full matrices, but they are positive definite on the treated and observed-control regions, respectively. 

Under \eqref{eq:vector_factor_model}, marginalizing out the latent factor
gives the conditional density of the diffused treated-region coordinates: $p_t(x_{\T,t}\mid x_\C)=[\int p(x_{\T,t},x_\C\mid f)p_{\core}(f)\,df]/p(x_\C)$. The denominator \(p(x_\C)\) does not depend on \(x_{\T,t}\) and, hence, does not affect the score with respect to \(x_{\T,t}\). Therefore, for the score calculation, it is enough to keep the numerator $\int p(x_{\T,t},x_\C\mid f)p_{\core}(f)\,df$. The relevant likelihood, viewed as a function of \(f\), is the joint likelihood \(p(x_{\T,t},x_\C\mid f)\). Due to the block-independent idiosyncratic noise, \(x_{\T,t}\) and \(x_\C\) are conditionally independent
given \(f\). Consequently, $p(x_{\T,t},x_\C\mid f)
=p(x_{\T,t}\mid f)p(x_\C\mid f)$. After omitting factors independent of \(f\), define the joint
likelihood kernel \(\mathcal{L}_t(f;x_{\T,t},x_\C)\) by
\begin{equation*}
\begin{aligned}
\mathcal{L}_t(f;x_{\T,t},x_\C)
&\propto
\exp\left\{
-\frac12\norm{x_{\T,t}-\alpha_tAf}_{\Lambda_{\T,t}}^2
\right\}
\exp\left\{
-\frac12\norm{x_\C-Af}_{\Lambda_\C}^2
\right\} \\
&=
\exp\left\{
-\frac12\norm{x_{\T,t}-\alpha_tAf}_{\Lambda_{\T,t}}^2
-\frac12\norm{x_\C-Af}_{\Lambda_\C}^2
\right\}.
\end{aligned}
\end{equation*}
Therefore, we have
\begin{equation}\label{eq:conditional_density}
p_t(x_{\T,t}\mid x_\C)
\propto
\int
\mathcal{L}_t(f;x_{\T,t},x_\C)
p_{\core}(f)\,df,
\end{equation}
where the proportionality is with respect to \(x_{\T,t}\); the omitted normalizing factor depends only on the fixed control vector \(x_\C\).

\medskip
\noindent\emph{Conditional score decomposition.}

Define the two precision Gram matrices as $H_{\T,t}=A^\top\Lambda_{\T,t}A$ and $H_\C=A^\top\Lambda_\C A$. Let $V_t=(H_{\T,t}+\alpha_t^{-2}H_\C)^{-1}$, and define the low-dimensional conditional encoder as $g_t
=
V_t
(
A^\top\Lambda_{\T,t}x_{\T,t}
+\alpha_t^{-1}A^\top\Lambda_\C x_\C
)$. The matrix \(V_t\) is the effective noise scale of the encoded factor information after combining the treated-region and observed-control-region branches. Since the observed-control-outcome term contributes additional precision \(\alpha_t^{-2}H_\C\), the effective factor noise is smaller than the treated-region-only factor noise whenever the observed-control-outcome coordinates carry nontrivial factor information. For \(g\in\R^r\), define the conditional core density as $p_{\con}^t(g)
=
\int
\phi(g;\alpha_tf,V_t)p_{\core}(f)df$. Equivalently, \(p_{\con}^t\) is the marginal density of $g\mid f\sim \N(\alpha_tf,V_t)$. Because the complementary masks provide positive precision on
every coordinate and \(A^\top A=I_r\), $(H_{\T,t}+\alpha_t^{-2}H_\C)$ is positive definite. Thus, \(V_t\) is well defined, and
the treated-region score satisfies
\begin{equation}\label{eq:score_decomp_vector}
\nabla_{x_{\T,t}}\log p_t(x_{\T,t}\mid x_\C)
=
\underbrace{
\Lambda_{\T,t}AV_t
\nabla_g\log p_{\con}^t(g_t)
}_{\text{factor score}}
-
\underbrace{
\Lambda_{\T,t}(x_{\T,t}-Ag_t)
}_{\text{residual score}} .
\end{equation}

Equivalently, let $\xi(g,t) =\E(\alpha_t f\mid g) =g+V_t\nabla_g\log p_{\con}^t(g)$, then
\begin{equation}\label{eq:masked_tucker_vector_form}
\nabla_{x_{\T,t}}\log p_t(x_{\T,t}\mid x_\C)
=
\Lambda_{\T,t}\{A\xi(g_t,t)-x_{\T,t}\}.
\end{equation}
\proof{Proof.}
By \eqref{eq:conditional_density}, we have $p_t(x_{\T,t}\mid x_\C)
\propto
\int
\mathcal{L}_t(f;x_{\T,t},x_\C)p_{\core}(f)\,df$. Taking the gradient with respect to \(x_{\T,t}\) and moving the derivative
inside the integral, we obtain
\[
\begin{aligned}
\nabla_{x_{\T,t}}\log p_t(x_{\T,t}\mid x_\C)
&=
\frac{
\int \nabla_{x_{\T,t}}\mathcal{L}_t(f;x_{\T,t},x_\C)p_{\core}(f)\,df
}{
\int \mathcal{L}_t(f;x_{\T,t},x_\C)p_{\core}(f)\,df
}  \\
&=
\Lambda_{\T,t}
\left\{
A
\frac{
\int \alpha_t f \mathcal{L}_t(f;x_{\T,t},x_\C)p_{\core}(f)\,df
}{
\int \mathcal{L}_t(f;x_{\T,t},x_\C)p_{\core}(f)\,df
}
-x_{\T,t}
\right\} \\
&=
\Lambda_{\T,t}
\left\{
A\E(\alpha_t f\mid x_{\T,t},x_\C)-x_{\T,t}
\right\}.
\end{aligned}
\]

Define the quadratic
coefficient of \(\alpha_t f\) as $V_t^{-1}
=
A^\top\Lambda_{\T,t}A+\alpha_t^{-2}A^\top\Lambda_\C A$, and define the corresponding center as $g_t
=
V_t
(
A^\top\Lambda_{\T,t}x_{\T,t}
+\alpha_t^{-1}A^\top\Lambda_\C x_\C
)$. As a function of \(f\), the exponent in
\(\mathcal{L}_t(f;x_{\T,t},x_\C)\) is a quadratic form.  Expanding the
two terms gives
\[
\begin{aligned}
\norm{x_{\T,t}-\alpha_tAf}_{\Lambda_{\T,t}}^2
+\norm{x_\C-Af}_{\Lambda_\C}^2 & =
(\alpha_tf)^\top A^\top\Lambda_{\T,t}A(\alpha_tf)
-2(\alpha_tf)^\top A^\top\Lambda_{\T,t}x_{\T,t}
+x_{\T,t}^\top\Lambda_{\T,t}x_{\T,t} \\
&\qquad
+f^\top A^\top\Lambda_\C Af
-2f^\top A^\top\Lambda_\C x_\C
+x_\C^\top\Lambda_\C x_\C .
\end{aligned}
\]
To express everything in terms of the scaled factor \(\alpha_tf\), note
that $f^\top A^\top\Lambda_\C Af
=
(\alpha_tf)^\top \alpha_t^{-2}H_\C(\alpha_tf)$ and $f^\top A^\top\Lambda_\C x_\C
=
(\alpha_tf)^\top \alpha_t^{-1}A^\top\Lambda_\C x_\C$. Therefore, we have
\[
\begin{aligned}
&\norm{x_{\T,t}-\alpha_tAf}_{\Lambda_{\T,t}}^2
+\norm{x_\C-Af}_{\Lambda_\C}^2 =
(\alpha_tf)^\top
\{H_{\T,t}+\alpha_t^{-2}H_\C\}
(\alpha_tf) \\
&\qquad
-2(\alpha_tf)^\top
\left\{
A^\top\Lambda_{\T,t}x_{\T,t}
+\alpha_t^{-1}A^\top\Lambda_\C x_\C
\right\}
+x_{\T,t}^\top\Lambda_{\T,t}x_{\T,t}
+x_\C^\top\Lambda_\C x_\C .
\end{aligned}
\]
By the definitions of \(V_t\) and \(g_t\), we know that $V_t^{-1}=H_{\T,t}+\alpha_t^{-2}H_\C$ and $V_t^{-1}g_t
=
A^\top\Lambda_{\T,t}x_{\T,t}
+\alpha_t^{-1}A^\top\Lambda_\C x_\C$. Substituting these identities into the previous display yields
\[
\begin{aligned}
\norm{x_{\T,t}-\alpha_tAf}_{\Lambda_{\T,t}}^2
+\norm{x_\C-Af}_{\Lambda_\C}^2 &=
(\alpha_tf)^\top V_t^{-1}(\alpha_tf)
-2(\alpha_tf)^\top V_t^{-1}g_t
+x_{\T,t}^\top\Lambda_{\T,t}x_{\T,t}
+x_\C^\top\Lambda_\C x_\C \\
& =
\norm{\alpha_tf-g_t}_{V_t^{-1}}^2
+
\left\{
x_{\T,t}^\top\Lambda_{\T,t}x_{\T,t}
+x_\C^\top\Lambda_\C x_\C
-g_t^\top V_t^{-1}g_t
\right\}.
\end{aligned}
\]
The remainder can be written as $R_t(x_{\T,t},x_\C)
=
x_{\T,t}^\top\Lambda_{\T,t}x_{\T,t}
+x_\C^\top\Lambda_\C x_\C
-g_t^\top V_t^{-1}g_t$.
which depends only on the observed pair \((x_{\T,t},x_\C)\). It shows that, as a function of \(f\), the likelihood kernel of the high-dimensional observation
\((x_{\T,t},x_\C)\) is proportional to $\exp\{
-\norm{\alpha_t f-g_t}_{V_t^{-1}}^2/2
\}$. All remaining factors are independent of \(f\). Hence, the posterior
distribution of \(f\) given \((x_{\T,t},x_\C)\) is the same as the
posterior distribution of \(f\) in the low-dimensional Gaussian
experiment $g_t=\alpha_t f+\eta_t$, for $\eta_t\sim\N(0,V_t)$. In particular, we have
\[
\E(\alpha_t f\mid x_{\T,t},x_\C)
=
\E(\alpha_t f\mid g_t).
\]

It remains to express this posterior mean through the score of the
low-dimensional marginal density. By definition, $p_{\con}^t(g) = \int \phi(g;\alpha_t f,V_t)p_{\core}(f)\,df$ is the marginal density of the experiment \(g=\alpha_t f+\eta_t\).
Differentiating under the integral gives $\nabla_g p_{\con}^t(g)
=
\int
\nabla_g\phi(g;\alpha_t f,V_t)p_{\core}(f)\,df =
\int
\{
-V_t^{-1}(g-\alpha_t f)
\}
\phi(g;\alpha_t f,V_t)p_{\core}(f)\,df$. Dividing by \(p_{\con}^t(g)\), we obtain $\nabla_g\log p_{\con}^t(g)
=
V_t^{-1}
\left\{
\E(\alpha_t f\mid g)-g
\right\}$. Therefore, $\E(\alpha_t f\mid g)
=
g+V_t\nabla_g\log p_{\con}^t(g)$. Evaluating this identity at the sufficient statistic \(g=g_t\) gives
\[
\E(\alpha_t f\mid x_{\T,t},x_\C)
=
\E(\alpha_t f\mid g_t)
=
g_t+V_t\nabla_g\log p_{\con}^t(g_t).
\]
Substituting this expression into the score identity $\nabla_{x_{\T,t}}\log p_t(x_{\T,t}\mid x_\C)
=
\Lambda_{\T,t}
\{
A\E(\alpha_t f\mid x_{\T,t},x_\C)-x_{\T,t}
\}$ proves \eqref{eq:score_decomp_vector} and
\eqref{eq:masked_tucker_vector_form}. 

Finally, it remains to return to the tensor expressions used in the main text. For the tensor model, let \(x_{\T,t}=\vecc(X_{\T,t})\), and
\(x_\C=\vecc(X_\C)\). Moreover, write $\Lambda_{\T,t}
=
\diag\{\vecc(\mathcal W_{\T,t})\}$ and $A_\otimes\xi(g_t,t)
=
\vecc\{
\E(\alpha_tF\mid G_t)
\times_{d=1}^D A_d
\}$. It follows that
\[
\begin{aligned}
\nabla_{x_{\T,t}}\log p_t(x_{\T,t}\mid x_\C)
&=
\diag\{\vecc(\mathcal W_{\T,t})\}
\left[
\vecc\left\{
\E(\alpha_tF\mid G_t)
\times_{d=1}^D A_d
\right\}
-
\vecc(X_{\T,t})
\right]\\
&=
\vecc\left[
\mathcal W_{\T,t}\od
\left\{
\E(\alpha_tF\mid G_t)
\times_{d=1}^D A_d
-
X_{\T,t}
\right\}
\right].
\end{aligned}
\]
Finally, vectorization preserves the Frobenius inner product, so $\vecc\{
\nabla_{X_{\T,t}}\log p_t(X_{\T,t}\mid X_\C)
\}
=\nabla_{x_{\T,t}}\log p_t(x_{\T,t}\mid x_\C)$. Since vectorization is one-to-one, the preceding identity proves
\eqref{eq:main-score-decomposition}.
\endproof
$\hfill\square$

\subsection{Proof of Theorem~\ref{thm:main-approximation}}
\proof{Proof.}
For clarity, we also prove Theorem~\ref{thm:main-approximation} in the vectorized notation. For \(R>0\), write $\mathcal G_R=\{g\in\R^r:\|g\|_\infty\le R\}$, and $K_R
=
1+\sup_{g\in\mathcal G_R,\ t\in[t_0,T]}
\|\xi(g,t)\|_2$. 

First, recall that conditioned on \(f\), $x_{\T,t}\sim \N(\alpha_t\MT A_{\otimes}f,\MT\Omega_t\MT)$, and $x_\C\sim \N(\MC A_{\otimes}f,\MC\Omega_0\MC)$. Because \(\Sigma_e^{\otimes}\) is diagonal, the treated-region and control-complement noises are independent. Therefore, $A_{\otimes}^{\top}\Lambda_{\T,t}x_{\T,t}
+\alpha_t^{-1}A_{\otimes}^{\top}\Lambda_\C x_\C$ has a conditional mean $\alpha_tH_{\T,t}f+\alpha_t^{-1}H_\C f= (H_{\T,t}+\alpha_t^{-2}H_\C)\alpha_tf$ and conditional covariance \(H_{\T,t}+\alpha_t^{-2}H_\C\). Multiplying by $V_t=(H_{\T,t}+\alpha_t^{-2}H_\C)^{-1}$ gives $g_t\mid f\sim \N(\alpha_tf,V_t)$.
Equivalently, we can write $g_t=\alpha_t f+\eta_t$, where $\eta_t\sim \N(0,V_t)$, and $\eta_t\indep f$. By Assumptions
\ref{ass:mask_core_tail} and \ref{ass:mask_xi_regular}, \(f\) is
sub-Gaussian, and \(\eta_t\) has covariance bounded by
\(v_{\max}I_r\), uniformly over \(t\in[t_0,T]\). Hence,
\(g_t\) is uniformly sub-Gaussian. Moreover, Assumption
\ref{ass:mask_xi_regular} gives $K_R
\le
K_0+\sqrt r\,L_gR$, so \(K_{R_\epsilon}=O(K_0+L_gR_\epsilon)\).

Second, we approximate the core regression on the truncated domain $\mathcal G_{R_\epsilon}=\{g\in\R^r:\|g\|_\infty\le R_\epsilon\}$. We can rescale $g'=(g+R_\epsilon\mathbf 1_r)/(2R_\epsilon)$ and $t'=t/T$. Then, define the rescaled regression function as
\[
\widetilde\xi(g',t')
=
\xi(2R_\epsilon g'-R_\epsilon\mathbf 1_r,\tau(t')),
\quad \text{where}\quad
\tau(t')=\max\{Tt',t_0\}.
\]
Under Assumption \ref{ass:mask_xi_regular}, \(\widetilde\xi\) is \(2R_\epsilon L_g\)-Lipschitz in \(g'\) and \(TL_t\)-Lipschitz in \(t'\). Partition \([0,1]^r\) into \(N_g^r\) cubes and \([0,1]\) into \(N_t\) intervals with
\[
N_g
=
\left\lceil
\frac{2^{r+3}\sqrt r\,R_\epsilon L_g}{\epsilon}
\right\rceil,
\qquad
N_t
=
\left\lceil
\frac{2^{r+2}T L_t}{\epsilon}
\right\rceil .
\]
Let \(\psi(a)=1\) for \(|a|<1\), \(\psi(a)=2-|a|\) for \(|a|\in[1,2]\), and \(\psi(a)=0\) for \(|a|>2\). For a multi-index \(q=(q_1,\ldots,q_r)^\top\in\{0,1,\ldots,N_g\}^r\) and \(j\in\{0,1,\ldots,N_t\}\), we follow the partition-of-unity construction in \citeEC{guo2026tucker} and consider
\[
\widetilde\zeta_i(g',t')
=
\sum_q\sum_j
\widetilde\xi_i(q/N_g,j/N_t)
\psi\left\{3N_t(t'-j/N_t)\right\}
\prod_{k=1}^r
\psi\left\{3N_g(g'_k-q_k/N_g)\right\},
\]
for each coordinate \(i\in[r]\). We use the partition-and-product construction in Appendix~B.1 of \citetEC{chen2023score}, which adapts the local
ReLU construction of \citetEC{Chen2022a}. We apply it coordinatewise to the centered functions
\(\widetilde\xi_i-\widetilde\xi_i(0,0)\), and restore
\(\widetilde\xi_i(0,0)\) through the output bias.  In the notation of
that construction, the input dimension is \(r+1\), the two coordinatewise
Lipschitz moduli are \(2R_\epsilon L_g\) and \(TL_t\), and the centered
range is bounded by a constant multiple of
\(K_{R_\epsilon}\). The stated uniform output bound then follows from
the approximation error and the bound on the target range:
\[
m =
O\{(1+L_gR_\epsilon)^r(1+TL_t)\epsilon^{-(r+1)}\},\qquad
K=O(K_0+L_gR_\epsilon),\qquad
L=O\{\log(K/\epsilon)+1\},
\]
\[
J=O(mL),\qquad \kappa
=
O\left(\max\{K_0+L_gR_\epsilon,TL_t,T^{-1}\}\right),
\qquad
\gamma_g'=C_rR_\epsilon L_g,
\qquad
\gamma_t'=C_rTL_t,
\]
there exists a rescaled ReLU network \(\widetilde\zeta_{\bar\theta}\) such that
\begin{align*}
&\sup_{g'\in[0,1]^r,\ t'\in[t_0/T,1]}
\|\widetilde\zeta_{\bar\theta}(g',t')-\widetilde\xi(g',t')\|_\infty
\le\epsilon,\quad\sup_{g'\in[0,1]^r,\ t'\in[t_0/T,1]}
\|\widetilde\zeta_{\bar\theta}(g',t')\|_2
\le C_r(K_{R_\epsilon}+\epsilon),\\
&\|\widetilde\zeta_{\bar\theta}(g_1',t')-
\widetilde\zeta_{\bar\theta}(g_2',t')\|_2
\le C_rR_\epsilon L_g\|g_1'-g_2'\|_2,\quad \|\widetilde\zeta_{\bar\theta}(g',t_1')-
\widetilde\zeta_{\bar\theta}(g',t_2')\|_2
\le C_rTL_t|t_1'-t_2'|.
\end{align*}
Transforming back to the original variables gives a ReLU network $\zeta_{\bar\theta}^{0}(g,t)
=
\widetilde\zeta_{\bar\theta}((g+R_\epsilon\mathbf 1_r)/(2R_\epsilon),t/T)$. Implementing \(t\mapsto t/T\) as an additional affine input layer adds
\(T^{-1}\) to the coefficient bound \(\kappa\); its constant depth and
sparsity costs are absorbed by the stated architecture orders.
Extend the network to \(\R^r\) using the coordinate-wise clipping map $\Pi_{R_\epsilon}(g)_j
=
-R_\epsilon+\operatorname{ReLU}(g_j+R_\epsilon)
-\operatorname{ReLU}(g_j-R_\epsilon)$ for $j\in[r]$, and set $\zeta_{\bar\theta}(g,t)
=
\zeta_{\bar\theta}^{0}(\Pi_{R_\epsilon}(g),t)$. The clipping map is represented exactly by an \(O(r)\)-size ReLU
module, is \(1\)-Lipschitz, maps \(\R^r\) into
\(\mathcal G_{R_\epsilon}\), and equals the identity on that box.
Therefore \(\zeta_{\bar\theta}\) agrees with
\(\zeta_{\bar\theta}^{0}\) on \(\mathcal G_{R_\epsilon}\), is globally
bounded by \(K\), and has the original-scale Lipschitz bounds
\[
\|\zeta_{\bar\theta}(g_1,t)-\zeta_{\bar\theta}(g_2,t)\|_2
\le
C_rL_g\|g_1-g_2\|_2,
\qquad
\|\zeta_{\bar\theta}(g,t_1)-\zeta_{\bar\theta}(g,t_2)\|_2
\le
C_rL_t|t_1-t_2|,
\]
globally on \(\R^r\times[t_0,T]\), up to changing the numerical
constant \(C_r\).  The additional width, sparsity, and coefficient size
are absorbed by the displayed architecture orders.  Therefore
\(\zeta_{\bar\theta}\in\mathcal F_{\core}(L,m,J,K,\kappa,\gamma_g,\gamma_t)\), and the constructed network satisfies
\begin{equation}\label{eq:proof_core_uniform_approx}
\sup_{g\in\mathcal G_{R_\epsilon},\ t\in[t_0,T]}
\|\zeta_{\bar\theta}(g,t)-\xi(g,t)\|_\infty
\le
\epsilon .
\end{equation}

Third, we convert the uniform approximation into an \(L^2\) core
approximation. For fixed \(t\in[t_0,T]\), define $\mathcal E_t=\{\|g_t\|_\infty\le R_\epsilon\}$. On \(\mathcal E_t\), \eqref{eq:proof_core_uniform_approx} gives $\|
\{\zeta_{\bar\theta}(g_t,t)
-\xi(g_t,t)\}
\mathbf I_{\mathcal E_t}
\|_2
\le
\sqrt r\,\epsilon$. On \(\mathcal E_t^c\), the clipped network remains bounded by \(K\).
Hence, we have
\[
\E\left[
\|\zeta_{\bar\theta}(g_t,t)
-\xi(g_t,t)\|_2^2
\mathbf I_{\mathcal E_t^c}
\right]\le
2K^2\mathbb{P}(\mathcal E_t^c)
+2\E\left[
\|\xi(g_t,t)\|_2^2
\mathbf I_{\mathcal E_t^c}
\right].
\]
By Jensen's inequality, we have $\|\xi(g,t)\|_2^2
=
\|\E(\alpha_tf\mid g)\|_2^2
\le
\E(\|\alpha_tf\|_2^2\mid g)
\le
\E(\|f\|_2^2\mid g)$. Since \(\mathcal E_t^c\) is measurable with respect to \(g_t\), the tower property gives
\[
\E\left[
\|\xi(g_t,t)\|_2^2
\mathbf I_{\mathcal E_t^c}
\right]
\le
\E\left[
\E(\|f\|_2^2\mid g_t)
\mathbf I_{\mathcal E_t^c}
\right]
=
\E\left[
\|f\|_2^2
\mathbf I_{\mathcal E_t^c}
\right].
\]
It remains to bound the right hand side uniformly in \(t\). From
\(g_t=\alpha_tf+\eta_t\), \(\alpha_t\le1\), and
\(\eta_t\sim\N(0,V_t)\) with \(V_t\preceq v_{\max}I_r\), we know that $\mathcal E_t^c
=
\{\|g_t\|_\infty>R_\epsilon\}
\subseteq
\{\|f\|_\infty>R_\epsilon/2\}
\cup
\{\|\eta_t\|_\infty>R_\epsilon/2\}$. Therefore, we can obtain
\[
\E\left[
\|f\|_2^2
\mathbf I_{\mathcal E_t^c}
\right]
\le
\E\left[
\|f\|_2^2
\mathbf I\{\|f\|_\infty>R_\epsilon/2\}
\right]
+
\E\left[
\|f\|_2^2
\mathbf I\{\|\eta_t\|_\infty>R_\epsilon/2\}
\right].
\]
The first term is controlled by the sub-Gaussian tail of \(f\). For the
second term, \(\eta_t\) is independent of \(f\), so we have $\E[
\|f\|_2^2
\mathbf I\{\|\eta_t\|_\infty>R_\epsilon/2\}
]
=
\E\|f\|_2^2
\mathbb{P}(\|\eta_t\|_\infty>R_\epsilon/2)$. The Gaussian tail bound and \(V_t\preceq v_{\max}I_r\) imply $\mathbb{P}(\|\eta_t\|_\infty>R_\epsilon/2)
\le
2r\exp(-cR_\epsilon^2)$ for a constant \(c>0\) that is independent of \(t\).
Together with the sub-Gaussian tail of \(f\), this gives, for a fixed
constant \(q_r>0\), $\mathbb{P}(\mathcal E_t^c)
+ \E[\|f\|_2^2\mathbf I_{\mathcal E_t^c}]\le C(1+R_\epsilon^{q_r})\exp(-cR_\epsilon^2)$ uniformly over \(t\in[t_0,T]\). Combining this bound with
\(K^2=O\{(K_0+L_gR_\epsilon)^2\}\), and enlarging \(q_r\) if
necessary, gives $\E[
\|\zeta_{\bar\theta}(g_t,t)
-\xi(g_t,t)\|_2^2
\mathbf I_{\mathcal E_t^c}
]\le
C(1+R_\epsilon^{q_r})\exp(-cR_\epsilon^2)$ uniformly over \(t\in[t_0,T]\). Let
\(u_\epsilon=c_{\mathrm{tail}}r/(t_0\epsilon)\) and choose \(C_R\)
such that \(cC_R^2\ge4\). Since \(u_\epsilon\ge1\), for the fixed
constant \(q_r\), there exists \(C_{q_r}>0\) such that $1+(\log u)^{q_r/2}\le C_{q_r}u^2$, where $u\ge1$,
we have
\begin{equation*}
(1+R_\epsilon^{q_r})\exp(-cR_\epsilon^2)
\le
C\{1+(\log u_\epsilon)^{q_r/2}\}u_\epsilon^{-4}\le
C u_\epsilon^{-2}
=C\left(\frac{t_0\epsilon}{c_{\mathrm{tail}}r}\right)^2
\le
\frac{C}{c_{\mathrm{tail}}^2}\epsilon^2.
\end{equation*}
Here, the last inequality follows from \(0<t_0\le1\) and \(r\ge1\).
Taking \(c_{\mathrm{tail}}\) sufficiently large makes the final upper
bound at most \(\epsilon^2\). Therefore, we have $\E[
\|\zeta_{\bar\theta}(g_t,t)
-\xi(g_t,t)\|_2^2
\mathbf I_{\mathcal E_t^c}
]
\le
\epsilon^2$ uniformly over \(t\in[t_0,T]\). Combining the bounds on \(\mathcal E_t\) and \(\mathcal E_t^c\) by the
triangle inequality in \(L^2\), we have that uniformly over \(t\in[t_0,T]\):
\begin{equation}\label{eq:proof_core_l2_error}
\left\{
\E
\left[
\left\|
\zeta_{\bar\theta}(g_t,t)
-
\xi(g_t,t)
\right\|_2^2
\right]
\right\}^{1/2}
\le
(\sqrt r+1)\epsilon.
\end{equation}

Finally, we lift the core approximation to the treated-region score. Set
\(\Gamma_d=A_d\) for every \(d\), equivalently \(\Gamma=A_\bullet\), and set
\(\omega=\omega^0\). Then, \(\Gamma_{\otimes}=A_{\otimes}\) and
\(g_t(A_\bullet,\omega^0)=g_t\). Vectorizing the network with core
\(\zeta_{\bar\theta}\) gives
\[
\vecc\{s_{A_\bullet,\omega^0,\bar\theta}(X_{\T,t},X_\C,t)\}
=
\Lambda_{\T,t}
\{A_{\otimes}\zeta_{\bar\theta}(g_t,t)-x_{\T,t}\}.
\]
Meanwhile, by \eqref{eq:main-score-decomposition},
\(\nabla_{x_{\T,t}}\log p_t(x_{\T,t}\mid x_\C)
=
\Lambda_{\T,t}
\{A_{\otimes}\xi(g_t,t)-x_{\T,t}\}\). The shortcut cancels exactly, so
\[
\vecc\{s_{A_\bullet,\omega^0,\bar\theta}(X_{\T,t},X_\C,t)\}
-
\nabla_{x_{\T,t}}\log p_t(x_{\T,t}\mid x_\C)
=
\Lambda_{\T,t}A_{\otimes}
\{\zeta_{\bar\theta}(g_t,t)-\xi(g_t,t)\}.
\]
Taking the nested conditional risk and applying \eqref{eq:proof_core_l2_error} gives
\[
\E_{X_\C}\E_{X_{\T,t}\mid X_\C}
\left[
\left\|
s_{A_\bullet,\omega^0,\bar\theta}(X_{\T,t},X_\C,t)
-
\nabla_{X_{\T,t}}\log p_t(X_{\T,t}\mid X_\C)
\right\|_F^2
\right]
\le
\|\Lambda_{\T,t}A_{\otimes}\|^2_{\mathrm{op}}(\sqrt r+1)^2\epsilon^2.
\]
Since \(\Omega_t\succeq h_tI_p\), the treatment precision
weights satisfy
\(\|\Lambda_{\T,t}\|_{\mathrm{op}}\le h_t^{-1}\). Since
\(A_{\otimes}^{\top}A_{\otimes}=I_r\), we have $\|\Lambda_{\T,t}A_{\otimes}\|_{\mathrm{op}}
\le
\|\Lambda_{\T,t}\|_{\mathrm{op}}\|A_{\otimes}\|_{\mathrm{op}}
\le
h_{t}^{-1}$. Combining the preceding two displays proves the result.
\endproof
$\hfill\square$

\subsection{Proof of Theorem~\ref{thm:main-generalization}}
\proof{Proof.}
Fix \(a\in(0,1)\) and \(\delta\in(0,1)\).  We first decompose the raw
DSM loss into statistical, truncation, and approximation terms.  We
begin with the truncation radii. For later use, define the Gaussian transition score and the
population raw DSM loss by $r_t(x_\T,x_{\T,t})
=h_t^{-1}(\alpha_tx_\T-x_{\T,t})$ and $\mathcal L_{\mask}(s)=\E\{\ell(X_0;s)\}$, respectively, where \(\ell\) is defined in \eqref{eq:mask_empirical_loss}. Let $d_{\C,\beta}=r+(p-p_\T)p^{-\beta}$, $d_+=d_{\T,\beta}+d_{\C,\beta}+1$, and $u_n=A_0\log(nd_+)$. Here, \(A_0>0\) is a sufficiently large constant. Meanwhile, take $C_{\T,x}^2=A_1d_{\T,\beta}u_n$ and $C_{\C,x}^2=A_1d_{\C,\beta}u_n$, where \(A_1>0\) is sufficiently large.  Hence, up to logarithmic
factors, $C_{\T,x}^2=\widetilde{\mathcal O}(d_{\T,\beta})$ and $C_{\C,x}^2=\widetilde{\mathcal O}\{d_{\C,\beta}\}$. Then, define the truncation event $\mathcal E_{\mathrm{tr}}
=
\{\|X_\T\|_F\le C_{\T,x},\ \|X_\C\|_F\le C_{\C,x}\}$. In vectorized notation, we can write
\[
X_\T=\MT A_{\otimes} f+p^{-\beta/2}\MT e,
\qquad
X_\C=\MC A_{\otimes} f+p^{-\beta/2}\MC e.
\]
Since \(A_{\otimes}^\top A_{\otimes}=I_r\) and \(\MT,\MC\) are
coordinate projections, we have $\|\MT A_{\otimes}f\|_2\le \|f\|_2$ and $\|\MC A_{\otimes}f\|_2\le \|f\|_2$. Then, by Assumption \ref{ass:mask_core_tail}, \(f\) is sub-Gaussian in the
sense that, for all \(u\ge1\), $\mathbb{P}\{\|f\|_2^2>C(r+u)\}\le Ce^{-cu}$. Moreover, Assumption \ref{ass:tensor} implies that the covariance of
\(p^{-\beta/2}\MT e\) is dominated by \(Cp^{-\beta}\MT\), while that
of \(p^{-\beta/2}\MC e\) is dominated by \(Cp^{-\beta}\MC\). Hence, for \(u\ge1\), the
standard Gaussian quadratic-form tail bound gives $\mathbb{P}\{p^{-\beta}\|\MT e\|_2^2>Cp^{-\beta}(p_\T+u)\}
\le Ce^{-cu}$ and $\mathbb{P}\{p^{-\beta}\|\MC e\|_2^2>Cp^{-\beta}(p-p_\T+u)\}
\le Ce^{-cu}$.

Combining the signal and idiosyncratic parts and increasing \(C\) if
necessary, we obtain $\mathbb{P}\{\|X_\T\|_F^2>Cd_{\T,\beta}u\}
+
\mathbb{P}\{\|X_\C\|_F^2>Cd_{\C,\beta}u\}
\le Ce^{-cu}$ for $u\ge1$. Taking \(u=u_n\) and choosing \(A_1\) sufficiently large yields
\[
\mathbb{P}(\|X_\T\|_F>C_{\T,x})
+
\mathbb{P}(\|X_\C\|_F>C_{\C,x})
\le
Ce^{-cu_n}
\le
(3n^2)^{-1},
\]
where the last inequality follows from the choice
\(u_n=A_0\log(nd_+)\) with \(A_0\) sufficiently large. Indeed, since \(d_+\ge1\), we have $Ce^{-cu_n}
=
C(nd_+)^{-cA_0}
\le
Cn^{-cA_0}$. Choose \(A_0\) such that \(cA_0\ge3\). Then, for all sufficiently
large \(n\), we have $Ce^{-cu_n}
\le
Cn^{-3}
\le
(3n^2)^{-1}$. Therefore, we have $\mathbb P(\mathcal E_{\mathrm{tr}}^c)\le (3n^2)^{-1}$.

It remains to control the second moment on the truncation complement.
The same tail bound, integrated over the tail, implies
\[
\E\left[
\|X_\T\|_F^2
\mathbf I\{\|X_\T\|_F>C_{\T,x}\}
\right]
\le
Cd_{\T,\beta}u_ne^{-cu_n},\quad
\E\left[
\|X_\C\|_F^2
\mathbf I\{\|X_\C\|_F>C_{\C,x}\}
\right]
\le
Cd_{\C,\beta}u_ne^{-cu_n}.
\]
Also, by the sub-Gaussian moment bound for \(f\) and the Gaussian moment
bound for \(e\), we have $\E\|X_\T\|_F^4\le Cd_{\T,\beta}^2$ and $\E\|X_\C\|_F^4\le Cd_{\C,\beta}^2$. Since $\mathbf I_{\mathcal E_{\mathrm{tr}}^c}
\le
\mathbf I\{\|X_\T\|_F>C_{\T,x}\}
+
\mathbf I\{\|X_\C\|_F>C_{\C,x}\}$, we need to control both own-tail and cross-tail terms. The own-tail
terms are bounded by the previous integrated tail estimate.  For the
cross-tail terms, the Cauchy–Schwarz inequality gives
\[
\E\left[
\|X_\T\|_F^2
\mathbf I\{\|X_\C\|_F>C_{\C,x}\}
\right]
\le
\{\E\|X_\T\|_F^4\}^{1/2}
\mathbb{P}(\|X_\C\|_F>C_{\C,x})^{1/2}
\le
Cd_{\T,\beta}e^{-cu_n/2},
\]
and similarly $\E[
\|X_\C\|_F^2
\mathbf I\{\|X_\T\|_F>C_{\T,x}\}
]
\le
Cd_{\C,\beta}e^{-cu_n/2}$. Combining the own-tail and cross-tail bounds, we have
\[
\E\{\|X_\T\|_F^2\mathbf I_{\mathcal E_{\mathrm{tr}}^c}\}
+
\E\{\|X_\C\|_F^2\mathbf I_{\mathcal E_{\mathrm{tr}}^c}\}
\le
C(d_{\T,\beta}+d_{\C,\beta})u_ne^{-cu_n}
+
C(d_{\T,\beta}+d_{\C,\beta})e^{-cu_n/2}.
\]

Since \(d_{\T,\beta}+d_{\C,\beta}\le d_+\) and
\(u_ne^{-cu_n}\le C e^{-cu_n/2}\), the right-hand side is bounded by $C d_+ e^{-cu_n/2}$. By the choice \(u_n=A_0\log(nd_+)\), we have $d_+ e^{-cu_n/2}=d_+(nd_+)^{-cA_0/2}=n^{-cA_0/2}d_+^{1-cA_0/2}$. Again, choosing \(A_0\) sufficiently large so that \(cA_0/2\ge3\), and using
\(d_+\ge1\), it follows for sufficiently large \(n\) that
$C d_+ e^{-cu_n/2}\le n^{-2}$. Therefore, we have $\E\{\|X_\T\|_F^2\mathbf I_{\mathcal E_{\mathrm{tr}}^c}\}
+
\E\{\|X_\C\|_F^2\mathbf I_{\mathcal E_{\mathrm{tr}}^c}\}
\le n^{-2}$. As a result, the selected radii
satisfy
\[
\mathbb P(\mathcal E_{\mathrm{tr}}^c)\le (3n^2)^{-1},
\qquad
\E\{\|X_\T\|_F^2\mathbf I_{\mathcal E_{\mathrm{tr}}^c}\}
+
\E\{\|X_\C\|_F^2\mathbf I_{\mathcal E_{\mathrm{tr}}^c}\}
\le n^{-2}.
\]

Then, define $\ell^{\mathrm{tr}}(X_0;s)=\ell(X_0;s)\mathbf I_{\mathcal E_{\mathrm{tr}}}$, $\mathcal L_{\mask}^{\mathrm{tr}}(s)=\E\ell^{\mathrm{tr}}(X_0;s)$, and $\widehat{\mathcal L}_{\mask}^{\mathrm{tr}}(s)
=\frac1n\sum_{i=1}^n\ell^{\mathrm{tr}}(X_0^{(i)};s)$.

Let \(s^\circ=s_{A_\bullet,\omega^0,\bar\theta}\) be the
approximator from Theorem \ref{thm:main-approximation}, and let $\mathcal E_{\mathrm{tr},n}
=
\bigcap_{i=1}^n\mathcal E_{\mathrm{tr},i}$, where
\(\mathcal E_{\mathrm{tr},i}\) denotes \(\mathcal E_{\mathrm{tr}}\) for
the sample \(X_0^{(i)}\). The event \(\mathcal E_{\mathrm{tr},n}\) is introduced to transfer the
empirical optimality of \(\widehat s\) from the original empirical loss
to the truncated empirical loss.  Indeed, \(\widehat s\) minimizes
\(\widehat{\mathcal L}_{\mask}\), while the empirical process argument
below is applied to the truncated loss class.  On
\(\mathcal E_{\mathrm{tr},n}\), every training sample lies in the
truncation set. Consequently, we have $\widehat{\mathcal L}_{\mask}^{\mathrm{tr}}(s)
=
\widehat{\mathcal L}_{\mask}(s)$, for all $s\in\mathcal S_{\MTClass}$.

Therefore, on \(\mathcal E_{\mathrm{tr},n}\), the empirical minimizer
\(\widehat s\) also satisfies the truncated empirical optimality
inequality $\widehat{\mathcal L}_{\mask}^{\mathrm{tr}}(\widehat s)
=
\widehat{\mathcal L}_{\mask}(\widehat s)
\le
\widehat{\mathcal L}_{\mask}(s^\circ)
=
\widehat{\mathcal L}_{\mask}^{\mathrm{tr}}(s^\circ)$. Using this inequality, we obtain the deterministic raw-loss
decomposition
\begin{align}
\mathcal L_{\mask}(\widehat s)
&=
\left\{
\mathcal L_{\mask}(\widehat s)
-
\mathcal L_{\mask}^{\mathrm{tr}}(\widehat s)
\right\}
+
\mathcal L_{\mask}^{\mathrm{tr}}(\widehat s) \nonumber\\
&=
\left\{
\mathcal L_{\mask}(\widehat s)
-
\mathcal L_{\mask}^{\mathrm{tr}}(\widehat s)
\right\}
+
\left\{
\mathcal L_{\mask}^{\mathrm{tr}}(\widehat s)
-
(1+a)\widehat{\mathcal L}_{\mask}^{\mathrm{tr}}(\widehat s)
\right\}
+
(1+a)\widehat{\mathcal L}_{\mask}^{\mathrm{tr}}(\widehat s)
\nonumber\\
&\le
\left\{
\mathcal L_{\mask}(\widehat s)
-
\mathcal L_{\mask}^{\mathrm{tr}}(\widehat s)
\right\}
+
\left\{
\mathcal L_{\mask}^{\mathrm{tr}}(\widehat s)
-
(1+a)\widehat{\mathcal L}_{\mask}^{\mathrm{tr}}(\widehat s)
\right\}
+
(1+a)\widehat{\mathcal L}_{\mask}^{\mathrm{tr}}(s^\circ)
\nonumber\\
&=
\underbrace{
\left\{
\mathcal L_{\mask}(\widehat s)
-
\mathcal L_{\mask}^{\mathrm{tr}}(\widehat s)
\right\}
}_{\mathcal E_{\mathrm{trunc}}}
+
\underbrace{
\left\{
\mathcal L_{\mask}^{\mathrm{tr}}(\widehat s)
-
(1+a)\widehat{\mathcal L}_{\mask}^{\mathrm{tr}}(\widehat s)
\right\}
}_{\mathcal E_{\mathrm{stat},1}}
\nonumber\\
&\quad+
(1+a)
\underbrace{
\left\{
\widehat{\mathcal L}_{\mask}^{\mathrm{tr}}(s^\circ)
-
(1+a)\mathcal L_{\mask}^{\mathrm{tr}}(s^\circ)
\right\}
}_{\mathcal E_{\mathrm{stat},2}}
+
(1+a)^2\mathcal L_{\mask}^{\mathrm{tr}}(s^\circ)
\nonumber\\
&\le
\mathcal E_{\mathrm{trunc}}
+
\mathcal E_{\mathrm{stat},1}
+
(1+a)\mathcal E_{\mathrm{stat},2}
+
(1+a)^2
\underbrace{\mathcal L_{\mask}(s^\circ)}_{\mathcal E_{\mathrm{app}}},
\label{eq:mask_raw_error_decomposition}
\end{align}
where the last inequality uses
\(\mathcal L_{\mask}^{\mathrm{tr}}(s^\circ)\le
\mathcal L_{\mask}(s^\circ)\).

\textit{Step 1: statistical error for the truncated raw loss.}
We first derive an upper bound for \(\ell^{\mathrm{tr}}\).  For
\(s=s_{\Gamma,\omega,\theta}\), let
\[
\mathcal W_\sigma
:=
\left\{\omega=(\omega_{dj}):
\sigma_{\min}^2\le\omega_{dj}\le\sigma_{\max}^2,
\ d\in[D],\ j\in[p_d]\right\}.
\]
Since $s(x_{\T,t},x_\C,t)
=
\Lambda_{\T,t}(\omega)
\{\Gamma_{\otimes}\zeta_\theta(g_t(\Gamma,\omega),t)-x_{\T,t}\}$, we have
\begin{equation}\label{eq:mask_raw_residual_identity}
s(x_{\T,t},x_\C,t)-r_t(x_\T,x_{\T,t}) =
\Lambda_{\T,t}(\omega)\Gamma_{\otimes}\zeta_\theta(g_t(\Gamma,\omega),t)
+\{h_t^{-1}\MT-\Lambda_{\T,t}(\omega)\}x_{\T,t}
-\alpha_th_t^{-1}x_\T.
\end{equation}
For all \(\omega\in\mathcal W_\sigma\),
$\|\Lambda_{\T,t}(\omega)\|_{\op}\le h_t^{-1}$, and
$\|h_t^{-1}\MT-\Lambda_{\T,t}(\omega)\|_{\op}
\le C p^{-\beta}h_t^{-2}$. Indeed, each diagonal entry \(s\) of
\(\Sigma_e^\otimes(\omega)\) lies in
\([\sigma_{\min}^{2D},\sigma_{\max}^{2D}]\), so that $h_t^{-1}-(h_t+\alpha_t^2p^{-\beta}s)^{-1}
=\alpha_t^2p^{-\beta}s/[h_t(h_t+\alpha_t^2p^{-\beta}s)]\le C p^{-\beta}h_t^{-2}$. Then, using \(\|\zeta_\theta\|_2\le K\) and
\(\|x_\T\|_2\le C_{\T,x}\) on \(\mathcal E_{\mathrm{tr}}\), we obtain
\[
\|s(X_{\T,t},X_\C,t)-r_t(X_\T,X_{\T,t})\|_2
\le
\frac{K+C_{\T,x}}{h_t}
+
C\frac{p^{-\beta}\|X_{\T,t}\|_2}{h_t^2}.
\]
Moreover, we have $\E(\|X_{\T,t}\|_2^2\mid X_0)
\le C_{\T,x}^2+p_\T h_t$ on $\mathcal E_{\mathrm{tr}}$. Therefore
\begin{equation}\label{eq:mask_bl_bound}
\sup_{s\in\mathcal S_{\MTClass}}
\sup_{X_0}\ell^{\mathrm{tr}}(X_0;s)
\le
B_{\ell,\mask}
:=
1+\frac{C}{T-t_0}\int_{t_0}^T
\left\{
\frac{K^2+C_{\T,x}^2}{h_t^2}
+
\frac{p^{-2\beta}C_{\T,x}^2}{h_t^4}
+
\frac{p_\T p^{-2\beta}}{h_t^3}
\right\}dt .
\end{equation}

For the Ornstein--Uhlenbeck schedule, \(h_t=1-\alpha_t^2\) is comparable to \(t\) near
zero and is bounded away from zero for \(t\) away from zero.  Since
\(t_0\le1\), there exists a constant \(c_h>0\) such that $h_t\ge c_h t_0$ for $t\in[t_0,T]$. Together with the early-stopping condition \(p^{-\beta}\le c_0t_0\),
this implies $p^{-\beta}\le C h_t$ for $t\in[t_0,T]$. We now simplify the three terms in the integrand of
\eqref{eq:mask_bl_bound}.  By the choice of the truncation radius, we have $C_{\T,x}^2=\widetilde{\mathcal O}(d_{\T,\beta})$. Moreover, Theorem \ref{thm:main-approximation} gives $K^2
=
O\{(K_0+L_gR_\epsilon)^2\}
=
O\{K_0^2+L_g^2\log[c_{\mathrm{tail}}r/(t_0\epsilon)]\}$. Since \(K_0\) and \(L_g\) are treated as fixed model constants and
\(\widetilde{\mathcal O}(\cdot)\) hides logarithmic factors in
\(r,t_0^{-1},\epsilon^{-1}\) and \(n\), this term is
\(\widetilde{\mathcal O}(1)\). Therefore, because
\(d_{\T,\beta}=r+p_\T p^{-\beta}\ge1\), we have $K^2+C_{\T,x}^2
=
\widetilde{\mathcal O}(d_{\T,\beta})$.

For the second term in \eqref{eq:mask_bl_bound}, the bound
\(p^{-\beta}\le Ch_t\) gives $h_t^{-4}p^{-2\beta}C_{\T,x}^2\leq C h_t^{-2}C_{\T,x}^2 =
\widetilde{\mathcal O}(h_t^{-2}d_{\T,\beta})$. For the last term, using \(p_\T p^{-\beta}\le d_{\T,\beta}\) and
\(p^{-\beta}\le Ch_t\), we obtain $h_t^{-3}p_\T p^{-2\beta}=[h_t^{-2}p_\T p^{-\beta}][h_t^{-1}p^{-\beta}]\le Ch_t^{-2}d_{\T,\beta}$. Thus, the whole integrand in \eqref{eq:mask_bl_bound} is bounded by
\(\widetilde{\mathcal O}(d_{\T,\beta}h_t^{-2})\).  Consequently,
\begin{equation}\label{eq:mask_bl_simplified}
B_{\ell,\mask}
\le
\widetilde{\mathcal O}
\left\{
\frac{d_{\T,\beta}}{T-t_0}
\int_{t_0}^T h_t^{-2}dt
\right\}.
\end{equation}
The bound $(T-t_0)^{-1} \int_{t_0}^T h_t^{-2}dt\le C(t_0^{-1}+T)$ follows by splitting \([t_0,T]\) at \(1\). On
\([t_0,1\wedge T]\), \(h_t\asymp t\), whereas \(h_t\asymp1\) on
\([1,T]\) when \(T>1\). Hence, the two contributions are bounded by
\(Ct_0^{-1}\) and \(CT\), respectively. Therefore
\begin{equation}\label{eq:mask_bl_simplified_rate}
B_{\ell,\mask}
\le
\widetilde{\mathcal O}
\left\{
d_{\T,\beta}
\left(\frac1{t_0}+T\right)
\right\}.
\end{equation}

Following the parameter-covering reduction in
\citetEC{guo2026tucker}, we cover the truncated loss class
\(\mathcal G_{\mathrm{tr}}=\{\ell^{\mathrm{tr}}(\cdot;s):
s\in\mathcal S_{\MTClass}\}\). Since the loss integrates out the
unbounded Gaussian variable \(X_{\T,t}\mid X_\T\), the intermediate
score metric is defined using the same conditional second moment. We
begin with a parameter perturbation bound. Let \(s_{\Gamma,\omega,\theta}\) and
\(s_{\widetilde\Gamma,\widetilde\omega,\widetilde\theta}\) be two elements
of \(\mathcal S_{\MTClass}\), and write $\Delta_\Gamma=\max_{d\in[D]}\|\Gamma_d-\widetilde\Gamma_d\|_{\op}$, $\Delta_\omega=\max_{d\in[D],\,j\in[p_d]}
|\omega_{dj}-\widetilde\omega_{dj}|$. Denote the admissible loading--variance parameter set by
\[
\mathcal P_{\Gamma,\omega}
:=
\Bigg\{
(\Gamma,\omega):
\Gamma_d^\top \Gamma_d=I_{r_d},\ d\in[D],\quad
\omega\in\mathcal W_\sigma,\quad
\sup_{t\in[t_0,T]}\|V_t(\Gamma,\omega)\|_{\op}\le C_V
\Bigg\}.
\]
Thus, the loading--variance components of both parameter triples belong
to \(\mathcal P_{\Gamma,\omega}\).
The telescoping identity
\[
\Gamma_D\otimes\cdots\otimes \Gamma_1-
\widetilde\Gamma_D\otimes\cdots\otimes\widetilde\Gamma_1
=
\sum_{d=1}^D
\widetilde\Gamma_D\otimes\cdots\otimes\widetilde\Gamma_{d+1}
\otimes(\Gamma_d-\widetilde\Gamma_d)
\otimes \Gamma_{d-1}\otimes\cdots\otimes \Gamma_1
\]
and \(\|\Gamma_d\|_{\op}=\|\widetilde\Gamma_d\|_{\op}=1\) imply
\begin{equation}\label{eq:mask_kronecker_perturb}
\|\Gamma_{\otimes}-\widetilde\Gamma_{\otimes}\|_{\op}
\le D\Delta_\Gamma.
\end{equation}

For a tensor index \(j=(j_1,\ldots,j_D)\), set
\(s_j(\omega)=\prod_{d=1}^D\omega_{d j_d}\). A second telescoping
identity gives $|s_j(\omega)-s_j(\widetilde\omega)|
\le
C_\sigma\Delta_\omega$, where \(C_\sigma\) depends only on
\(D,\sigma_{\min},\sigma_{\max}\). On a treatment coordinate, the
corresponding precision is
\(\{h_t+\alpha_t^2p^{-\beta}s_j(\omega)\}^{-1}\). Hence, we have $|[h_t+\alpha_t^2p^{-\beta}s_j(\omega)]^{-1}
-[h_t+\alpha_t^2p^{-\beta}s_j(\widetilde\omega)]^{-1}
|\le C p^{-\beta}h_t^{-2}\Delta_\omega$. On a control coordinate, the precision is \(p^\beta/s_j(\omega)\), and
the lower variance bound gives
\[
\left|
\frac{p^\beta}{s_j(\omega)}-
\frac{p^\beta}{s_j(\widetilde\omega)}
\right|
\le C p^\beta\Delta_\omega.
\]
Since all precision matrices are diagonal, these entry-wise bounds are
also operator-norm bounds. Thus
\begin{equation}\label{eq:mask_lambda_t_perturb}
\|\Lambda_{\T,t}(\omega)-
\Lambda_{\T,t}(\widetilde\omega)\|_{\op}
\le C p^{-\beta}h_t^{-2}\Delta_\omega,\quad
\|\Lambda_\C(\omega)-\Lambda_\C(\widetilde\omega)\|_{\op}
\le C p^\beta\Delta_\omega.
\end{equation}
In addition, we shall have
\begin{equation}\label{eq:mask_lambda_norms}
\|\Lambda_{\T,t}(\omega)\|_{\op}\le h_t^{-1},
\qquad
\|\Lambda_\C(\omega)\|_{\op}\le C p^\beta.
\end{equation}

For the matrices \(H\) and \(V_t\), adding and subtracting
\(\widetilde\Gamma_{\otimes}^{\top}\Lambda_{\T,t}(\omega)\Gamma_{\otimes}\)
and
\(\widetilde\Gamma_{\otimes}^{\top}
\Lambda_{\T,t}(\widetilde\omega)\Gamma_{\otimes}\) gives
\[
H_{\T,t}(\Gamma,\omega)-H_{\T,t}(\widetilde\Gamma,\widetilde\omega) =
(\Gamma_{\otimes}-\widetilde\Gamma_{\otimes})^\top
\Lambda_{\T,t}(\omega)\Gamma_{\otimes}+
\widetilde\Gamma_{\otimes}^\top
\{\Lambda_{\T,t}(\omega)-
\Lambda_{\T,t}(\widetilde\omega)\}\Gamma_{\otimes}+
\widetilde\Gamma_{\otimes}^\top\Lambda_{\T,t}(\widetilde\omega)
(\Gamma_{\otimes}-\widetilde\Gamma_{\otimes}).
\]
Equations \eqref{eq:mask_kronecker_perturb}--
\eqref{eq:mask_lambda_norms} therefore imply
\begin{align}
\|H_{\T,t}(\Gamma,\omega)-
H_{\T,t}(\widetilde\Gamma,\widetilde\omega)\|_{\op}
&\le
C\{h_t^{-1}\Delta_\Gamma+p^{-\beta}h_t^{-2}\Delta_\omega\},
\label{eq:mask_ht_perturb}\\
\|H_\C(\Gamma,\omega)-H_\C(\widetilde\Gamma,\widetilde\omega)\|_{\op}
&\le
C p^\beta(\Delta_\Gamma+\Delta_\omega).
\label{eq:mask_hc_perturb}
\end{align}
Define $B_t(\Gamma,\omega)
=H_{\T,t}(\Gamma,\omega)+\alpha_t^{-2}H_\C(\Gamma,\omega)$ and $V_t(\Gamma,\omega)=B_t(\Gamma,\omega)^{-1}$. Both loading–variance pairs belong to \(\mathcal P_{\Gamma,\omega}\), so
the two inverses exist and have an operator norm of at most \(C_V\). The inverse
identity $V_t(\Gamma,\omega)-V_t(\widetilde\Gamma,\widetilde\omega)
=
V_t(\Gamma,\omega)
\{B_t(\widetilde\Gamma,\widetilde\omega)-B_t(\Gamma,\omega)\}
V_t(\widetilde\Gamma,\widetilde\omega)$ and \eqref{eq:mask_ht_perturb}--\eqref{eq:mask_hc_perturb} yield
\begin{equation}\label{eq:mask_v_perturb_explicit}
\|V_t(\Gamma,\omega)-V_t(\widetilde\Gamma,\widetilde\omega)\|_{\op}
\le
L_V(t)(\Delta_\Gamma+\Delta_\omega),\quad\text{where} \quad L_V(t)
=
C C_V^2
\{h_t^{-1}+p^{-\beta}h_t^{-2}+\alpha_t^{-2}p^\beta\}.
\end{equation}

For the low-dimensional encoder, define $a_t^{\mathrm{enc}}(\Gamma,\omega)
=
\Gamma_{\otimes}^\top\Lambda_{\T,t}(\omega)X_{\T,t}
+\alpha_t^{-1}\Gamma_{\otimes}^\top\Lambda_\C(\omega)X_\C$ and  $g_t(\Gamma,\omega)=V_t(\Gamma,\omega)a_t^{\mathrm{enc}}(\Gamma,\omega)$. On \(\mathcal E_{\mathrm{tr}}\), set $Q_\T(t)=\{C_{\T,x}^2+p_\T h_t\}^{1/2}$. The forward transition gives
\begin{equation}\label{eq:mask_forward_moment_cover}
\left\{\E(\|X_{\T,t}\|_2^2\mid X_\T)\right\}^{1/2}
\le Q_\T(t).
\end{equation}
By Minkowski's inequality and \eqref{eq:mask_lambda_norms}, we have
\begin{equation}\label{eq:mask_encoder_linear_moment}
\left\{
\E_{X_{\T,t}\mid X_\T}
\|a_t^{\mathrm{enc}}(\Gamma,\omega)\|_2^2
\right\}^{1/2}
\le
G_a(t),
\qquad
G_a(t)=h_t^{-1}Q_\T(t)+C\alpha_t^{-1}p^\beta C_{\C,x}.
\end{equation}

To compare the two encoder inputs, we first decompose the treated-region branch as
\[
\Gamma_{\otimes}^\top\Lambda_{\T,t}(\omega)X_{\T,t}
-\widetilde\Gamma_{\otimes}^\top
\Lambda_{\T,t}(\widetilde\omega)X_{\T,t}=
(\Gamma_{\otimes}-\widetilde\Gamma_{\otimes})^\top
\Lambda_{\T,t}(\omega)X_{\T,t}
+\widetilde\Gamma_{\otimes}^\top
\{\Lambda_{\T,t}(\omega)-
\Lambda_{\T,t}(\widetilde\omega)\}X_{\T,t}.
\]
The control branch satisfies the corresponding identity
\[
\alpha_t^{-1}\Gamma_{\otimes}^\top\Lambda_\C(\omega)X_\C
-\alpha_t^{-1}\widetilde\Gamma_{\otimes}^\top
\Lambda_\C(\widetilde\omega)X_\C=
\alpha_t^{-1}(\Gamma_{\otimes}-\widetilde\Gamma_{\otimes})^\top
\Lambda_\C(\omega)X_\C +
\alpha_t^{-1}\widetilde\Gamma_{\otimes}^\top
\{\Lambda_\C(\omega)-\Lambda_\C(\widetilde\omega)\}X_\C.
\]
Equations
\eqref{eq:mask_kronecker_perturb}--
\eqref{eq:mask_forward_moment_cover} then give
\begin{equation}\label{eq:mask_encoder_linear_perturb}
\left\{
\E_{X_{\T,t}\mid X_\T}
\|a_t^{\mathrm{enc}}(\Gamma,\omega)
-a_t^{\mathrm{enc}}(\widetilde\Gamma,\widetilde\omega)\|_2^2
\right\}^{1/2}
\le
L_a(t)(\Delta_\Gamma+\Delta_\omega),\quad\text{where}
\end{equation}
\begin{equation}\label{eq:mask_la_definition}
L_a(t)
=
C\left[
\{h_t^{-1}+p^{-\beta}h_t^{-2}\}Q_\T(t)
+\alpha_t^{-1}p^\beta C_{\C,x}
\right].
\end{equation}
Finally, add and subtract
\(V_t(\widetilde\Gamma,\widetilde\omega)
a_t^{\mathrm{enc}}(\Gamma,\omega)\) to obtain
\[
g_t(\Gamma,\omega)-g_t(\widetilde\Gamma,\widetilde\omega)=
\{V_t(\Gamma,\omega)-V_t(\widetilde\Gamma,\widetilde\omega)\}
a_t^{\mathrm{enc}}(\Gamma,\omega)+
V_t(\widetilde\Gamma,\widetilde\omega)
\{a_t^{\mathrm{enc}}(\Gamma,\omega)
-a_t^{\mathrm{enc}}(\widetilde\Gamma,\widetilde\omega)\}.
\]
Combining \eqref{eq:mask_v_perturb_explicit},
\eqref{eq:mask_encoder_linear_moment}, and
\eqref{eq:mask_encoder_linear_perturb} yields
\begin{align}
&\sup_{X_0\in\mathcal E_{\mathrm{tr}}}
\left\{
\E_{X_{\T,t}\mid X_\T}
\|g_t(\Gamma,\omega)-g_t(\widetilde\Gamma,\widetilde\omega)\|_2^2
\right\}^{1/2}
\le
L_{\mathrm{enc}}(t)(\Delta_\Gamma+\Delta_\omega),
\label{eq:mask_g_perturb}\\
&\sup_{(\Gamma,\omega)\in\mathcal P_{\Gamma,\omega}}
\sup_{X_0\in\mathcal E_{\mathrm{tr}}}
\left\{
\E_{X_{\T,t}\mid X_\T}
\|g_t(\Gamma,\omega)\|_2^2
\right\}^{1/2}
\le
G_{\mathrm{enc}}(t),\quad\text{where}
\label{eq:mask_g_moment}
\end{align}
\begin{equation}\label{eq:mask_encoder_constants}
L_{\mathrm{enc}}(t)=L_V(t)G_a(t)+C_VL_a(t),
\qquad
G_{\mathrm{enc}}(t)=C_VG_a(t).
\end{equation}
These bounds are finite in the conditional \(L^2\) metric.

The network depth is included as a discrete index in the cover below,
so it suffices to compare two networks of a common depth
\(\bar L\le L\). Pad every hidden layer to a width of \(m\). Enlarging
\(m\) to \(m\vee(r+1)\), if necessary, does not change its order and
ensures that every layer dimension is at most \(m\). For a common
input \(z_0=(g^\top,t)^\top\), write $z_\ell=\rho(W_\ell z_{\ell-1}+b_\ell),
\quad 1\le\ell<\bar L$ and $\zeta_\theta(g,t)=W_{\bar L}z_{\bar L-1}+b_{\bar L}$, and define \(\widetilde z_\ell\) analogously. Put
\(\delta_\theta=\|\theta-\widetilde\theta\|_\infty\). Since every
matrix has at most \(m\) rows and columns, and every parameter is
bounded by \(\kappa\), we have $\|W_\ell\|_{\op}\vee\|\widetilde W_\ell\|_{\op}
\le m\kappa$ and $\|W_\ell-\widetilde W_\ell\|_{\op}
+\|b_\ell-\widetilde b_\ell\|_2
\le 2m\delta_\theta$. Because ReLU is \(1\)-Lipschitz, the activations satisfy the two
recursions
\begin{equation*}
\|z_\ell\|_2\le
(m+1)(1\vee\kappa)\{1+\|z_{\ell-1}\|_2\},\quad
\|z_\ell-\widetilde z_\ell\|_2\le
m\kappa\|z_{\ell-1}-\widetilde z_{\ell-1}\|_2 +
2m\delta_\theta\{1+\|z_{\ell-1}\|_2\}.
\end{equation*}
Let \(M_\kappa=(m+1)(1\vee\kappa)\), which is at least \(2\). Starting from
\(z_0=\widetilde z_0\), induction over the preceding recursions gives $\|z_\ell\|_2
\le
2M_\kappa^\ell\{1+\|z_0\|_2\}$ and $\|z_\ell-\widetilde z_\ell\|_2
\le
C\ell(m+1)M_\kappa^{\ell-1}
\{1+\|z_0\|_2\}\delta_\theta$. Applying the same difference decomposition to the final affine layer
therefore gives
\begin{equation}\label{eq:mask_network_parameter_perturb}
\|\zeta_\theta(g,t)-\zeta_{\widetilde\theta}(g,t)\|_2
\le
L_\theta^{\mathrm{net}}(1+|t|+\|g\|_2)\delta_\theta,
\end{equation}
where the following uniform choice is valid for every \(\bar L\le L\):
\begin{equation}\label{eq:mask_network_parameter_constant}
L_\theta^{\mathrm{net}}
=
C L^2(m+1)^{L+2}(1\vee\kappa)^L.
\end{equation}

For different encoder and network parameters, insert the intermediate
term
\(\zeta_\theta(g_t(\widetilde\Gamma,\widetilde\omega),t)\). The global input
Lipschitz constraint in \(\mathcal F_{\core}\) and
\eqref{eq:mask_network_parameter_perturb} imply
\[
\|\zeta_\theta(g_t(\Gamma,\omega),t)
-\zeta_{\widetilde\theta}(g_t(\widetilde\Gamma,
\widetilde\omega),t)\|_2\le
\gamma_g\|g_t(\Gamma,\omega)-g_t(\widetilde\Gamma,
\widetilde\omega)\|_2+
L_\theta^{\mathrm{net}}
\{1+T+\|g_t(\widetilde\Gamma,\widetilde\omega)\|_2\}
\delta_\theta.
\]
Taking the conditional \(L^2\) norm and applying
\eqref{eq:mask_g_perturb}--\eqref{eq:mask_g_moment} yields
\begin{equation}\label{eq:mask_core_output_perturb}
\sup_{X_0\in\mathcal E_{\mathrm{tr}}}
\left\{
\E_{X_{\T,t}\mid X_\T}
\|\zeta_\theta(g_t(\Gamma,\omega),t)
-\zeta_{\widetilde\theta}(g_t(\widetilde\Gamma,
\widetilde\omega),t)\|_2^2
\right\}^{1/2}
\le
L_{\mathrm{core}}(t)
(\Delta_\Gamma+\Delta_\omega+\delta_\theta),
\end{equation}
where $L_{\mathrm{core}}(t)
=\gamma_gL_{\mathrm{enc}}(t)
+L_\theta^{\mathrm{net}}\{1+T+G_{\mathrm{enc}}(t)\}$.

For the score output itself, write \(g=g_t(\Gamma,\omega)\) and
\(\widetilde g=g_t(\widetilde\Gamma,\widetilde\omega)\). Then we have
\[
\begin{aligned}
s_{\Gamma,\omega,\theta}(x_{\T,t},x_\C,t)
-&s_{\widetilde\Gamma,\widetilde\omega,\widetilde\theta}
(x_{\T,t},x_\C,t) =
\{\Lambda_{\T,t}(\omega)-\Lambda_{\T,t}(\widetilde\omega)\}
\{\Gamma_{\otimes}\zeta_\theta(g,t)-x_{\T,t}\}\\
&+
\Lambda_{\T,t}(\widetilde\omega)
(\Gamma_{\otimes}-\widetilde\Gamma_{\otimes})\zeta_\theta(g,t)+
\Lambda_{\T,t}(\widetilde\omega)\widetilde\Gamma_{\otimes}
\{\zeta_\theta(g,t)-\zeta_{\widetilde\theta}(\widetilde g,t)\}.
\end{aligned}
\]
By \eqref{eq:mask_lambda_t_perturb},
\(\|\Gamma_{\otimes}\zeta_\theta(g,t)\|_2\le K\), and
\eqref{eq:mask_forward_moment_cover}, the first term satisfies
\[
\left\{\E_{X_{\T,t}\mid X_\T}
\left\|
\{\Lambda_{\T,t}(\omega)-
\Lambda_{\T,t}(\widetilde\omega)\}
\{\Gamma_{\otimes}\zeta_\theta(g,t)-X_{\T,t}\}
\right\|_2^2\right\}^{1/2}\le
C p^{-\beta}h_t^{-2}\{K+Q_\T(t)\}\Delta_\omega.
\]
For the second term, Equations \eqref{eq:mask_kronecker_perturb} and
\eqref{eq:mask_lambda_norms} give $\|
\Lambda_{\T,t}(\widetilde\omega)
(\Gamma_{\otimes}-\widetilde\Gamma_{\otimes})
\zeta_\theta(g,t)
\|_2 \le C h_t^{-1}K\Delta_\Gamma$. Finally, \eqref{eq:mask_lambda_norms} and
\eqref{eq:mask_core_output_perturb} give
\[
\left\{\E_{X_{\T,t}\mid X_\T}
\left\|
\Lambda_{\T,t}(\widetilde\omega)\widetilde\Gamma_{\otimes}
\{\zeta_\theta(g,t)-
\zeta_{\widetilde\theta}(\widetilde g,t)\}
\right\|_2^2\right\}^{1/2}\le
h_t^{-1}L_{\mathrm{core}}(t)
(\Delta_\Gamma+\Delta_\omega+\delta_\theta).
\]
Consequently, we have
\begin{equation}\label{eq:mask_score_parameter_perturb}
\sup_{X_0\in\mathcal E_{\mathrm{tr}}}
\left\{
\E_{X_{\T,t}\mid X_\T}
\|s_{\Gamma,\omega,\theta}(X_{\T,t},X_\C,t)
-s_{\widetilde\Gamma,\widetilde\omega,\widetilde\theta}
(X_{\T,t},X_\C,t)\|_2^2
\right\}^{1/2}
\le
L_s(t)(\Delta_\Gamma+\Delta_\omega+\delta_\theta),
\end{equation}
\begin{equation}\label{eq:mask_score_parameter_constant}
L_s(t)
=
C\left[
p^{-\beta}h_t^{-2}\{K+Q_\T(t)\}
+h_t^{-1}K+h_t^{-1}L_{\mathrm{core}}(t)
\right],\qquad \overline L_s=\sup_{t\in[t_0,T]}L_s(t)
\end{equation}

The supremum above is finite because \(t_0>0\), \(\alpha_T>0\), and all
members of the score class satisfy the uniform \(C_V\), \(K\), and
parameter bounds. Define the time-averaged score metric
\[
d_{\mathrm{av}}^s(s,\widetilde s)
=
\sup_{X_0\in\mathcal E_{\mathrm{tr}}}
\left\{
\frac1{T-t_0}\int_{t_0}^T
\E_{X_{\T,t}\mid X_\T}
\|s(X_{\T,t},X_\C,t)-\widetilde s(X_{\T,t},X_\C,t)\|_2^2dt
\right\}^{1/2}.
\]
Then we have $d_{\mathrm{av}}^s(s_{\Gamma,\omega,\theta},
s_{\widetilde\Gamma,\widetilde\omega,\widetilde\theta})
\le
\overline L_s(\Delta_\Gamma+\Delta_\omega+\delta_\theta)$. For \(R_s(X_{\T,t},X_\C,t)=s(X_{\T,t},X_\C,t)-r_t(X_\T,X_{\T,t})\),
recall that
\[
\frac1{T-t_0}\int_{t_0}^T
\E_{X_{\T,t}\mid X_\T}\|R_s(X_{\T,t},X_\C,t)\|_2^2dt
\le B_{\ell,\mask}
\]
on \(\mathcal E_{\mathrm{tr}}\), uniformly in \(s\).  Define
\(d_\infty^\ell\) by $d_\infty^\ell
\{\ell^{\mathrm{tr}}(\cdot;s),
\ell^{\mathrm{tr}}(\cdot;\widetilde s)\}
=
\sup_{X_0}
|\ell^{\mathrm{tr}}(X_0;s)-\ell^{\mathrm{tr}}(X_0;\widetilde s)|$. Then, since
\(|\|R_s\|_2^2-\|R_{\widetilde s}\|_2^2|
\le(\|R_s\|_2+\|R_{\widetilde s}\|_2)\|s-\widetilde s\|_2\),
the difference of the two truncated losses is zero outside
\(\mathcal E_{\mathrm{tr}}\), and the Cauchy–Schwarz inequality in the
time-forward expectation gives
\[
\begin{aligned}
&|\ell^{\mathrm{tr}}(X_0;s)-\ell^{\mathrm{tr}}(X_0;\widetilde s)|=
\frac{\mathbf I_{\mathcal E_{\mathrm{tr}}}}{T-t_0}
\left|
\int_{t_0}^T
\E_{X_{\T,t}\mid X_\T}
\{\|R_s\|_2^2-\|R_{\widetilde s}\|_2^2\}dt
\right|\\
&\quad\le
\left[
\frac1{T-t_0}\int_{t_0}^T
\E_{X_{\T,t}\mid X_\T}
\{\|R_s\|_2+\|R_{\widetilde s}\|_2\}^2dt
\right]^{1/2}
d_{\mathrm{av}}^s(s,\widetilde s)\le
2B_{\ell,\mask}^{1/2}d_{\mathrm{av}}^s(s,\widetilde s),
\end{aligned}
\]
uniformly over \(X_0\in\mathcal E_{\mathrm{tr}}\), where the last step
uses \((a+b)^2\le2a^2+2b^2\) and the definition of
\(B_{\ell,\mask}\) for each of \(s\) and \(\widetilde s\). Hence, we have $d_\infty^\ell
\{\ell^{\mathrm{tr}}(\cdot;s),
\ell^{\mathrm{tr}}(\cdot;\widetilde s)\}
\le
2B_{\ell,\mask}^{1/2}d_{\mathrm{av}}^s(s,\widetilde s)$. Let \(\tau_\ell=B_{\ell,\mask}/n\) and choose
\(\tau_s=\tau_\ell/(4B_{\ell,\mask}^{1/2})\). It is enough to construct a
\(\tau_s\)-cover of the score class under \(d_{\mathrm{av}}^s\), because
such a cover gives a loss-cover radius of at most \(\tau_\ell\). By
\eqref{eq:mask_score_parameter_perturb}, it suffices to take
\[
\tau_\Gamma=\frac{\tau_s}{3\overline L_s},
\qquad
\tau_\omega=\frac{\tau_s}{3\overline L_s},
\qquad
\tau_\theta=\frac{\tau_s}{3\overline L_s}
\]
for the loading, variance, and core-network parameter covers,
respectively.

It remains to calculate the three parameter covering numbers. Under the
metric \(\max_d\|\Gamma_d-\widetilde\Gamma_d\|_{\op}\), a product of
\(\tau_\Gamma/2\)-covers of the individual Stiefel manifolds is a
\(\tau_\Gamma/2\)-cover of the loading space. The volumetric bound for
\(\operatorname{St}(p_d,r_d)\) in Lemma 8 of
\citetEC{chen2019inference}, together with
\(\|A\|_{\op}\le\|A\|_F\), gives
\[
\log\mathcal N\left(
\frac{\tau_\Gamma}{2},
\{\Gamma_d^\top \Gamma_d=I_{r_d}\}_{d=1}^D,
\max_d\|\cdot\|_{\op}
\right)\le
C\sum_{d=1}^Dp_dr_d
\log\left(1+\frac{C}{\tau_\Gamma}\right)
\le
C p_{\max}\log\left(1+\frac{C}{\tau_\Gamma}\right),
\]
because \(D\) and the mode ranks are fixed. Similarly,
\(\mathcal W_\sigma\) is a bounded rectangle of dimension
\(\sum_{d=1}^Dp_d\le Dp_{\max}\), so
\[
\log\mathcal N\left(
\frac{\tau_\omega}{2},\mathcal W_\sigma,\|\cdot\|_\infty
\right)
\le
C p_{\max}\log\left(1+\frac{C}{\tau_\omega}\right).
\]

For the core network, write \(\Theta_{\core}\) for the set of parameter
vectors satisfying the architecture, coefficient, and sparsity
constraints of \(\mathcal F_{\core}\). For a fixed depth, pad all
hidden layers to width \(m\), and let \(P_\theta\) be the maximum number
of weight and bias entries among the padded architectures. Since the
input and output dimensions are fixed at \(r+1\) and \(r\), respectively, we have $P_\theta\le C L(m+1)^2$. Put \(J_\theta=J\wedge P_\theta\). Since the architecture budgets are
positive, \(1\le J_\theta\le P_\theta\), and
\[
\sum_{j=0}^{J_\theta}\binom{P_\theta}{j}
\le
\left(\frac{eP_\theta}{J_\theta}\right)^{J_\theta}.
\]
For a fixed support of cardinality at most \(J_\theta\), the active
parameters lie in \([-\kappa,\kappa]^{J_\theta}\), which admits an
\(\ell_\infty\)-cover of radius \(\tau_\theta/2\) with at most
\(\{1+4\kappa/\tau_\theta\}^{J_\theta}\) elements. Taking the union over
the admissible depths and supports therefore gives
\[
\log\mathcal N\left(
\frac{\tau_\theta}{2},\Theta_{\core},\|\cdot\|_\infty
\right)
\le
\log L
+CJ_\theta\log\left(\frac{eP_\theta}{J_\theta}\right)
+
CJ_\theta\log\left(1+\frac{C\kappa}{\tau_\theta}\right).
\]

The preceding parameter covers may have centers outside
\(\mathcal S_{\MTClass}\), because the inverse and global network
constraints need not be preserved by coordinate discretization. Fix an
ordering of the centers in each half-radius cover and assign every
admissible parameter to the first center whose covering ball contains
it. Together with the network depth, the three resulting center indices
partition \(\mathcal S_{\MTClass}\) into no more cells than the product
of the three covering numbers. From every nonempty cell, select one
admissible parameter triple. Two triples assigned to the same cell are
within \(\tau_\Gamma\), \(\tau_\omega\), and \(\tau_\theta\) in the loading,
variance, and network metrics, respectively. It follows from
\eqref{eq:mask_score_parameter_perturb} that their score distance is at
most \(\tau_s\). The selected triples therefore form an internal
\(\tau_s\)-cover of \(\mathcal S_{\MTClass}\); in particular, every
center satisfies the defining bound on \(V_t\) and the common loss
envelope \(B_{\ell,\mask}\).

Multiplying the three parameter covering numbers gives
\begin{align}
\log\mathcal N(\tau_\ell,\mathcal G_{\mathrm{tr}},d_\infty^\ell)&\le
C p_{\max}\log\left(1+\frac{C}{\tau_\Gamma}\right)
+C p_{\max}\log\left(1+\frac{C}{\tau_\omega}\right)
+\log L\nonumber\\
&\qquad+
CJ_\theta\log\left(\frac{eP_\theta}{J_\theta}\right)
+CJ_\theta\log\left(1+\frac{C\kappa}{\tau_\theta}\right).
\label{eq:mask_loss_entropy_explicit}
\end{align}
Since
\[
\tau_s=\frac{B_{\ell,\mask}^{1/2}}{4n},
\qquad
\tau_\Gamma=\tau_\omega=\tau_\theta
=\frac{B_{\ell,\mask}^{1/2}}{12n\overline L_s},
\]
each metric-entropy logarithm in
\eqref{eq:mask_loss_entropy_explicit} is bounded by a constant multiple
of \(\log(n\Xi_{\mask})\), where we may take
\begin{equation}\label{eq:mask_xi_definition}
\Xi_{\mask}
=
2+\overline L_s+P_\theta+\kappa+L
+B_{\ell,\mask}^{-1/2}.
\end{equation}
Indeed, \(h_{t_0}^{-1}\le Ct_0^{-1}\),
\(p^\beta\le p\le p_{\max}^D\), and
\(p_\T\le p_{\max}^D\). The definitions
\eqref{eq:mask_v_perturb_explicit}, \eqref{eq:mask_la_definition},
\eqref{eq:mask_encoder_constants},
\eqref{eq:mask_network_parameter_constant}, and
\eqref{eq:mask_score_parameter_constant} imply
\[
\begin{aligned}
\log\Xi_{\mask}
\le C\bigg[&1+\log(1+C_V)
+\log\{1+B_{\ell,\mask}^{-1/2}\}
+\log(1+K+C_{\T,x}+C_{\C,x})\\
&+\log(1+p_{\max})+\log(t_0^{-1})
+\log(\alpha_T^{-1})+\log(1+L)
+L\log\{(m+1)(1\vee\kappa)\}\bigg].
\end{aligned}
\]
Under the configuration of Theorem \ref{thm:main-approximation} and the
chosen truncation radii, the right-hand side consists only of
polylogarithmic factors in the quantities suppressed by
\(\widetilde{\mathcal O}\). Consequently, \(\log\Xi_{\mask}\) is
absorbed into the logarithmic factor in the statistical rate.
Since \(J_\theta\le J\), we conclude that
\begin{equation}\label{eq:mask_loss_entropy}
\log\mathcal N(\tau_\ell,\mathcal G_{\mathrm{tr}},d_\infty^\ell)
\le
C\{J+p_{\max}\}\log(n\Xi_{\mask}),
\end{equation}
where the contribution \(J\) comes from the core network cover, while
the contribution \(p_{\max}\) comes from the finite-dimensional loading
and variance parameter covers; \(D\) and the ranks \(r_d\) are treated
as fixed.

Write \(\mathbb P g=\E g(X_0)\) and
\(\mathbb P_ng=n^{-1}\sum_{i=1}^ng(X_0^{(i)})\). Since every
\(g\in\mathcal G_{\mathrm{tr}}\) takes values in
\([0,B_{\ell,\mask}]\), the two one-sided relative-deviation
inequalities in Lemma 15 of \citetEC{chen2023score} imply
\[
\sup_{g\in\mathcal G_{\mathrm{tr}}}
\{\mathbb Pg-(1+a)\mathbb P_ng\}\vee
\sup_{g\in\mathcal G_{\mathrm{tr}}}
\{\mathbb P_ng-(1+a)\mathbb Pg\}\le
\frac{CB_{\ell,\mask}}{na}
\left[
\log\mathcal N(\tau_\ell,\mathcal G_{\mathrm{tr}},d_\infty^\ell)
+\log(2/\delta)
\right]
+C\tau_\ell.
\]
with probability at least \(1-\delta\), where the two directions are
assigned a failure probability of \(\delta/2\). This is the relative
empirical-process reduction used in the proof of the score-estimation
theorem of \citetEC{guo2026tucker}. Substituting
\(\tau_\ell=B_{\ell,\mask}/n\) and
\eqref{eq:mask_loss_entropy}, the following inequalities therefore hold
simultaneously for every \(s\in\mathcal S_{\MTClass}\): $\mathcal L_{\mask}^{\mathrm{tr}}(s)
\le
(1+a)\widehat{\mathcal L}_{\mask}^{\mathrm{tr}}(s)
+\Delta_{\mathrm{stat}}$ and $\widehat{\mathcal L}_{\mask}^{\mathrm{tr}}(s)
\le
(1+a)\mathcal L_{\mask}^{\mathrm{tr}}(s)
+\Delta_{\mathrm{stat}}$, where
\begin{equation}\label{eq:mask_delta_stat}
\Delta_{\mathrm{stat}}
\le
\frac{C B_{\ell,\mask}}{na}
\left[
\{J+p_{\max}\}\log(n\Xi_{\mask})+\log(2/\delta)
\right]
+
\frac{C B_{\ell,\mask}}{n}.
\end{equation}
The first term in \(\Delta_{\mathrm{stat}}\) is the relative
empirical-process contribution, and the final
\(CB_{\ell,\mask}/n\) is the discretization contribution.

\textit{Step 2: truncation error.} The residual identity \eqref{eq:mask_raw_residual_identity} gives, without imposing
\(\mathcal E_{\mathrm{tr}}\),
\[
\|s(X_{\T,t},X_\C,t)-r_t(X_\T,X_{\T,t})\|_2^2
\le
C\left\{
\frac{K^2+\|X_\T\|_2^2}{h_t^2}
+
\frac{p^{-2\beta}\|X_{\T,t}\|_2^2}{h_t^4}
\right\}.
\]
Here, the right-hand side is independent of the particular core network
except through the uniform output bound \(K\).  Since $\E(\|X_{\T,t}\|_2^2\mid X_0)
\le
C\|X_\T\|_2^2+p_\T h_t$, we have uniformly over \(s\in\mathcal S_{\MTClass}\) that
\[
\begin{aligned}
&\mathcal L_{\mask}(s)-\mathcal L_{\mask}^{\mathrm{tr}}(s)\\
&\quad\le
\frac{C}{T-t_0}\int_{t_0}^T
\left[
\frac{K^2\mathbb{P}(\mathcal E_{\mathrm{tr}}^c)
+\E\{\|X_\T\|_2^2\mathbf I_{\mathcal E_{\mathrm{tr}}^c}\}}{h_t^2}
+
\frac{p^{-2\beta}\E\{\|X_\T\|_2^2\mathbf I_{\mathcal E_{\mathrm{tr}}^c}\}}{h_t^4}
+
\frac{p_\T p^{-2\beta}h_t\mathbb{P}(\mathcal E_{\mathrm{tr}}^c)}{h_t^4}
\right]dt .
\end{aligned}
\]
The selected radii in \(\mathcal E_{\mathrm{tr}}\) satisfy $\mathbb{P}(\mathcal E_{\mathrm{tr}}^c)\le (3n^2)^{-1}$ and $\E\{\|X_\T\|_2^2\mathbf I_{\mathcal E_{\mathrm{tr}}^c}\}
\le n^{-2}$. Substituting these two estimates gives
\[
\mathcal L_{\mask}(s)-\mathcal L_{\mask}^{\mathrm{tr}}(s)\le
\frac{C}{n^2(T-t_0)}\int_{t_0}^T
\left[
\frac{K^2+1}{h_t^2}
+
\frac{p^{-2\beta}}{h_t^4}
+
\frac{p_\T p^{-2\beta}h_t}{h_t^4}
\right]dt .\]
This last integral is dominated term-by-term by the deterministic
envelope in \eqref{eq:mask_bl_bound}: \(K^2+1\le C(K^2+C_{\T,x}^2)\),
\(p^{-2\beta}\le p^{-2\beta}C_{\T,x}^2\) since \(C_{\T,x}^2\ge1\) for
\(n\) sufficiently large, and the final
\(p_\T p^{-2\beta}h_t/h_t^4\) term is exactly the third envelope term.
Since \(n^{-2}\le n^{-1}\), the preceding display is bounded by
\(C B_{\ell,\mask}/n\).  Hence
\begin{equation}\label{eq:mask_trunc_remainder}
\Delta_{\mathrm{tr}}
:=
\sup_{s\in\mathcal S_{\MTClass}}
\{\mathcal L_{\mask}(s)-\mathcal L_{\mask}^{\mathrm{tr}}(s)\}
\le
\frac{C B_{\ell,\mask}}{n}.
\end{equation}
Moreover, we have
\[
\mathbb P\left(\bigcap_{i=1}^n\mathcal E_{\mathrm{tr},i}\right)
\ge
1-\sum_{i=1}^n\mathbb{P}(\mathcal E_{\mathrm{tr},i}^c)
\ge 1-\frac1{3n},
\]
and on this event
\(\widehat{\mathcal L}_{\mask}^{\mathrm{tr}}(s)
=\widehat{\mathcal L}_{\mask}(s)\) for all \(s\).

\textit{Step 3: approximation term and raw-loss oracle inequality.}
Let \(\mathcal A_{\mathrm{stat}}\) be the concentration event from
Step 1 and let
\(\mathcal A_{\mathrm{tr}}=\cap_{i=1}^n\mathcal E_{\mathrm{tr},i}\) be
the sample truncation event from Step 2. On
\(\mathcal A_{\mathrm{stat}}\cap\mathcal A_{\mathrm{tr}}\), the two
statistical terms in \eqref{eq:mask_raw_error_decomposition} are at
most \(\Delta_{\mathrm{stat}}\), the truncation term is at most
\(\Delta_{\mathrm{tr}}\), and the empirical optimality used in that
decomposition is valid. Consequently,
\begin{equation}\label{eq:mask_raw_oracle_inequality}
\mathcal L_{\mask}(\widehat s)
\le
(1+a)^2\mathcal L_{\mask}(s^\circ)
+
(2+a)\Delta_{\mathrm{stat}}
+
\Delta_{\mathrm{tr}}.
\end{equation}

\textit{Step 4: conversion from the DSM loss to score risk.}
Finally, let $s_t(x_{\T,t}\mid x_\C)
=\nabla_{x_{\T,t}}\log p_t(x_{\T,t}\mid x_\C)$. For fixed \(t\), \(r_t(X_\T,X_{\T,t})\) is the score of the forward Gaussian kernel \(p(X_{\T,t}\mid X_\T)\) with respect to the treatment
coordinate \(X_{\T,t}\). Therefore, the usual denoising identity gives $s_t(X_{\T,t}\mid X_\C)
= \E\{r_t(X_\T,X_{\T,t})\mid X_{\T,t},X_\C\}$. This identity is the conditional version of the \(L^2\)-projection property of denoising score matching. To use it, write $s-r_t=(s-s_t)+(s_t-r_t)$. For every square-integrable \(s\), expanding the square and conditioning
on \((X_{\T,t},X_\C)\) yields
\[
\begin{aligned}
\E_{X_\C}\E_{X_\T\mid X_\C}\E_{X_{\T,t}\mid X_\T}
\|s(X_{\T,t},X_\C,t)&-r_t(X_\T,X_{\T,t})\|_2^2=
\E_{X_\C}\E_{X_{\T,t}\mid X_\C}
\|s(X_{\T,t},X_\C,t)-s_t(X_{\T,t}\mid X_\C)\|_2^2\\
&\qquad+
\E_{X_\C}\E_{X_\T\mid X_\C}\E_{X_{\T,t}\mid X_\T}\|r_t(X_\T,X_{\T,t})-s_t(X_{\T,t}\mid X_\C)\|_2^2.
\end{aligned}
\]
Indeed, the cross term is $2\E[
\langle
s(X_{\T,t},X_\C,t)-s_t(X_{\T,t}\mid X_\C),
s_t(X_{\T,t}\mid X_\C)-r_t(X_\T,X_{\T,t})
\rangle]$, which vanishes because
\(\E\{r_t-s_t\mid X_{\T,t},X_\C\}=0\).  Integrating over time,
we obtain the exact decomposition $\mathcal L_{\mask}(s)=\mathcal R_{\mask}(s)+C_{\mathrm{DSM}}$, where
\[
C_{\mathrm{DSM}}
=
\frac1{T-t_0}\int_{t_0}^T
\E_{X_\C}\E_{X_\T\mid X_\C}\E_{X_{\T,t}\mid X_\T}
\|r_t(X_\T,X_{\T,t})-s_t(X_{\T,t}\mid X_\C)\|_2^2dt,\quad\text{and}
\]
$$\mathcal R_{\mask}(s)=(T-t_0)^{-1}\int_{t_0}^T
\E_{X_{\C}}\E_{X_{\T,t}\mid X_{\C}}
\|s(X_{\T,t},X_{\C},t)-\nabla_{x_{\T,t}}\log p_t(X_{\T,t}\mid X_{\C})\|_2^2dt.$$
Note that $C_{\mathrm{DSM}}$ does not depend on \(s\). Applying this identity to the raw oracle
inequality \eqref{eq:mask_raw_oracle_inequality} yields
\begin{equation}\label{eq:mask_oracle_score}
\mathcal R_{\mask}(\widehat s)
\le
(1+a)^2\mathcal R_{\mask}(s^\circ)
+
\{(1+a)^2-1\}C_{\mathrm{DSM}}
+
(2+a)\Delta_{\mathrm{stat}}
+
\Delta_{\mathrm{tr}}.
\end{equation}

The term \(\{(1+a)^2-1\}C_{\mathrm{DSM}}\) results from the
multiplicative raw-loss inequality. We next bound
\(C_{\mathrm{DSM}}\), which is independent of \(s\).
By the exact score decomposition, we know that $s_t(X_{\T,t}\mid X_\C)
=
\Lambda_{\T,t}
\{A_{\otimes}\xi(g_t,t)-X_{\T,t}\}$. Hence,
\[
\begin{aligned}
r_t-s_t
&=
\alpha_th_t^{-1}X_\T
-h_t^{-1}X_{\T,t}
-\Lambda_{\T,t}A_{\otimes}\xi(g_t,t)
+\Lambda_{\T,t}X_{\T,t}\\
&=
\alpha_th_t^{-1}X_\T
-\{h_t^{-1}\MT-\Lambda_{\T,t}\}X_{\T,t}
-\Lambda_{\T,t}A_{\otimes}\xi(g_t,t).
\end{aligned}
\]

The same precision-difference bound used in
\eqref{eq:mask_raw_residual_identity} gives $\|h_t^{-1}\MT-\Lambda_{\T,t}\|_{\op}
\le C p^{-\beta}h_t^{-2}$. By Jensen's inequality and the definition
\(\xi(g_t,t)=\E(\alpha_t f\mid g_t)\), we have $\E\|\xi(g_t,t)\|_2^2
\le
\E\|\alpha_t f\|_2^2
\le
\E\|f\|_2^2
\le C_r$. Also, $\E\|X_\T\|_2^2\le Cd_{\T,\beta}$, while $\E\|X_{\T,t}\|_2^2\le C(d_{\T,\beta}+p_\T h_t)$. Using \((a+b+c)^2\le3(a^2+b^2+c^2)\), the three pieces in
\(r_t-s_t\) are bounded as follows:
\[
\E\|\alpha_th_t^{-1}X_\T\|_2^2
\le
Cd_{\T,\beta}h_t^{-2},\quad \E\|\{h_t^{-1}\MT-\Lambda_{\T,t}\}X_{\T,t}\|_2^2
\le
Cp^{-2\beta}h_t^{-4}(d_{\T,\beta}+p_\T h_t),
\]
\[
\text{and}\quad \E\|\Lambda_{\T,t}A_{\otimes}
\xi(g_t,t)\|_2^2
\le
Ch_t^{-2}
\le
Cd_{\T,\beta}h_t^{-2},
\]
where the last inequality uses \(d_{\T,\beta}\ge1\).  After time
averaging, these bounds imply
\[
C_{\mathrm{DSM}}
\le
C
\left\{
\frac{d_{\T,\beta}}{T-t_0}\int_{t_0}^T h_t^{-2}dt
+
\frac{d_{\T,\beta}p^{-2\beta}}{T-t_0}\int_{t_0}^T h_t^{-4}dt
+
\frac{p_\T p^{-2\beta}}{T-t_0}\int_{t_0}^T h_t^{-3}dt
\right\}.
\]
Under \(p^{-\beta}\le c_0t_0\), the same argument used in the envelope
bound gives \(p^{-\beta}\le Ch_t\) for all \(t\in[t_0,T]\). Therefore
, the second term in the preceding display satisfies $d_{\T,\beta}p^{-2\beta}h_t^{-4}
\le
C d_{\T,\beta}h_t^{-2}$, and the third term satisfies $p_\T p^{-2\beta}h_t^{-3}
=
p_\T p^{-\beta}h_t^{-2}(p^{-\beta}h_t^{-1})
\le
C d_{\T,\beta}h_t^{-2}$. Consequently, we have
\begin{equation}\label{eq:mask_cdsm_bound}
C_{\mathrm{DSM}}
\le
C
\frac{d_{\T,\beta}}{T-t_0}
\int_{t_0}^T h_t^{-2}dt
\le
C d_{\T,\beta}
\left(\frac1{t_0}+T\right).
\end{equation}
Thus, \(C_{\mathrm{DSM}}/d_{\T,\beta}\) has the same time singularity as
the statistical upper bound after the early-stopping normalization.

\textit{Step 5: balancing.}
Theorem \ref{thm:main-approximation} gives $\mathcal R_{\mask}(s^\circ)
\le
(T-t_0)^{-1}\int_{t_0}^T
h_t^{-2}(\sqrt r+1)^2\epsilon^2dt$. Set \(\delta=1/(3n)\) and \(a=c_a\epsilon^2\), where \(c_a>0\) is a
small numerical constant.  Since \(a\in(0,1)\),
\((1+a)^2\le4\) and \((1+a)^2-1\le3a\). Define
\[
\mathfrak C_\epsilon
:=(1+L_gR_\epsilon)^r(1+TL_t)
\left[1+\log\left\{1+
\frac{K_0+L_gR_\epsilon}{\epsilon}\right\}\right].
\]
The displayed choices of \(m\) and \(L\) in
Theorem~\ref{thm:main-approximation} then give $J=O\{\mathfrak C_\epsilon\epsilon^{-(r+1)}\}$. Set $\mathfrak L_\epsilon
=
\log(n\Xi_{\mask})+\log(6n)$. By \eqref{eq:mask_bl_simplified_rate},
\eqref{eq:mask_delta_stat}, and \(\delta=1/(3n)\), we have
\[
\frac{\Delta_{\mathrm{stat}}}{d_{\T,\beta}}
\le
\widetilde{\mathcal O}
\left[
\left(\frac1{t_0}+T\right)
\left\{
\frac{
\{\mathfrak C_\epsilon\epsilon^{-(r+1)}+p_{\max}\}
\mathfrak L_\epsilon}
{n\epsilon^2}
+\frac1n
\right\}
\right],
\]
where \(a^{-1}=c_a^{-1}\epsilon^{-2}\). The estimate following
\eqref{eq:mask_xi_definition} shows that
\(\mathfrak L_\epsilon\) is polylogarithmic in
\(\{t_0^{-1},\epsilon^{-1},\alpha_T^{-1},p_{\max},n\}\) under the
network configuration of Theorem~\ref{thm:main-approximation}.  In
particular, the layerwise factor
\((m+1)^{L+2}(1\vee\kappa)^L\) enters only through $\log L_\theta^{\mathrm{net}}
\le
C\left[
\log(1+L)+L\log\{(m+1)(1\vee\kappa)\}
\right]$, and therefore introduces no additional polynomial dependence on
\(\epsilon^{-1}\) or \(p_{\max}\). The truncation bound
\eqref{eq:mask_trunc_remainder} gives $d_{\T,\beta}^{-1}\Delta_{\mathrm{tr}}=\widetilde{\mathcal O}\{n^{-1}t_0^{-1}+n^{-1}T
\},$ which is dominated by the preceding statistical bound because
\(\epsilon\in(0,1)\),
\(\mathfrak C_\epsilon\epsilon^{-(r+1)}+p_{\max}\ge1\), and
\(\mathfrak L_\epsilon\ge1\). For the approximation term, Theorem
\ref{thm:main-approximation} gives the pointwise bound
\[
\E_{X_\C}\E_{X_{\T,t}\mid X_\C}
\left[
\left\|
s^\circ(X_{\T,t},X_\C,t)
-
s_t(X_{\T,t}\mid X_\C)
\right\|_2^2
\right]
\le
C h_t^{-2}\epsilon^2.
\]
As \(r\) is treated as fixed, integration and \eqref{eq:mask_cdsm_bound} give
\[
\frac{\mathcal R_{\mask}(s^\circ)}{d_{\T,\beta}}
\le
\widetilde{\mathcal O}
\left[
\left(\frac1{t_0}+T\right)\epsilon^2
\right],\quad \frac{aC_{\mathrm{DSM}}}{d_{\T,\beta}}
\le
\widetilde{\mathcal O}
\left[
\left(\frac1{t_0}+T\right)\epsilon^2.
\right]
\]
Thus, the oracle inequality \eqref{eq:mask_oracle_score}, after division
by \(d_{\T,\beta}\), gives
\[
\frac{\mathcal R_{\mask}(\widehat s)}{d_{\T,\beta}}
\le
\widetilde{\mathcal O}
\left[
\left(\frac1{t_0}+T\right)
\left\{
\epsilon^2
+
\frac{
\{\mathfrak C_\epsilon\epsilon^{-(r+1)}+p_{\max}\}
\mathfrak L_\epsilon}
{n\epsilon^2}
\right\}
\right].
\]
The concentration event has a probability of at least \(1-\delta\), while
the sample truncation event has a probability of at least \(1-1/(3n)\).
With \(\delta=1/(3n)\), a union bound gives a failure probability of at most
\(2/(3n)\le1/n\). Hence, the displayed bound holds with a probability of at least
\(1-1/n\). Finally, the choice of \(\epsilon\) gives $\epsilon^2=n^{-(2-2\delta_n)/(r+5)}$. By the definition of \(\mathfrak C_\epsilon\), the identity
\(R_\epsilon^2=O\{\log((t_0\epsilon)^{-1})\}\),
\(T=2\log(\alpha_T^{-1})\), and the bound following
\eqref{eq:mask_xi_definition}, the explicit \(n\)-logarithmic order,
apart from the polylogarithmic factors already suppressed in the theorem,
can be stated precisely as follows: there is a factor \(\Pi_n\), polynomial
in \(\log(1+t_0^{-1})\), \(\log(1+p_{\max})\), and
\(\log(1+\alpha_T^{-1})\), such that $\mathfrak C_\epsilon\mathfrak L_\epsilon
\le C\Pi_n(\log n)^{(r+8)/2}$. Moreover,
\(n^{-\delta_n}=(\log n)^{-(r+12)/2}\). Consequently, $(n\epsilon^2)^{-1}\mathfrak C_\epsilon\epsilon^{-(r+1)}
\mathfrak L_\epsilon =\epsilon^2\mathfrak C_\epsilon\mathfrak L_\epsilon n^{-\delta_n}
=O\{\Pi_n\epsilon^2(\log n)^{-2}\}
=\widetilde{\mathcal O}(\epsilon^2)$. Hence, the enlarged covering-number bound does
not alter the polynomial balance between approximation and statistical
error. Moreover, we have $(n\epsilon^2)^{-1}p_{\max}\mathfrak L_\epsilon=\widetilde{\mathcal O}(
p_{\max}n^{-\frac{r+3+2\delta_n}{r+5}}
)$; substitution into the preceding oracle bound proves
\eqref{eq:main-generalization-rate}.
\endproof
$\hfill\square$

\subsection{Proof of Theorem~\ref{thm:main-distribution}}
\proof{Proof.}
For a sufficiently large \(C_\beta\), the choices in the theorem imply
\(c_0^{-1}p^{-\beta}<t_0\). Moreover, \(T-t_0\ge c_T\) for all
sufficiently large \(n\). Hence, Theorem
\ref{thm:main-generalization} applies. On the event given in Theorem
\ref{thm:main-generalization}, we can fix the training sample and hence
\(\widehat s\). Fix \(x_\C\) in the full-measure set considered below. On reverse time
\([0,T-t_0]\), let \(\mathbb P_x\), \(\widetilde{\mathbb P}_x\), and
\(\mathbb Q_x\) be, respectively, the path laws of the exact-score process
initialized from \(P_T^\T(\cdot\mid x_\C)\), the estimated-score process
initialized from the same law, and the trained process initialized from
\(\gamma_\T=\N(0,I_{p_\T})\). Their endpoint laws are
\(P_{t_0}^\T(\cdot\mid x_\C)\),
\(\widetilde P_{t_0}^\T(\cdot\mid x_\C)\), and
\(\widehat P_{t_0}^\T(\cdot\mid x_\C)\).

We first control the score-estimation term.  Let $s_t(x_{\T,t}\mid x_\C)
=
\nabla_{x_{\T,t}}\log p_t(x_{\T,t}\mid x_\C)$. We verify the square-integrability required below.  The DSM
projection identity in Step 4 of the proof of Theorem
\ref{thm:main-generalization} and Jensen's inequality give
\[
\E_{X_{\T,t}\mid x_\C}
\|s_t(X_{\T,t}\mid x_\C)\|_2^2
\le
\E\{\|r_t(X_\T,X_{\T,t})\|_2^2\mid X_\C=x_\C\}
=
\frac{p_\T}{h_t}.
\]
Here the last equality follows because
\(r_t=-h_t^{-1/2}\MT z_t\), where \(z_t\sim\N(0,I_p)\) is independent of
\(X_\C\).
Moreover, every \(s_{\Gamma,\omega,\theta}\in\mathcal S_{\MTClass}\) satisfies
\(\|\Lambda_{\T,t}(\omega)\|_{\op}\le h_t^{-1}\),
\(\|\Gamma_{\otimes}\|_{\op}=1\), and
\(\|\zeta_\theta\|_2\le K\).  Consequently,
\[
\E_{X_{\T,t}\mid x_\C}
\|\widehat s(X_{\T,t},x_\C,t)\|_2^2
\le
C h_t^{-2}
\left\{K^2+
\E(\|X_\T\|_2^2\mid X_\C=x_\C)+p_\T h_t\right\}
\]
for \(P_\C\)-almost every \(x_\C\). Since
\(\E(\|X_\T\|_2^2\mid X_\C)<\infty\) almost surely and
\(h_t\ge h_{t_0}>0\), the preceding bounds imply $\sup_{t\in[t_0,T]}
\E_{X_{\T,t}\mid x_\C}
\|\widehat s(X_{\T,t},x_\C,t)
-s_t(X_{\T,t}\mid x_\C)\|_2^2
<\infty$.
Moreover, for fixed \(x_\C\), every score in
\(\mathcal S_{\MTClass}\) is globally Lipschitz with at most linear
growth in \(x_{\T,t}\) on \([t_0,T]\). Hence, the estimated reverse SDE
defines a Markov kernel in the treated region. 

Then, to apply the change-of-measure argument on a non-degenerate state space,
let \(J_{\T}\in\R^{p\times p_{\T}}\) be the treated-region-coordinate
selection matrix, so that $J_{\T}^{\top}J_{\T}=I_{p_{\T}}$ and $J_{\T}J_{\T}^{\top}=\MT$. Every state and score in the reverse process is supported on \(\T\).
Thus, write \(y_u=J_{\T}^{\top}x_{\T,u}^{\leftarrow}\)
for the exact-score reverse process and
\(\widetilde y_u=J_{\T}^{\top}\widetilde x_{\T,u}^{\leftarrow}\)
for the auxiliary reverse process, and set
\(b_t(y,x_\C)=J_{\T}^{\top}s_t(J_{\T}y\mid x_\C)\) and
\(\widehat b_t(y,x_\C)=J_{\T}^{\top}\widehat s(J_{\T}y,x_\C,t)\).
The true and auxiliary reverse processes, in these coordinates, have
the respective SDEs
\[
\mathrm d y_u=\{\tfrac12y_u+b_{T-u}(y_u,x_\C)\}\,\mathrm d u+\mathrm d B_u,
\qquad
\mathrm d\widetilde y_u=\{\tfrac12\widetilde y_u+\widehat b_{T-u}(\widetilde y_u,x_\C)\}\,\mathrm d u+\mathrm d\widetilde B_u,
\]
with the same initial law and a common identity diffusion coefficient; \(B_u\) and
\(\widetilde B_u\) are standard \(p_\T\)-dimensional Brownian motions. Since \(J_{\T}\) is an isometry
between the treated region and \(\R^{p_{\T}}\), relative entropy
and the squared drift difference are unchanged by this
coordinate representation. Clearly, the preceding moment bounds verify the change-of-measure condition
used in Lemma D.4 of \citetEC{fu2024unveil}. The first line below is the
additive property of relative entropy under disintegration
\citepEC[Theorem~2.4]{leonard2014path}, applied to the initial-coordinate
map. The second line is the localized Girsanov bound for the two path laws
with a common initial distribution
\citepEC{girsanov1960transforming,follmer2005entropy,fu2024unveil}:
\begin{align*}
\operatorname{KL}(\mathbb P_x\|\mathbb Q_x)
&=\operatorname{KL}(\mathbb P_x\|\widetilde{\mathbb P}_x)
+\operatorname{KL}\!\left\{P_T^\T(\cdot\mid x_\C)
\middle\|\gamma_\T\right\}\notag\\
&\le \frac12\int_{t_0}^T
\E_{X_{\T,t}\mid x_\C}
\|\widehat s(X_{\T,t},x_\C,t)-s_t(X_{\T,t}\mid x_\C)\|_2^2dt
+\operatorname{KL}\!\left\{P_T^\T(\cdot\mid x_\C)
\middle\|\gamma_\T\right\}.
\end{align*}
The Girsanov bound is understood by first stopping when the accumulated
squared drift difference reaches a finite level and then passing to the
limit by lower semicontinuity of relative entropy; thus the finite-energy
condition verified above suffices.
Indeed, \(\widetilde{\mathbb P}_x\) and \(\mathbb Q_x\) use the same
estimated reverse dynamics, so their likelihood ratio depends only on the
initial state. Data processing under the endpoint map
\citepEC{csiszar1967information} and averaging over
\(X_\C\) therefore yield
\begin{equation}\label{eq:proof-endpoint-kl}
\E_{X_\C}\operatorname{KL}\!\left\{
P_{t_0}^\T(\cdot\mid X_\C)
\middle\|\widehat P_{t_0}^\T(\cdot\mid X_\C)\right\}\notag\le \frac12(T-t_0)\mathcal R_{\mask}(\widehat s)
+\E_{X_\C}\operatorname{KL}\!\left\{
P_T^\T(\cdot\mid X_\C)\middle\|\gamma_\T\right\}.
\end{equation}

The definition of \(\mathcal R_{\mask}\) and Theorem
\ref{thm:main-generalization} give
\begin{equation}\label{eq:dist_score_part}
\frac{(T-t_0)\mathcal R_{\mask}(\widehat s)}{d_{\T,\beta}}
\le
\widetilde{\mathcal O}\!\left[
\left(\frac1{t_0}+T\right)
\left\{
n^{-\frac{2-2\delta_n}{r+5}}
+p_{\max}n^{-\frac{r+3+2\delta_n}{r+5}}
\right\}
\right],
\end{equation}
where the factor \(T-t_0\) is absorbed into the logarithmic term because
\(T\asymp\log n+\log p_{\max}\).

It remains to bound the terminal relative-entropy term in \eqref{eq:proof-endpoint-kl}. The forward process for the missing outcomes is the
Ornstein-Uhlenbeck process on \(\R^{p_\T}\) with invariant law
\(\gamma_\T\). Before applying entropy contraction pointwise in
\(x_\C\), we verify that its initial conditional relative entropy is
finite almost surely. Identifying the treated region with
\(\R^{p_\T}\),
write $x_\T=\MT A_{\otimes}f+G_\T$, where $G_\T\sim \N(0,p^{-\beta}(\Sigma_e^{\otimes})_\T)$, \((\Sigma_e^{\otimes})_\T\) is the treated-region principal submatrix
and \(G_\T\) is independent of \((f,X_\C)\). Conditional on
\(X_\C=x_\C\), the law of \(x_\T\) is therefore the convolution of the
conditional distribution of \(\MT A_{\otimes}f\) with the law of \(G_\T\).
More explicitly, if \(Y_\T\) has the conditional distribution of
\(\MT A_{\otimes}f\) given \(X_\C=x_\C\), then $h(Y_\T+G_\T)-h(G_\T)
=
I(Y_\T;Y_\T+G_\T)
\ge0$, where the identity uses the independence of \(Y_\T\) and \(G_\T\). Indeed, the conditional entropy and KL
identities used next are
integrable since Gaussian convolution bounds the conditional entropy from below, while the finite conditional second moment, the Gaussian maximum-entropy bound, and Jensen's inequality control it from above after averaging over \(X_\C\). Therefore,
\[
h(x_\T\mid X_\C=x_\C)
\ge
h(G_\T)
=
\frac12\log\left\{(2\pi \mathrm e)^{p_\T}
(p^\beta)^{-p_\T}\det((\Sigma_e^{\otimes})_\T)\right\},
\]
and $\E_{X_\C}\E(\|x_\T\|_2^2\mid X_\C)
=
\E\|x_\T\|_2^2
\le
C(r+p_\T p^{-\beta})$. Using the identity
\(\mathrm{KL}(P\|\gamma_\T)=\frac12\E_P\|X\|_2^2-h(P)
+\frac{p_\T}{2}\log(2\pi)\), we obtain
\[
\begin{aligned}
\E_{X_\C}
\mathrm{KL}\{P_0^\T(\cdot\mid X_\C)\|\gamma_\T\}& =
\frac12\E\|x_\T\|_2^2
-\E_{X_\C}h(x_\T\mid X_\C)
+\frac{p_\T}{2}\log(2\pi)\\
&\quad\le
C(r+p_\T p^{-\beta})
-h(G_\T)
+\frac{p_\T}{2}\log(2\pi).
\end{aligned}
\]
Moreover, the eigenvalues of \(\Sigma_e^{\otimes}\) are bounded above
and below by fixed constants.  Since \(p\le p_{\max}^D\) and \(D\) is
fixed,
\[
\begin{aligned}
-h(G_\T)+\frac{p_\T}{2}\log(2\pi)
&=
\frac{p_\T}{2}\log(p^\beta)
-\frac{p_\T}{2}
-\frac12\log\det\{(\Sigma_e^{\otimes})_\T\}\\
&\le
C p_\T\{1+\log(p^\beta)\}\le
C p_\T\log(1+p_{\max}),
\end{aligned}
\]
where the last two bounds use the fixed variance bounds and
\(\log(p^\beta)=\sum_{d=1}^D\beta_d\log p_d
\le D\log p_{\max}\).
Combining the preceding two displays gives
\begin{equation}\label{eq:mask_initial_conditional_kl}
\E_{X_\C}
\mathrm{KL}\{P_0^\T(\cdot\mid X_\C)\|\gamma_\T\}
\le
C\{r+p_\T\log(1+p_{\max})\}
\end{equation}
for a constant depending only on the fixed variance bounds and the
core-factor tail constants. Since relative entropy is nonnegative,
\eqref{eq:mask_initial_conditional_kl} implies
\(\mathrm{KL}\{P_0^\T(\cdot\mid x_\C)\|\gamma_\T\}<\infty\) for
\(P_\C\)-almost every \(x_\C\). For each such context, the standard
Ornstein–Uhlenbeck entropy contraction \citepEC{bakry2014analysis}
therefore applies pointwise; see also the proof of Theorem~4.2 in
\citetEC{fu2024unveil}. Integrating that pointwise inequality and
using \eqref{eq:mask_initial_conditional_kl} yields
\begin{equation*}
\E_{X_\C}\mathrm{KL}\{P_T^\T(\cdot\mid X_\C)\|\gamma_\T\}
\le e^{-cT}\E_{X_\C}
\mathrm{KL}\{P_0^\T(\cdot\mid X_\C)\|\gamma_\T\}
\le C e^{-cT}\{r+p_\T\log(1+p_{\max})\}.
\end{equation*}
Since \(p_\T\le p\le p_{\max}^D\) and \(D\) is fixed, taking
\(C_T\) sufficiently large in
\(T=C_T(\log n+\log p_{\max})\) gives directly
\[
\E_{X_\C}\operatorname{KL}\!\left\{
P_T^\T(\cdot\mid X_\C)\middle\|\gamma_\T\right\}
=o\!\left(d_{\T,\beta}\mathfrak R_n^2\right).
\]

Finally, substitute \(t_0\asymp n^{-a_n}\) into the score-estimation
bound \eqref{eq:dist_score_part}. Since
\(T=C_T(\log n+\log p_{\max})\), we have 
$(1/t_0+T)
=\widetilde{\mathcal O}(n^{a_n})$. For the first score-estimation term, we have
\[
n^{a_n}n^{-\frac{2-2\delta_n}{r+5}}
=
n^{-\frac{5(1-\delta_n)}{3(r+5)}},
\]
because \(a_n=(1-\delta_n)/\{3(r+5)\}\). As for the second term, we have
\[
n^{a_n}p_{\max}n^{-\frac{r+3+2\delta_n}{r+5}}
=
p_{\max}\,
n^{-\frac{3r+8+7\delta_n}{3(r+5)}}.
\]
Combining these rates with \eqref{eq:proof-endpoint-kl} proves
\eqref{eq:main-distribution-kl}.

\endproof
$\hfill\square$

\subsection{Proof of Corollaries \ref{cor:main-linear-functional} and \ref{cor:counterfactual_prediction_interval}}
\proof{Proof.}
Fix a realization of \(\mathcal F_{\rm tr}\) in \(\mathcal A_n\) and condition on
\(X_\C^{1:n_e}=x_\C^{1:n_e}\). For \(\ell=1,\ldots,n_e\), write $P_\ell=P_{t_0}^{\T}(\cdot\mid x_\C^{(\ell)})$ and $Q_\ell=\widehat P_{t_0}^{\T}(\cdot\mid x_\C^{(\ell)})$. Let \(J_\T\in\R^{p\times p_\T}\) select the positions in the treated region, with
\(J_\T^\top J_\T=I_{p_\T}\) and \(J_\T J_\T^\top=\MT\), and set
\(w_\T=J_\T^\top w\). Extend the aggregate notation to the forward process by defining
\[
u_{\ell,t}=J_\T^\top\vecc(X_{\T,t}^{(\ell)})\in\R^{p_\T},
\qquad
U_{w,\ell,t}=w^\top\vecc(X_{\T,t}^{(\ell)}),
\qquad
V_{n_e,w,t}=\frac{1}{n_e}\sum_{\ell=1}^{n_e}U_{w,\ell,t}.
\]
Thus, \(U_{w,\ell,t}=w_\T^\top u_{\ell,t}\). Conditional independence gives the joint laws
\(\bigotimes_{\ell=1}^{n_e}P_\ell\) and \(\bigotimes_{\ell=1}^{n_e}Q_\ell\).
Relative entropy is additive over product measures and contracts under measurable maps. Applying these properties
\citepEC{csiszar1967information} to $\Phi_{n_e,w}(u_1,\ldots,u_{n_e})
=(n_e)^{-1/2}\sum_{\ell=1}^{n_e}w_\T^\top u_\ell$ gives
\begin{equation*}
\operatorname{KL}\!\left[
\mathcal L\{\sqrt{n_e}V_{n_e,w,t_0}\mid x_\C^{1:n_e}\}
\middle\|
\mathcal L\{\sqrt{n_e}\widehat V_{n_e,w,t_0}\mid x_\C^{1:n_e}\}
\right]\le
\operatorname{KL}\!\left(\bigotimes_{\ell=1}^{n_e}P_\ell
\middle\|\bigotimes_{\ell=1}^{n_e}Q_\ell\right)
=\sum_{\ell=1}^{n_e}\operatorname{KL}(P_\ell\|Q_\ell).
\end{equation*}
Averaging over the contexts and applying \eqref{eq:main-distribution-kl}, we shall have
\[
\E_{X_\C^{1:n_e}}\operatorname{KL}\!\left[
\mathcal L\{\sqrt{n_e}V_{n_e,w,t_0}\mid X_\C^{1:n_e}\}
\middle\|
\mathcal L\{\sqrt{n_e}\widehat V_{n_e,w,t_0}\mid X_\C^{1:n_e}\}
\right]
\le C n_e d_{\T,\beta}\mathfrak R_n^2.
\]
Pinsker's and Jensen's inequalities, therefore, give
\begin{equation}\label{eq:proof-aggregate-tv}
\E_{X_\C^{1:n_e}}\operatorname{TV}\!\left[
\mathcal L\{\sqrt{n_e}V_{n_e,w,t_0}\mid X_\C^{1:n_e}\},
\mathcal L\{\sqrt{n_e}\widehat V_{n_e,w,t_0}\mid X_\C^{1:n_e}\}
\right]
\le C\sqrt{n_e}\,d_{\T,\beta}^{1/2}\mathfrak R_n.
\end{equation}

For each analysis tensor, couple the original and early-stopped variables using an independent copy \(Z_\ell\) of the Gaussian tensor in \eqref{eq:main-forward}. Then
\begin{equation}\label{eq:aggregate-ou-coupling}
\sqrt{n_e}\bigl(V_{n_e,w,t}-V_{n_e,w,0}\bigr)
=(\alpha_t-1)\sqrt{n_e}V_{n_e,w,0}
+\frac{h_t^{1/2}}{\sqrt{n_e}}\sum_{\ell=1}^{n_e}w^\top\vecc(Z_\ell).
\end{equation}
Assumption~\ref{ass:tensor} implies \(\E U_{w,\ell,0}=0\) and $\E U_{w,\ell,0}^2
=w^\top A_\otimes\Sigma_fA_\otimes^\top w
+p^{-\beta}w^\top\Sigma_e^\otimes w
\le C\|w\|_2^2$. Independence across analysis tensors and \eqref{eq:aggregate-ou-coupling} imply, for \(0<t\le1\),
\begin{equation}\label{eq:aggregate-ou-second-moment}
\E\left|\sqrt{n_e}\bigl(V_{n_e,w,t}-V_{n_e,w,0}\bigr)\right|^2
\le C\|w\|_2^2\{(1-\alpha_t)^2+h_t\}
\le C\|w\|_2^2t.
\end{equation}
For almost every context vector, \eqref{eq:aggregate-ou-coupling} is a conditional coupling. Averaging its cost and applying Jensen's inequality yield
\[
\E_{X_\C^{1:n_e}}d_{\mathrm{BL}}\!\left[
\mathcal L\{\sqrt{n_e}V_{n_e,w,0}\mid X_\C^{1:n_e}\},
\mathcal L\{\sqrt{n_e}V_{n_e,w,t_0}\mid X_\C^{1:n_e}\}
\right]\le C\|w\|_2\sqrt{t_0}
\le C\|w\|_2n^{-a_n/2}.
\]
Combining this display with \eqref{eq:proof-aggregate-tv},
\(d_{\mathrm{BL}}\le2\operatorname{TV}\), and the triangle inequality proves
\eqref{eq:main-linear-functional-bound}.

For Corollary~\ref{cor:counterfactual_prediction_interval}, write the exact learned-quantile interval as
\([a(X_\C^{1:n_e}),b(X_\C^{1:n_e})]\). Its learned law has a conditional probability of \(1-\alpha\). Multiplying the statistic and both endpoints by \(\sqrt{n_e}\), \eqref{eq:proof-aggregate-tv} gives
\begin{equation*}
\E_{X_\C^{1:n_e}}\Big|
\mathbb P\!\left\{\sqrt{n_e}V_{n_e,w,t_0}\in
[\sqrt{n_e}a(X_\C^{1:n_e}),\sqrt{n_e}b(X_\C^{1:n_e})]
\mid X_\C^{1:n_e}\right\}-(1-\alpha)
\Big|\le C\sqrt{n_e}\,d_{\T,\beta}^{1/2}\mathfrak R_n.
\end{equation*}
Under \eqref{eq:aggregate-ou-coupling}, the membership of
\(\sqrt{n_e}V_{n_e,w,0}\) and \(\sqrt{n_e}V_{n_e,w,t_0}\) in the scaled interval can differ only if
\(\sqrt{n_e}|V_{n_e,w,t_0}-V_{n_e,w,0}|>\eta\), or if
\(\sqrt{n_e}V_{n_e,w,0}\) lies within \(\eta\) of an endpoint. The conditional density bound, Markov's inequality, and \eqref{eq:aggregate-ou-second-moment} give, for every \(\eta>0\),
\begin{align*}
&\E_{X_\C^{1:n_e}}\Big|
\mathbb P\!\left\{\sqrt{n_e}V_{n_e,w,0}\in
[\sqrt{n_e}a,\sqrt{n_e}b]\mid X_\C^{1:n_e}\right\}
-\mathbb P\!\left\{\sqrt{n_e}V_{n_e,w,t_0}\in
[\sqrt{n_e}a,\sqrt{n_e}b]\mid X_\C^{1:n_e}\right\}
\Big|\\
&\qquad\le
\frac{C\|w\|_2^2n^{-a_n}}{\eta^2}+4M_{n,n_e,w}\eta,
\end{align*}
where the context-dependent endpoints are suppressed in the second display. Optimizing over \(\eta\) and combining the last two displays gives the right-hand side of \eqref{eq:main-original-interval}. Finally, as $\{V_{n_e,w,0}\in[a,b]\}
=\{\sqrt{n_e}V_{n_e,w,0}\in[\sqrt{n_e}a,\sqrt{n_e}b]\}$, it proves the stated coverage result on the average-outcome scale.

For the finite-\(B\) result, fix \(X_\C^{1:n_e}=x\). Let \(\widehat F_x\) be the conditional cumulative distribution function of \(\widehat V_{n_e,w,t_0}\), let \(\widehat P_x\) be its law, and define $\widehat F_{B,x}(v)=B^{-1}\sum_{b=1}^B
\mathbf I\{\widehat V_{n_e,w,t_0}^{(b)}\le v\}$, $D_{B,x}=\sup_{v\in\R}|\widehat F_{B,x}(v)-\widehat F_x(v)|$. Conditionally on \((x,\mathcal F_{\rm tr})\), the generated values are independent draws from \(\widehat P_x\). The Dvoretzky-Kiefer-Wolfowitz-Massart inequality \citepEC{massart1990tight} gives
\[
\mathbb P\{D_{B,x}>\varepsilon\mid x,\mathcal F_{\rm tr}\}
\le2\exp(-2B\varepsilon^2),
\qquad
\E(D_{B,x}\mid x,\mathcal F_{\rm tr})
\le\sqrt{\frac{\pi}{2B}}.
\]
Let \(\widehat q_{B,x}(u)=\inf\{v:\widehat F_{B,x}(v)\ge u\}\).
Continuity of \(\widehat F_x\) and absence of ties imply $|\widehat F_x\{\widehat q_{B,x}(u)\}-u|
\le D_{B,x}+B^{-1}$, for $0<u<1$. Applying this inequality to both endpoints yields
\[
\left|
\widehat P_x\!\left\{\widehat I_{n_e,w,1-\alpha}^{(B)}\right\}
-(1-\alpha)
\right|
\le2D_{B,x}+2B^{-1},
\]
whose conditional expectation is at most \(CB^{-1/2}\). Common multiplication of the generated values and empirical endpoints by \(\sqrt{n_e}\) leaves this error unchanged.

Construct a fresh set of
\((\sqrt{n_e}V_{n_e,w,0},\sqrt{n_e}V_{n_e,w,t_0})\) whose Gaussian variables in \eqref{eq:aggregate-ou-coupling} are independent of the finite generation sample. Conditional on \(\mathcal F_B^{\mathrm{gen}}\), the random interval has fixed endpoints. After multiplying these endpoints by \(\sqrt{n_e}\), the total-variation and coupling arguments above apply pointwise. Averaging over the contexts and the finite generation sample adds \(CB^{-1/2}\) to \eqref{eq:main-original-interval}, proving \eqref{eq:finite_generation_interval}.
\endproof
$\hfill\square$

\section{Framework Details}\label{appendix_sec:Framework Details}
In this section, we report the score-network architecture and the discrete training and generation procedures used for \CFTDiff. The central construction of the candidate network class is a masked Tucker score architecture: the observed control outcomes and the noised missing region are encoded through two masked Tucker pathways, fused in a low-dimensional core, decoded back to the original matrix coordinates, and connected to a missing-region residual.

The numerical experiments use matrix-valued inputs. Although the iFlex data can be stored as a $3{,}746\times74\times24$ household-by-day-by-hour array, the trained tensor object is the household-level day-by-hour matrix
\[
X_0^{(i)}=\bigl[Y_{ith}(0)\bigr]_{t\in[74],\,h\in[24]}\in\mathbb R^{74\times24}.
\]
Thus, the empirical implementation takes $D=2$, with $(p_1,p_2)=(74,24)$, while the household index $i$ identifies training samples or analysis samples. Matrices from the control-group households form the training sample. For each group-specific intervention calendar, its fixed $74\times24$ mask is applied synthetically to these training matrices during training and to the corresponding treatment-group matrices during conditional generation. Accordingly, the implementation description and architecture figure are written for the order-two specialization $X\in\mathbb R^{H\times W}$; batch and feature-channel indices are suppressed when they are not material.

\subsection{Architecture}\label{appendix_subsec:Architecture}

Let $M_{\mathcal C}$ denote the external binary mask received by the code, with one on the observed control outcomes, and set $M_{\mathcal T}=\mathbf 1-M_{\mathcal C}$. Thus, $M_{\mathcal T}$ is one on the treated region. At diffusion step $j$, the score network receives
\[
    x_j=M_{\mathcal C}\odot X_{\mathcal C}
        +M_{\mathcal T}\odot X_{\mathcal T,j},
\]
together with $M_{\mathcal C}$ and the time label $j$. The first term is kept fixed, and the second term is the noised missing region. Figure~\ref{fig:ec2-masked-tucker-architecture} displays the score architecture used in the experiments, which consists of two masked Tucker encoders, a low-dimensional nonlinear core, a Tucker decoder, and a missing-region shortcut connection. Note that the two encoders are not duplicate input channels, as they carry complementary information from the noised missing region and the observed control outcomes, respectively.

\begin{figure}[htb]
\FIGURE
{\includegraphics[width=0.98\textwidth]{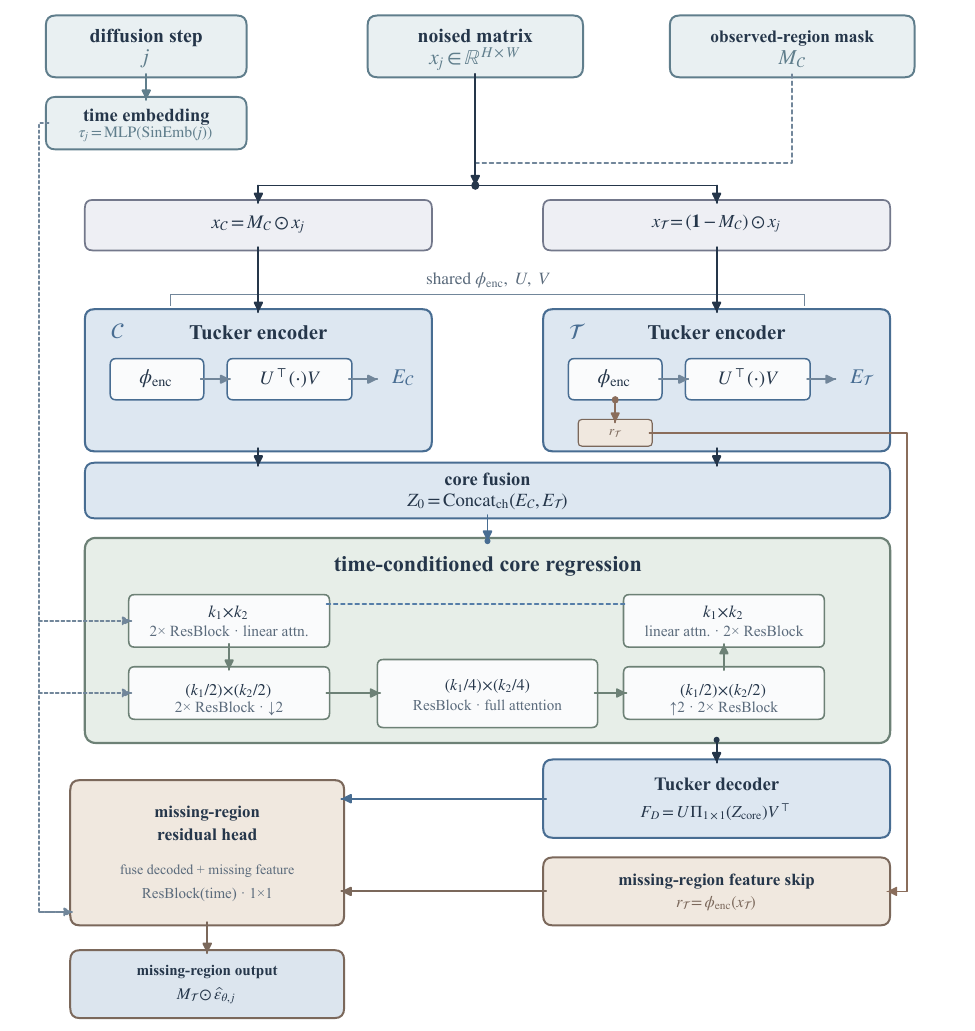}}
{CFT-Diff masked Tucker score architecture.
\label{fig:ec2-masked-tucker-architecture}}
{The masked matrix input is split into the fixed observed control outcomes and a noised missing region. Each part is mapped by the same Tucker encoder $E_{\mathcal S}=U^\top\phi_{\text{enc}}(M_{\mathcal S}\odot x_j)V$, $\mathcal S\in\{\mathcal C,\mathcal T\}$, where $\phi_{\text{enc}}$ is an application-specific full-resolution feature stem. The two core maps are concatenated along the channel dimension and processed by a time-conditioned core regression on the low-dimensional Tucker core; the displayed multi-resolution residual and attention stages are a representative realization of this map. A $1\times1$ projection followed by the Tucker decoder combines the core output with the shortcut $r_{\mathcal T}=\phi_{\text{enc}}(M_{\mathcal T}\odot x_j)$. The network outputs the learned noise prediction in the missing region. The observed control outcomes remain fixed during the conditional reverse update. The brighter blue dashed link denotes a representative skip; slate dashed links carry time and mask conditioning; and the right-side rail denotes the missing-region feature shortcut.}
\end{figure}

For the order-two implementation, $\phi_{\text{enc}}$ is a shared full-resolution feature stem, and $U\in\mathbb R^{H\times k_1}$ and $V\in\mathbb R^{W\times k_2}$ are shared, trainable Tucker bases. In our experiments, $\phi_{\text{enc}}$ is instantiated as a zero-padded $7\times7$ convolution. The two core tensors are concatenated in the channel direction as $[E_{\mathcal C};E_{\mathcal T}]_{\rm ch}$. A sinusoidal time embedding followed by an MLP modulates a conventional two-dimensional ResNet/attention U-Net in the $k_1\times k_2$ core space. Its output is reduced by a $1\times1$ channel projection and decoded by $U(\cdot)V^\top$ to the original matrix resolution. The decoded feature is concatenated with the missing-region feature $\phi_{\text{enc}}(M_{\mathcal T}\odot x_j)$ and passed, together with the time embedding, through a residual output block and a $1\times1$ head. In the matrix-factor experiments, $U,V$ are warm started with loading-space estimates and remain trainable.

\subsection{Core-Net Function Class}\label{appendix_subsec:core-network-class}

The Core-Net below is the theoretical sieve used to control approximation and stochastic complexity. All full-dimensional masking, encoding, decoding, and residual operations are specified separately; the low-dimensional network learns only the map from the $r$-dimensional core statistic and diffusion time to the core output. This makes the theoretical function class follow the same Tucker encoder-decoder organization displayed in Figure~\ref{fig:ec2-masked-tucker-architecture}.

\begin{definition}[Core-Net function class]
The class $\mathcal F_{\core}(L,m,J,K,\kappa,\gamma_g,\gamma_t)$ consists of feed-forward ReLU maps
\[
    \zeta_\theta:
    \mathbb R^r\times[t_0,T]
    \to
    \mathbb R^r
\]
with an affine output layer, depth at most $L$, hidden-layer width at most $m$, at most $J$ nonzero weights and biases, and parameter magnitudes bounded by $\kappa$. Each map satisfies
\[
    \sup_{(g,t)\in\mathbb R^r\times[t_0,T]}
    \|\zeta_\theta(g,t)\|_2
    \leq K,
\]
\[
    \|\zeta_\theta(g,t)-\zeta_\theta(g',t)\|_2
    \leq
    \gamma_g\|g-g'\|_2,
    \qquad
    g,g'\in\mathbb R^r,\quad t\in[t_0,T],
\]
and
\[
    \|\zeta_\theta(g,t)-\zeta_\theta(g,t')\|_2
    \leq
    \gamma_t|t-t'|,
    \qquad
    g\in\mathbb R^r,\quad t,t'\in[t_0,T].
\]
For a tensor-valued Core-Net $\mathcal Z_\theta$, the identification is
\[
    \zeta_\theta(g,t)
    =
    \vecc\left[
        \mathcal Z_\theta\{
            \operatorname{Tucker}(g),t
        \}
    \right].
\]
\end{definition}

The boundedness and Lipschitz restrictions in this definition are the regularity conditions needed for the low-dimensional approximation and covering arguments. In the experiments, the corresponding low-dimensional Core-Net is implemented using the time-conditioned U-Net displayed in Figure~\ref{fig:ec2-masked-tucker-architecture}, with residual blocks, SiLU/GELU activations, normalization, and linear or full attention. In both cases, nonlinear computation is confined to the Tucker core, while dual masking, Tucker encoding and decoding, and the missing-region shortcut remain explicit structural components.

\subsection{Masked DDPM Training}
For a given CFT-Diff checkpoint, training uses training tensors and a fixed mask. For each minibatch, we draw $j\sim\operatorname{Unif}\{0,\ldots,N_{\mathrm{step}}-1\}$ and $\epsilon\sim\mathcal N(0,I)$, and construct the masked discrete forward input
\[
    x_j
    =M_{\mathcal C}\odot X_0+
    M_{\mathcal T}\odot
    \left(\sqrt{\bar\alpha_j}X_0+
    \sqrt{1-\bar\alpha_j}\epsilon\right).
\]
Here, only the missing region is noised, while the observed control outcomes remain fixed. The model uses the DDPM noise-prediction parameterization \citepEC{ho2020denoising} and minimizes the masked noise-prediction loss
\[
    \widehat{\mathcal L}
    =\frac{1}{B}\sum_{i=1}^{B}
    \frac{\left\|M_{\mathcal T}\odot
    \left\{\epsilon_i-\widehat\epsilon_\theta(x_{j_i}^{(i)},j_i;M_{\mathcal C})\right\}\right\|_F^2}
    {\|M_{\mathcal T}\|_0}.
\]
The external mask is synthetically applied to every complete training tensor, and the loss is evaluated only over its missing outcomes.

\paragraph{Simulation configuration.}
For the matrix-factor simulations, the input has one channel of shape $64\times64$. The Tucker core dimension is $(k_1,k_2)=(16,16)$, and the core U-Net has resolutions $16^2$, $8^2$, and $4^2$ with widths $64$, $128$, and $256$, respectively. Each resolution uses two time-conditioned residual blocks and attention; the bottleneck uses full attention. We use a cosine beta schedule with $N_{\mathrm{step}}=200$ diffusion steps, $\epsilon$-prediction, and 300 training epochs. Optimization uses AdamW with a learning rate of $10^{-4}$, weight decay of $0.01$, a batch size of 60, gradient accumulation of 2, a 10-epoch warm-up followed by cosine scheduling, exponential moving average decay of $0.999$, and automatic mixed precision. These numerical settings describe the reported matrix-factor simulation configuration; the architectural mask convention and the conditional objective remain unchanged when the matrix size or Tucker ranks are adapted to another design.

\paragraph{Empirical demand-data configuration.}
For the iFlex experiment, each empirical input is a one-channel $74\times24$ matrix from a single household. We use $(k_1,k_2)=(16,16)$, so the nonlinear core is evaluated on $16\times16$, $8\times8$, and $4\times4$ Tucker grids with a base width of 64. Similar to the simulation setting, each represented level contains two time-conditioned residual blocks. For each mask-specific fit, the day and hour loading spaces are initialized from the leading empirical eigenvectors computed from the complete household training matrices. The loading matrices $U$ and $V$ are warm starts and remain trainable with a learning rate of $10^{-5}$, while the other AdamW parameters use a learning rate of $10^{-4}$. We train for 300 epochs with a batch size of 64, gradient accumulation of 2, weight decay of $0.01$, and automatic mixed precision. The learning-rate schedule uses a linear 10-epoch warm-up followed by cosine annealing with $T_{\max}=40$ and $\eta_{\min}=10^{-6}$. The empirical runs use the same 200-step cosine DDPM noise-prediction objective as above.

\subsection{Conditional Generation}

The reported experiments use DDIM generation \citepEC{song2020denoising} with $N_{\mathrm{step}}=200$ reverse steps. Thus, each reverse trajectory starts from independent Gaussian noise in the missing region and produces one sample from the trained conditional generator. At every reverse step, we update only the missing region while keeping the observed control outcomes fixed. This projection makes the generated tensor agree exactly with the observed control outcomes throughout the trajectory. The target $\widehat X_0$ inferred from the predicted noise is clipped before its DDIM update. We experimented with clipping widths from 1.5 to 2.5 (with 2.0 as the default).

 \begin{algorithm}[htb]
\caption{Conditional DDIM generation for \CFTDiff}\label{alg:ec2-conditional-ddim}
\small
\begin{algorithmic}[1]
\Require observed control outcomes $X_{\mathcal C}$, observed mask $M_{\mathcal C}$, $M_{\mathcal T}=\mathbf1-M_{\mathcal C}$, noise predictor $\widehat\epsilon_\theta$, and cumulative schedule $\{\bar\alpha_j\}_{j=0}^{N_{\mathrm{step}}-1}$
\Ensure conditional completions $\{\widetilde X_0^{(b)}\}_{b=1}^{B}$
\For{$b=1,\ldots,B$}
    \State draw $z^{(b)}\sim\mathcal N(0,I)$ and set $x_{N_{\mathrm{step}}-1}\gets M_{\mathcal C}\odot X_{\mathcal C}+M_{\mathcal T}\odot z^{(b)}$
    \For{$j=N_{\mathrm{step}}-1,N_{\mathrm{step}}-2,\ldots,1$}
        \State $\widehat\epsilon\gets\widehat\epsilon_\theta(x_j,j;M_{\mathcal C})$
        \State $\widehat X_0\gets\operatorname{clip}_{[-x_{\max},x_{\max}]}\!\left\{\bigl(x_j-\sqrt{1-\bar\alpha_j}\widehat\epsilon\bigr)/\sqrt{\bar\alpha_j}\right\}$
        \State $x_{j-1}^{\mathcal T}\gets\sqrt{\bar\alpha_{j-1}}\widehat X_0+\sqrt{1-\bar\alpha_{j-1}}\widehat\epsilon$
        \State $x_{j-1}\gets M_{\mathcal C}\odot X_{\mathcal C}+M_{\mathcal T}\odot x_{j-1}^{\mathcal T}$ \Comment{hard conditional projection}    \EndFor
    \State $\widetilde X_0^{(b)}\gets  x_0$
\EndFor
\end{algorithmic}
\end{algorithm}

For the reported simulations, we generate $B=100$ conditional completions for each test matrix. The resulting empirical distribution supports both point estimates obtained by averaging draws and the distributional measures and interval summaries reported in the paper.

Meanwhile, for each held-out household in the iFlex recovery experiments, we supply the sampler only with the observed control outcomes $M_{\mathcal C}\odot X_{\rm obs}$. We generate $B=100$ completions with a seed matched to the corresponding training split, a 200-step DDIM, and clipping of the inferred $\widehat X_0$ to $[-2,2]$. Repeated household inputs receive independent initial Gaussian states in the missing region. Point-recovery estimates average the 100 draws, whereas the unreduced draw set is retained for distributional evaluation, including the joint energy score.

The recovery masks are constructed at the household--day level and then expanded over all 24 hourly entries of the selected day. Under simultaneous adoption, every held-out household uses the same bottom-suffix mask, with 5, 11, 16, 21, 27, 32, 37, 43, or 48 missing working days. Under staggered adoption, households are partitioned into balanced cohorts with different missing-suffix lengths whose weighted mean equals the selected target ratio; each cohort is sampled with the same-seed CFT-Diff checkpoint matched to that suffix mask. Under switchback, the held-out households are deterministically allocated, using the fixed mask seed, to three nearly balanced groups. A policy calendar, rather than a held-out outcome, marks a day as hidden when that policy occurs for at least one source household during at least one hour; each group receives a distinct union of these calendars, expanded over all 24 hourly entries. CFT-Diff is trained once for each resulting $74\times24$ calendar template, using the same split-specific complete training data, and reused for every household assigned to that template. Thus, a heterogeneous empirical panel is implemented as a collection of fixed-mask conditional generation problems while preserving the observed-entry pattern of every individual household.

\section{Simulation Details and Additional Results}\label{app:Additional Simulations}

This appendix provides the full simulation design, defines the evaluation measures, and reports additional results across factor strengths and numbers of latent factors.

\subsection{Detailed Simulation Design}\label{app:simulation_design}

\paragraph{Data-generating process.}
For each replication, we generate 372 independent $64\times64$ matrices according to
\begin{equation}
X_{\mathrm{raw}}^{(\ell)}=A_1F^{(\ell)}A_2^\top+E^{(\ell)},\qquad \ell=1,\ldots,372,
\label{eq:simulation_dgp}
\end{equation}
where $X_{\mathrm{raw}}^{(\ell)} =p^{\beta/2}X$, $F^{(\ell)}\in\mathbb R^{16\times16}$, and $E^{(\ell)}\in\mathbb R^{64\times64}$ are mutually independent across $\ell$, with $\vecc(F^{(\ell)})\sim\mathcal N(0,4I_{256})$ and $\vecc(E^{(\ell)})\sim\mathcal N(0,I_{4096})$. To construct the loading matrices, we obtain $Q_1,Q_2\in\mathbb R^{64\times16}$ by orthonormalizing two independent standard Gaussian matrices and set $A_d=64^{\beta/2}Q_d$ for $d=1,2$, where $\beta\in\{0.50,0.75\}$. Thus, increasing $\beta$ strengthens the common factor component while preserving the loading directions.

The first 360 matrices constitute the training sample, and the remaining 12 are used for evaluation. To place all methods on a common numerical scale, let $c_{\mathrm{sim}}=\max_{1\leq\ell\leq360}\|X_{\mathrm{raw}}^{(\ell)}\|_\infty$ and define $X^{(\ell)}=X_{\mathrm{raw}}^{(\ell)}/c_{\mathrm{sim}}$. We repeat the entire experiment five times using independently generated loading spaces, latent factors, and idiosyncratic errors. Within each replication, the same training and evaluation matrices are used across methods and missing rates.

\paragraph{Missing Regions.}
Each evaluation matrix is partitioned into a rectangular missing region and its observed control outcomes. The target missing rates are $\rho\in\{0.25,0.50,0.75\}$. Using one-based indexing, the 25\% design removes rows 17--48 and columns 33--64, yielding a $32\times32$ block. The 50\% design removes rows 12--52 and columns 15--64, yielding a $41\times50$ block and an actual missing rate of 0.5005. The 75\% design removes rows 6--58 and columns 7--64, yielding a $53\times58$ block and an actual missing rate of 0.7505. In each design, the block is vertically centered and reaches the final matrix column. The same block definition is used for every evaluation matrix within a design, and accuracy is evaluated only on its missing region. During \CFTDiff\ training, the corresponding mask is applied synthetically to each complete training matrix.

Because \eqref{eq:simulation_dgp} is Gaussian, the exact conditional distribution used in the pointwise diagnostics is available analytically. Let $x_{\mathcal T}$ and $x_{\mathcal C}$ denote the vectorized missing region and observed control outcomes, respectively, and partition the covariance matrix of $\vecc(X^{(\ell)})$ conformably. Then
\begin{equation}
x_{\mathcal T}\mid x_{\mathcal C}
\sim
\mathcal N\!\left(
\Sigma_{\mathcal T\mathcal C}\Sigma_{\mathcal C\mathcal C}^{-1}x_{\mathcal C},
\Sigma_{\mathcal T\mathcal T}-
\Sigma_{\mathcal T\mathcal C}\Sigma_{\mathcal C\mathcal C}^{-1}\Sigma_{\mathcal C\mathcal T}
\right).
\label{eq:simulation_exact_conditional_law}
\end{equation}

\paragraph{Methods and implementation.}
The comparison includes the five point estimators and three diffusion methods described in Section~\ref{subsec:simulation_setting}. For the diffusion methods, the generated mean is used as the point estimate. \ConvDiff\ and \TuckerDiff\ use one trained model within each factor-strength setting and replication across the three missing rates. Because \CFTDiff\ incorporates the fixed mask in its conditional score architecture, it is trained for each missing-rate design. Each diffusion model is trained for 300 epochs with batch size 60, gradient accumulation over two batches, and 200 diffusion steps. The aggregate evaluation uses 100 conditional draws for each of the 12 evaluation matrices in every replication. The architectures and discrete training and generation procedures are described in Appendix~\ref{appendix_sec:Framework Details}.

\paragraph{Evaluation measures.}
For a given replication and missing rate, let $\mathcal J$ index the missing entries pooled across the 12 evaluation matrices, let $N_{\mathcal J}=|\mathcal J|$, and write $y_i$ for the realized value at entry $i\in\mathcal J$. For a diffusion method, let $x_i^{(1)},\ldots,x_i^{(B)}$ be its $B=100$ conditional draws and set $\widehat y_i=B^{-1}\sum_{b=1}^B x_i^{(b)}$. The point-recovery measures are
\begin{equation*}
\operatorname{MAE}=\frac{1}{N_{\mathcal J}}\sum_{i\in\mathcal J}|\widehat y_i-y_i|,
\qquad
\operatorname{RMSE}=\left\{\frac{1}{N_{\mathcal J}}\sum_{i\in\mathcal J}(\widehat y_i-y_i)^2\right\}^{1/2}.
\end{equation*}

The entrywise continuous ranked probability score is estimated by
\begin{equation}
\operatorname{CRPS}_i
=\frac{1}{B}\sum_{b=1}^B|x_i^{(b)}-y_i|
-\frac{1}{2B^2}\sum_{b=1}^B\sum_{b'=1}^B|x_i^{(b)}-x_i^{(b')}|,
\label{eq:simulation_crps}
\end{equation}
and the reported CRPS averages \eqref{eq:simulation_crps} over $i\in\mathcal J$. The energy score evaluates the missing region jointly. For a realized missing-region vector $y$ and generated vectors $x^{(1)},\ldots,x^{(B)}$, it is estimated by
\begin{equation*}
\operatorname{Energy}
=\frac{1}{B}\sum_{b=1}^B\|x^{(b)}-y\|_2
-\frac{1}{2}\,\widehat{\mathbb E}\|x-x'\|_2,
\end{equation*}
where the second expectation uses 200 fixed pairs of generated draws. Both scores are negatively oriented.

Let $L_i$ and $U_i$ be the 5th and 95th empirical percentiles of the draws at entry $i$. The reported 90\% coverage and width are $N_{\mathcal J}^{-1}\sum_{i\in\mathcal J}\mathbb I\{L_i\leq y_i\leq U_i\}$ and $N_{\mathcal J}^{-1}\sum_{i\in\mathcal J}(U_i-L_i)$, respectively. With $\delta=0.10$, the corresponding interval score is
\begin{equation*}
\operatorname{IS}_{\delta,i}
=(U_i-L_i)
+\frac{2}{\delta}(L_i-y_i)\mathbb I\{y_i<L_i\}
+\frac{2}{\delta}(y_i-U_i)\mathbb I\{y_i>U_i\}.
\end{equation*}
The weighted interval score combines the predictive median with central intervals at the 50\%, 60\%, 70\%, 80\%, 90\%, and 95\% levels, using the standard weights $\delta_k/2$, where $\delta_k$ is one minus the corresponding coverage level. IS and WIS reward calibration and sharpness jointly; smaller values are preferred. Each measure is first pooled over the 12 evaluation matrices within a replication. The tables report the mean and standard deviation across the five replications.

\paragraph{Pointwise diagnostic.}
The pointwise figures use the 25\% missing region and 10,000 conditional draws from each diffusion method for the first evaluation matrix in one replication. Within the missing region, we compute the empirical percentile of each realized entry and consider entries within 0.03 of 0.25, 0.50, or 0.75. For displayed entry $i$ and method $m$, let $\mu_i$ and $\sigma_i$ denote the exact conditional mean and standard deviation from \eqref{eq:simulation_exact_conditional_law}, and let $\widehat\mu_{mi}$ and $\widehat\sigma_{mi}$ denote the corresponding moments of the generated draws. We summarize pointwise agreement by
\begin{equation}
D_{mi}
=
\left[
\left(\frac{\widehat\mu_{mi}-\mu_i}{\sigma_i}\right)^2
+
\left\{\log\left(\frac{\widehat\sigma_{mi}}{\sigma_i}\right)\right\}^2
\right]^{1/2}.
\label{eq:simulation_pointwise_discrepancy}
\end{equation}
The display is designed to make differences among the conditional generators visually informative. We therefore select, within each percentile neighborhood, an entry for which the \CFTDiff\ discrepancy is no larger than those of the two nested diffusion models; among the remaining candidates, selection balances the \CFTDiff\ discrepancy, its improvement over the benchmarks, and proximity to the target percentile. The aggregate tables, rather than these selected entries, provide the full-sample comparison.

\subsection{Additional Recovery Results}\label{app:simulation_additional_results}

Table~\ref{tab:recovery_beta075} repeats the aggregate comparison under $\beta=0.75$. The stronger common factor improves recovery for methods that exploit low-rank structure, while the ranking among the diffusion methods remains unchanged. \CFTDiff\ has the lowest MAE, RMSE, Energy, CRPS, IS, and WIS among the three conditional generators at every missing rate. Relative to \TuckerDiff, it lowers Energy and CRPS by 14\%--42\% and produces intervals that are 17\%--45\% narrower. These results show that the gains from masked conditional generation persist as the factor component becomes stronger.

\begin{table}[htb]
\TABLE
{Counterfactual recovery across missing rates ($\beta=0.75$).
\label{tab:recovery_beta075}}
{
\scriptsize
\setlength{\tabcolsep}{1.25pt}
\resizebox{\textwidth}{!}{%
\begin{tabular}{@{}lcccccccc@{}}
\toprule
& \multicolumn{2}{c}{Point Recovery} & \multicolumn{2}{c}{Distributional Recovery} & \multicolumn{4}{c}{Interval Recovery} \\[-1pt]
\cmidrule(lr){2-3}\cmidrule(lr){4-5}\cmidrule(lr){6-9}
Method & MAE & RMSE & Energy & CRPS & Coverage & Width & IS & WIS \\
\midrule
\multicolumn{9}{c}{\textit{Panel A: Missing Rate $=25\%$}} \\
\midrule
\NMC & 0.0208 (0.0012) & 0.0261 (0.0016) & -- & -- & -- & -- & -- & -- \\
\DID & 0.1682 (0.0109) & 0.2165 (0.0142) & -- & -- & -- & -- & -- & -- \\
\SC & 0.0205 (0.0011) & 0.0259 (0.0014) & -- & -- & -- & -- & -- & -- \\
\SDID & 0.1164 (0.0070) & 0.1495 (0.0087) & -- & -- & -- & -- & -- & -- \\
\TFI & 0.0513 (0.0030) & 0.0656 (0.0040) & -- & -- & -- & -- & -- & -- \\
\ConvDiff & 0.1182 (0.0068) & 0.1518 (0.0085) & 3.4132 (0.1944) & 0.0844 (0.0048) & 0.8738 (0.0120) & 0.4610 (0.0167) & 0.6418 (0.0367) & 0.0635 (0.0037) \\
\TuckerDiff & 0.0360 (0.0016) & 0.0473 (0.0017) & 1.1161 (0.0347) & 0.0254 (0.0010) & 0.9728 (0.0104) & 0.2153 (0.0105) & 0.2257 (0.0073) & 0.0199 (0.0007) \\
\CFTDiff & 0.0211 (0.0011) & 0.0275 (0.0014) & 0.6560 (0.0316) & 0.0148 (0.0008) & 0.9567 (0.0044) & 0.1195 (0.0060) & 0.1303 (0.0065) & 0.0115 (0.0006) \\[2pt]
\midrule
\multicolumn{9}{c}{\textit{Panel B: Missing Rate $=50\%$}} \\
\midrule
\NMC & 0.0235 (0.0016) & 0.0297 (0.0020) & -- & -- & -- & -- & -- & -- \\
\DID & 0.1704 (0.0124) & 0.2189 (0.0160) & -- & -- & -- & -- & -- & -- \\
\SC & 0.0245 (0.0021) & 0.0313 (0.0029) & -- & -- & -- & -- & -- & -- \\
\SDID & 0.1180 (0.0079) & 0.1516 (0.0099) & -- & -- & -- & -- & -- & -- \\
\TFI & 0.0705 (0.0050) & 0.0906 (0.0061) & -- & -- & -- & -- & -- & -- \\
\ConvDiff & 0.1197 (0.0079) & 0.1538 (0.0099) & 4.8924 (0.3178) & 0.0855 (0.0057) & 0.8739 (0.0117) & 0.4670 (0.0203) & 0.6510 (0.0435) & 0.0644 (0.0043) \\
\TuckerDiff & 0.0600 (0.0034) & 0.0784 (0.0043) & 2.5256 (0.1340) & 0.0428 (0.0024) & 0.9526 (0.0115) & 0.3082 (0.0180) & 0.3378 (0.0173) & 0.0325 (0.0018) \\
\CFTDiff & 0.0375 (0.0025) & 0.0498 (0.0033) & 1.5945 (0.1057) & 0.0265 (0.0018) & 0.9321 (0.0193) & 0.1792 (0.0154) & 0.2104 (0.0136) & 0.0201 (0.0014) \\[2pt]
\midrule
\multicolumn{9}{c}{\textit{Panel C: Missing Rate $=75\%$}} \\
\midrule
\NMC & 0.0797 (0.0051) & 0.1051 (0.0066) & -- & -- & -- & -- & -- & -- \\
\DID & 0.1685 (0.0108) & 0.2169 (0.0141) & -- & -- & -- & -- & -- & -- \\
\SC & 0.0881 (0.0042) & 0.1173 (0.0058) & -- & -- & -- & -- & -- & -- \\
\SDID & 0.1176 (0.0073) & 0.1514 (0.0091) & -- & -- & -- & -- & -- & -- \\
\TFI & 0.0978 (0.0056) & 0.1260 (0.0070) & -- & -- & -- & -- & -- & -- \\
\ConvDiff & 0.1188 (0.0074) & 0.1529 (0.0091) & 5.9528 (0.3577) & 0.0849 (0.0052) & 0.8779 (0.0135) & 0.4692 (0.0189) & 0.6468 (0.0393) & 0.0639 (0.0039) \\
\TuckerDiff & 0.0884 (0.0056) & 0.1147 (0.0071) & 4.4728 (0.2710) & 0.0630 (0.0039) & 0.9268 (0.0117) & 0.4031 (0.0187) & 0.4717 (0.0270) & 0.0472 (0.0029) \\
\CFTDiff & 0.0761 (0.0056) & 0.0994 (0.0072) & 3.8687 (0.2775) & 0.0540 (0.0040) & 0.9159 (0.0145) & 0.3339 (0.0188) & 0.4032 (0.0297) & 0.0405 (0.0030) \\
\bottomrule
\end{tabular}%
}
}
{The table reports recovery at three missing rates. Entries are means across five independent simulation replications, with standard deviations in parentheses. MAE and RMSE evaluate point recovery; Energy and CRPS evaluate the generated conditional distribution. Coverage and Width denote the empirical coverage and average width of nominal 90\% intervals, and IS and WIS denote the interval score and weighted interval score. Smaller MAE, RMSE, Energy, CRPS, IS, and WIS indicate better performance. Width is interpreted jointly with Coverage, whose nominal level is 0.90. Dashes indicate measures that are not applicable to point estimators. Each diffusion method uses 100 conditional draws per evaluation matrix.}
\end{table}

\subsection{Additional Pointwise Distribution Diagnostics}\label{app:simulation_pointwise}

Figure~\ref{fig:pointwise_conditional_distribution_beta075} gives the pointwise diagnostic for $\beta=0.75$ in the baseline design. Figures~\ref{fig:pointwise_conditional_distribution_beta050_4f}--\ref{fig:pointwise_conditional_distribution_beta075_8f} extend the comparison to designs with four and eight factors. Across these additional designs, \CFTDiff\ continues to track both the shape and mean of the exact conditional distribution at all three selected entries. \ConvDiff\ remains substantially overdispersed and exhibits clear mean deviations in several tail cases, while \TuckerDiff\ reduces these discrepancies but still produces excess dispersion, particularly with eight factors. The relative performance of the three methods therefore remains stable as the number of factors increases.

\begin{figure}[htb]
\FIGURE
{\includegraphics[width=\textwidth]{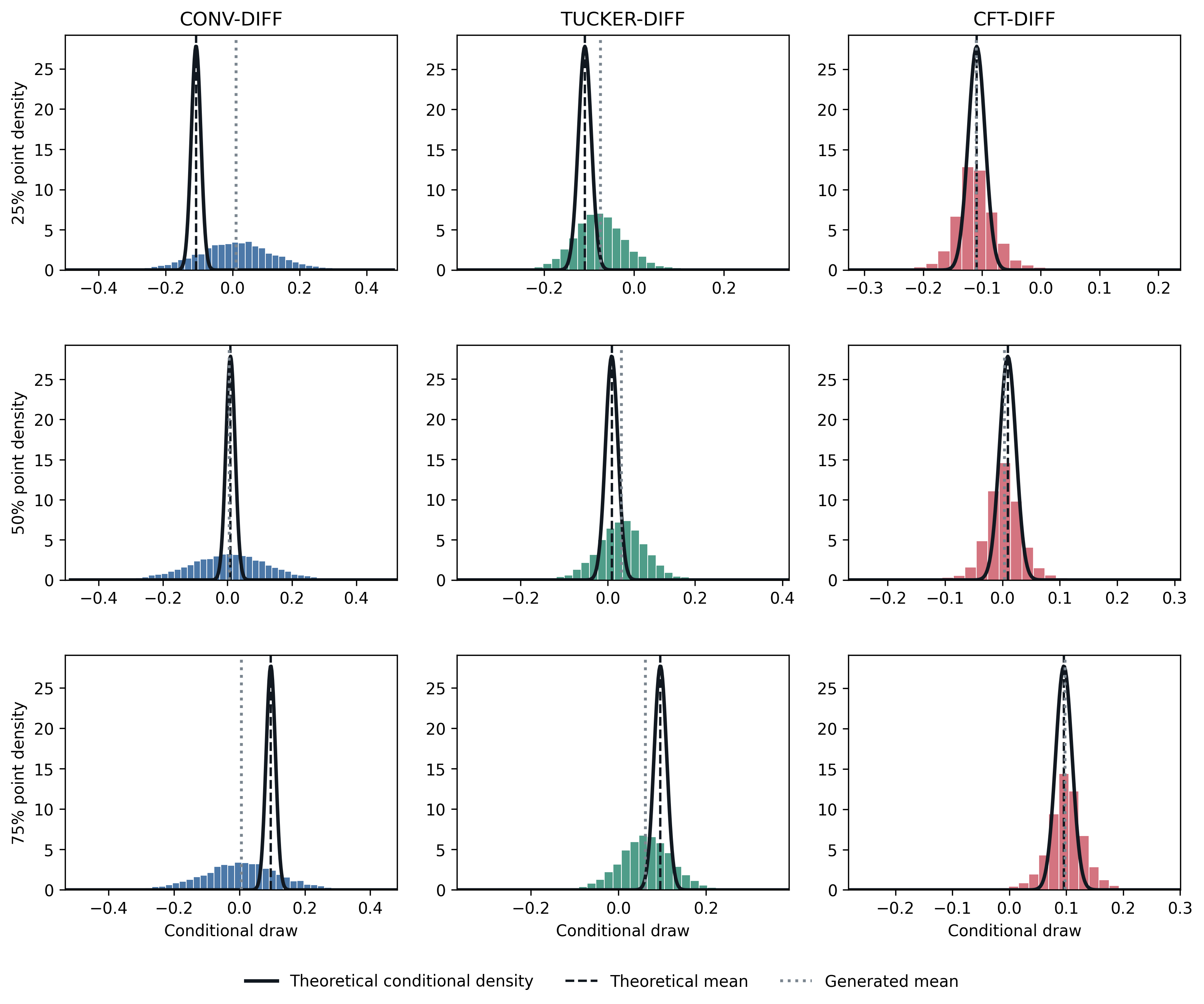}}
{Pointwise conditional distribution recovery ($\beta=0.75$).
\label{fig:pointwise_conditional_distribution_beta075}}
{Columns correspond to \ConvDiff, \TuckerDiff, and \CFTDiff.
Rows display entries whose realized missing values are near the 25th,
50th, and 75th percentiles within the missing region.
Each histogram is based on 10,000 conditional draws, and the missing
rate is 25\%. The solid curve shows the exact conditional density;
the dashed and dotted vertical lines indicate the exact conditional mean and the mean of the generated draws, respectively.}
\end{figure}

\begin{figure}[htb]
\FIGURE
{\includegraphics[width=\textwidth]{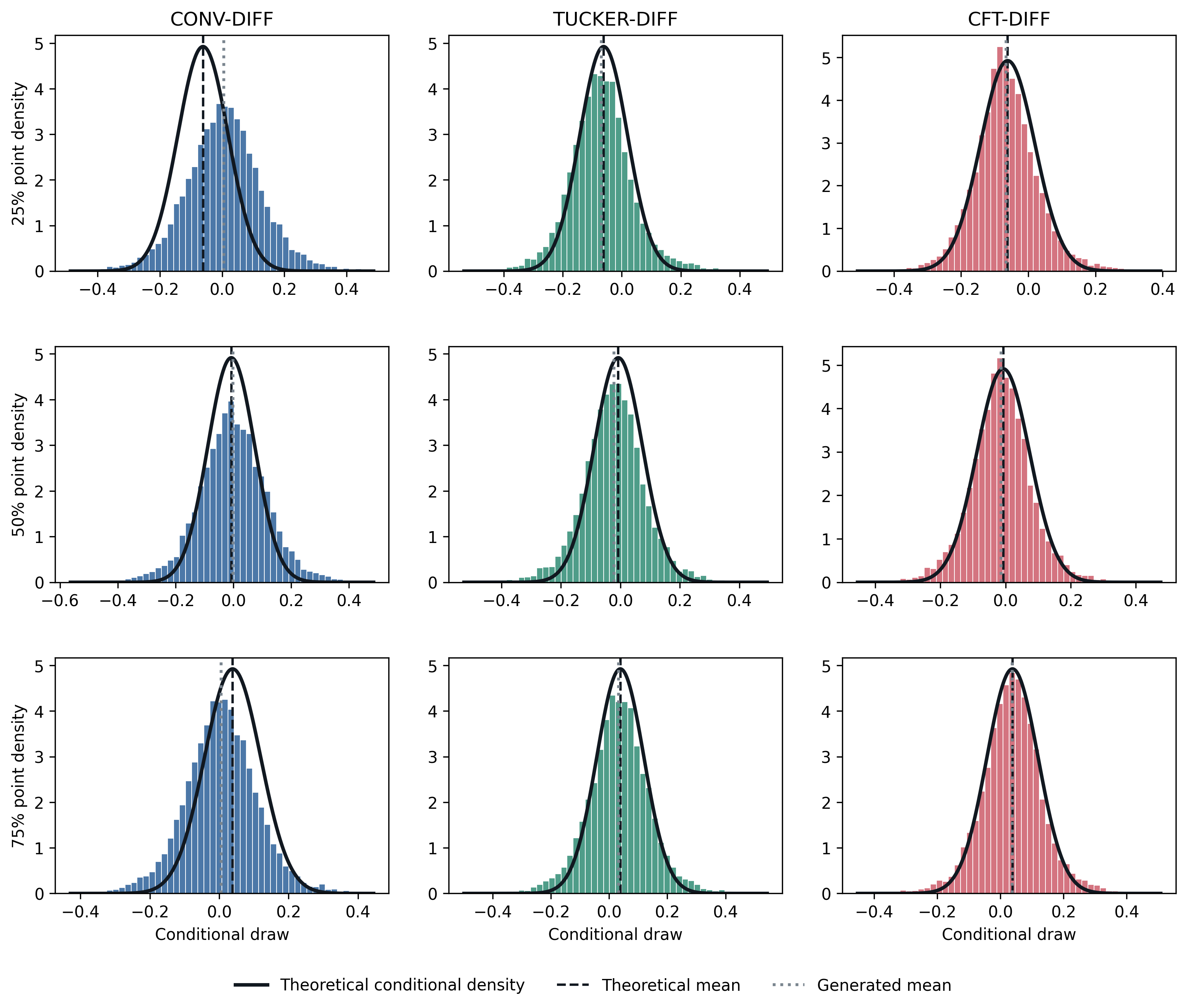}}
{Pointwise conditional distribution recovery with four factors ($\beta=0.50$).
\label{fig:pointwise_conditional_distribution_beta050_4f}}
{Columns correspond to \ConvDiff, \TuckerDiff, and \CFTDiff. The design contains four latent factors. Rows display entries whose realized missing values are near the 25th, 50th, and 75th percentiles within the missing region. Each histogram is based on 10,000 conditional draws, and the missing rate is 25\%. The solid curve shows the exact conditional density; the dashed and dotted vertical lines indicate the exact conditional mean and the mean of the generated draws, respectively.}
\end{figure}

\begin{figure}[htb]
\FIGURE
{\includegraphics[width=\textwidth]{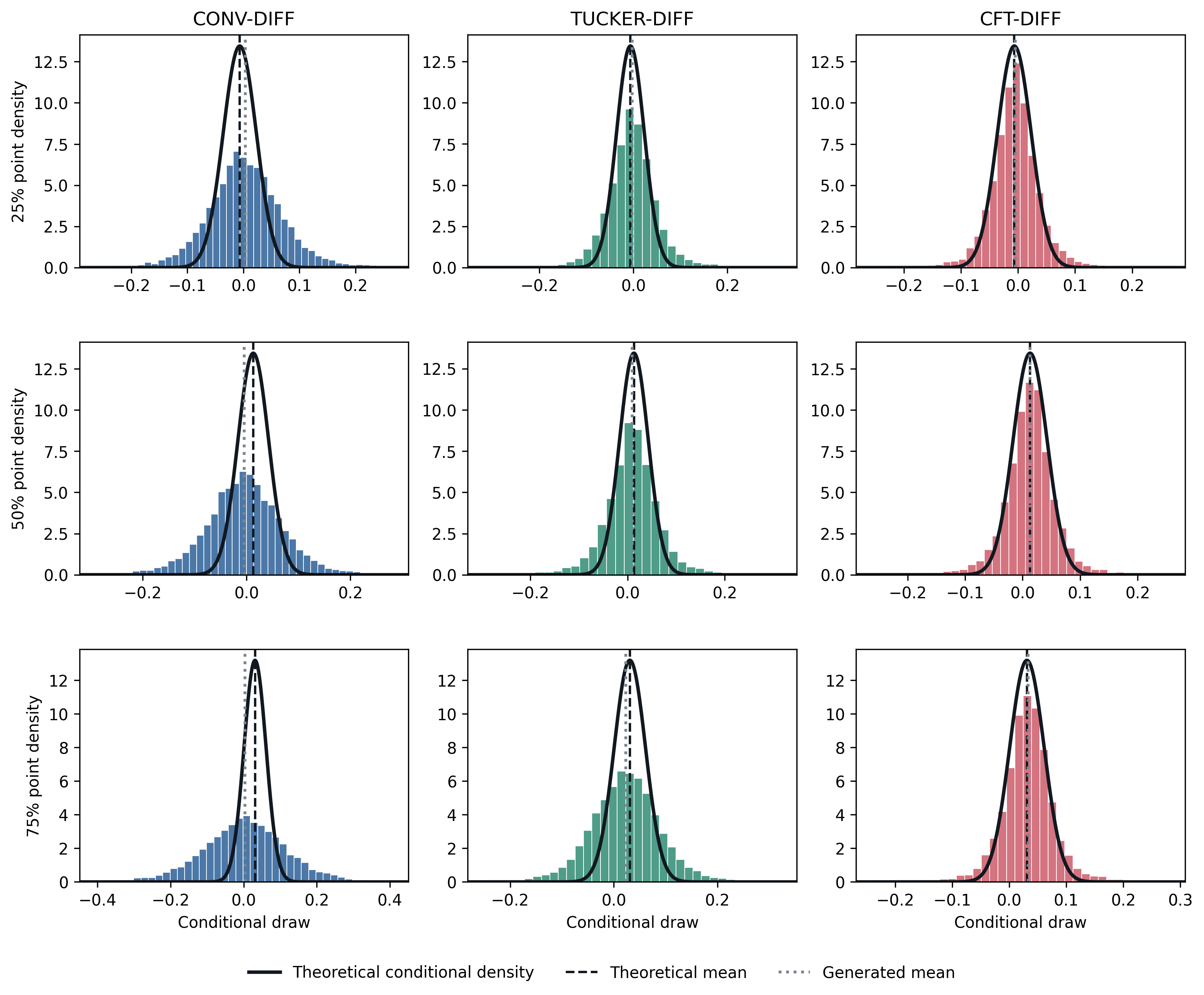}}
{Pointwise conditional distribution recovery with four factors ($\beta=0.75$).
\label{fig:pointwise_conditional_distribution_beta075_4f}}
{Columns correspond to \ConvDiff, \TuckerDiff, and \CFTDiff. The design contains four latent factors. Rows display entries whose realized missing values are near the 25th, 50th, and 75th percentiles within the missing region. Each histogram is based on 10,000 conditional draws, and the missing rate is 25\%. The solid curve shows the exact conditional density; the dashed and dotted vertical lines indicate the exact conditional mean and the mean of the generated draws, respectively.}
\end{figure}

\begin{figure}[htb]
\FIGURE
{\includegraphics[width=\textwidth]{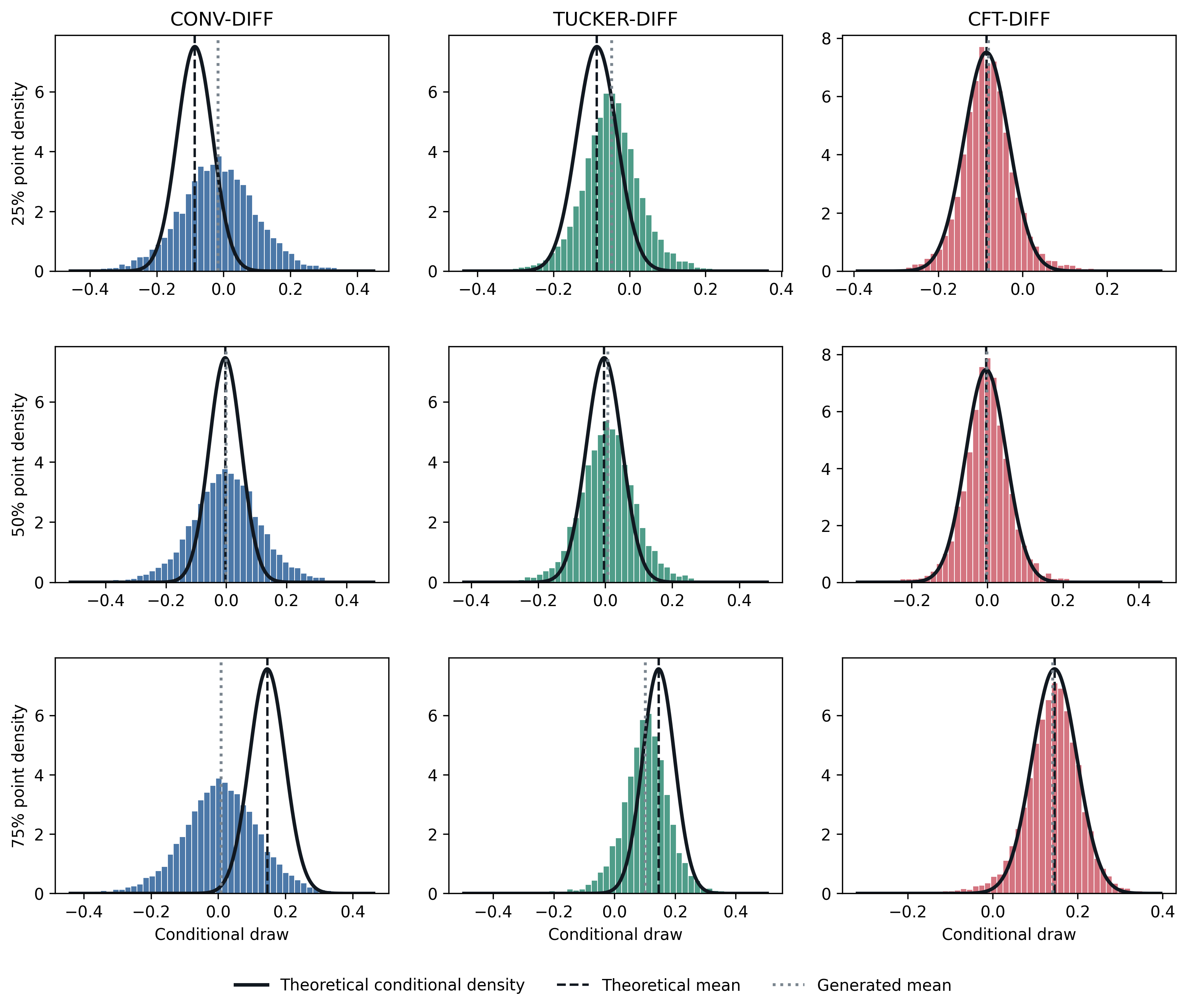}}
{Pointwise conditional distribution recovery with eight factors ($\beta=0.50$).
\label{fig:pointwise_conditional_distribution_beta050_8f}}
{Columns correspond to \ConvDiff, \TuckerDiff, and \CFTDiff. The design contains eight latent factors. Rows display entries whose realized missing values are near the 25th, 50th, and 75th percentiles within the missing region. Each histogram is based on 10,000 conditional draws, and the missing rate is 25\%. The solid curve shows the exact conditional density; the dashed and dotted vertical lines indicate the exact conditional mean and the mean of the generated draws, respectively.}
\end{figure}

\begin{figure}[htb]
\FIGURE
{\includegraphics[width=\textwidth]{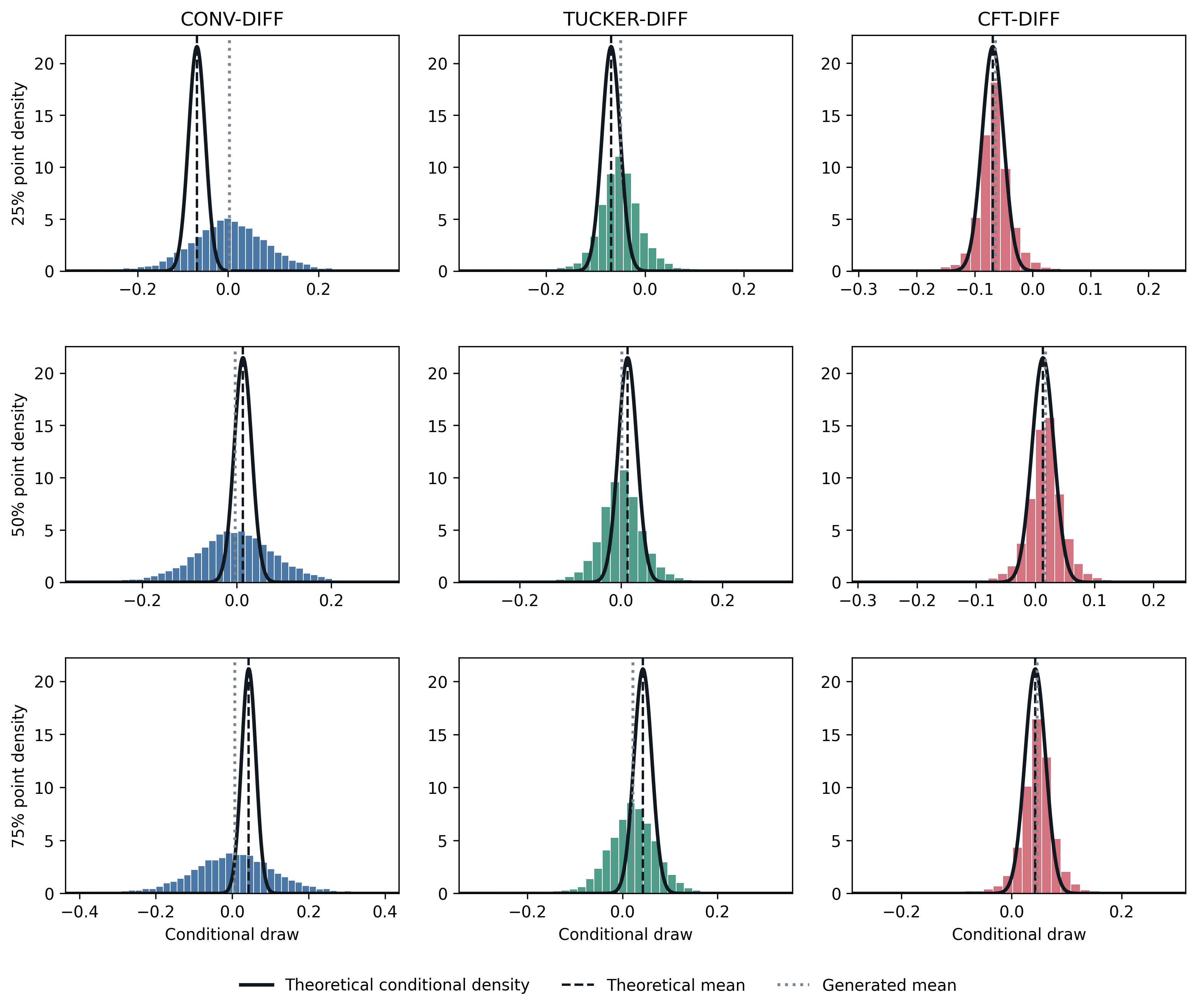}}
{Pointwise conditional distribution recovery with eight factors ($\beta=0.75$).
\label{fig:pointwise_conditional_distribution_beta075_8f}}
{Columns correspond to \ConvDiff, \TuckerDiff, and \CFTDiff. The design contains eight latent factors. Rows display entries whose realized missing values are near the 25th, 50th, and 75th percentiles within the missing region. Each histogram is based on 10,000 conditional draws, and the missing rate is 25\%. The solid curve shows the exact conditional density; the dashed and dotted vertical lines indicate the exact conditional mean and the mean of the generated draws, respectively.}
\end{figure}

Table~\ref{tab:pointwise_conditional_fit} reports the corresponding location, scale, and discrepancy values for both factor-strength settings. \CFTDiff\ has the smallest $D$ at every displayed entry. At $\beta=0.50$, its discrepancies are between 0.298 and 0.361, compared with values exceeding 1.17 for both diffusion benchmarks. At $\beta=0.75$, \CFTDiff\ remains substantially closer to the exact law even though the exact conditional distributions become more concentrated.

\begin{table}[htb]
\TABLE
{Pointwise conditional distribution fit at the displayed entries.
\label{tab:pointwise_conditional_fit}}
{
\scriptsize
\setlength{\tabcolsep}{1.5pt}
\begin{tabular*}{\textwidth}{@{\extracolsep{\fill}}ccccccccccccc@{}}
\toprule
& & \multicolumn{2}{c}{Exact Law} & \multicolumn{3}{c}{Conv-Diff} & \multicolumn{3}{c}{Tucker-Diff} & \multicolumn{3}{c}{CFT-Diff} \\[-1pt]
\cmidrule(lr){3-4}\cmidrule(lr){5-7}\cmidrule(lr){8-10}\cmidrule(lr){11-13}
Empirical Quantile & Realized Value & Mean & Std. Dev. & Mean & Std. Dev. & $D$ & Mean & Std. Dev. & $D$ & Mean & Std. Dev. & $D$ \\
\midrule
\multicolumn{13}{c}{\textit{Panel A: $\beta=0.50$}} \\
\midrule
$25\%$ & $-0.096$ & $-0.108$ & $0.043$ & $-0.007$ & $0.176$ & $2.709$ & $-0.044$ & $0.115$ & $1.754$ & $-0.108$ & $0.062$ & $0.361$ \\
$50\%$ & $0.006$ & $0.027$ & $0.041$ & $-0.009$ & $0.141$ & $1.525$ & $-0.009$ & $0.087$ & $1.173$ & $0.028$ & $0.055$ & $0.298$ \\
$75\%$ & $0.089$ & $0.028$ & $0.044$ & $-0.005$ & $0.186$ & $1.634$ & $-0.024$ & $0.124$ & $1.572$ & $0.026$ & $0.059$ & $0.299$ \\[2pt]
\midrule
\multicolumn{13}{c}{\textit{Panel B: $\beta=0.75$}} \\
\midrule
$25\%$ & $-0.089$ & $-0.109$ & $0.014$ & $0.010$ & $0.126$ & $8.582$ & $-0.075$ & $0.067$ & $2.867$ & $-0.110$ & $0.042$ & $1.076$ \\
$50\%$ & $0.006$ & $0.009$ & $0.014$ & $0.006$ & $0.134$ & $2.249$ & $0.032$ & $0.065$ & $2.169$ & $0.004$ & $0.040$ & $1.091$ \\
$75\%$ & $0.102$ & $0.096$ & $0.014$ & $0.006$ & $0.129$ & $6.657$ & $0.061$ & $0.070$ & $2.869$ & $0.098$ & $0.042$ & $1.085$ \\
\bottomrule
\end{tabular*}
}
{The table reports the entries displayed in Figures~\ref{fig:pointwise_conditional_distribution_beta050} and \ref{fig:pointwise_conditional_distribution_beta075}. Empirical Quantile gives the approximate rank of the realized value within the missing region. Under Exact Law, Mean and Std. Dev. are the conditional mean and standard deviation implied by \eqref{eq:simulation_exact_conditional_law}; for each diffusion method, they are calculated from 10,000 conditional draws. The discrepancy $D$, defined in \eqref{eq:simulation_pointwise_discrepancy}, accounts for deviations in both conditional location and scale. Smaller values indicate closer agreement, with zero corresponding to an exact match. The missing rate is 25\%.}
\end{table}

\section{Additional Empirical Analysis}\label{sec:Additional Empirical Analysis}
\subsection{Overview of the iFlex Experiment}\label{subsec:Overview of the iFlex Experiment}

The experimental design and data construction are documented by \citetEC{hofmann2023iflexdata}, while \citetEC{hofmann2024demandflexibility} examine household demand responses in the full-scale experiment. Table~\ref{tab:iflex_group_signal_sets} summarizes the group-specific intervention design. Treatment groups were organized by region and assigned prespecified sets of four candidate price signals, typically varying the peak-price level within a common profile or comparing alternative profile shapes. On each intervention day, one signal was selected from the corresponding group-specific set.

\begin{table}
\footnotesize
\TABLE
{Group-specific candidate price signals in the iFlex experiment.
\label{tab:iflex_group_signal_sets}}
{
\setlength{\tabcolsep}{8pt}
\begin{tabular*}{\textwidth}
{@{\extracolsep{\fill}}lll}
\toprule
Region
& Treatment Group
& Candidate Price Signals \\
\midrule
Bergen
& $\mathrm{Ber\_1}$
& $\{B_2,\,B_5,\,B_{10},\,B_{15}\}$ \\

Bod{\o}
& $\mathrm{Bo\_1}$
& $\{B_2,\,B_5,\,B_{10},\,B_{15}\}$ \\

\addlinespace[2pt]
Oslo
& $\mathrm{Os\_1}$
& $\{B_2,\,B_5,\,B_{10},\,B_{15}\}$ \\

& $\mathrm{Os\_2}$
& $\{P_2,\,P_5,\,P_{10},\,P_{15}\}$ \\

& $\mathrm{Os\_3}$
& $\{A_{10},\,B_{10},\,P_{10},\,P0_{10}\}$ \\

& $\mathrm{Os\_4}$
& $\{A_5,\,A_{10},\,B_{10},\,C\}$ \\

& $\mathrm{Os\_5}$
& $\{P0_2,\,P0_{10},\,P0_{30},\,P_{10}\}$ \\

& $\mathrm{Os\_6}$
& $\{A_5,\,B_2,\,C,\,P_{15}\}$ \\

\addlinespace[2pt]
Troms{\o}
& $\mathrm{Trom\_1}$
& $\{B_2,\,B_5,\,B_{10},\,B_{15}\}$ \\

& $\mathrm{Trom\_2}$
& $\{P_2,\,P_5,\,P_{10},\,P_{15}\}$ \\

\addlinespace[2pt]
Trondheim
& $\mathrm{Trond\_1}$
& $\{B_2,\,B_5,\,B_{10},\,B_{15}\}$ \\
\bottomrule
\end{tabular*}
}
{Each treatment group was associated with four prespecified candidate price signals. On each intervention day, one of these signals was selected for the group, and all households within that group were assigned the same 24-hour price path. The five regional control groups received no experimental price signal and are therefore omitted from the table. A signal label combines the price-profile family with the peak price level in NOK/kWh. Profile $C$ denotes the single asymmetric price signal with different morning and afternoon peak-price levels.}
\end{table}

Because households index samples in the empirical implementation, the Tucker structure used by \CFTDiff pertains to the day and hour modes of the $74\times24$ household matrices. We assess this structure using three collections of fully observed no-policy outcomes: globally nonintervention days, the pretreatment period, and all working days for households in the control groups. For each collection, we form the corresponding household-by-day-by-hour array for diagnostic purposes and examine its mode-wise spectral concentration. The day- and hour-mode spectra directly assess the low-rank structure used by the trained model; the household-direction spectrum is reported only as a descriptive measure of common cross-household variation.

\begin{figure}[htb]
	\FIGURE
    {\includegraphics[scale=0.78]{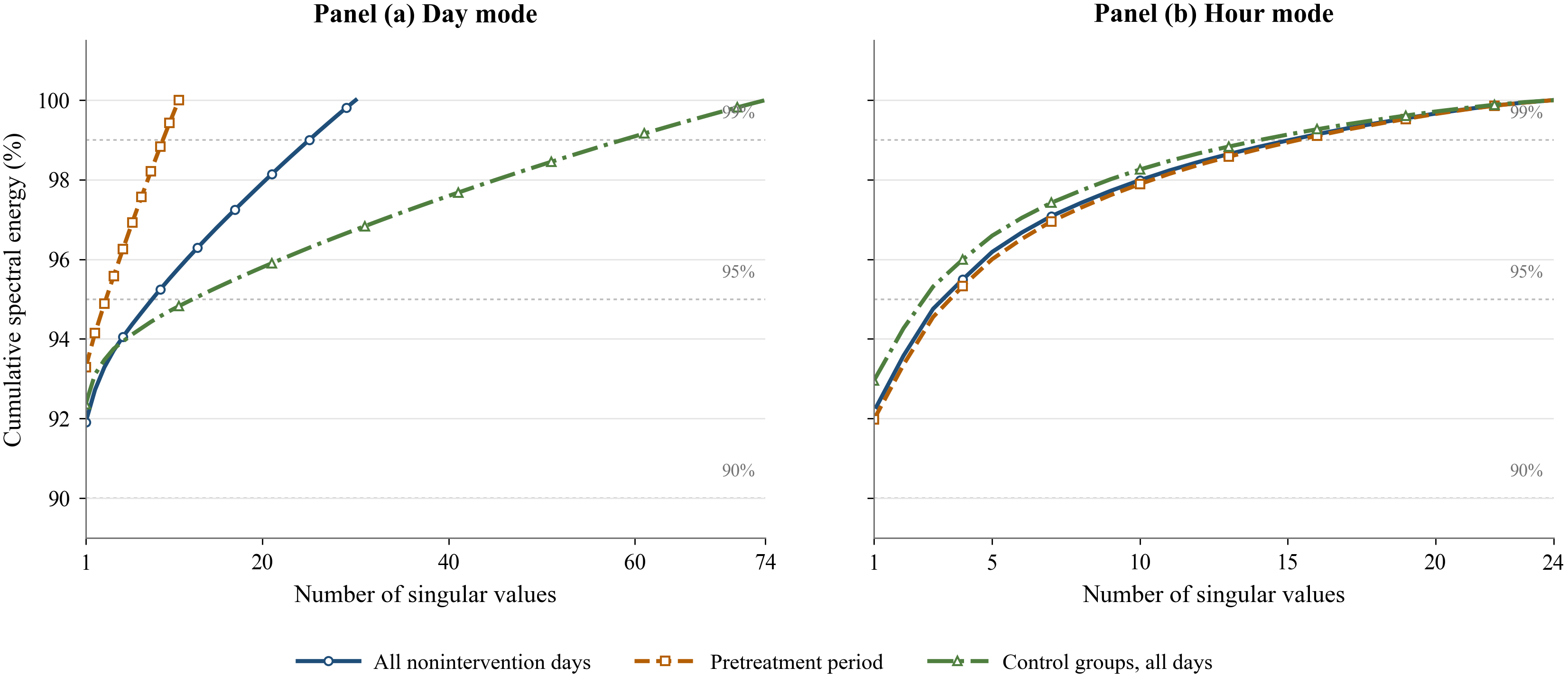}}	
	{Mode-wise spectral profiles of fully observed no-policy electricity-consumption tensors. \label{fig:iflex_tensor_spectral_profiles}} 
	{Panels (a) and (b) report cumulative spectral energy along the day and hour modes, respectively. The three curves are constructed from globally nonintervention days, the pretreatment period, and all working days for households in the control groups. Horizontal lines mark the 90\%, 95\%, and 99\% energy thresholds.}
\end{figure}

Figure~\ref{fig:iflex_tensor_spectral_profiles} shows pronounced spectral concentration in both modes used by the empirical model. Across the three sample constructions, the day-mode spectra reach the 95\% threshold with only a small fraction of the available components, while the hour-mode spectra exceed 95\% within the first few components and approach 99\% well before the full 24-dimensional mode is used. The profiles are also similar across the three samples, indicating that the observed concentration is not confined to either the short pretreatment period or the control-group households. These results support the approximately low-rank day-by-hour representation underlying the empirical \CFTDiff specification.

\subsection{Missingness Patterns}\label{sec:appendix_missingness_patterns}
To evaluate counterfactual recovery in a controlled setting, we artificially mask fully observed no-policy outcomes and compare the resulting imputations with their held-out values. We consider three treatment-induced missingness patterns commonly encountered in causal panels: switchback, staggered adoption, and simultaneous adoption. As illustrated in Figure~\ref{fig:iflex_missing_patterns}, control households remain fully observed, whereas complete household-day trajectories are masked after the first intervention day according to the corresponding design. These patterns allow us to assess imputation performance under both intermittent and monotone block missingness.
\begin{figure}[htb]
	\FIGURE
	{\includegraphics[scale=0.56]{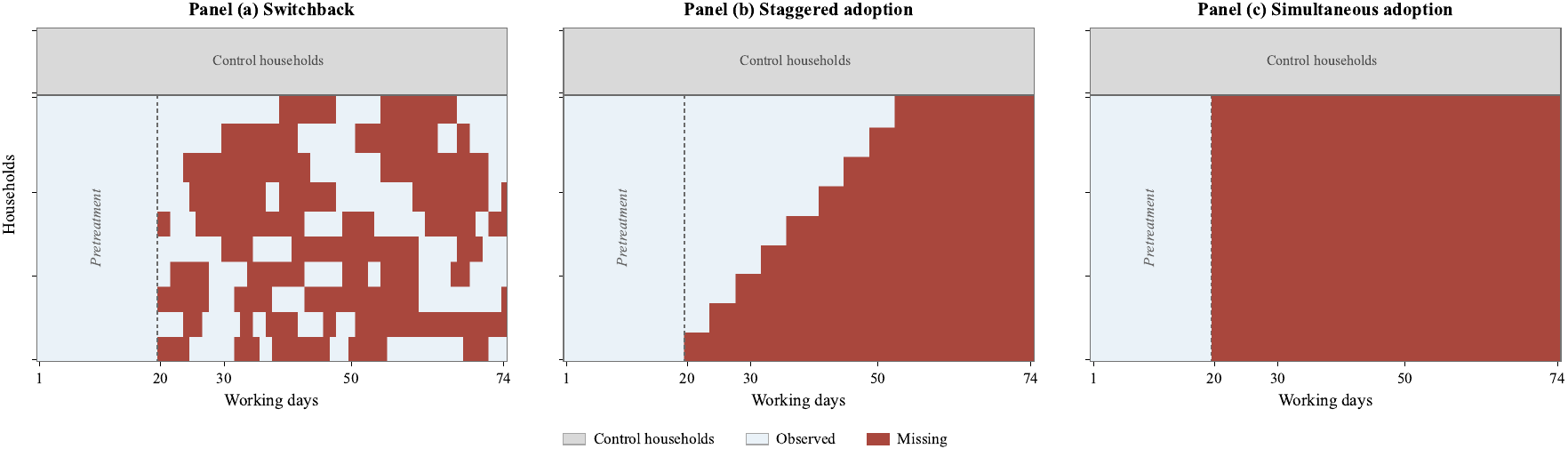}}	
	{Illustrative intervention patterns. \label{fig:iflex_missing_patterns}} 
	{Panels (a)--(c) illustrate the switchback, staggered-adoption, and simultaneous-adoption patterns, respectively. Rows represent households and columns represent working days. Gray rows denote fully observed control households, blue cells denote observed no-policy outcomes, and red cells denote missing no-policy outcomes. The vertical dashed line marks the first intervention day; observations to its left constitute the pretreatment period.}
\end{figure}

\subsection{Additional Results on Counterfactual Recovery Accuracy}\label{sec:appendix_recovery_metrics}
\subsubsection{Evaluation Measures}

This section defines the measures used to evaluate point and counterfactual distribution recovery. Let $\mathcal{M}$ denote the set of artificially masked outcomes and let $n_{\mathcal M}=|\mathcal{M}|$. For each $j\in\mathcal{M}$, let $Y_j$ denote the observed no-policy outcome and $\widehat Y_j$ its recovered value. For a diffusion method with $B$ generated draws, the point prediction is $\widehat Y_j=\frac{1}{B}\sum_{b=1}^{B}\widetilde Y_j^{(b)}$.

\paragraph{Point recovery.}
We evaluate point recovery using the mean absolute error (MAE) and root mean squared error (RMSE):
\[
    \operatorname{MAE}=\frac{1}{n_{\mathcal M}}\sum_{j\in\mathcal M}\left|\widehat Y_j-Y_j\right|,\qquad  \operatorname{RMSE}=\left[\frac{1}{n_{\mathcal M}}\sum_{j\in\mathcal M}\left(\widehat Y_j-Y_j\right)^2\right]^{1/2}.
\]
Smaller values indicate more accurate point recovery.

\paragraph{Energy score.}
The energy score evaluates the joint distribution of the missing 24-hour trajectory. Let $\mathbf Y_j\in\mathbb R^{24}$ denote the observed trajectory for an artificially masked household--day, and let $\widetilde{\mathbf Y}_j^{(1)},\ldots,\widetilde{\mathbf Y}_j^{(B)}$ denote the corresponding generated trajectories. For a predictive distribution $Q$, the population energy score is
\[
    \operatorname{ES}(Q,\mathbf Y_j)=\mathbb E_Q\left[\|\widetilde{\mathbf Y}_j-\mathbf Y_j\|_2\right]-\frac{1}{2}\mathbb E_Q
    \left[\|\widetilde{\mathbf Y}_j-\widetilde{\mathbf Y}_j'\|_2\right],
\]
where $\widetilde{\mathbf Y}_j$ and $\widetilde{\mathbf Y}_j'$ are independent draws from $Q$. Using the generated draws, the first expectation is evaluated by
\[
    \frac{1}{B}\sum_{b=1}^{B}\left\|\widetilde{\mathbf Y}_j^{(b)}-\mathbf Y_j\right\|_2.
\]
The second expectation is evaluated using randomly sampled pairs of generated draws. We average the resulting scores across all artificially masked household--days and across the five test samples. A smaller energy score indicates better recovery of the joint counterfactual distribution.

\paragraph{Continuous ranked probability score.}
CRPS evaluates the marginal distribution of an individual missing hourly outcome. Let $\widetilde Y_j^{(1)},\ldots,\widetilde Y_j^{(B)}$ be the generated draws for outcome $Y_j$. The empirical CRPS is
\[
    \operatorname{CRPS}_j=\frac{1}{B}\sum_{b=1}^{B}\left|\widetilde Y_j^{(b)}-Y_j\right|-\frac{1}{2B^2}\sum_{b=1}^{B}\sum_{b'=1}^{B}
    \left|\widetilde Y_j^{(b)}-\widetilde Y_j^{(b')}\right|.
\]
The reported CRPS is averaged across the missing hourly outcomes. Smaller values indicate better marginal distributional recovery.

\paragraph{Coverage and interval width.}
For the interval-based measures, we use the household--day total of the missing trajectory. Let
\[
    U_j=\sum_{h\in\mathcal M_j}Y_{jh},\qquad\widetilde U_j^{(b)}=\sum_{h\in\mathcal M_j}\widetilde Y_{jh}^{(b)},
\]
where $\mathcal M_j$ denotes the missing hours for household--day $j$. Under the block-missingness designs used in the main analysis, $\mathcal M_j$ contains all 24 hours.

For a central interval with nominal coverage $1-\alpha$, let
\[
    L_{j,\alpha}=\widehat Q_{\alpha/2}\left(\widetilde U_j^{(1)},\ldots,\widetilde U_j^{(B)}\right),\qquad H_{j,\alpha}=\widehat Q_{1-\alpha/2}\left(\widetilde U_j^{(1)},\ldots,\widetilde U_j^{(B)}\right).
\]
The empirical coverage is
\[
    \operatorname{Coverage}_{1-\alpha}=\frac{1}{n_B}\sum_{j=1}^{n_B}\mathbb I\left\{L_{j,\alpha}\leq U_j\leq H_{j,\alpha}\right\},
\]
where $n_B$ is the number of evaluated household--day blocks. A well-calibrated interval has empirical coverage close to its nominal level $1-\alpha$.

The average interval width is
\[
    \operatorname{Width}_{1-\alpha}=\frac{1}{n_B}\sum_{j=1}^{n_B}\left(H_{j,\alpha}-L_{j,\alpha}\right).
\]
For comparable coverage, a smaller width indicates a sharper counterfactual distribution.

\paragraph{Interval score.}
The interval score combines interval width and coverage in a single measure. For a central $(1-\alpha)$ interval,
\[
\begin{aligned}
    \operatorname{IS}_{\alpha,j}={}&H_{j,\alpha}-L_{j,\alpha}+\frac{2}{\alpha}\left(L_{j,\alpha}-U_j\right)\mathbb I\left\{U_j<L_{j,\alpha}\right\}+\frac{2}{\alpha}\left(U_j-H_{j,\alpha}\right)\mathbb I\left\{U_j>H_{j,\alpha}\right\}.
\end{aligned}
\]
The first term penalizes wide intervals, while the remaining terms penalize observations that fall below or above the interval. The table reports the average interval score for the 90\% central interval. A smaller value indicates a better combination of coverage and interval width.

\paragraph{Weighted interval score.}
The weighted interval score aggregates information from several central intervals and the predictive median. Let $m_j$ denote the predictive median of $\widetilde U_j^{(1)},\ldots,\widetilde U_j^{(B)}$, and let the $K_C$ central coverage levels be $1-\alpha_1,\ldots,1-\alpha_{K_C}$. We compute
\[
    \operatorname{WIS}_j=\frac{\frac{1}{2}|U_j-m_j|+\displaystyle\sum_{k=1}^{K_C}\frac{\alpha_k}{2}\operatorname{IS}_{\alpha_k,j}}{K_C+\frac{1}{2}}.
\]
In the main analysis, the central coverage levels are $50\%$, $60\%$, $70\%$, $80\%$, $90\%$, and $95\%$. The reported WIS is averaged across evaluated household--days. A smaller WIS indicates better overall distributional recovery across the different parts of the counterfactual distribution.

\subsubsection{Results on Counterfactual Recovery Accuracy}\label{sec:appendix_additional_recovery_results}

We provide several additional checks on the counterfactual recovery results in Section~\ref{sec:counterfactual_recovery}. These analyses examine whether the main findings are robust to an alternative point-recovery criterion, continue to hold for aggregate counterfactual outcomes, and extend beyond the energy score used in the main figure.

\begin{figure}[htb]
	\FIGURE
	{\includegraphics[scale=0.52]{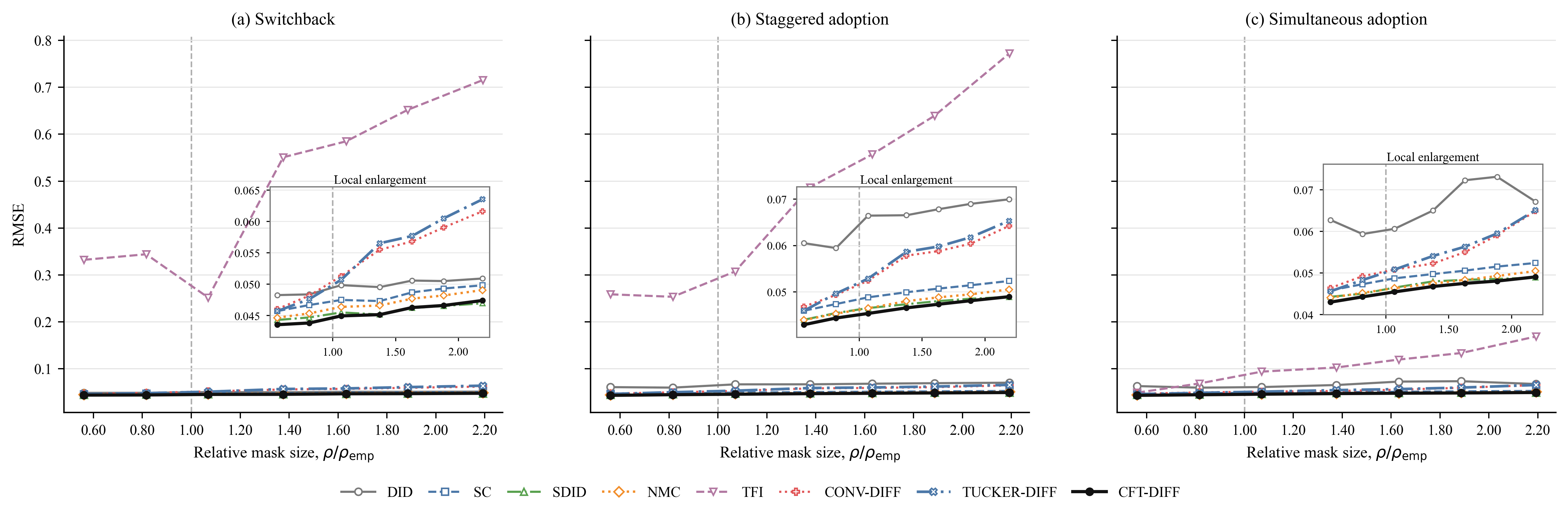}}	
	{Additional point counterfactual recovery measured by RMSE.
	\label{fig:appendix_recovery_rmse}}  
	{The figure reports RMSE across different missingness levels under the switchback, staggered-adoption, and simultaneous-adoption patterns. The horizontal axis reports the realized missing proportion relative to the empirical treated share, $\rho/\rho_{\mathrm{emp}}$, and the vertical dashed line indicates $\rho=\rho_{\mathrm{emp}}$. Point predictions for the three diffusion methods are obtained by averaging the 100 generated counterfactual draws. Lower values indicate better recovery. The inset in each panel is a local enlargement of the corresponding region in the original panel.}
\end{figure}

Figure~\ref{fig:appendix_recovery_rmse} reports the results over the same missingness patterns and levels as in the main analysis. The ranking is similar to that obtained with MAE. \CFTDiff remains highly stable as the missing region expands and yields the lowest RMSE throughout the three designs. The two alternative diffusion models deteriorate more rapidly with missingness, while \TFI is particularly sensitive to the size and structure of the missing region. Thus, the point-recovery advantage documented in the main text is not specific to the absolute-error criterion.

\begin{figure}[htb]
	\FIGURE
	{\includegraphics[scale=0.52]{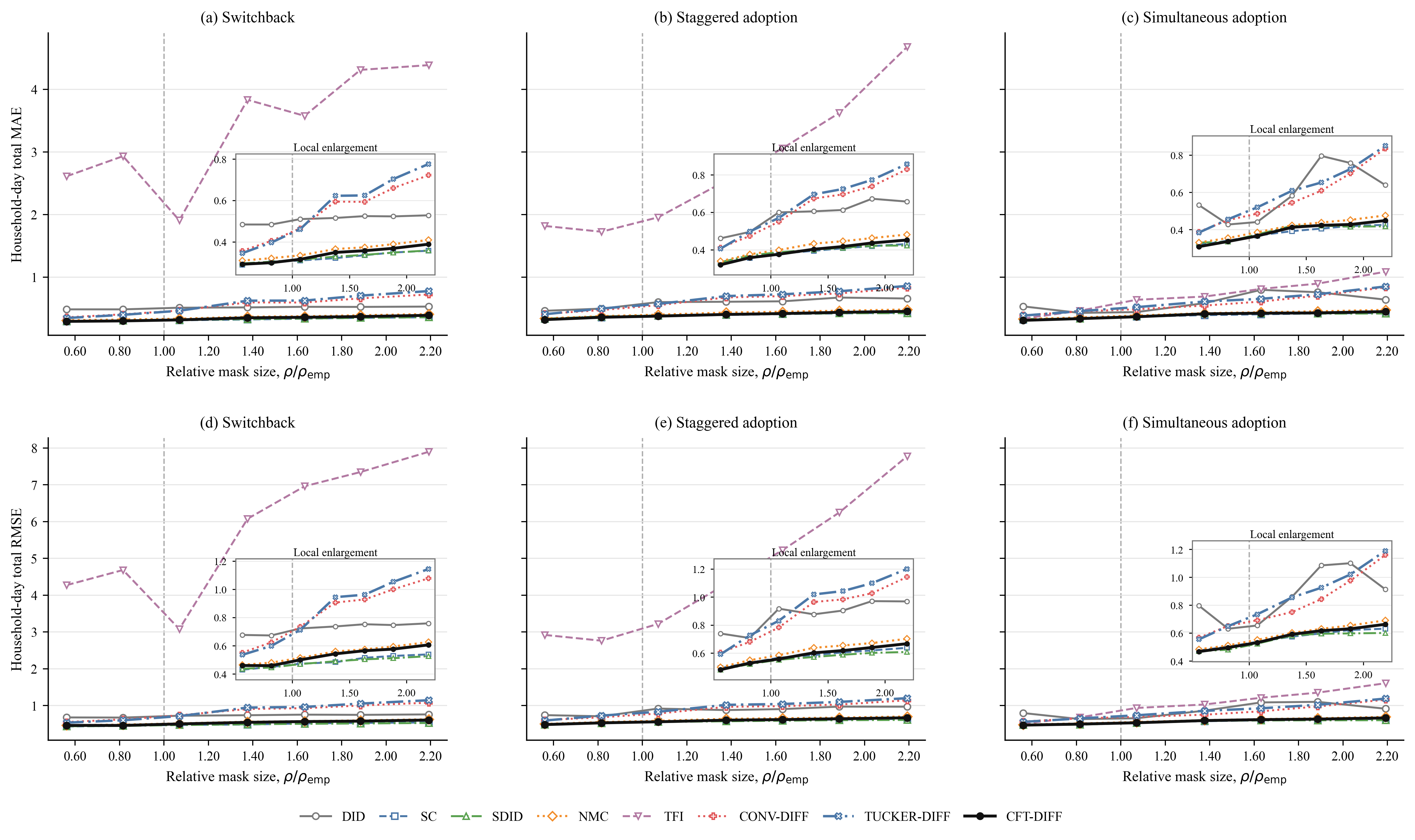}}	
	{Recovery of aggregate household--day counterfactual outcomes.
	\label{fig:appendix_recovery_aggregate}}
	{The figure evaluates recovery of household--day totals across different missingness levels. For each artificially masked household--day, the recovered and observed outcomes are first aggregated over the missing 24-hour trajectory. Panels (a)--(c) report MAE and Panels (d)--(f) report RMSE under the switchback, staggered-adoption, and simultaneous-adoption patterns, respectively. The horizontal axis reports $\rho/\rho_{\mathrm{emp}}$, and the vertical dashed line indicates $\rho=\rho_{\mathrm{emp}}$. Lower values indicate better recovery.}
\end{figure}

Pointwise accuracy need not imply accurate recovery of aggregates formed from the missing trajectory. This distinction is relevant for the empirical analysis because the causal estimands aggregate counterfactual outcomes across hours. We sum the 24 hourly outcomes within each artificially masked household--day before computing the recovery error. Figure~\ref{fig:appendix_recovery_aggregate} reports MAE and RMSE for these household--day totals. \CFTDiff remains among the most accurate methods under all three missingness patterns and is substantially more stable than the alternative diffusion models as the missing proportion increases. The results indicate that its pointwise recovery performance carries over to aggregate counterfactual outcomes rather than disappearing after temporal aggregation.

\begin{figure}[htb]
	\FIGURE
	{\includegraphics[scale=0.52]{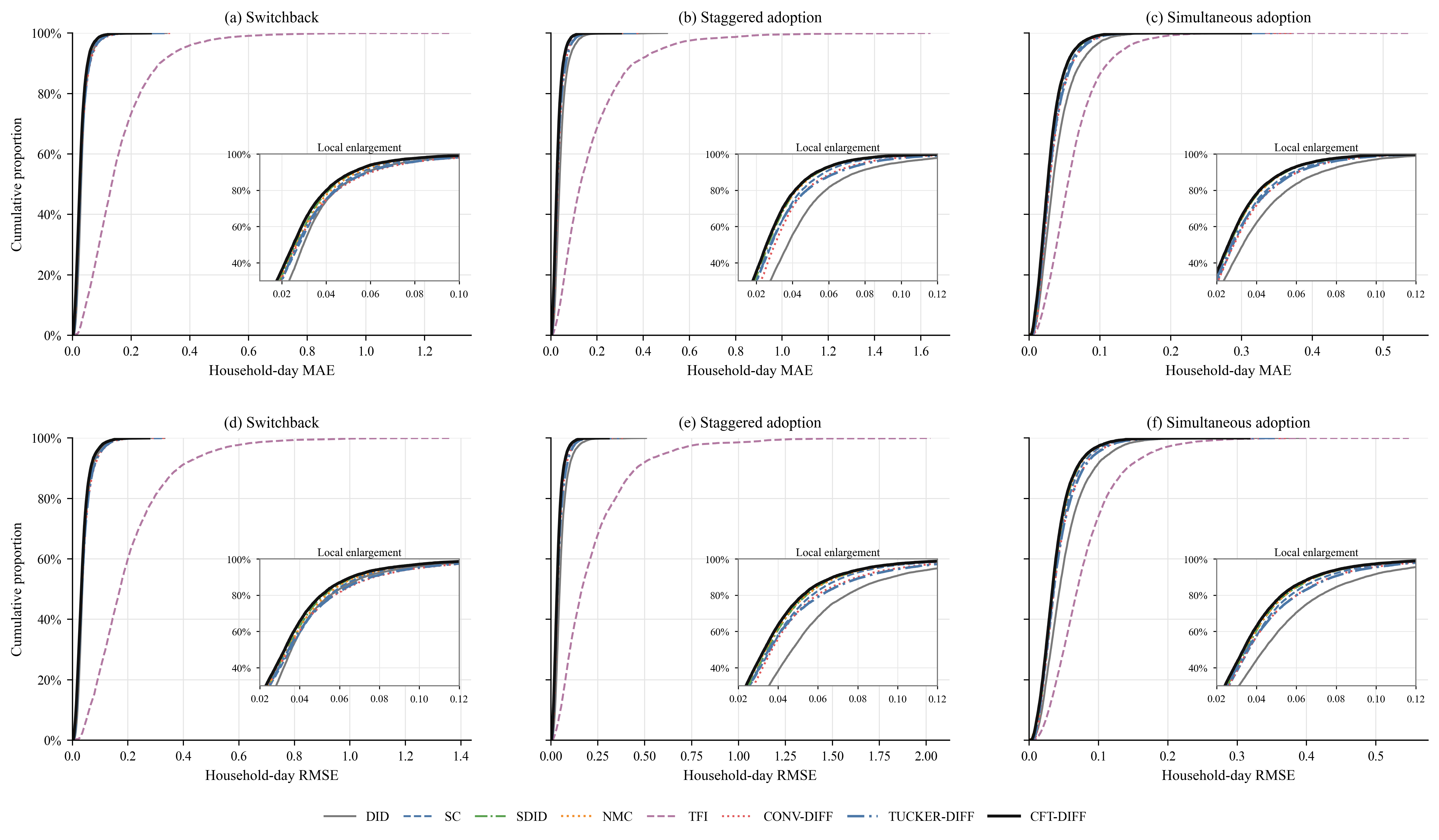}}	
	{Distribution of household--day counterfactual recovery errors.
	\label{fig:appendix_recovery_ecdf}}
	{The figure reports empirical cumulative distribution functions of household--day recovery errors at the available missingness level closest to the empirical treated share. Panels (a)--(c) report household--day MAE and Panels (d)--(f) report household--day RMSE under the switchback, staggered-adoption, and simultaneous-adoption patterns, respectively. Errors are computed separately for each artificially masked household--day over its 24-hour trajectory. Curves farther to the left indicate smaller recovery errors. The inset in each panel provides a local enlargement of the low-error region.}
\end{figure}

Figure~\ref{fig:appendix_recovery_ecdf} reports the empirical distributions of household--day MAE and RMSE at the missingness level closest to the empirical treated share. The distributions for \CFTDiff are concentrated in the low-error region under all three designs. The differences are particularly visible relative to \TFI and remain present relative to the two diffusion benchmarks. Hence, the favorable average performance of \CFTDiff reflects broadly lower recovery errors across the masked household--days rather than improvements concentrated in a small subset of observations.

\begin{figure}[htb]
	\FIGURE
	{\includegraphics[scale=0.52]{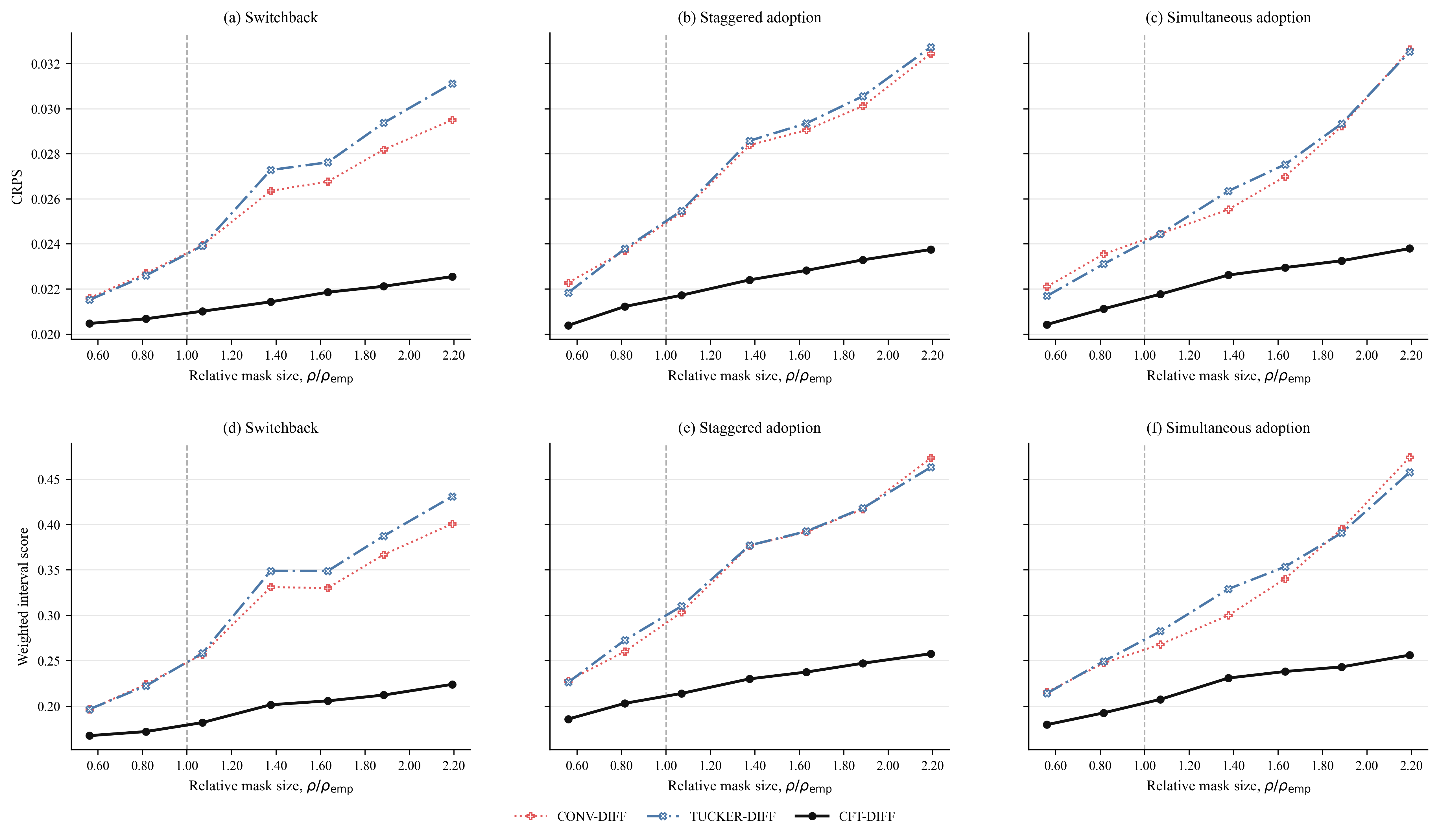}}	
	{Additional counterfactual distribution recovery across missingness levels.
	\label{fig:appendix_recovery_crps_wis}} 
	{The figure evaluates counterfactual distribution recovery for the three diffusion methods across different missingness levels. Panels (a)--(c) report CRPS for individual hourly outcomes, and Panels (d)--(f) report WIS for household--day totals under the switchback, staggered-adoption, and simultaneous-adoption patterns, respectively. The horizontal axis reports $\rho/\rho_{\mathrm{emp}}$, and the vertical dashed line indicates $\rho=\rho_{\mathrm{emp}}$. Lower values indicate better distributional recovery.}
\end{figure}

Finally, we provide additional evidence on distributional recovery. Whereas the main figure uses the energy score to evaluate the joint 24-hour trajectory, Figure~\ref{fig:appendix_recovery_crps_wis} considers two complementary criteria. CRPS evaluates the marginal distribution of individual hourly outcomes, while WIS evaluates the distribution of household--day totals through multiple central intervals. Across all three missingness patterns, both scores increase as the missing region expands, but substantially more slowly for \CFTDiff. \CFTDiff has the lowest CRPS and WIS at every reported missingness level, showing that its distributional advantage is present for both hourly marginals and aggregate counterfactual outcomes.

\begin{table}[htbp]
\TABLE
{Counterfactual distribution recovery across missingness patterns.
\label{tab:appendix_distributional_recovery}}
{
\begin{tabular*}{\textwidth}{@{\extracolsep{\fill}}lrrrrrr}
\hline
\multicolumn{7}{l}{\textit{Panel A: Switchback}} \\
Method & \multicolumn{1}{c}{Energy score} & \multicolumn{1}{c}{CRPS} & \multicolumn{1}{c}{90\% coverage} & \multicolumn{1}{c}{90\% width} & \multicolumn{1}{c}{IS} & \multicolumn{1}{r}{WIS} \\
\hline
\ConvDiff & 0.1479 & 0.0239 & 0.887 & \underline{1.7246} & \underline{2.6548} & \underline{0.2569} \\
\TuckerDiff & \underline{0.1476} & \underline{0.0239} & \underline{0.907} & 1.8352 & 2.6898 & 0.2583 \\
\CFTDiff & \textbf{0.1314} & \textbf{0.0210} & \textbf{0.893} & \textbf{1.2559} & \textbf{2.0126} & \textbf{0.1819} \\
\hline
\multicolumn{7}{l}{\textit{Panel B: Staggered adoption}} \\
Method & \multicolumn{1}{c}{Energy score} & \multicolumn{1}{c}{CRPS} & \multicolumn{1}{c}{90\% coverage} & \multicolumn{1}{c}{90\% width} & \multicolumn{1}{c}{IS} & \multicolumn{1}{r}{WIS} \\
\hline
\ConvDiff & \underline{0.1546} & \underline{0.0254} & 0.950 & 2.8132 & 3.1717 & \underline{0.3034} \\
\TuckerDiff & 0.1553 & 0.0255 & \textbf{0.905} & \underline{2.2037} & \underline{3.0402} & 0.3102 \\
\CFTDiff & \textbf{0.1358} & \textbf{0.0217} & \underline{0.920} & \textbf{1.7198} & \textbf{2.2905} & \textbf{0.2140} \\
\hline
\multicolumn{7}{l}{\textit{Panel C: Simultaneous adoption}} \\
Method & \multicolumn{1}{c}{Energy score} & \multicolumn{1}{c}{CRPS} & \multicolumn{1}{c}{90\% coverage} & \multicolumn{1}{c}{90\% width} & \multicolumn{1}{c}{IS} & \multicolumn{1}{r}{WIS} \\
\hline
\ConvDiff & 0.1510 & \underline{0.0244} & 0.942 & 2.2955 & \underline{2.7606} & \underline{0.2680} \\
\TuckerDiff & \underline{0.1509} & 0.0245 & \textbf{0.896} & \underline{1.9824} & 2.8150 & 0.2825 \\
\CFTDiff & \textbf{0.1364} & \textbf{0.0218} & \underline{0.919} & \textbf{1.6321} & \textbf{2.2463} & \textbf{0.2075} \\
\hline
\end{tabular*}
}
{The table evaluates the three diffusion methods at the available missingness level closest to the empirical treated share, $\rho_{\mathrm{emp}}=0.265$. The common target missing proportion is 0.284, with realized missing proportions of 0.283, 0.284, and 0.284 under the switchback, staggered-adoption, and simultaneous-adoption patterns, respectively. The energy score evaluates the joint 24-hour trajectory, CRPS evaluates individual hourly outcomes, and the interval-based measures are computed for household--day totals. Errors, scores, and interval widths are reported in the normalized experimental scale. Lower values are better for the energy score, CRPS, IS, and WIS; coverage is evaluated by closeness to 0.90, and a smaller width indicates a sharper interval when coverage is comparable. Boldface and underlining indicate the best and second-best values within each panel, respectively.}
\end{table}

Table~\ref{tab:appendix_distributional_recovery} reports all distributional measures at the missingness level closest to that observed in the application. Under each missingness pattern, \CFTDiff has the lowest energy score, CRPS, interval score, and WIS. Its 90\% intervals are also narrower than those of the alternative diffusion models, with coverage remaining reasonably close to 0.90. Under staggered adoption, for example, the interval width decreases from 2.2037 for \TuckerDiff to 1.7198 for \CFTDiff, with coverage of 0.920. The switchback and simultaneous-adoption designs show the same combination of lower distributional scores and sharper intervals.

\subsubsection{Counterfactual Trajectory Recovery}\label{sec:appendix_trajectory_recovery}

The preceding analyses summarize counterfactual recovery using scalar error and distributional measures. These measures do not reveal whether a method recovers the intraday shape of the missing outcome, including the timing and magnitude of demand peaks. We therefore complement them with trajectory-level comparisons. For each artificially masked household group, the first column of the figures compares the observed no-policy trajectory with the five point estimators. The remaining columns report \ConvDiff, \TuckerDiff, and \CFTDiff. The solid curve is the median of the generated group-average trajectories, and the shaded regions are the central 50\% and 90\% counterfactual prediction intervals.

\begin{figure}[htb]
	\FIGURE
	{\includegraphics[scale=0.52]{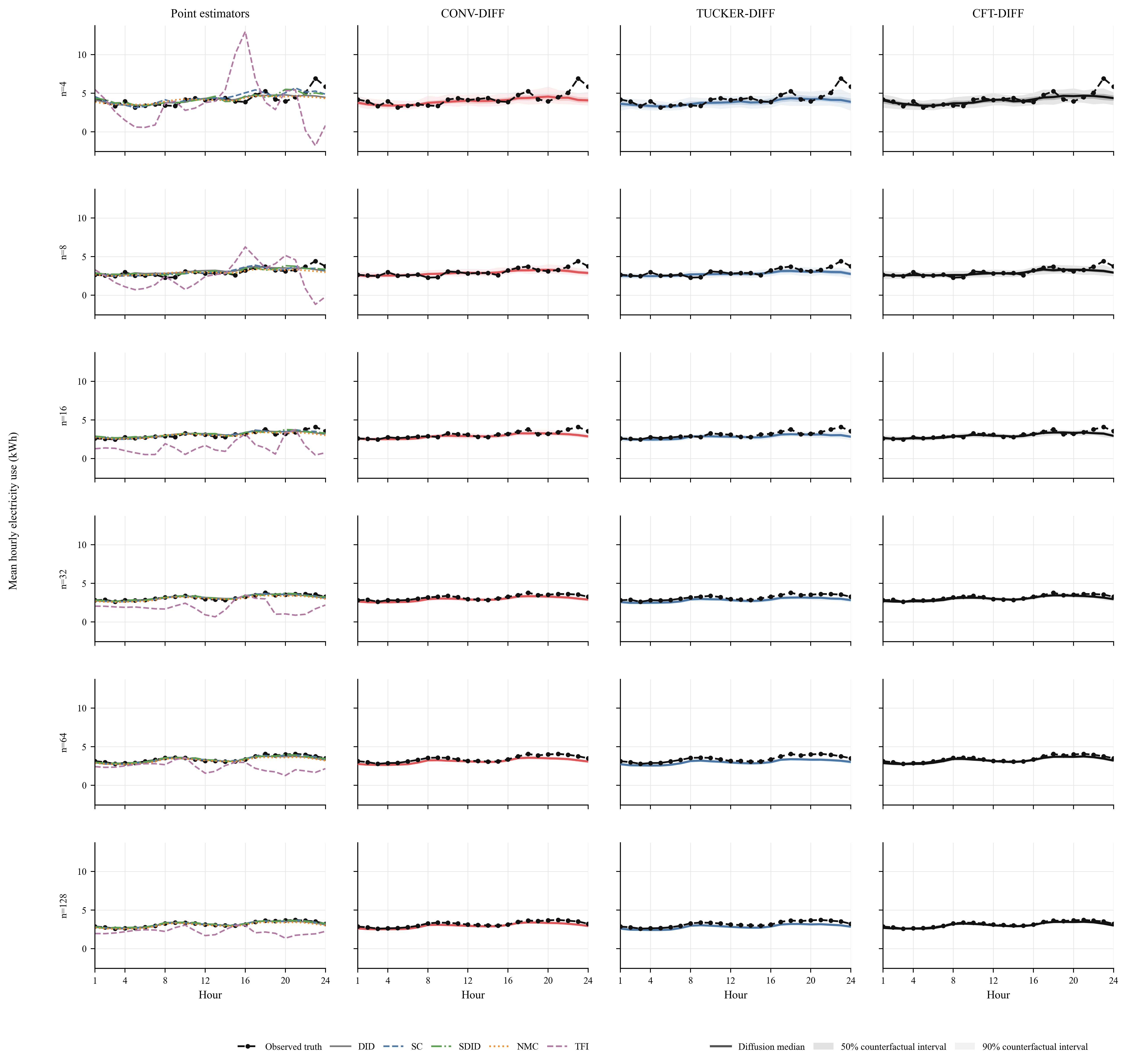}}	
	{Counterfactual trajectories across nested household groups under switchback missingness.
	\label{fig:trajectory_switchback_nested}} 
	{The figure compares recovered 24-hour counterfactual trajectories as the number of households included in the group average increases under the switchback pattern. Rows correspond to nested household groups of increasing size. The first column reports the observed trajectory together with the five point estimators; the remaining columns report \ConvDiff, \TuckerDiff, and \CFTDiff, respectively. For each diffusion method, the solid line denotes the median generated trajectory and the shaded regions denote the 50\% and 90\% central counterfactual prediction intervals.}
\end{figure}

\begin{figure}[htb]
	\FIGURE
	{\includegraphics[scale=0.52]{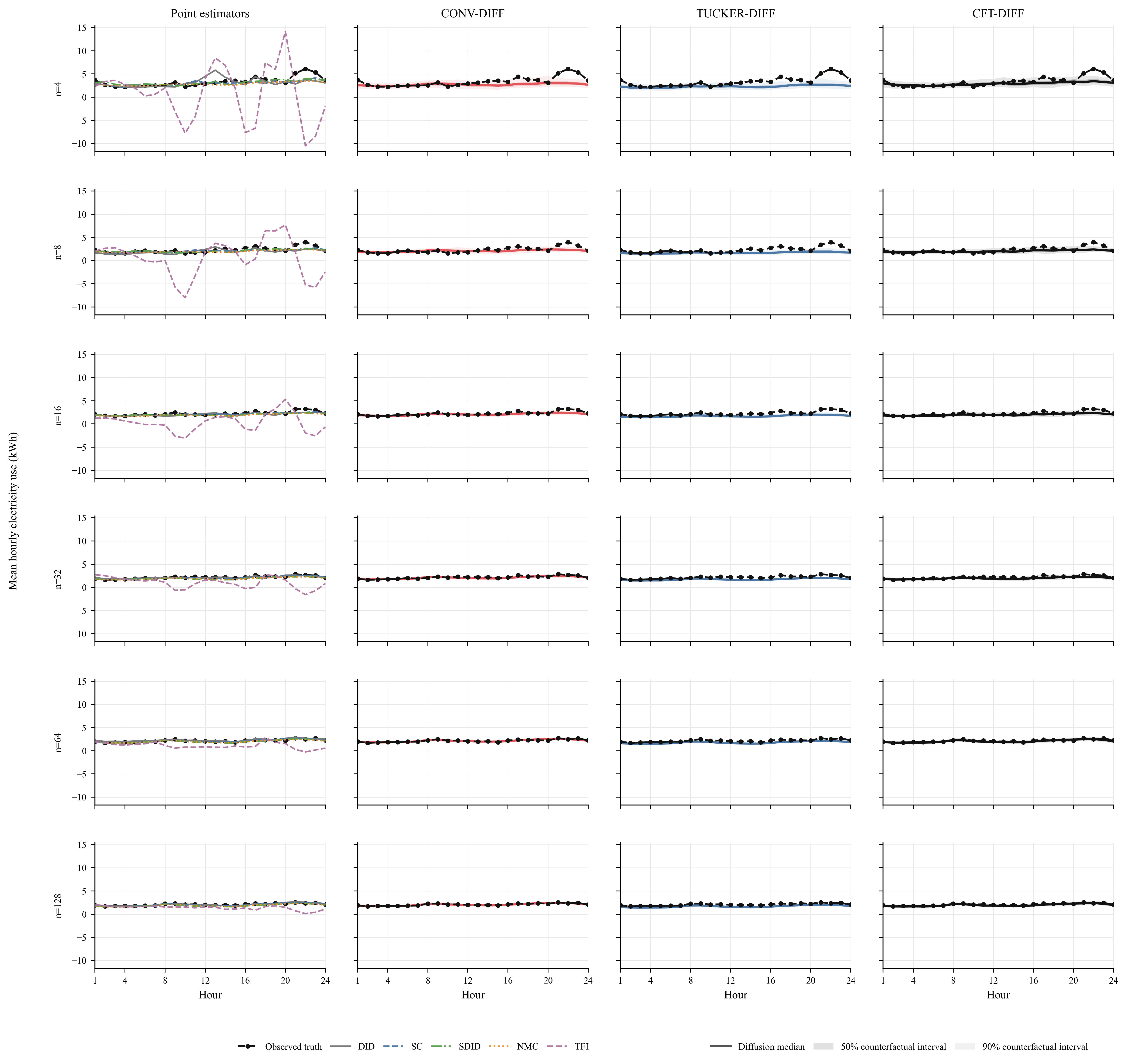}}	
	{Counterfactual trajectories across nested household groups under staggered adoption.
	\label{fig:trajectory_stagger_nested}} 
	{The figure compares recovered 24-hour counterfactual trajectories as the number of households included in the group average increases under staggered adoption. Rows correspond to nested household groups of increasing size. The first column reports the observed trajectory together with the five point estimators; the remaining columns report \ConvDiff, \TuckerDiff, and \CFTDiff, respectively. For each diffusion method, the solid line denotes the median generated trajectory and the shaded regions denote the 50\% and 90\% central counterfactual prediction intervals.}
\end{figure}

\begin{figure}[htb]
	\FIGURE
	{\includegraphics[scale=0.52]{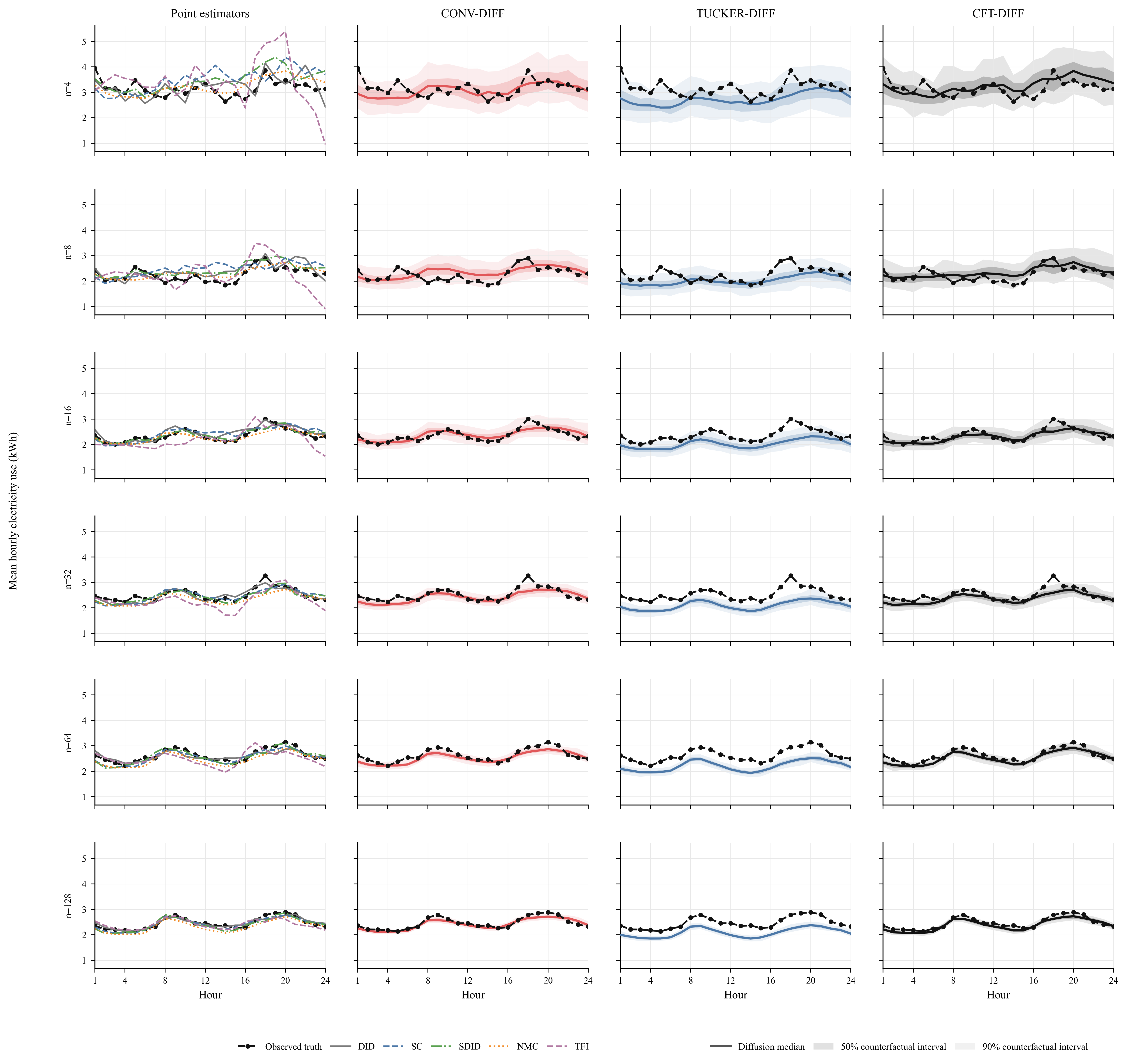}}	
	{Counterfactual trajectories across nested household groups under simultaneous adoption.
	\label{fig:trajectory_block_nested}} 
	{The figure compares recovered 24-hour counterfactual trajectories as the number of households included in the group average increases under simultaneous adoption. Rows correspond to nested household groups of increasing size. The first column reports the observed trajectory together with the five point estimators; the remaining columns report \ConvDiff, \TuckerDiff, and \CFTDiff, respectively. For each diffusion method, the solid line denotes the median generated trajectory and the shaded regions denote the 50\% and 90\% central counterfactual prediction intervals.}
\end{figure}

Figures~\ref{fig:trajectory_switchback_nested}, \ref{fig:trajectory_stagger_nested}, and \ref{fig:trajectory_block_nested} examine trajectory recovery as the number of households in each group increases. This design separates errors that are specific to small groups from errors that persist after aggregation. For small groups, the realized trajectories contain substantial hour-to-hour variation, and the differences across methods are correspondingly more visible. Among the point estimators, \DID, \SC, \SDID, and \NMC generally recover the overall level and shape, whereas \TFI can be unstable when only a few households are averaged. The three diffusion models produce smoother counterfactual trajectories, but \ConvDiff and especially \TuckerDiff tend to attenuate local peaks. \CFTDiff follows the observed trajectory more closely while retaining the main intraday variation.

As the group size increases, idiosyncratic household variation is averaged out and the trajectories become progressively smoother. The \CFTDiff median becomes tightly aligned with the observed group-average trajectory, and its counterfactual prediction intervals also become more concentrated. The same pattern is visible under switchback, staggered adoption, and simultaneous adoption. Thus, the recovery performance is not confined to large aggregates: \CFTDiff remains informative for relatively small groups and becomes increasingly precise as the counterfactual object is aggregated over more households.


\begin{figure}[htb]
	\FIGURE
	{\includegraphics[scale=0.52]{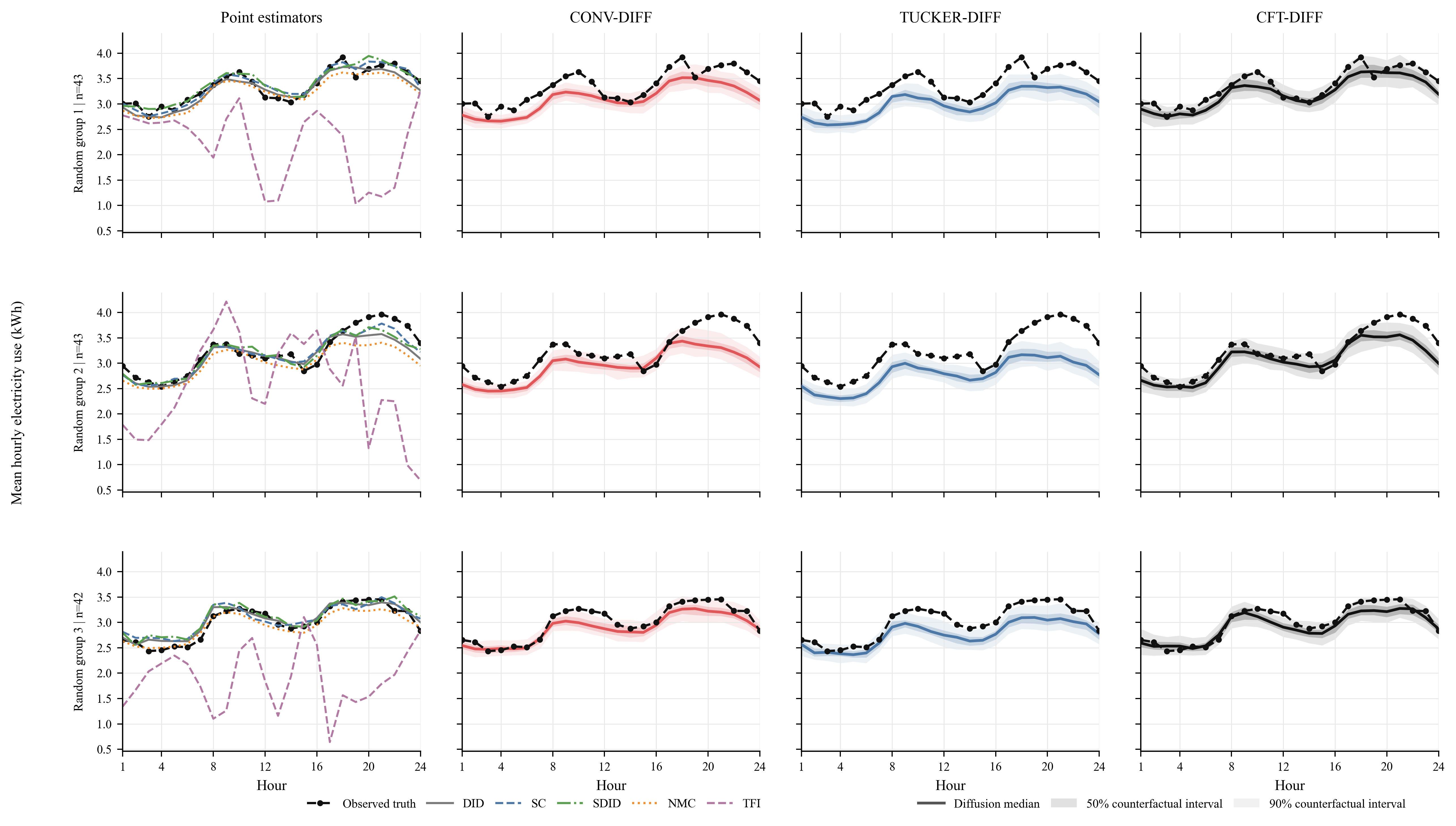}}	
	{Counterfactual trajectories across random household groups under switchback missingness.
	\label{fig:trajectory_switchback_random}} 
	{The figure examines whether trajectory recovery depends on the composition of the household group under the switchback pattern. Each row corresponds to a different randomly selected household group. The first column reports the observed trajectory and the five point estimators; the remaining columns report the three diffusion methods. Solid diffusion curves denote median generated trajectories, with 50\% and 90\% central counterfactual prediction intervals shown by the shaded regions.}
\end{figure}

\begin{figure}[htb]
	\FIGURE
	{\includegraphics[scale=0.52]{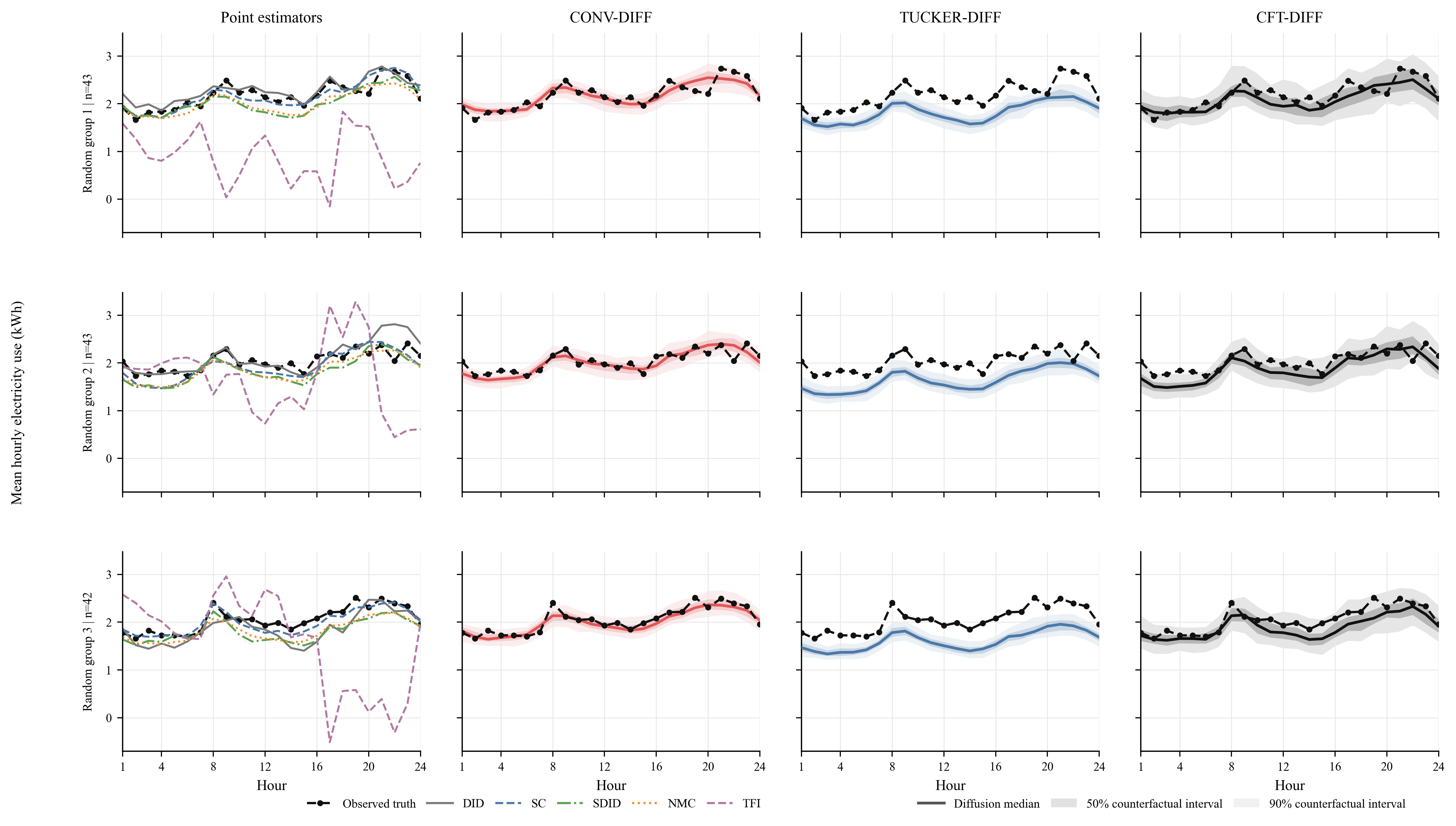}}	
	{Counterfactual trajectories across random household groups under staggered adoption.
	\label{fig:trajectory_stagger_random}} 
	{The figure examines whether trajectory recovery depends on the composition of the household group under staggered adoption. Each row corresponds to a different randomly selected household group. The first column reports the observed trajectory and the five point estimators; the remaining columns report the three diffusion methods. Solid diffusion curves denote median generated trajectories, with 50\% and 90\% central counterfactual prediction intervals shown by the shaded regions.}
\end{figure}

\begin{figure}[htb]
	\FIGURE
	{\includegraphics[scale=0.52]{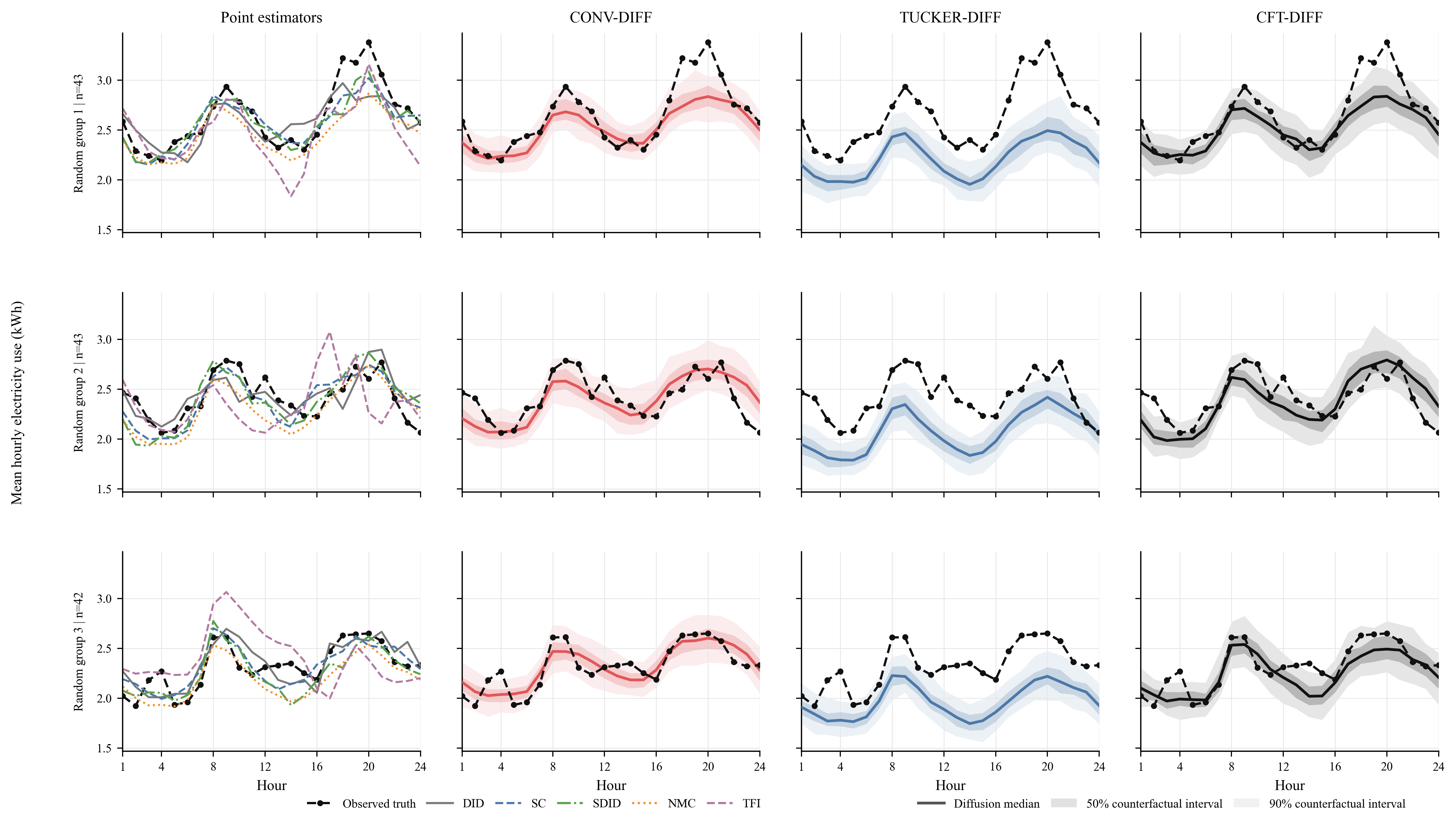}}	
	{Counterfactual trajectories across random household groups under simultaneous adoption.
	\label{fig:trajectory_block_random}} 
	{The figure examines whether trajectory recovery depends on the composition of the household group under simultaneous adoption. Each row corresponds to a different randomly selected household group. The first column reports the observed trajectory and the five point estimators; the remaining columns report the three diffusion methods. Solid diffusion curves denote median generated trajectories, with 50\% and 90\% central counterfactual prediction intervals shown by the shaded regions.}
\end{figure}

Figures~\ref{fig:trajectory_switchback_random}, \ref{fig:trajectory_stagger_random}, and \ref{fig:trajectory_block_random} repeat the comparison for nonoverlapping random household groups. The resulting trajectories vary meaningfully across groups in both their levels and peak patterns, providing a more demanding check on the cross-sectional stability of the recovery methods. The main conclusions are similar across the three groups and missingness patterns. The stronger point estimators generally reproduce the broad group-average profile, while \TFI exhibits noticeably larger deviations in several cases. Among the diffusion models, \CFTDiff consistently preserves more of the observed intraday variation. By comparison, \ConvDiff and \TuckerDiff tend to smooth the morning and evening peaks and, for some groups, shift the trajectory downward. The similarity of these results across independently formed groups indicates that the trajectory-level performance is not driven by a particular household composition.

\begin{figure}[htb]
	\FIGURE
	{\includegraphics[scale=0.52]{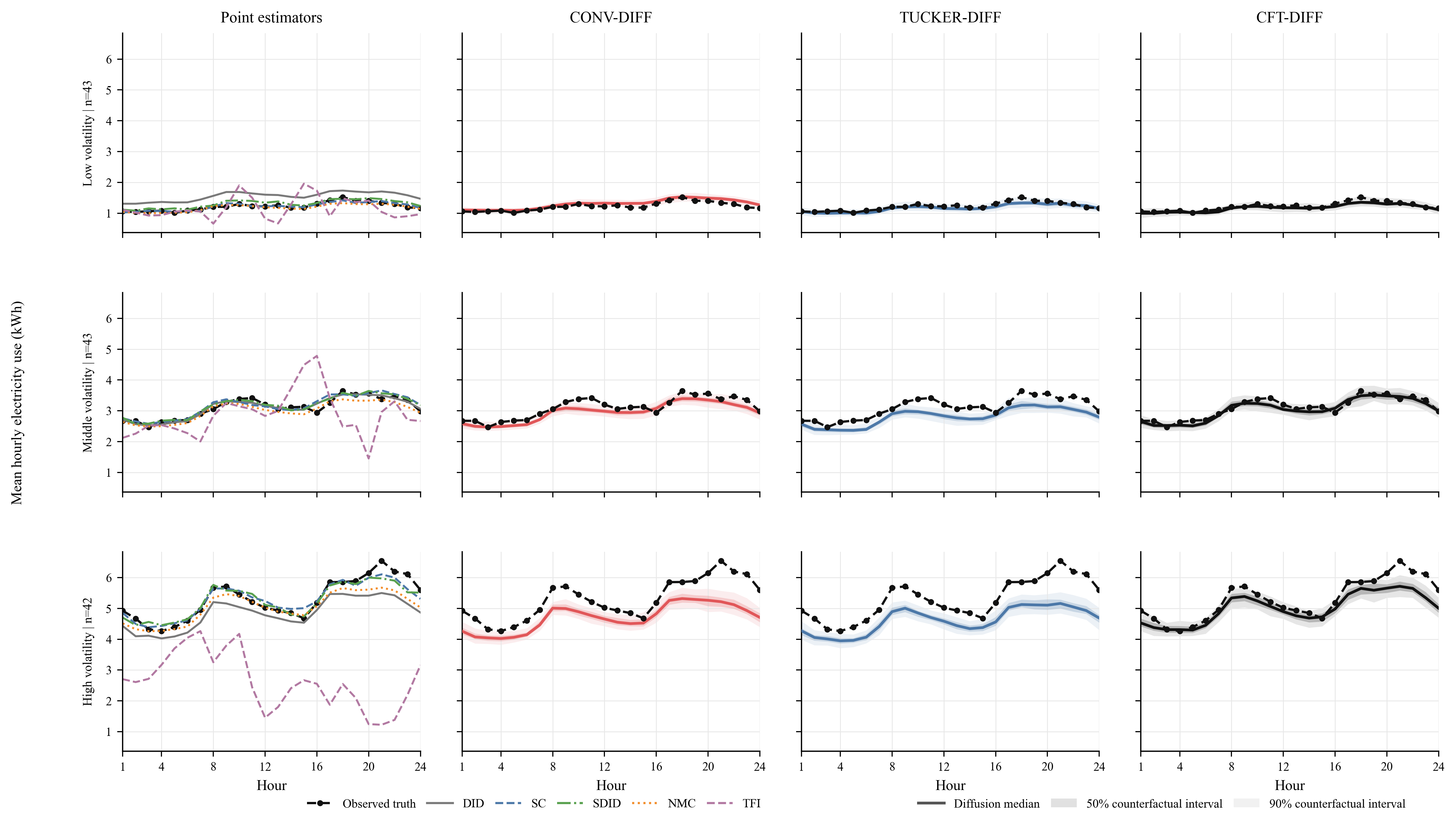}}	
	{Counterfactual trajectories across household demand-volatility groups under switchback missingness.
	\label{fig:trajectory_switchback_volatility}} 
	{The figure evaluates trajectory recovery across household groups with different levels of demand volatility under the switchback pattern. Households are grouped using demand variation measured from the observed portion of the data. Rows correspond to the resulting volatility groups. The first column reports the observed trajectory and the five point estimators; the remaining columns report the three diffusion methods together with their median trajectories and 50\% and 90\% central counterfactual prediction intervals.}
\end{figure}

\begin{figure}[htb]
	\FIGURE
	{\includegraphics[scale=0.52]{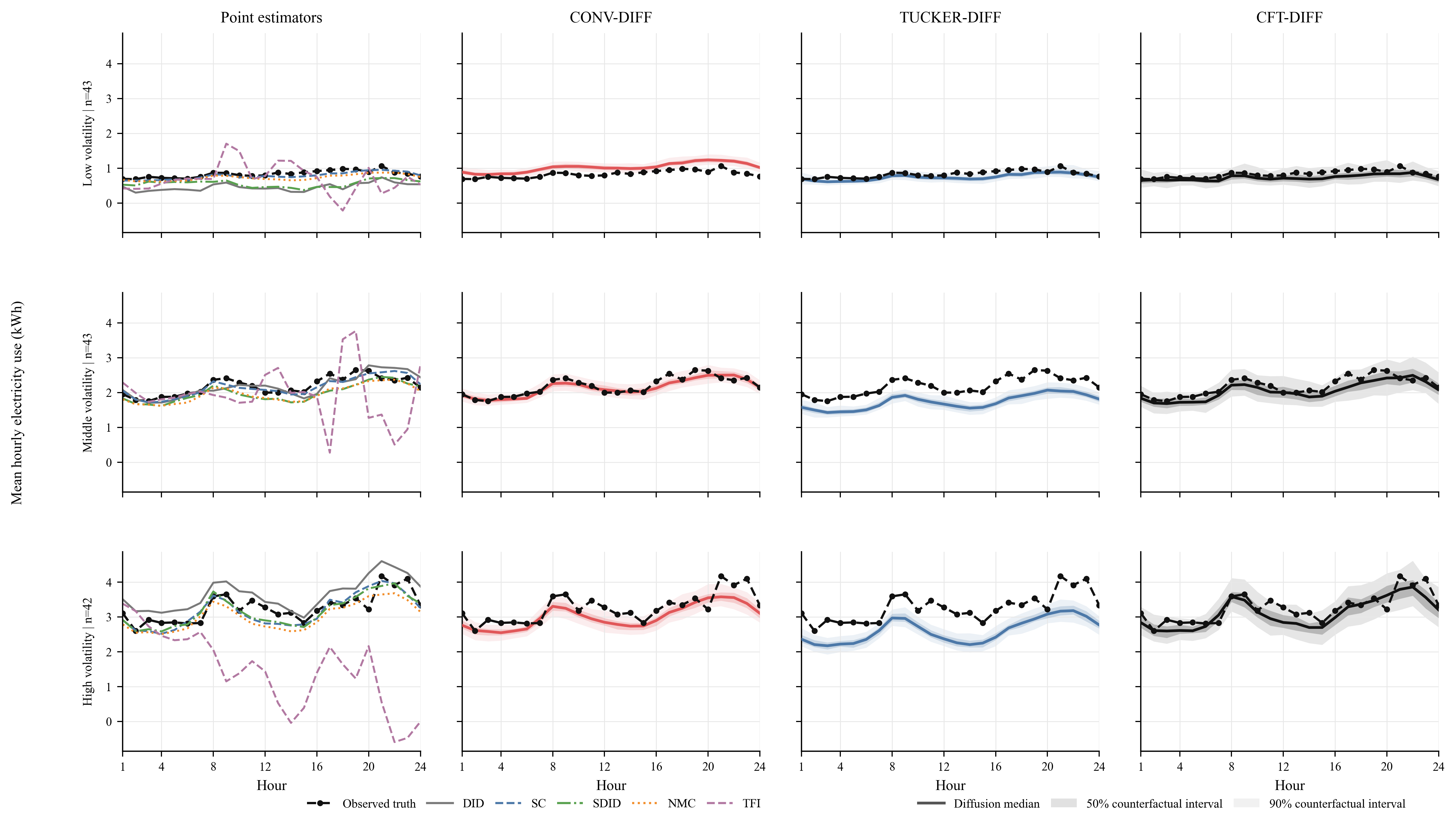}}	
	{Counterfactual trajectories across household demand-volatility groups under staggered adoption.
	\label{fig:trajectory_stagger_volatility}} 
	{The figure evaluates trajectory recovery across household groups with different levels of demand volatility under staggered adoption. Households are grouped using demand variation measured from the observed portion of the data. Rows correspond to the resulting volatility groups. The first column reports the observed trajectory and the five point estimators; the remaining columns report the three diffusion methods together with their median trajectories and 50\% and 90\% central counterfactual prediction intervals.}
\end{figure}

\begin{figure}[htb]
	\FIGURE
	{\includegraphics[scale=0.52]{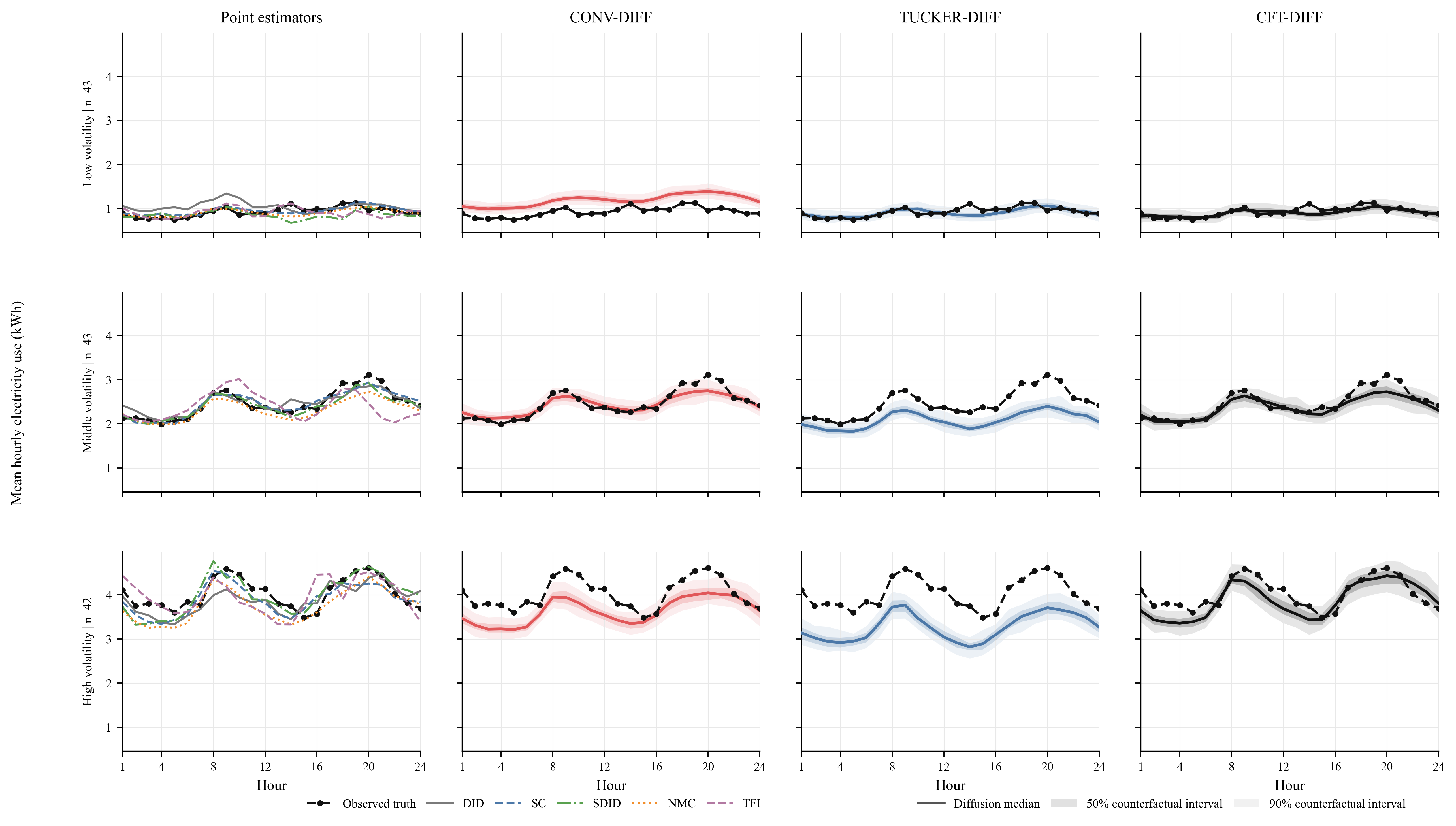}}	
	{Counterfactual trajectories across household demand-volatility groups under simultaneous adoption.
	\label{fig:trajectory_block_volatility}} 
	{The figure evaluates trajectory recovery across household groups with different levels of demand volatility under simultaneous adoption. Households are grouped using demand variation measured from the observed portion of the data. Rows correspond to the resulting volatility groups. The first column reports the observed trajectory and the five point estimators; the remaining columns report the three diffusion methods together with their median trajectories and 50\% and 90\% central counterfactual prediction intervals.}
\end{figure}

Figures~\ref{fig:trajectory_switchback_volatility}, \ref{fig:trajectory_stagger_volatility}, and \ref{fig:trajectory_block_volatility} stratify households by demand volatility measured from the observed outcomes. Differences across methods are relatively small for the low-volatility groups but become more pronounced as volatility increases. In the high-volatility groups, the observed trajectory displays larger intraday movements and more pronounced peaks. \ConvDiff and \TuckerDiff tend to attenuate these movements, whereas the \CFTDiff median remains much closer to the observed trajectory. The counterfactual bands also become wider for the more volatile groups, reflecting the greater uncertainty associated with recovering their missing trajectories. The same qualitative pattern appears under all three missingness structures.


\begin{figure}[htb]
	\FIGURE
	{\includegraphics[scale=0.52]{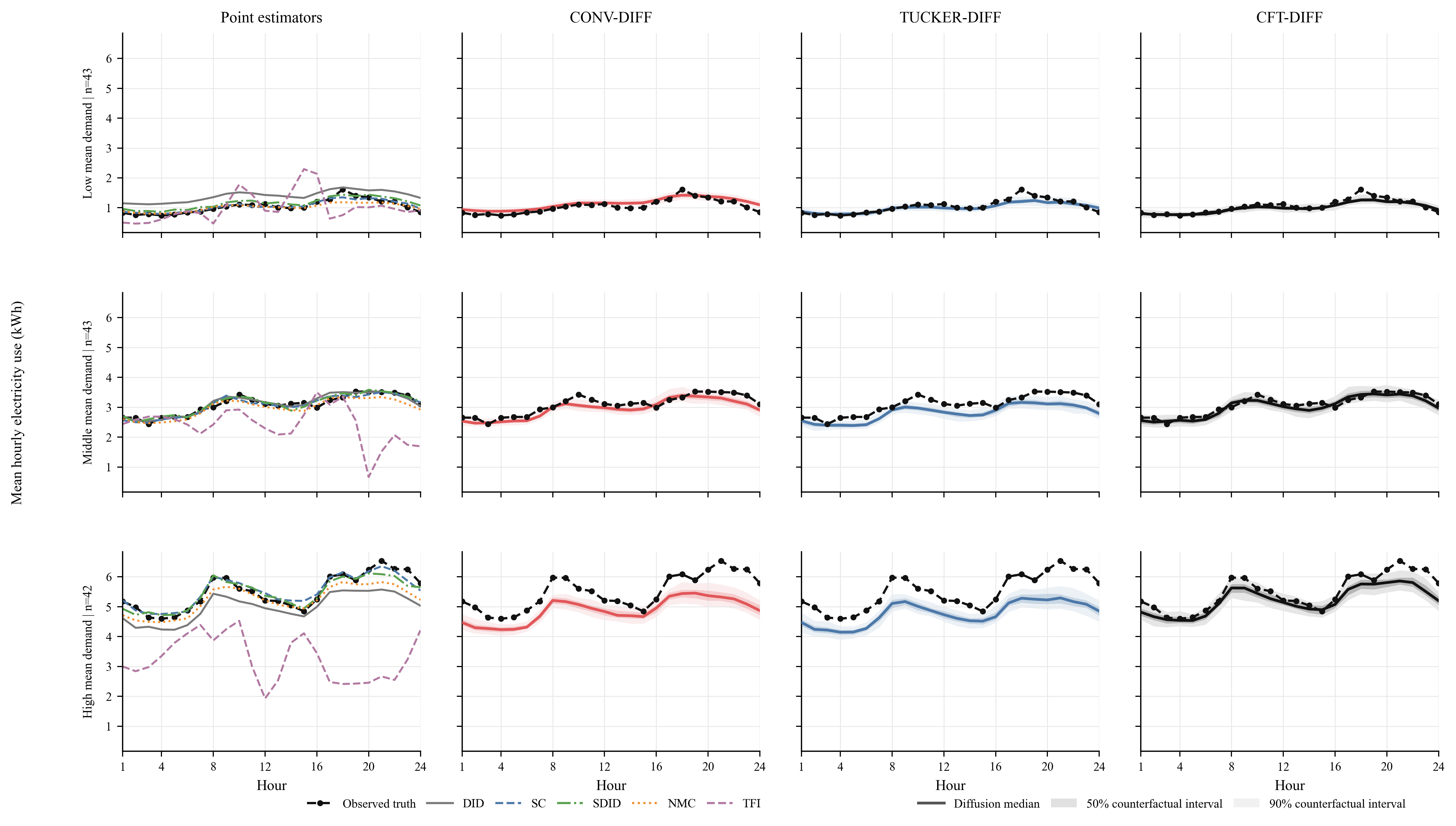}}	
	{Counterfactual trajectories across household demand-level groups under switchback missingness.
	\label{fig:trajectory_switchback_mean}} 
	{The figure evaluates trajectory recovery across household groups with different levels of average electricity use under the switchback pattern. Households are grouped according to their mean demand computed from the observed portion of the data. Rows correspond to the resulting demand-level groups. The first column reports the observed trajectory and the five point estimators; the remaining columns report the three diffusion methods together with their median trajectories and 50\% and 90\% central counterfactual prediction intervals.}
\end{figure}

\begin{figure}[htb]
	\FIGURE
	{\includegraphics[scale=0.52]{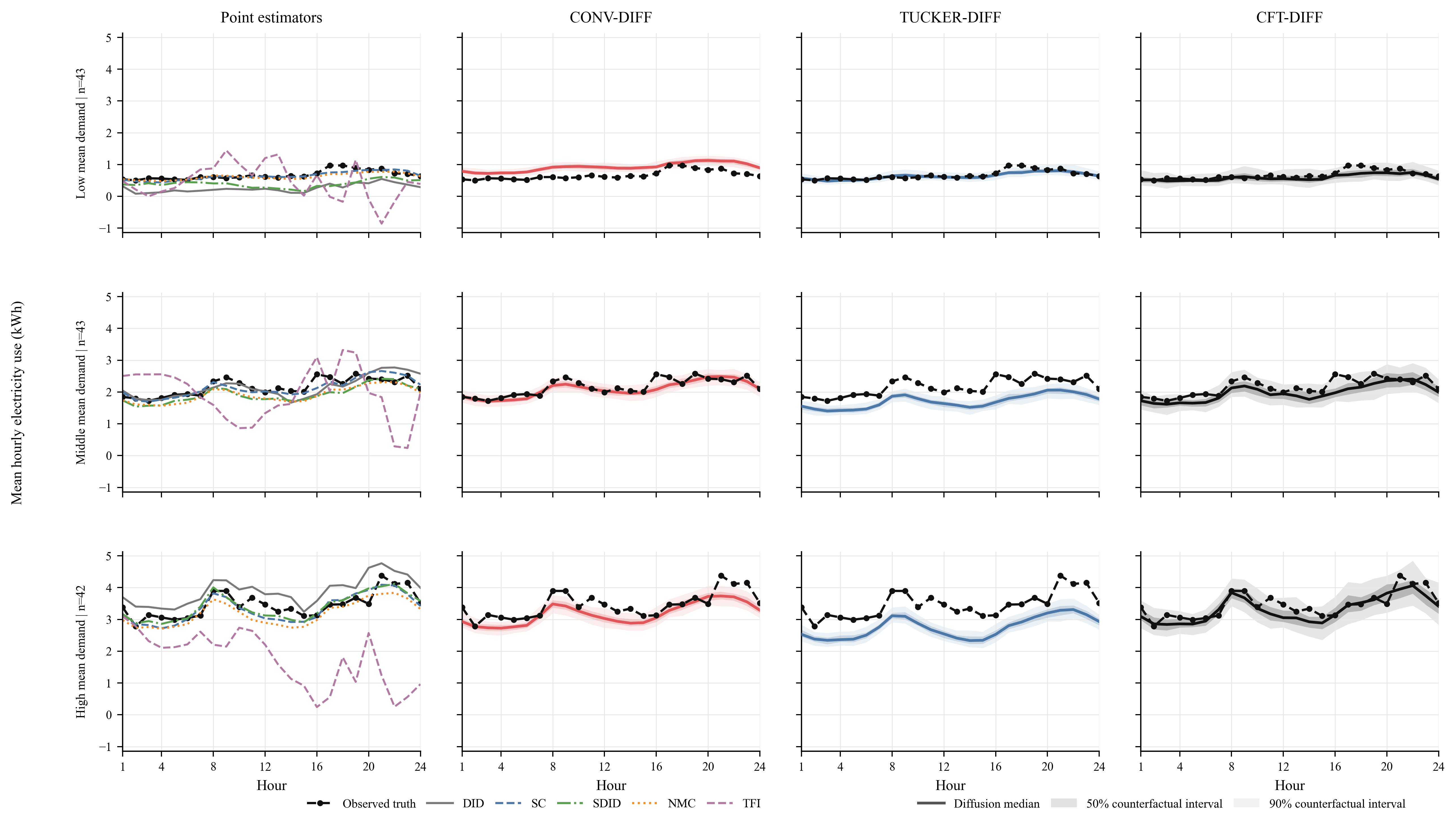}}	
	{Counterfactual trajectories across household demand-level groups under staggered adoption.
	\label{fig:trajectory_stagger_mean}} 
	{The figure evaluates trajectory recovery across household groups with different levels of average electricity use under staggered adoption. Households are grouped according to their mean demand computed from the observed portion of the data. Rows correspond to the resulting demand-level groups. The first column reports the observed trajectory and the five point estimators; the remaining columns report the three diffusion methods together with their median trajectories and 50\% and 90\% central counterfactual prediction intervals.}
\end{figure}

\begin{figure}[htb]
	\FIGURE
	{\includegraphics[scale=0.52]{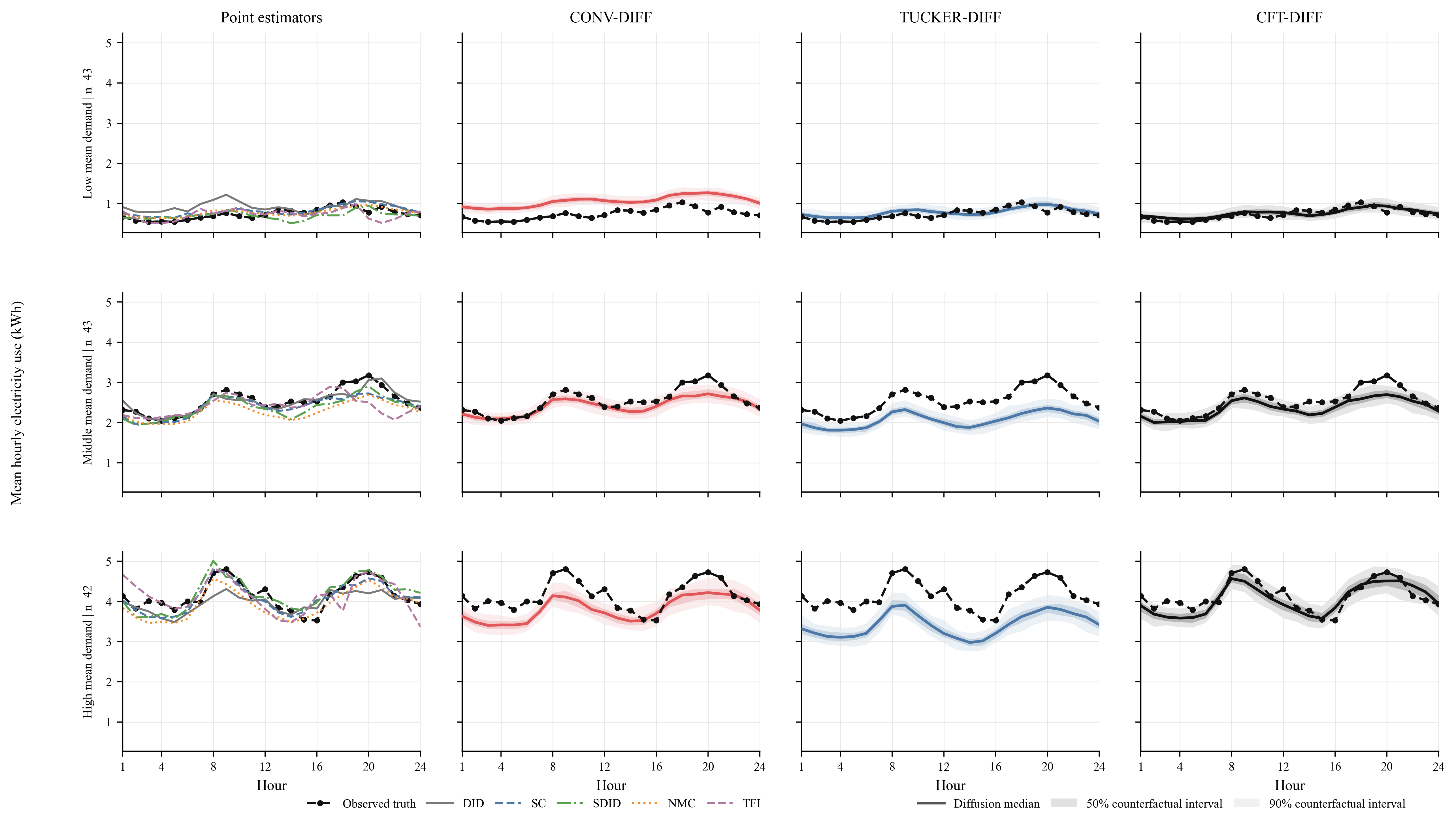}}	
	{Counterfactual trajectories across household demand-level groups under simultaneous adoption.
	\label{fig:trajectory_block_mean}} 
	{The figure evaluates trajectory recovery across household groups with different levels of average electricity use under simultaneous adoption. Households are grouped according to their mean demand computed from the observed portion of the data. Rows correspond to the resulting demand-level groups. The first column reports the observed trajectory and the five point estimators; the remaining columns report the three diffusion methods together with their median trajectories and 50\% and 90\% central counterfactual prediction intervals.}
\end{figure}

Figures~\ref{fig:trajectory_switchback_mean}, \ref{fig:trajectory_stagger_mean}, and \ref{fig:trajectory_block_mean} divide households into low-, middle-, and high-demand groups using their observed outcomes. Recovery is relatively similar across methods for the low-demand groups, where the trajectories are comparatively flat. The differences become clearer for the middle- and high-demand groups, whose trajectories exhibit larger morning and evening peaks. \CFTDiff adjusts to these changes in both level and shape and remains close to the observed trajectory. In contrast, the two nested diffusion benchmarks tend to understate the higher-demand portions of the curve, particularly around peak hours. The point estimators other than \TFI also track the broad demand profile reasonably well, but they do not provide the conditional distribution around the recovered trajectory. These results suggest that the gains from counterfactual conditioning are particularly useful when the missing trajectory has more pronounced systematic variation.


\begin{figure}[htb]
	\FIGURE
	{\includegraphics[scale=0.52]{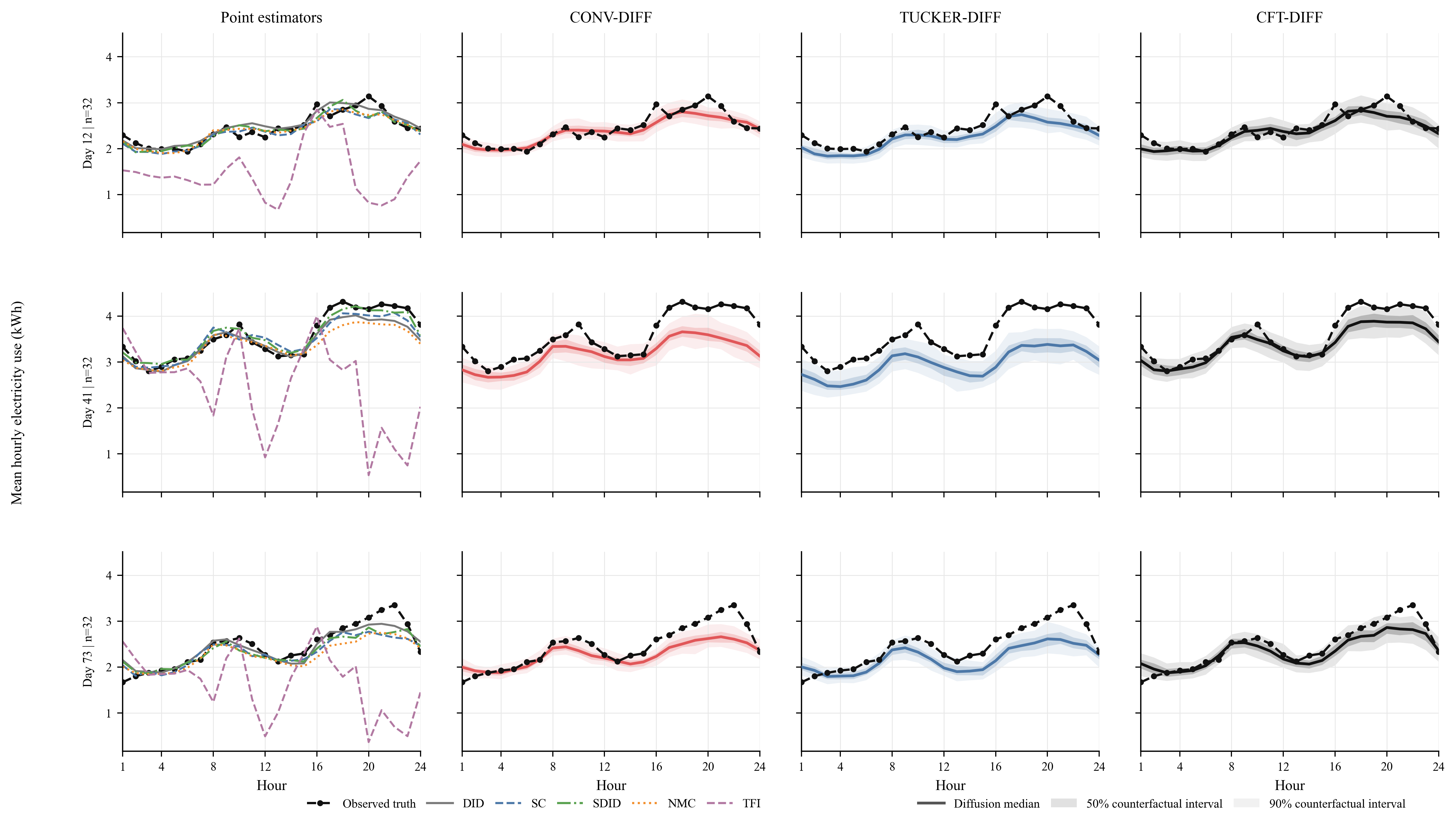}}	
	{Counterfactual trajectories across missing days under switchback missingness.
	\label{fig:trajectory_switchback_days}} 
	{The figure evaluates trajectory recovery on different artificially masked days under the switchback pattern. Each row corresponds to a different missing day and reports the average trajectory for the selected households on that day. The first column reports the observed trajectory and the five point estimators; the remaining columns report \ConvDiff, \TuckerDiff, and \CFTDiff, respectively. For each diffusion method, the solid line denotes the median generated trajectory and the shaded regions denote the 50\% and 90\% central counterfactual prediction intervals.}
\end{figure}

\begin{figure}[htb]
	\FIGURE
	{\includegraphics[scale=0.52]{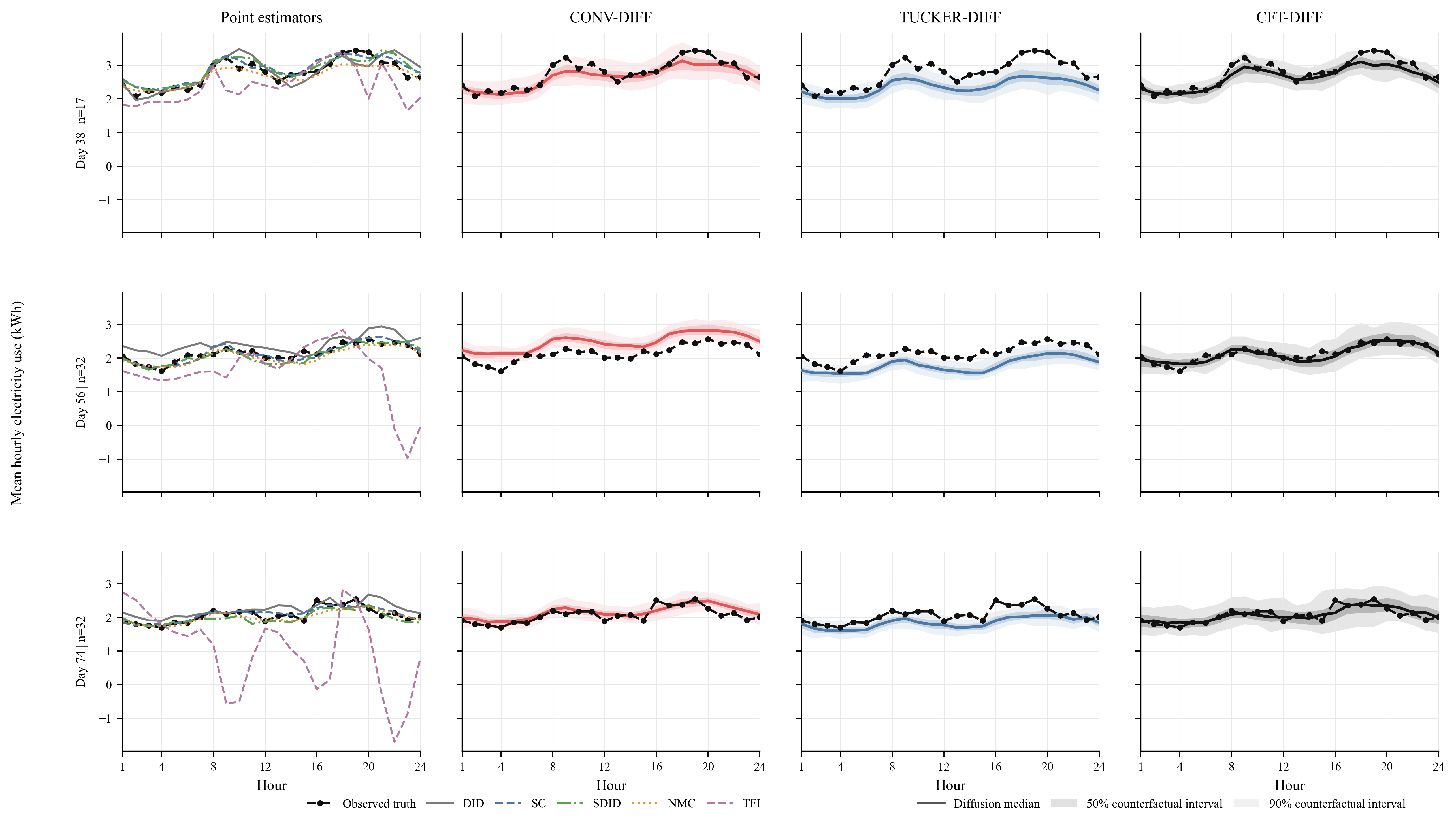}}	
	{Counterfactual trajectories across missing days under staggered adoption.
	\label{fig:trajectory_stagger_days}} 
	{The figure evaluates trajectory recovery on different artificially masked days under staggered adoption. Each row corresponds to a different missing day and reports the average trajectory for the selected households on that day. The first column reports the observed trajectory and the five point estimators; the remaining columns report \ConvDiff, \TuckerDiff, and \CFTDiff, respectively. For each diffusion method, the solid line denotes the median generated trajectory and the shaded regions denote the 50\% and 90\% central counterfactual prediction intervals.}
\end{figure}

\begin{figure}[htb]
	\FIGURE
	{\includegraphics[scale=0.52]{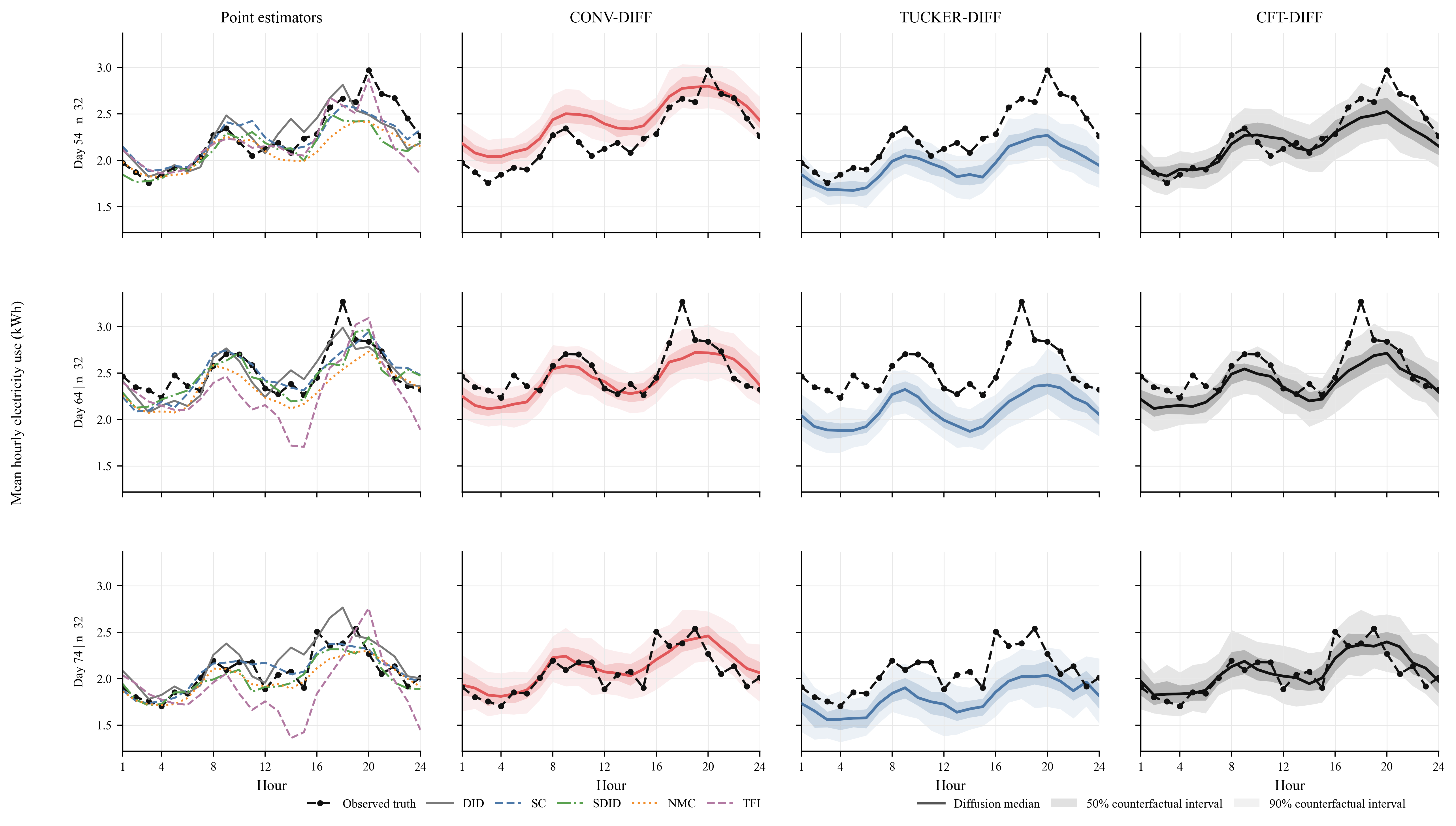}}	
	{Counterfactual trajectories across missing days under simultaneous adoption.
	\label{fig:trajectory_block_days}} 
	{The figure evaluates trajectory recovery on different artificially masked days under simultaneous adoption. Each row corresponds to a different missing day and reports the average trajectory for the selected households on that day. The first column reports the observed trajectory and the five point estimators; the remaining columns report \ConvDiff, \TuckerDiff, and \CFTDiff, respectively. For each diffusion method, the solid line denotes the median generated trajectory and the shaded regions denote the 50\% and 90\% central counterfactual prediction intervals.}
\end{figure}

Figures~\ref{fig:trajectory_switchback_days}, \ref{fig:trajectory_stagger_days}, and \ref{fig:trajectory_block_days} repeat the comparison across different artificially masked days. The selected days differ in both average electricity use and the shape of the intraday profile. Nevertheless, the same broad ranking remains visible. \CFTDiff closely follows changes in the level, timing, and magnitude of the observed demand peaks across days, while \ConvDiff and \TuckerDiff generally produce smoother trajectories and can understate relatively sharp movements. Several point estimators remain competitive for the average trajectory, although \TFI again exhibits substantial instability on some days. Overall, the trajectory evidence complements the scalar recovery measures in Section~\ref{sec:appendix_additional_recovery_results}: the advantage of \CFTDiff is not limited to average error reductions, but is also reflected in its ability to recover the shape of the missing 24-hour counterfactual trajectory across different aggregation levels, household groups, and time periods.

\subsection{Additional Results on Demand Response}\label{sec:appendix_treatment_results}

This section provides additional analyses of the demand responses reported in Section~\ref{sec:demand_response}. We examine hourly response heterogeneity, price intensity and profile design, repeated intervention exposure, and heterogeneity across temperature, regions, and household characteristics.

\begin{figure}[htb]
	\FIGURE
	{\includegraphics[scale=0.44]{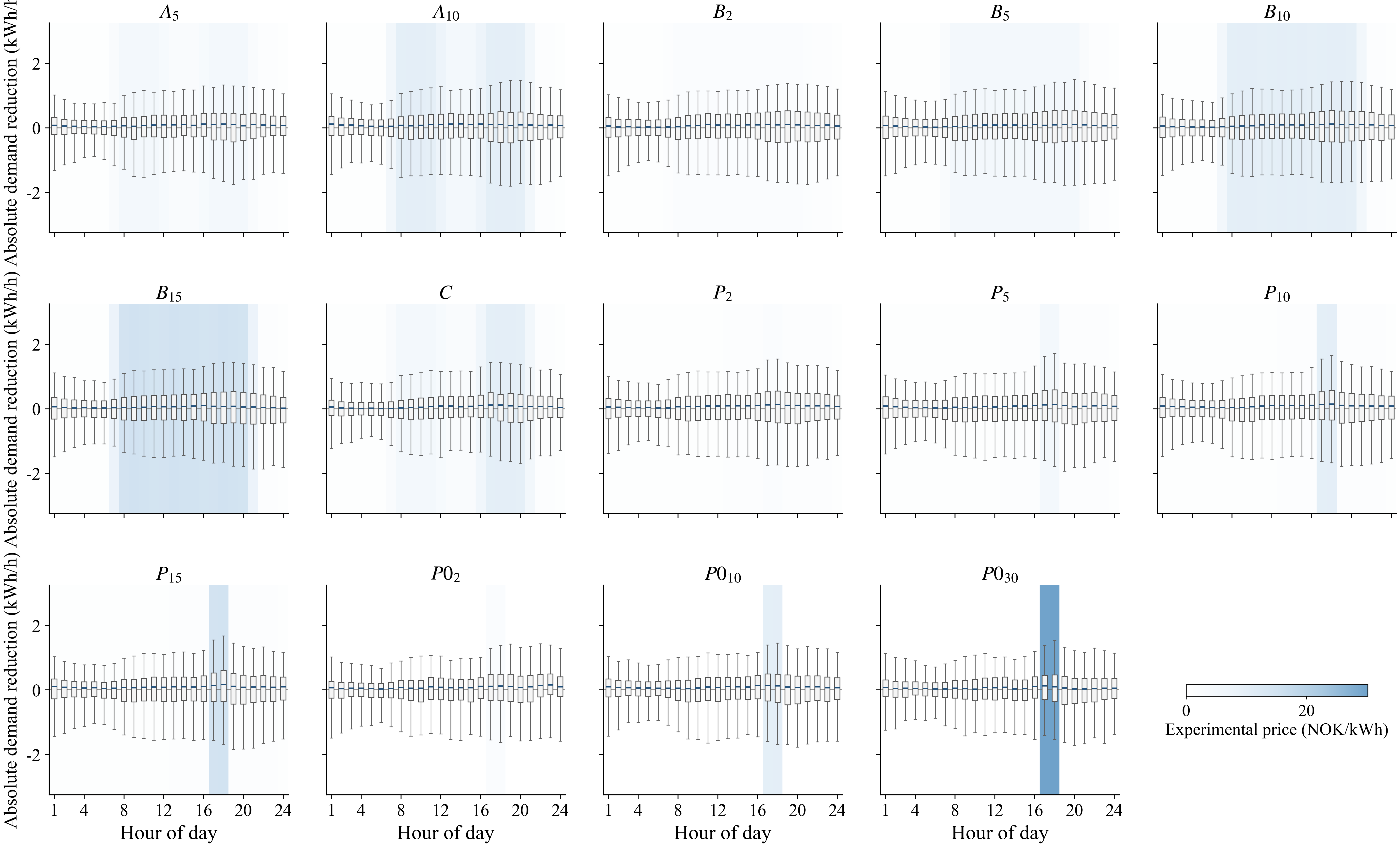}}	
	{Distribution of hourly demand reductions by price signal.
	\label{fig:causal_hourly_effect_distribution}} 
    {Each boxplot reports the distribution of hourly demand reductions across household--policy-day observations. The no-policy counterfactual is first averaged over 100 \CFTDiff conditional draws, and the demand reduction is defined as counterfactual demand minus observed demand, so positive values indicate lower demand under the assigned price signal. Boxes denote the interquartile range, the center line denotes the median, and whiskers extend from the 5th to the 95th percentile. Individual outliers are omitted. Background shading indicates the experimental electricity price.}
\end{figure}

Aggregate effects can conceal substantial variation across treated observations. Figure~\ref{fig:causal_hourly_effect_distribution} therefore reports the distribution of hourly demand reductions across household--policy-day observations, using the mean of 100 \CFTDiff conditional draws for each observation. The medians are generally close to zero, while the distributions remain wide at many hours, indicating considerable variation across households and intervention days. This variation is especially visible during the daytime and evening periods, when electricity use is relatively high. The aggregate estimates therefore do not characterize all treated observations uniformly: under the same price signal, some household--day observations exhibit sizable demand reductions, whereas others show little change or an increase at particular hours.

\begin{figure}[htb]
	\FIGURE
	{\includegraphics[scale=0.45]{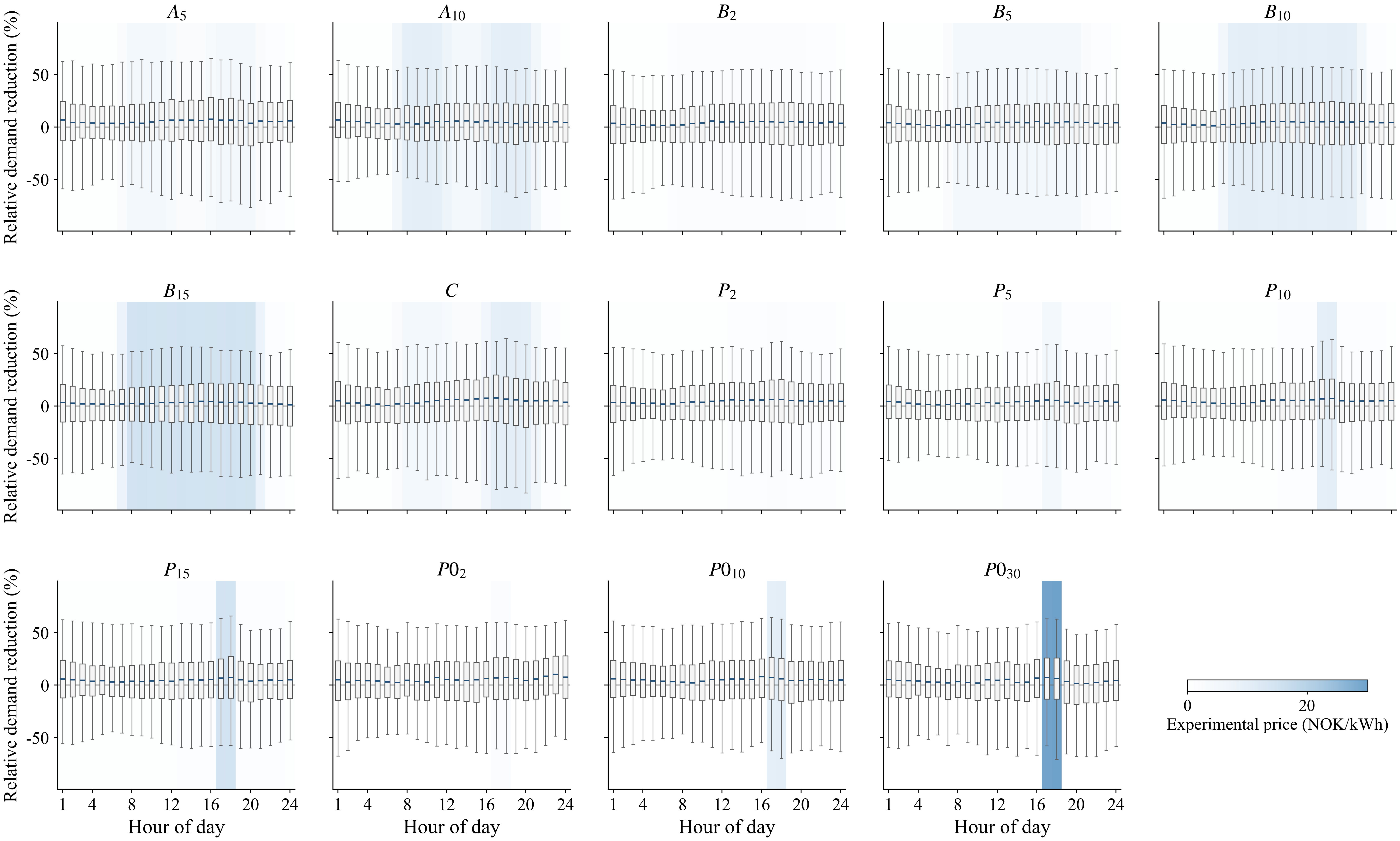}}	
	{Distribution of hourly relative demand reductions by price signal.
	\label{fig:appendix_hourly_relative_effect_distribution}} 
	{Each boxplot reports the distribution of hourly relative demand reductions across household--policy-day observations. For each observation, the no-policy counterfactual is first averaged over the 100 \CFTDiff draws. Positive values indicate lower observed demand relative to the estimated no-policy counterfactual. Boxes denote the interquartile range, the center line denotes the median, and whiskers extend from the 5th to the 95th percentile. Individual outliers are omitted. Background shading indicates the experimental electricity price.}
\end{figure}

Figure~\ref{fig:appendix_hourly_relative_effect_distribution} reports the corresponding distributions of relative demand reductions. The results are qualitatively similar to the absolute-effect distributions in Figure~\ref{fig:appendix_hourly_relative_effect_distribution}. The medians remain generally close to zero, while the distributions are wide across households and intervention days, indicating substantial heterogeneity in hourly demand responses. Thus, expressing the effects relative to the estimated no-policy counterfactual does not materially change the main patterns documented in Section~\ref{sec:demand_response}.

\begin{figure}[htb]
	\FIGURE
	{\includegraphics[scale=0.45]{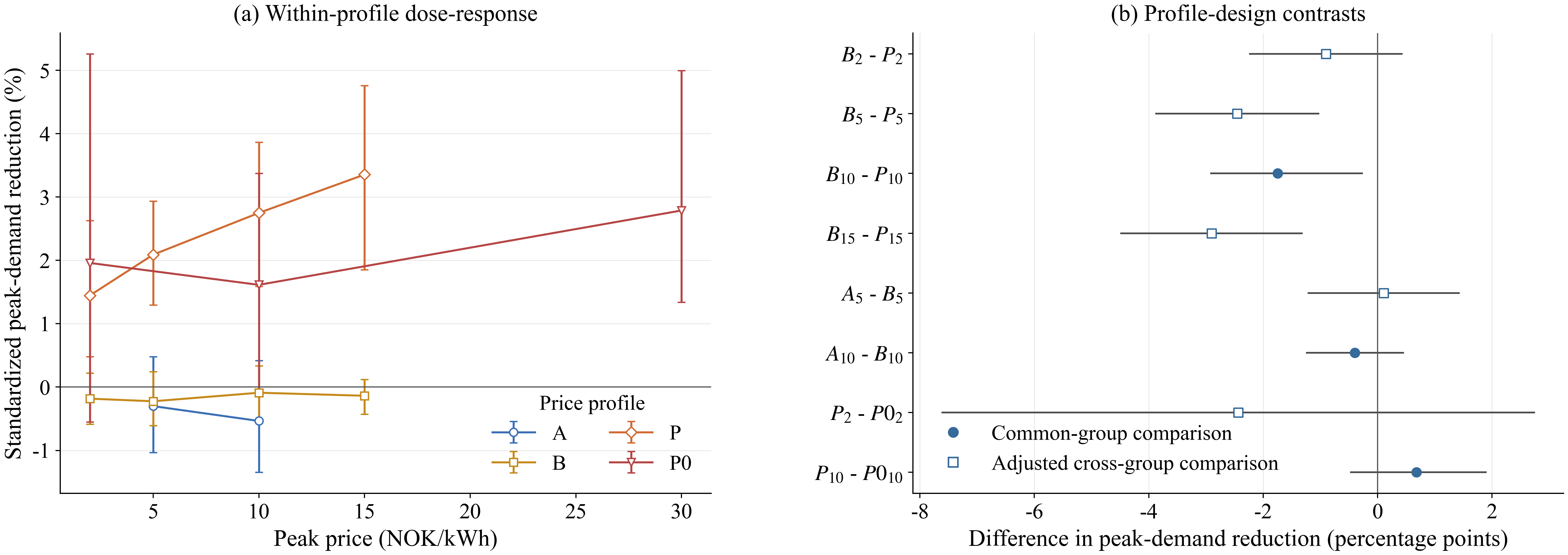}}	
	{Demand response across price intensity and profile design.
	\label{fig:appendix_price_intensity_design}} 
	{Panel (a) reports peak-period relative demand reductions across peak-price levels within each price profile, using experimental groups observed under all price levels within the corresponding profile. Profile C is omitted because it does not have a single peak-price level. Panel (b) reports pairwise differences in peak-period relative demand reductions across price-signal designs. Filled circles denote common-group comparisons, and open squares denote covariate-adjusted cross-group comparisons. Error bars report 95\% intervals. In Panel (b), positive values indicate a larger demand reduction for the first price signal in each comparison.}
\end{figure}

Figure~\ref{fig:appendix_price_intensity_design} further examines whether the demand response varies with price intensity and profile design. Panel~(a) shows that the peak-demand reduction increases with the peak price for profile $P$. The estimates for profile $B$ remain close to zero, while the responses for profiles $A$ and $P0$ are not monotone and several estimates remain imprecise. Panel~(b) shows larger estimated reductions for the $P$ profile than for the B profile at peak prices of 5, 10, and 15 NOK/kWh, with the corresponding 95\% intervals excluding zero. Most other profile-design comparisons remain imprecise, with intervals including zero. These results indicate that the relationship between price intensity and demand response depends on the price profile, with the clearest dose--response pattern occurring for profile $P$.

\begin{figure}[htb]
	\FIGURE
	{\includegraphics[scale=0.55]{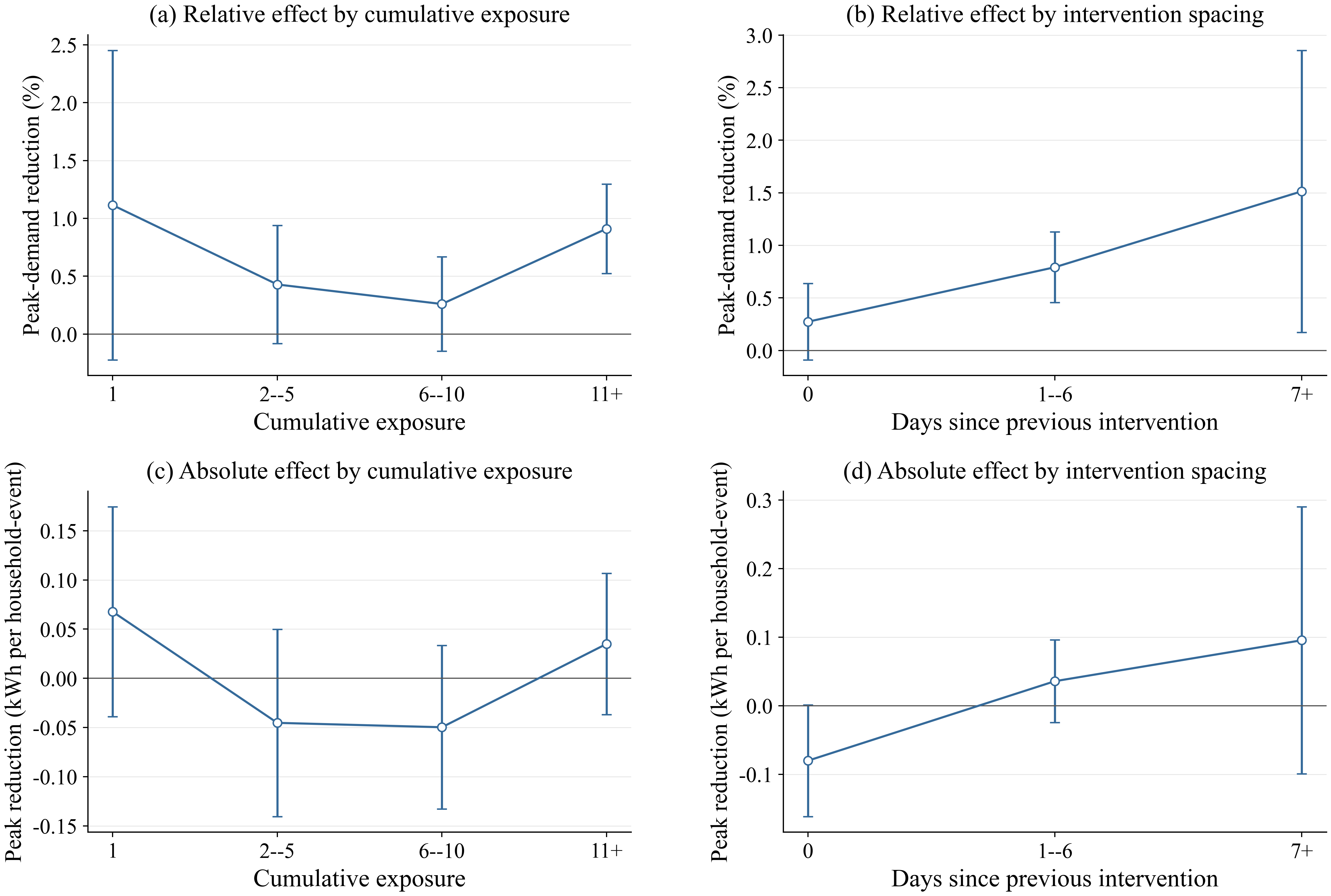}}	
	{Demand response by cumulative exposure and intervention spacing.
	\label{fig:appendix_repeated_exposure}} 
	{Panels (a) and (c) report relative and absolute peak-demand reductions by cumulative intervention exposure, grouped into 1, 2--5, 6--10, and 11 or more prior and current intervention events. Panels (b) and (d) report the corresponding effects by the number of days since the previous intervention, grouped into 0, 1--6, and 7 or more days. Estimates adjust for price signal, region, and baseline demand, as well as temperature when available; the intervention-spacing specifications additionally control for intervention order. Error bars report 95\% intervals. Positive values indicate demand reductions.}
\end{figure}

Figure~\ref{fig:appendix_repeated_exposure} next examines whether the demand response changes with repeated intervention exposure. Panels~(a) and~(c) show no monotone relationship between the response and cumulative exposure. The estimates decline after the first intervention, remain similar for the 2--5 and 6--10 exposure groups, and then increase for the 11+ group. Panels~(b) and~(d) show a clearer pattern for intervention spacing: the estimated peak-demand reduction increases with the time since the previous intervention and is largest for the 7+ day group, although the intervals remain wide. Taken together, these results provide limited evidence of a systematic learning or fatigue pattern across repeated interventions, but suggest a stronger response when interventions are spaced farther apart.

The preceding analysis shows that demand response varies with the design and timing of the intervention. We next examine whether it also varies with outdoor temperature, an important source of variation in residential electricity demand. We classify the previous 24-hour average outdoor temperature into six intervals: $\leq -15^\circ$C, $(-15,-10]^\circ$C, $(-10,-5]^\circ$C, $(-5,0]^\circ$C, $(0,5]^\circ$C, and $>5^\circ$C. We first estimate the average peak-demand response within each temperature interval, controlling for price signal, region, baseline demand, and intervention order. We then examine whether the temperature relationship differs across price-profile designs. For this analysis, the price signals are grouped into three profile classes: short profiles ($P$ and $P0$), extended profiles ($B$), and dual-peak profiles ($A$ and $C$). Panels (b) and (d) report the corresponding temperature-response relationships for these three profile classes, controlling for region, baseline demand, and intervention order. The analysis is conducted at the region--intervention-event level, with observations weighted by the number of households in each event.

\begin{figure}[htb]
	\FIGURE
	{\includegraphics[scale=0.52]{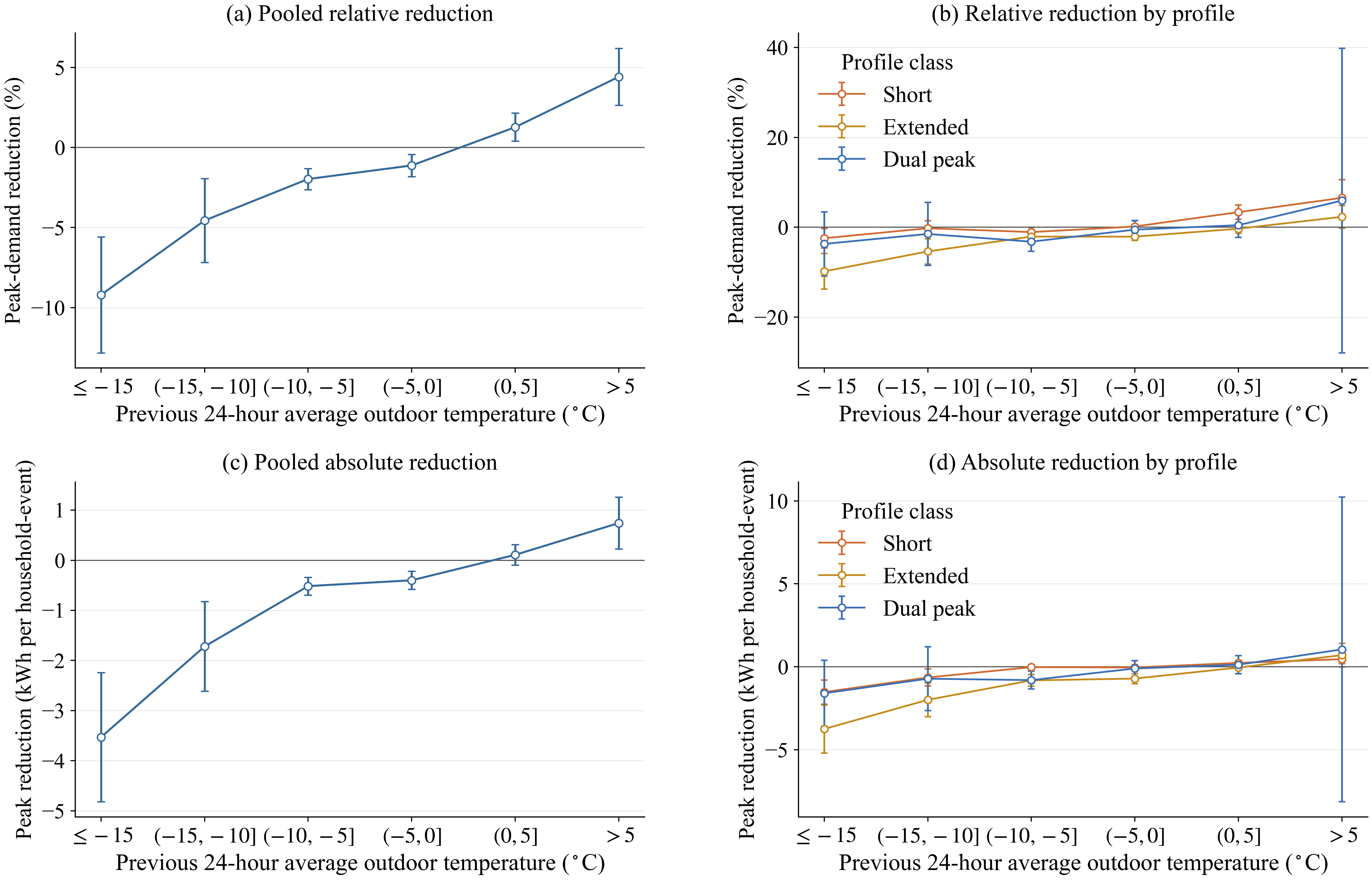}}	
	{Demand response by outdoor temperature.
	\label{fig:appendix_temperature_heterogeneity}} 
	{Panels (a) and (c) report pooled relative and absolute peak-demand reductions across outdoor-temperature bins. Panels (b) and (d) report the corresponding estimates separately for short, extended, and dual-peak price profiles. The pooled estimates adjust for price signal, region, baseline demand, and intervention order; the profile-specific estimates adjust for region, baseline demand, and intervention order. Error bars report 95\% intervals. Positive values indicate demand reductions.}
\end{figure}

Figure~\ref{fig:appendix_temperature_heterogeneity} shows a clear relationship between outdoor temperature and demand response. Panels (a) and (c) reveal a strong monotone pattern in the pooled estimates: as outdoor temperature increases, both relative and absolute peak-demand reductions increase steadily. At temperatures below $-15^\circ$C, the estimated response is strongly negative, whereas it approaches zero around $0^\circ$C and becomes positive at higher temperatures. Panels (b) and (d) show that this pattern is not driven by a single price-profile design. The short, extended, and dual-peak profiles all exhibit generally increasing demand reductions as temperature rises, although the estimates are less precise for some temperature--profile combinations. Overall, the results indicate that households respond more strongly to dynamic pricing under warmer conditions, while peak-period electricity use is less responsive during very cold periods.

\begin{figure}[htb]
	\FIGURE
	{\includegraphics[scale=0.55]{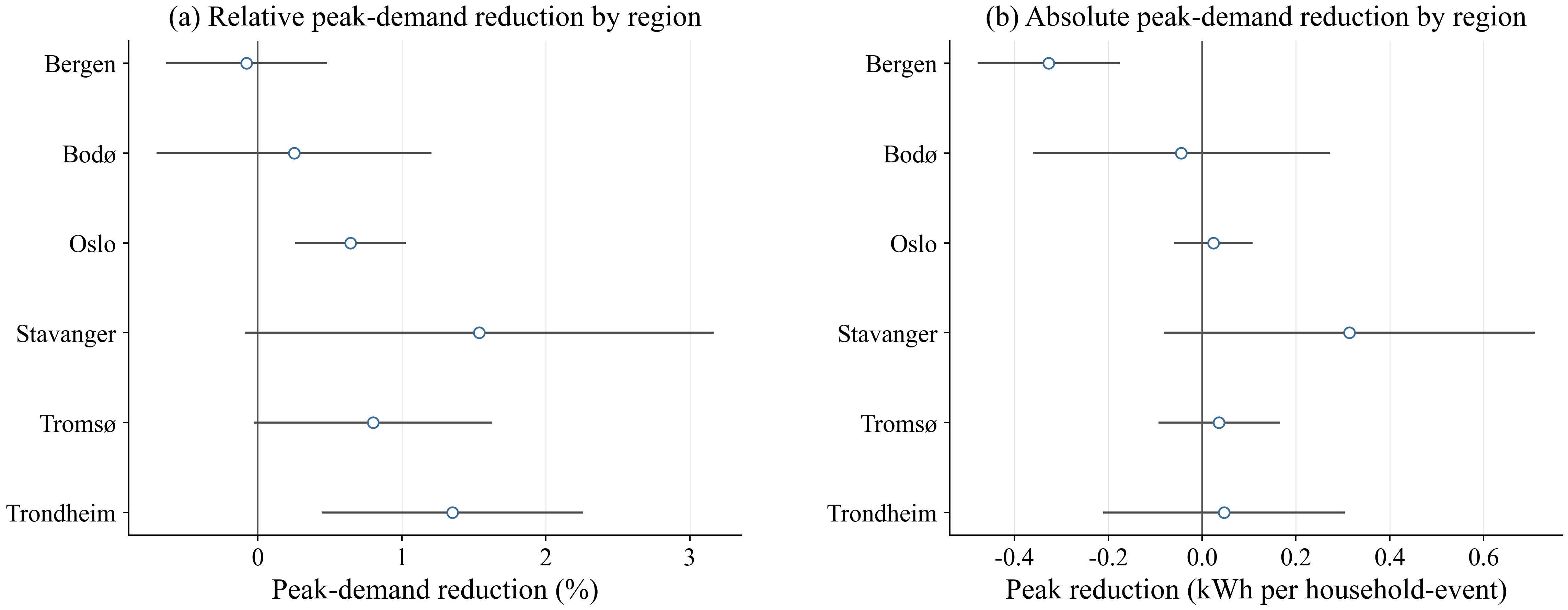}}	
	{Regional heterogeneity in peak-demand reductions.
	\label{fig:appendix_region_heterogeneity}} 
	{Panel (a) reports relative peak-demand reductions by region, and Panel (b) reports the corresponding absolute reductions. Estimates adjust for price signal, baseline demand, intervention order, and temperature when available. Error bars report 95\% intervals. Positive values indicate demand reductions.}
\end{figure}

Given the strong temperature pattern in Figure~Figure~\ref{fig:appendix_temperature_heterogeneity}, we next examine whether regional differences remain after controlling for temperature. Figure~\ref{fig:appendix_region_heterogeneity} shows that some regional heterogeneity persists. Bergen has a relative-effect estimate close to zero but a negative absolute-effect estimate whose confidence interval excludes zero. Relative peak-demand reductions are positive in the other regions and are largest in Stavanger and Trondheim, although the estimate for Stavanger is imprecise. The corresponding intervals exclude zero for Oslo and Trondheim. In absolute terms, the estimates for the other regions are close to zero or positive, with relatively wide confidence intervals. The estimates therefore do not indicate a simple geographic pattern. Because the specifications already control for outdoor temperature, these differences are unlikely to reflect temperature alone. They may instead capture remaining regional differences in household characteristics, housing conditions, heating technologies, or electricity-use patterns that are not separately identified in our data. We therefore interpret the results as evidence of residual regional heterogeneity rather than as causal effects of location.

\begin{figure}[htb]
	\FIGURE
	{\includegraphics[scale=0.48]{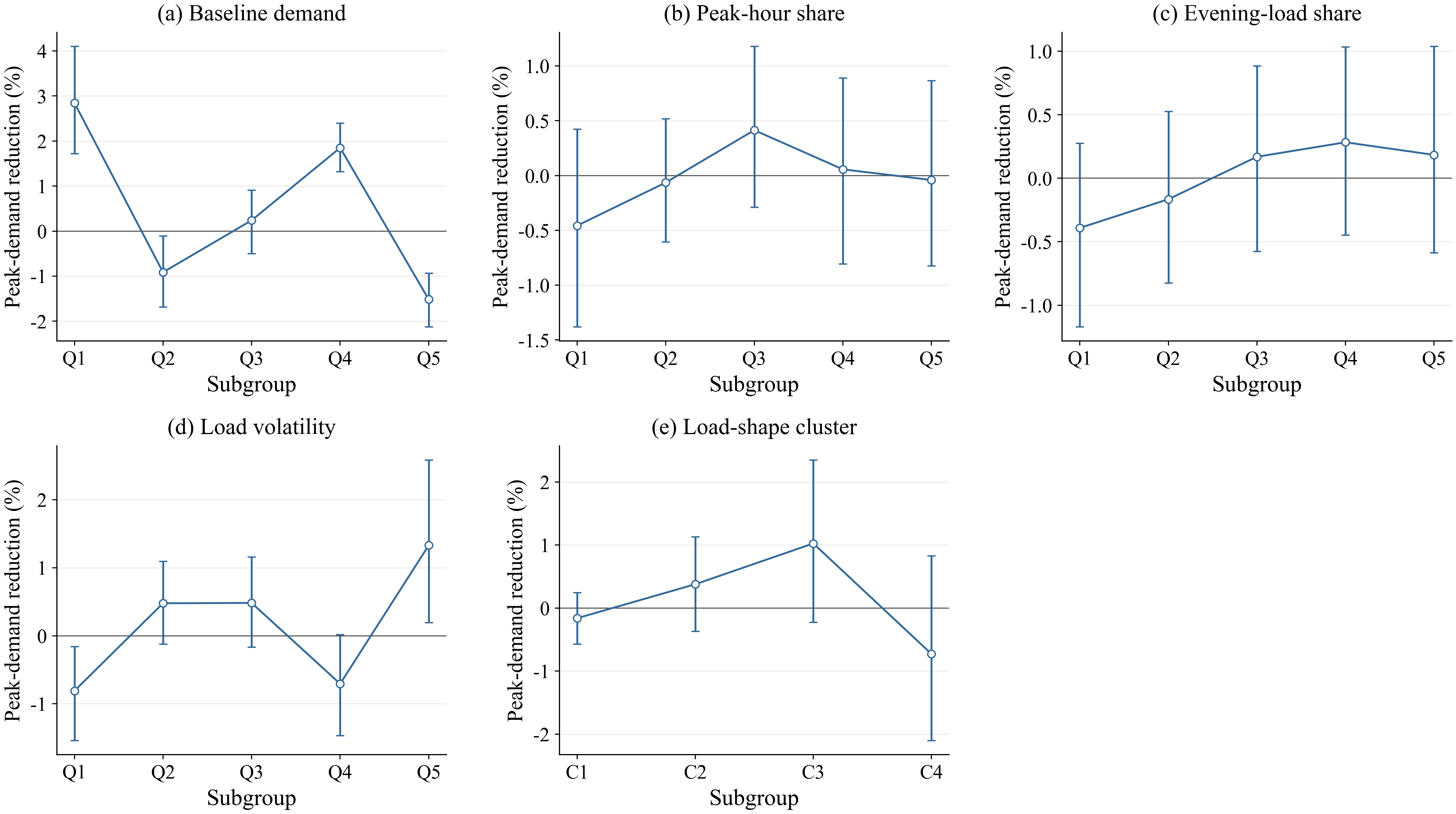}}	
	{Household heterogeneity in peak-demand reductions.
	\label{fig:appendix_household_heterogeneity}} 
	{Panels (a)--(e) report relative peak-demand reductions across subgroups defined by pre-intervention baseline demand, peak-hour load share, evening-load share, load volatility, and load shape, respectively. Household characteristics are constructed using electricity use observed before each household's first intervention. Q1--Q5 denote quintiles, and C1--C4 denote clusters of normalized 24-hour pre-intervention load profiles. Error bars report 95\% intervals obtained from a household bootstrap. Positive values indicate demand reductions.}
\end{figure}

We finally examine heterogeneity across household electricity-use characteristics measured before the first intervention. For each household, we construct four pre-intervention characteristics: average electricity demand, the share of electricity use during peak hours, the share during evening hours, and load volatility. Households are ranked separately by each characteristic and divided into five equally sized groups, with Q1 denoting the lowest quintile and Q5 the highest. We also construct each household's average 24-hour pre-intervention load profile and normalize it by total daily use, so that the resulting profile captures the timing rather than the level of electricity consumption. We then apply $K$-means clustering to these normalized profiles and classify households into four load-shape clusters, C1--C4. The cluster labels are ordered by the center of the daily load profile, from earlier to later concentration of electricity use. These classifications are constructed using only electricity use observed before treatment and therefore do not depend on the estimated treatment effects.

Figure~\ref{fig:appendix_household_heterogeneity} shows substantial heterogeneity in peak-demand response across these household groups, although the patterns are generally nonmonotone. Panel~(a) exhibits the largest differences across baseline-demand quintiles: households in Q1 and Q4 show positive peak-demand reductions, whereas Q2 and Q5 show negative responses and Q3 is close to zero. The differences across peak-hour and evening-load shares in Panels~(b) and~(c) are more modest, and all corresponding confidence intervals include zero. Panel~(d) also shows a nonmonotone pattern across load volatility, with Q5 exhibiting the largest positive response. Panel~(e) shows variation across intraday load-shape clusters, with the highest point estimate for C3 and the lowest for C4, although the intervals are wide and overlap substantially. Overall, the results indicate meaningful household heterogeneity in demand response, but no single pre-intervention characteristic provides a simple ordering of households by their response to dynamic pricing.

\begin{table}[htbp]
\TABLE
{Detailed estimates of household heterogeneity in peak-demand reductions.
\label{tab:appendix_household_heterogeneity}}
{\begin{tabular}{llrrrr}
\hline
Characteristic & Subgroup & Households & Relative \% & 95\% CI & Absolute kWh \\
\hline
Baseline demand quintile & 1 & 621 & 2.84 & [1.72, 4.10] & 0.168 \\
Baseline demand quintile & 2 & 620 & -0.91 & [-1.68, -0.10] & -0.133 \\
Baseline demand quintile & 3 & 621 & 0.24 & [-0.50, 0.91] & 0.051 \\
Baseline demand quintile & 4 & 620 & 1.85 & [1.32, 2.40] & 0.581 \\
Baseline demand quintile & 5 & 621 & -1.52 & [-2.13, -0.94] & -0.676 \\
Peak-hour share quintile & 1 & 621 & -0.46 & [-1.38, 0.42] & -0.103 \\
Peak-hour share quintile & 2 & 620 & -0.06 & [-0.60, 0.52] & -0.017 \\
Peak-hour share quintile & 3 & 621 & 0.41 & [-0.29, 1.18] & 0.108 \\
Peak-hour share quintile & 4 & 620 & 0.06 & [-0.80, 0.89] & 0.013 \\
Peak-hour share quintile & 5 & 621 & -0.04 & [-0.82, 0.87] & -0.008 \\
Evening-load share quintile & 1 & 621 & -0.39 & [-1.17, 0.27] & -0.100 \\
Evening-load share quintile & 2 & 620 & -0.17 & [-0.83, 0.53] & -0.045 \\
Evening-load share quintile & 3 & 621 & 0.17 & [-0.58, 0.88] & 0.046 \\
Evening-load share quintile & 4 & 620 & 0.28 & [-0.45, 1.04] & 0.066 \\
Evening-load share quintile & 5 & 621 & 0.18 & [-0.59, 1.04] & 0.026 \\
Load-volatility quintile & 1 & 621 & -0.81 & [-1.54, -0.16] & -0.258 \\
Load-volatility quintile & 2 & 620 & 0.48 & [-0.12, 1.10] & 0.142 \\
Load-volatility quintile & 3 & 621 & 0.48 & [-0.17, 1.16] & 0.121 \\
Load-volatility quintile & 4 & 620 & -0.71 & [-1.47, 0.02] & -0.147 \\
Load-volatility quintile & 5 & 621 & 1.33 & [0.19, 2.58] & 0.133 \\
Load-shape cluster & 1 & 1893 & -0.16 & [-0.57, 0.24] & -0.043 \\
Load-shape cluster & 2 & 521 & 0.38 & [-0.37, 1.13] & 0.097 \\
Load-shape cluster & 3 & 392 & 1.02 & [-0.23, 2.35] & 0.140 \\
Load-shape cluster & 4 & 297 & -0.73 & [-2.10, 0.83] & -0.098 \\
\hline
\end{tabular}}
{Subgroups are constructed from electricity use observed before each household's first intervention. Relative effects are pooled ratio-of-sums estimates. Absolute effects are kWh reductions per household--policy day over the corresponding signal-specific peak-price hours. Confidence intervals are obtained from a household bootstrap. Positive values indicate demand reductions.}
\end{table}

Table~\ref{tab:appendix_household_heterogeneity} provides the corresponding estimates for the household groups in Figure~\ref{fig:appendix_household_heterogeneity}. As described above, the four continuous pre-intervention characteristics are divided into quintiles across treated households, while normalized 24-hour load profiles are classified into four load-shape clusters. For each household, we first average the 100 \CFTDiff draws to obtain the estimated no-policy counterfactual and then aggregate the counterfactual and observed demand over all of that household's intervention days and the corresponding signal-specific peak-price hours. For a subgroup $\mathcal{G}$, let $Y_i^{0}$ and $Y_i^{\mathrm{obs}}$ denote these aggregated counterfactual and observed outcomes for household $i$, and let $n_i$ denote its number of intervention days. We report the pooled relative peak-demand reduction
\[
\widehat{R}_{\mathcal{G}}=100\,\frac{\sum_{i\in\mathcal{G}}\left(Y_i^{0}-Y_i^{\mathrm{obs}}\right)}{\sum_{i\in\mathcal{G}}Y_i^{0}},
\]
and the corresponding absolute reduction per household--policy day,
\[
\widehat{A}_{\mathcal{G}}=\frac{\sum_{i\in\mathcal{G}}\left(Y_i^{0}-Y_i^{\mathrm{obs}}\right)}{\sum_{i\in\mathcal{G}}n_i}.
\]
Positive values therefore indicate lower observed demand relative to the estimated no-policy counterfactual. To quantify sampling uncertainty, we resample households with replacement within each subgroup, keeping each sampled household's complete intervention history together, and recompute $\widehat{R}_{\mathcal{G}}$ in each of 1,000 bootstrap samples. The reported 95\% confidence interval is given by the 2.5th and 97.5th percentiles of these bootstrap estimates.

The estimates confirm substantial and nonmonotone household heterogeneity. The clearest differences occur across baseline-demand quintiles. Q1 has the largest positive relative reduction, at 2.84\%, with a 95\% confidence interval of [1.72, 4.10], while Q4 also has a positive reduction of 1.85\%. By contrast, Q2 and Q5 have negative estimates, and the estimate for Q3 is close to zero. The differences across peak-hour and evening-load shares are more modest, with all corresponding confidence intervals including zero. Load volatility exhibits a nonmonotone pattern: Q1 has a negative response, whereas Q5 has a positive response, and both intervals exclude zero. The estimates also vary across load-shape clusters, although their confidence intervals are wide and include zero. The absolute estimates show that relative and absolute responses need not rank households in the same way. For example, Q1 of baseline demand has the largest relative reduction, whereas Q4 has the largest absolute reduction, at 0.581 kWh per household--policy day. Overall, the table reinforces the conclusion that household response to dynamic pricing varies with pre-intervention electricity-use patterns, but this heterogeneity cannot be summarized by a simple monotone relationship with any single household characteristic.



\begingroup
\def\bibfont{\fontsize{10pt}{12pt}\selectfont}
\def\bibsep{0pt}
\bibliographystyleEC{informs2014}
\bibliographyEC{ref}
\endgroup

\end{document}